\documentclass[12pt,letterpaper,openany,oneside]{book} %PDF version
\usepackage[letterpaper,top=1in, bottom=1.5in, left=1.5in, right=1in]{geometry}

\usepackage{amsmath,amsfonts,amssymb,amsthm}
\usepackage{graphicx,graphics,caption,color}
\usepackage[vlined,boxruled,linesnumbered,algochapter]{algorithm2e}
\usepackage{tikz} % for rendering graphical models
\usetikzlibrary{fit,positioning}
\usepackage{microtype,csquotes}
\usepackage[english]{babel}
\usepackage{placeins}

\usepackage{fancyhdr, setspace}

\fancypagestyle{plain}{
	\fancyhf{} % clear all header and footer fields
	\fancyfoot[R]{\thepage $\qquad$}
}
\usepackage[style=ieee,backend=bibtex,sorting=nyt,firstinits=true,sortcites=true]{biblatex}
\newcommand{\E}{\mathbb{E}}
\newcommand{\V}{\mathbb{V}}

\newcommand{\Sec}{\textbf{\S}}

\newcommand{\argmax}[1]{\underset{#1}{\operatorname{argmax}} \;}
\newcommand{\argmin}[1]{\underset{#1}{\operatorname{argmin}} \;}
\newcommand{\mnorm}[2][{\Sigma_j}]{\|#2\|^2_{\Sigma_j}}
\newcommand{\minnprod}[3][{\Sigma_j}]{{\langle #2,#3 \rangle}_{\Sigma_j}}

\def\layersep{3.5cm}
\def\layersepp{1.75cm}

\theoremstyle{plain}
\newtheorem{prop}{Proposition}[chapter]
\newtheorem{lem}{Lemma}[chapter]
\newtheorem{cor}{Corollary}[chapter]
\newtheorem{thm}{Theorem}[chapter]

\makeatletter
\def\blfootnote{\xdef\@thefnmark{}\@footnotetext}
\makeatother

\makeatletter
\AtEveryCite{%
  }
\makeatother

\begin{document}
\frontmatter

%\subfile{./Chapters/Title}
\thispagestyle{empty}
	\begin{titlepage}
	\begin{center}
		\rule[-.3\baselineskip]{0pt}{0.5in}
		\begin{spacing}{2} 
			\large{
				NOISE BENEFITS 
				\\IN 
				\\EXPECTATION-MAXIMIZATION
				\\ALGORITHMS\\[0.35in]
			}
		\end{spacing}
		by\\[0.25in]
		\large{Osonde Adekorede Osoba}\\[0.45in]
		
		\rule{5in}{1pt}\\[0.45in]
		
		A Dissertation Presented to the\\
		FACULTY OF THE USC GRADUATE SCHOOL\\
		UNIVERSITY OF SOUTHERN CALIFORNIA\\
		In Partial Fulfillment of the\\
		Requirements for the Degree\\
		DOCTOR OF PHILOSOPHY\\
		(ELECTRICAL ENGINEERING)\\[0.2in]		
		August 2013\\				
	\end{center}
	\vfill
	Copyright 2013 \hfill Osonde Adekorede Osoba
	\end{titlepage}
	
%\subfile{./Chapters/Acks}
\chapter*{Acknowledgements}
\addcontentsline{toc}{chapter}{Acknowledgements}

%Kosko, Kartik, Sanya

\doublespacing
I am deeply grateful to my advisor Professor Bart Kosko for his patience, faith, and guidance throughout the development of this work. I am indebted to him for fostering much of my academic and personal growth over the past few years.

I thank the other committee members Professor Antonio Ortega and Professor James Moore II. Their insightful feedback helped shape this dissertation.

I had the privilege of collaborating closely with Kartik Audhkhasi and Professor Sanya Mitaim. I am grateful for our many late-night math sessions.

Finally I would like to thank my friends and family for supporting me through my graduate career. In particular I would like to thank my friend and teacher Michael Davis and the United Capoeira Association community in Los Angeles.

\tableofcontents
\setcounter{tocdepth}{1}
\clearpage
\listoftables
\addcontentsline{toc}{chapter}{List of Tables}
\clearpage
\listoffigures
\addcontentsline{toc}{chapter}{List of Figures}

%\subfile{./Chapters/preamble}
\chapter*{Preface} %Thesis Outline
\addcontentsline{toc}{chapter}{Preface}

This dissertation's main finding is that noise can speed up the convergence of the Expectation-Maximization (EM) algorithm. The Noisy Expectation Maximization (NEM) theorem states a sufficient condition under which controlled noise injection causes the EM algorithm to converge faster on average. NEM algorithms modify the basic EM scheme to take advantage of the NEM theorem's average speed-up guarantee. The dissertation also proves that neural, fuzzy, and other function approximators can uniformly approximate the more general posterior probabilities in Bayesian inference. The dissertation derives theorems and properties about the use of approximators in general Bayesian inference models.

The dissertation consists of 10 chapters. Chapter~\ref{ch:overt} is a short preview of all the main theorems and results in this dissertation. Chapter~\ref{ch:EM} gives a broad review of the EM algorithm. It also presents EM variations and examples of EM algorithms for different data models. This review lays the groundwork for discussing NEM and its convergence properties.

Chapter~\ref{ch:NEM} introduces Noisy Expectation Maximization. This chapter discusses the intuition behind the NEM theorem. Then it derives the NEM theorem which gives a condition under which noise injection improves the average convergence speed of the EM algorithm. Examples of NEM algorithms show this average noise benefit. This chapter includes a full derivation and demonstration of NEM for the popular Gaussian mixture model. The chapter ends with a corollary to the NEM theorem which states that noise-enhanced EM estimation is more beneficial for sparse data sets than for large data sets.% This corollary shows the explanatory power of the NEM theorem.

Chapters~\ref{ch:NEM-CNBT},~\ref{ch:NEM-HMM}, and~\ref{ch:NEM-BP} apply NEM to three important EM models: mixture models for clustering, hidden Markov models, and feedfoward artificial neural networks.  The training algorithms for these three models are standard and heavily-used: $k$-means clustering, the Baum-Welch algorithm, and backpropagation training respectively. All three algorithms are EM algorithms. These chapters present proofs showing that these algorithms are EM algorithms. The EM-subsumption proof for backpropagation is new. The subsumption proofs for $k$-means and Baum-Welch are not new. These subsumptions imply that the NEM  theorem applies to these algorithms. Chapters~\ref{ch:NEM-CNBT},~\ref{ch:NEM-HMM}, and~\ref{ch:NEM-BP} derive the NEM conditions for each model and show that NEM improves the speed of these popular training algorithms.

Chapter~\ref{ch:Bayes} reviews Bayesian statistics in detail. The chapter highlights the difference between the Bayesian approach to statistics and the frequentist approach to statistics (which underlies maximum likelihood techniques like EM). It also shows that the Bayesian approach subsumes the frequentist approach. This subsumption implies that subsequent Bayesian inference results also apply to EM models.

Chapters~\ref{ch:BAT} and~\ref{ch:ExBAT} give analyze the effects of model-function approximation in Bayesian inference. These chapters present the Bayesian Approximation theorem and its extension, the Extended Bayesian Approximation theorem. These theorems guarantee that uniform approximators for Bayesian model functions produce uniform approximators for the posterior pdf via Bayes theorem. Simulations with uniform fuzzy function approximators show sample posterior pdf approximations that validate the theorems. The use of fuzzy function approximators has the effect of subsuming most closed functional form Bayesian inference models via the recent Watkins Representation Theorem.

Chapter~\ref{ch:finale} discusses ongoing and future research that extends these results.

\chapter*{Abstract}
\addcontentsline{toc}{chapter}{Abstract}

\doublespacing
This dissertation shows that careful injection of noise into sample data can substantially speed up Expectation-Maximization algorithms. Expectation-Maximization algorithms are a class of iterative algorithms for extracting maximum likelihood estimates from corrupted or incomplete data. The convergence speed-up is an example of a noise benefit or ``stochastic resonance'' in statistical signal processing. The dissertation presents derivations of sufficient conditions for such noise-benefits and demonstrates the speed-up in some ubiquitous signal-processing algorithms. These algorithms include parameter estimation for mixture models, the $k$-means clustering algorithm, the Baum-Welch algorithm for training hidden Markov models, and backpropagation for training feedforward artificial neural networks. This dissertation also analyses the effects of data and model corruption on the more general Bayesian inference estimation framework. The main finding is a theorem guaranteeing that uniform approximators for Bayesian model functions produce uniform approximators for the posterior pdf via Bayes theorem. This result also applies to hierarchical and multidimensional Bayesian models.

\mainmatter

\onehalfspacing

\newtheorem*{thm*}{Theorem}
\newtheorem*{cor*}{Corollary}
\captionsetup{list=no}

%\subfile{./Chapters/Overture}
\chapter{Preview of Dissertation Results}\label{ch:overt}

The main aim of this dissertation is to demonstrate that noise injection can improve the average speed of Expectation-Maximization (EM) algorithms. The EM discussion in Chapter~\ref{ch:EM} gives an idea of the power and generality of the EM algorithm schema. But EM algorithms have a key weakness: they converge slowly especially on high-dimensional incomplete data. Noise injection can address this problem. The Noisy Expectation Maximization (NEM) theorem (Theorem~\ref{thm:NEM}) in Chapter~\ref{ch:NEM} describes a condition under which injected noise causes faster EM convergence on average. This general condition reduces to a simpler condition (Corollary \ref{cor:GMM-NEM}) for Gaussian mixture models (GMMs). The GMM noise benefit leads to EM speed-ups in clustering algorithms and in the training of hidden Markov models. The general NEM noise benefit also applies to the backpropagation algorithm for training feedforward neural network. This noise benefit relies on the fact that backpropagation is indeed a type of EM algorithm (Theorem~\ref{thm:bp_gem_equiv}).

The secondary aim of this dissertation is to show that uniform function approximators can expand the set of model functions (likelihood functions, prior pdfs, and hyperprior pdfs) available for Bayesian inference. Bayesian statisticians often limit themselves to a small set of closed-form model functions either for ease of analysis or because they have no robust method for approximating arbitrary model functions. This dissertation shows a simple robust method for uniform model function approximation in Chapters~\ref{ch:BAT} and~\ref{ch:ExBAT}. Theorem~\ref{thm:BAT} and Theorem~\ref{thm:ExBAT} guarantee that uniform approximators for model functions lead to uniform approximators for posterior pdfs.

\section{Noisy Expectation-Maximization}
The Noisy Expectation Maximization (NEM) theorem (Theorem~\ref{thm:NEM}) is the major result in this dissertation. 
\begin{thm*}{\bf{[Noisy Expectation Maximization (NEM)]:}}\\
	An EM iteration noise benefit occurs on average if
	\begin{equation}
		\E_{Y,Z,N|\theta_*} \left[ \ln\left( \frac{f(Y+N,Z|\theta_k)}{f(Y,Z|\theta_k)} \right) \right] \geq 0\;.
	\end{equation}
\end{thm*}
The theorem gives a sufficient condition under which adding noise $N$ to the observed data $Y$ leads to an increase in the average convergence speed of the EM algorithm. This is the first description of a noise benefit for EM algorithms. It relies on the insight that noise can sometimes perturb the likelihood function favorably. Thus noise injection can lead to better iterative estimates for parameters. This sufficient condition is general and applies to \emph{any} EM data model. 

\begin{figure}[ht!]
	\centerline{ \includegraphics[width=0.75\textwidth]{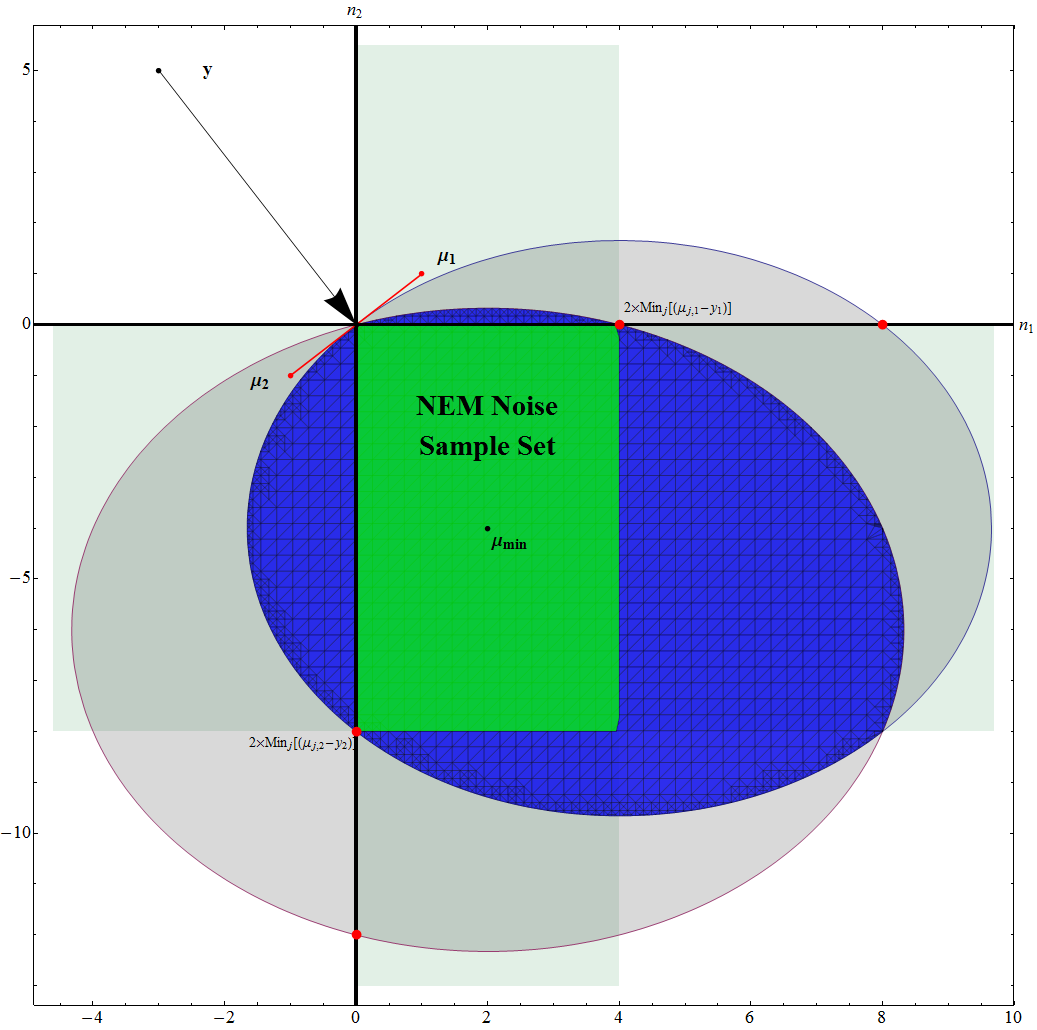} }
	\caption{
		Geometry of NEM noise for a GMM. Noise samples in the blue overlapping region satisfy the NEM sufficient condition and lead to faster EM convergence. Noise samples in the green box satisfy a simpler quadratic NEM sufficient condition and also lead to faster EM convergence. Sampling from the green box is easier. This geometry is for a sample $\mathbf{y}$ of a $2$-D GMM with sub-populations centered at $\boldsymbol{\mu}_1$ and $\boldsymbol{\mu}_2$. \Sec\ref{subsec:NEM-Geom} and \Sec\ref{subsec:JGMM-NEM} discuss these geometries in more detail.
	}
\label{fg:Geom-intro}
\end{figure}

The first major corollary (Corollary \ref{cor:GMM-NEM}) applies this sufficient condition to EM algorithms on the Gaussian mixture model. This results in a simple quadratic noise screening condition for the average noise benefit.
\begin{cor*}{\bf{[NEM Condition for GMMs (in $1$-D)]:}}\\
	The NEM sufficient condition holds for a GMM if the additive noise samples $n$ satisfy the following algebraic condition
	\begin{align}
		n^2 &\leq 2n\left(\mu_j-y\right) \textrm{ for all GMM sub-populations} \quad j \;.
	\end{align}
\end{cor*}
This quadratic condition defines the geometry (Figure~\ref{fg:Geom-intro}) of the set of noise samples that can speed up the EM algorithm.

\FloatBarrier

Noise injection subject to the NEM condition leads to better EM estimates on average at each iteration and faster EM convergence. Combining NEM noise injection with a noise decay per iteration leads to much faster overall EM convergence. We refer to the combination of NEM noise injection and noise cooling as the \emph{NEM algorithm}~(\Sec\ref{sec:NEM-Alg}). A comparison of the evolution of EM and NEM algorithms on a sample estimation problem shows that the NEM algorithm reaches the stationary point of the likelihood function in $30\%$ fewer steps than the EM algorithm (see Figure~\ref{fg:intro-demo}).
\begin{figure}[ht!]
	\centerline{\bf{Log-likelihood Comparison of EM and Noise-enhanced EM}}
	\centerline{\includegraphics[width=0.95\textwidth]{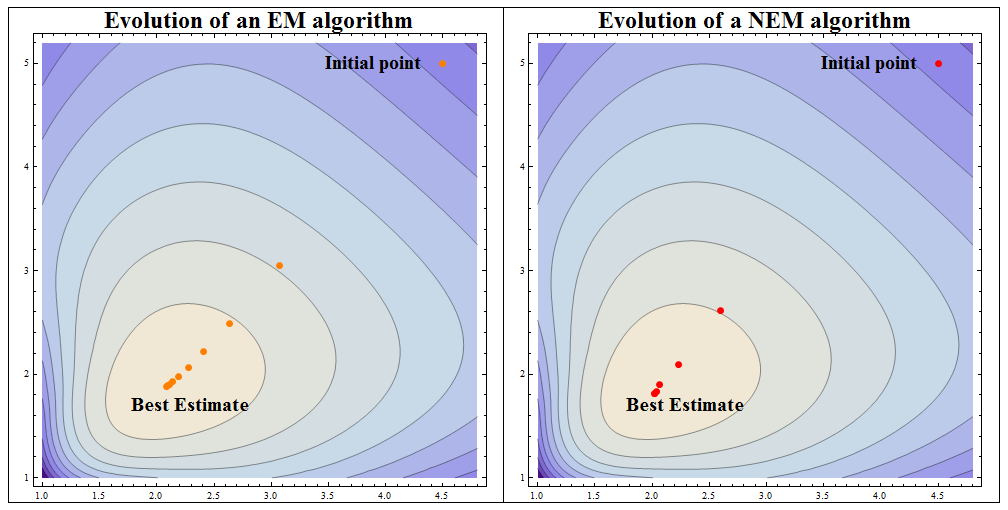}}
	\caption{
		NEM noise injection can speed up the convergence of the EM algorithm. The plot shows the evolution of an EM algorithm on a log-likelihood surface with and without noise injection. Both algorithms start at the same initial estimate and converge to the same point on the log-likelihood surface. The EM algorithm converges in $10$ iterations while the noise-enhanced algorithm converges in $7$ iterations---$30\%$ faster than the EM algorithm.
	}
\label{fg:intro-demo}
\end{figure}

\section{Applications of Noisy Expectation-Maximization}
Finding the NEM noise benefit led to recasting other iterative statistical algorithms as EM algorithms to allow a noise boost. The NEM theorem is a general prescriptive tool for extracting noise benefits from arbitrary EM algorithms. So these reinterpretations serve as a basis for introducing NEM noise benefits into other standard iterative estimation algorithms. This dissertation shows NEM noise benefits in three such algorithms: the $k$-means clustering algorithm (Chapter~\ref{ch:NEM-CNBT}), the Baum-Welch algorithm (Chapter~\ref{ch:NEM-HMM}), and the backpropagation algorithm (Chapter~\ref{ch:NEM-BP}). 

The most important of these algorithms is the backpropagation algorithm for feedforward neural network training. We show for the first time that the backpropagation algorithm is in fact a generalized EM (GEM) algorithm (Theorem~\ref{thm:bp_gem_equiv}) and thus benefits from proper noise injection:
\vspace{-6pt}
\begin{thm*}{\bf{[Backpropagation is a GEM Algorithm]:}}\\
	The backpropagation update equation for a feedforward neural-network likelihood function equals the GEM update equation. Thus backpropagation is a GEM algorithm.
\end{thm*}
\goodbreak

This theorem illustrates a general theme in recasting estimation algorithms as EM algorithms: iterative estimation algorithms that make use of missing information and increase a data log-likelihood are usually (G)EM algorithms. Chapter~\ref{ch:NEM-BP} provides proof details and simulations of NEM noise benefits for backpropagation. The NEM condition for backpropagation (Theorem~\ref{thm:sph-ndnn-bp}) has interesting geometric properties as the backpropagation noise ball in Figure~\ref{fg:BP-demo} illustrates.
\begin{figure}[ht!]
	\centerline{\bf{Geometry of NEM Noise for Backpropagation}}
%	\centerline{\includegraphics[width=0.7\textwidth]{Figures/Custom/Spheres.png}}
	\centerline{\includegraphics[height=0.5\textheight]{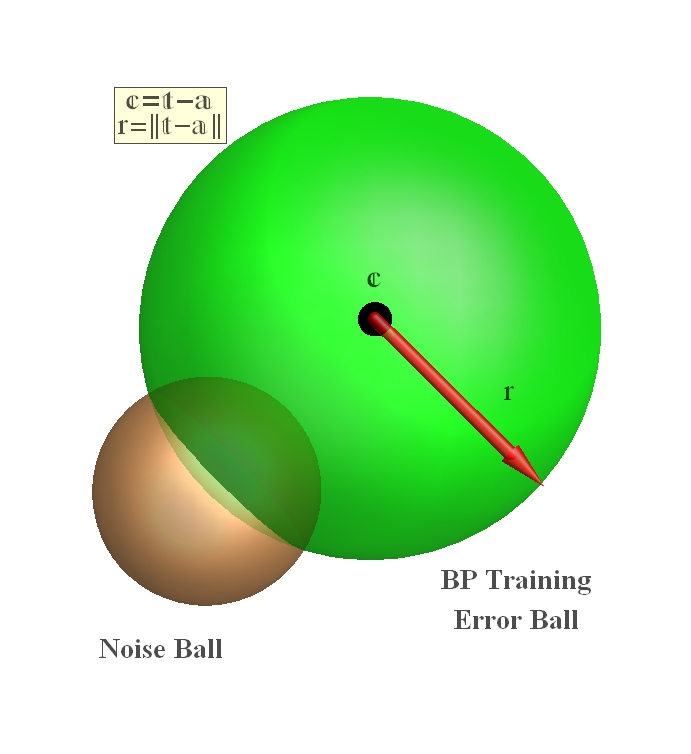}}
	\caption{
		NEM noise for faster backpropagation using Gaussian output neurons. The NEM noise must fall inside the backpropagation ``training error'' sphere. This is the sphere with center $\mathbf{c} = \mathbf{t}-\mathbf{a}$ (the error between the target output $\mathbf{t}$ and the actual output $\mathbf{a}$) with radius $r=\|\mathbf{c}\|$. Noise from the noise ball section that intersects with the error sphere will speed up backpropagation training according to the NEM theorem. The error ball changes at each training iteration.
	}
\label{fg:BP-demo}
\end{figure}

Deep neural networks can also benefit from the NEM noise benefit. Deep neural networks are ``deep'' stacks of restricted Boltzmann machines (RBMs). The depth of the network may help the network identify complicated patterns or concepts in complex data like video or speech. These deep networks are in fact bidirectional associative memories (BAMs). The stability and fast training properties of deep networks are direct consequences of the global stability property of BAMs.

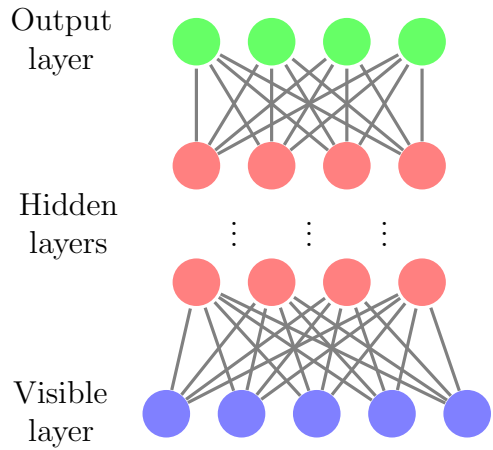
\begin{figure}[ht!]
	\centering
	\begin{tikzpicture}[shorten >=1pt,->,draw=black!50, node distance=\layersep]
		\tikzstyle{every pin edge}=[<-,shorten <=2pt]
		\tikzstyle{neuron}=[circle,fill=black!25,minimum size=18pt,inner sep=0pt]
		\tikzstyle{input neuron}=[neuron, fill=blue!50];
		\tikzstyle{hidden neuron}=[neuron, fill=red!50];
		\tikzstyle{output neuron}=[neuron, fill=green!60];
		\tikzstyle{annot} = [text width=4em, text centered]
		\foreach \name / \y in {1,...,5}
			\node[input neuron] (I-\name) at (-\y-0.4,0) {};
		\foreach \name / \y in {2,...,5}
				\node[hidden neuron] (H-\name) at (-\y cm, \layersepp) {};
		\foreach \name / \y in {2,...,5}
				\node[hidden neuron] (H2-\name) at (-\y cm, \layersep-0.2cm) {};
		\foreach \name / \y in {2,...,5}
				\node[output neuron] (O-\name) at (-\y cm, \layersep+\layersepp-0.3cm) {};
		
		\foreach \source in {1,...,5}
			\foreach \dest in {2,...,5}
				\draw[-, very thick] (I-\source) edge (H-\dest);

		\foreach \source in {2,...,5}
			\foreach \dest in {2,...,5}
				\draw[-, very thick] (H2-\source) edge (O-\dest);
				
		\node[inner sep=0,minimum size=0] (phantom) at (-3.5 cm, 1.75cm) {}; % invisible node
		\node[annot, above of=phantom, node distance=0.75cm](dots) {$\mathbf{\vdots}$};
		\node[annot, right of=dots, node distance=1cm] {$\mathbf{\vdots}$};
		\node[annot, left of=dots, node distance=1cm] (hs) {$\mathbf{\vdots}$};
				
		\node[annot,left of=hs, node distance=2.2cm] (hl) {Hidden layers};
		\node[annot,left of=I-5, node distance=1.4cm] {Visible layer};
		\node[annot,left of=O-5, node distance=1.8cm] {Output layer};
		
	\end{tikzpicture}
	\caption{A Deep Neural Network consists of a stack of restricted Boltzmann machines (RBMs) or bidirectional associative memories (BAMs).}
\end{figure}

The so-called Contrastive Divergence algorithm is the current standard algorithm for pre-training deep networks. It is an iterative algorithm for approximate maximum likelihood estimation (\Sec\ref{subsec:deepNN}). CD is also a GEM algorithm. Theorem~\ref{thm:hyp-ndnn} and Theorem~\ref{thm:sph-ndnn} give the NEM noise benefit conditions for training the RBMs in a deep network. The NEM condition for RBMs shares many geometrical properties with the NEM condition for backpropagation.
\FloatBarrier

% The DNN result needs one of those noise spheres.  It can go next to the architecture drawing.

\section{Results on Bayesian Approximation}
The last major results in this dissertation are the Bayesian approximation theorems in Chapters~\ref{ch:BAT} and~\ref{ch:ExBAT}. They address the effects of using approximate model functions for Bayesian inference. Approximate model functions are common in Bayesian statistics because statisticians often have to estimate the true model functions from data or experts. This dissertation presents the first general proof that these model approximations do not degrade the quality of the approximate posterior pdf. Below is a combined statement of the two approximation theorems in this dissertation (Theorem~\ref{thm:BAT} and Theorem~\ref{thm:ExBAT}):
\vspace{12pt}
\begin{thm*}{\bf{[The Unified Bayesian Approximation Theorem]:}}\\
	Suppose the model functions (likelihoods $g$, prior $h$, and hyperpriors $\pi$) for a Bayesian inference problem are bounded and continuous. Suppose also that the joint product of the model functions' uniform approximators $G H \Pi$ is non-zero almost everywhere on the domain of interest $\mathcal{D}$.
	
	Then the posterior pdf approximator $F = \frac{G H \Pi}{\int_{\mathcal{D}} G H \Pi}$ also uniformly approximates the true posterior pdf $f = \frac{g h \pi}{\int_{\mathcal{D}} g h \pi}$
\end{thm*}

This approximation theorem gives statisticians the freedom to use approximators to approximate arbitrary model functions---even model functions that have no closed functional form---without worrying about the quality of their posterior pdfs.

Statisticians can choose any uniform approximation method to reap the benefits of this theorem. Standard additive model (SAM) fuzzy systems are one such tool for uniform function approximation. Fuzzy systems can use linguistic information to build model functions. Figure~\ref{fg:FAT-Demo} below shows an example of a SAM system approximating a pdf using $5$ fuzzy rules. \Sec\ref{sec:AdaptFAT} discusses Fuzzy function approximation in detail. Chapter~\ref{ch:ExBAT} addresses the complexities of approximate Bayesian inference in hierarchical or iterative inference contexts.
\begin{figure}[h!]
%\centerline{\includegraphics[width=0.8\textwidth]...
\centerline{\includegraphics[height=0.31\textheight, width=0.75\textwidth]{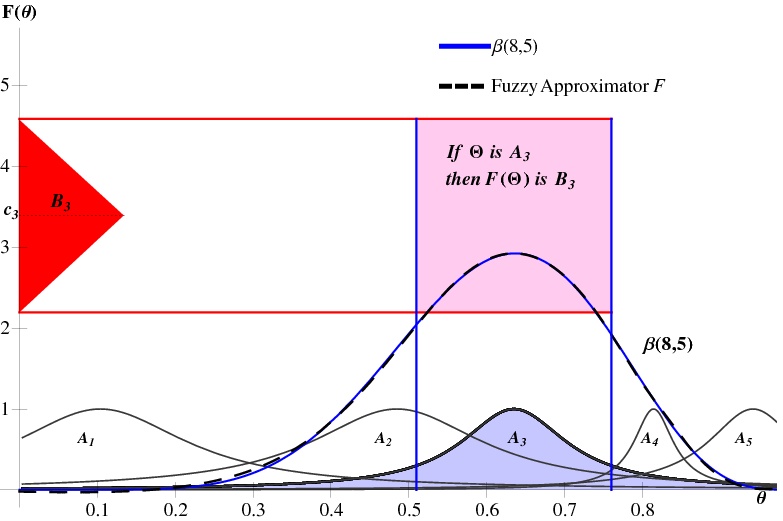}}
	\caption{
		A fuzzy function approximation for a $\beta(8, 5)$ prior pdf. An adaptive SAM (standard additive model) fuzzy system tuned five fuzzy sets to give a nearly exact approximation of the beta prior.  Each fuzzy rule defines a patch or 3-D surface above the input-output planar state space.  The third rule has the form ``If $\Theta = A_3$ then $B_3$'' where then-part set $B_3$ is a fuzzy number centered at centroid $c_3$.  This rule might have the linguistic form ``If $\Theta$ is {\it approximately} $\frac{1}{2}$ then $F(\Theta)$ is {\it large}.''
	}
\label{fg:FAT-Demo}
\end{figure}

This dissertation also contains other minor results of note including a convergence theorem (Theorem~\ref{thm:MM-converge}) for a subset of minorization-maximization (MM) algorithms. MM algorithms generalize EM algorithms. But there are no published proofs of MM convergence. There is an extension of the GMM-NEM condition to mixtures of \emph{jointly} Gaussian sub-populations (Corollary~\ref{cor:JGMM-NEM}). There is also an alternate proof showing that the $k$-means algorithm is a specialized EM algorithm. Other proofs of this subsumption already exists in the literature. This dissertation ends with discussions about ongoing work to establish or demonstrate NEM noise benefits in genomics and medical imaging applications (\Sec\ref{subsec:genomics}, \Sec\ref{subsec:PET-NEM}, \& \Sec\ref{subsec:MRI-NEM}).

\clearpage

\captionsetup{list=yes}

%\subfile{./Chapters/chapEM}
\chapter{The Expectation--Maximization (EM) Algorithm}\label{ch:EM}

The Expectation--Maximization (EM) algorithm~\parencite{dempster-laird-rubin1977, mclachlan-krishnan2007, gupta-chen2010} is an iterative statistical algorithm that estimates high--likelihood parameters from incomplete or corrupted data. This popular algorithm has a wide array of applications that includes data clustering~\parencite{celeux-govaert1992,ambroise-dang-govaert1997}, automated speech recognition~\parencite{rabiner1989,juang-rabiner1991}, medical imaging~\parencite{shepp-vardi1982, zhang-brady-smith2001}, genome-sequencing~\parencite{lawrence-reilly1990,bailey-elkan1995a}, radar denoising~\parencite{wang-Dogandzic-Nehorai2006}, and infectious-diseases tracking~\parencite{reilly-lawlor1999,bacchetti1990}. A prominent mathematical modeler even opined that the EM algorithm is ``as close as data analysis algorithms come to a free lunch''~\parencite[][p. 177]{gershenfeld1999}. Other data-mining researchers consider the EM algorithm to be one of the top ten algorithms for data-mining~\parencite{top-ten-algs2008}.

EM algorithms split an estimation problem into two steps: the \emph{Expectation} (E) step and the \emph{Maximization} (M) step. The E step describes the best possible complete model for the incomplete data given all current information. The M step uses that new complete model to pick higher likelihood estimates of the distribution parameters for the incomplete data. The improved parameter estimates from the M step lead to better complete models in the E step. The EM algorithm iterates between the E and M steps until the parameter estimates stop improving.

This chapter presents the EM algorithm in full detail. The EM algorithm is a generalization of maximum likelihood estimation (MLE). It inherits many properties from MLE. The next section reviews MLE and motivates the EM generalization. Then we formulate the EM algorithm and examine some of its convergence properties. The chapter ends with a survey of instances, notable variants, and a generalization of the EM algorithm.

\section{A Review of Maximum Likelihood Estimation}\label{sec:ell-expo}

Maximum likelihood estimation methods search for parameter values $\theta$ which maximize a likelihood function $\ell(\theta)$. The likelihood $\ell(\theta)$ is a statistic of the data $y$ that defines a preference map over all possible parameters $\theta \in \Theta$ for the data model~\parencite{pawitan2001}. This preference map quantifies how well a parameter $\theta_0$ describes or summarizes the sample $y$ relative to all other parameters $\theta_i \in \Theta$~\parencite{pawitan2001,efron-mle-dec1982}.

The likelihood $\ell(\theta)$ measures the quality of a parameter estimate $\theta$ using the observed sample $y$ and the data's parametric pdf $f(y|\theta)$. Fisher~\parencite*{fisher1922,fisher1925} formalized the use of the likelihood as a method for evaluating candidate parameter estimates. He defined the likelihood $\ell(\theta)$ of $\theta$ for a sample $y$ of the observed random variable (rv) $Y$ as
\begin{equation} 
	\ell(\theta) = f(y|\theta)
\end{equation}
where $Y$ has the probability density function (pdf) $f(y|\theta)$. $\ell(\theta)$ has the unusual property of being both a function of $\theta$ and a statistic of the data $y$. If an estimate $\theta_0$ has low likelihood $\ell(\theta_0)$ for a given sample $y$ then there is a low probability of observing $y$ under the hypothesis that the true data pdf is $f(y|\theta_0)$. A parameter $\theta_1$ is more likely than $\theta_0$ if it gives a higher likelihood than $\theta_0$. So there is a higher probability of observing $y$ if the true sampling pdf was $f(y|\theta_1)$ instead of $f(y|\theta_0)$.

One approach to statistics asserts that this likelihood function contains all the objective evidence the experiment (i.e. the random sampling of $Y$) provides about the  unknown parameter $\theta$. This assertion is the \emph{likelihood principle} and it is often an implicit assumption in statistics applications~\parencite{berger-wolpert1988} including Bayesian applications.

\subsection{MLE Overview} \label{subsec:MLE-over}

A basic question in statistics is: which pdf $f(y)$ \emph{best} describes a set of observed random samples $\{y_i\}_i^n$? Statisticians often invoke the simplifying assumption that the data samples come from one pdf in a class of parametric sampling pdfs $\{f(y|\theta)\}_{\theta \in \Theta}$\footnote{
	This assumption may be insufficient if the parametric pdf class is not general enough. A rigorous statistician would ensure that the parametric pdf class $\{f(y|\theta)\}_{\theta \in \Theta}$ is provably dense in the class of appropriate, possibly non-parametric pdfs $\{f_\psi(y)\}_{\psi\in \Psi}$ for $y$. Most estimation applications just assume a convenient pdf family and do not address this assumption rigorously.
}.
Then the question becomes: which parameter estimate $\hat{\theta}$ gives the \emph{best} parametric pdf $f(y|\theta)$ for describing the observed random samples $\{y_i\}_i^n$?

The preceding discussion on likelihoods suggests that the best parameter estimate $\hat{\theta}$ is the value of the parameter that gives the highest observation probability for $\{y_i\}_i^n$ i.e. the most likely parameter. This is the \emph{Maximum Likelihood Estimate} (MLE)~\parencite{fisher1922, efron-mle-dec1982, aldrich1997, pawitan2001}. The likelihood $\ell(\theta)$ of the random samples $\{y_i\}_i^n$ is
\begin{equation} 
	\ell(\theta) = \prod_i^n f(y_i|\theta) \;.
\end{equation}
Thus the ML estimate $\hat{\theta}$ is
\begin{equation}
	\hat{\theta}_n = \argmax{\theta \in \Theta} \ell(\theta) \;. 
\end{equation}
The product structure of the joint pdf and the exponential nature of many pdfs make it easier to optimize the logarithm of the likelihood $\ln \ell(\theta)$ instead of the likelihood $\ell(\theta)$ itself. So we define the log-likelihood $L(\theta)$ as\footnote{
	I use the $L(\theta)$ and $\ell(\theta)$ to denote $L(\theta|y)$ and $\ell(\theta|y)$ when the data random variable $Y$ is unambiguous.
}
\begin{equation}
	L(\theta) = \ln \ell(\theta) \;.
\end{equation}
The logarithmic transformation preserves the likelihood function's stationary points because $L'(\theta) = k~\ell'(\theta)$ where $k$ is a strictly positive scalar for all viable values\footnote{
	$\theta$ is viable if it has strictly positive likelihood $\ell(\theta) > 0$. i.e. $f(y|\theta) > 0 $ for the observed sample $y$. The viability condition is necessary because $L'(\theta)=\frac{\ell'(\theta)}{\ell(\theta)}$. Thus $L$ preserves the stationary points of $\ell$ only if $k=\frac{1}{\ell(\theta)} \neq 0$. The log-likelihood $L(\theta)$ also preserves the preference order specified by $\ell(\theta)$ since $\ell(\theta_0) \leq \ell(\theta_1) \iff \ln \ell(\theta_0) \leq \ln \ell(\theta_1)$ by the monotone increasing property of the log-transformation.
}
of $\theta$. So the ML estimate is equivalent to
\begin{equation}
	\hat{\theta}_n = \argmax{\theta \in \Theta} L(\theta) \;.
\label{eq:MLE}
\end{equation}
The EM algorithm applies a generalized MLE approach to find optimal parameters to fit complicated pdfs for incomplete data samples.

\subsection{A Brief History of Maximum Likelihood Methods}\label{subsec:hist-mle}

Ronald Fisher introduced\footnote{
	Francis Y. Edgeworth have preempted Fisher \parencite{pratt1976, edgeworth1908} in formulating the maximum likelihood principle in 1908.  But his work did not gain attention until Fisher reformulated an equivalent idea outside the Bayesian setting many years later.
} maximum likelihood methods with his 1912 paper \parencite{fisher1912}. He disparaged the then state-of-the-art parameter estimation methods: the method of moments and the minimum mean squared estimation (MMSE) method. He criticized the fact that both methods select \emph{optimal} parameters based on arbitrary criteria: matching finite-order moments and minimizing the mean squared error respectively. But his main criticism was that these methods give different estimates under different parameterizations of the underlying distribution. 

Suppose the sampling pdf $f(y|\theta)$ has a different functional representation ${\tilde f}(y|\phi)$ where $\phi = t(\theta)$. Then the best estimate $\hat \theta$ for $\theta$ should identify the best estimate $\hat{\phi}$ for $\phi$ via the parameter transformation
\begin{equation}
	\hat{\phi} = t(\hat{\theta}) \;.
\end{equation}
The moments and MMSE methods do not obey this invariance principle i.e. $\hat{\phi} \neq t(\hat{\theta})$ in general. Fisher argued that the invariance principle should hold for any truly \emph{optimal} estimate since the observed samples do not change with re-parameterizations of the pdf. Fisher then proposed an alternate method of statistical estimation that maximized parameters for parametric pdfs. He showed that this method is invariant under parameter transformations. 

Fisher did not name the method until he wrote his 1922 paper~\parencite{fisher1922} in which he formalized the idea of a likelihood as distinct from a probability. Likelihoods measure degree of certainty for parameters just as probabilities do for events. But likelihoods have no independent axiomatic foundation. Likelihoods, for example, do not have to integrate to one. And so there is no clear concept of likelihood marginalization over unwanted variables \parencite{bayarri-degroot1992}. EM provides one likelihood equivalent to Bayesian posterior marginalization over unobserved \emph{data} random variables.
% "nuisance rvs" not same as "nuisance parameters"... Profile likelihood applies for nuisance param. EM for nuisance rv.

Fisher, Doob, and Wald later found other attractive statistical properties of maximum likelihood estimates. The ML estimate $\hat{\theta}_n$ for a parameter $\theta$ is consistent ($\hat{\theta}_n \rightarrow \theta$ in probability) \parencite{doob1934}. It is \emph{strongly} consistent ($\hat{\theta}_n \rightarrow \theta$ with {probability-one}) under additional weak conditions \parencite{wald1949}. The MLE is asymptotically efficient \parencite{fisher1925}:
\begin{equation}
	\lim_n \V[\hat{\theta}_n] =  \mathcal{I}_n^{-1}(\theta)
\end{equation}
where $\mathcal{I}(\theta)$ is the Fisher information $\mathcal{I}(\theta) = -\E_{Y|\theta} \left[L''(\theta) \right]$ \parencite{hogg-mckean-craig2005}. And the MLE is asymptotically normal \parencite{doob1934, wald1943}:
\begin{equation}
	\lim_n (\hat{\theta}_n - \theta) \sim \mathcal{N}(0, \mathcal{I}_{n}^{-1}) \;. \label{eq:MLE-AsymNorm}
\end{equation}

\subsection{Numerical Maximum Likelihood Estimation}

Maximum Likelihood Estimation converts a statistical estimation problem into an optimization problem. We can solve the optimization problem analytically. But numerical methods are necessary when equation (\ref{eq:MLE}) has no analytic solution. ML estimates are roots of the derivative $L'(\theta)$ of the log likelihood $L(\theta)$. This derivative $L'(\theta)$ is also the score statistic $S(\theta)$
\begin{equation}
	S(\theta) = L'(\theta) = \frac{\ell'(\theta)}{\ell(\theta)}.
\end{equation}

The Newton-Raphson (NR) method \parencite{ypma1995} is a Taylor-series-based iterative method for root-finding. NR uses the first-derivative of the objective function to tune update directions. The first-derivative of the score $S'(\theta)$ is the second derivative of the log-likelihood $L''(\theta)$. This is the \emph{observed Fisher information}~\parencite{fisher1934, pawitan2001}: \nocite{efron-hinkley1978}
\begin{align}
	I(\theta) &= -L''(\theta) \\	
	\text{Or} \quad \big{[}\mathbf{I}(\boldsymbol{\theta}) \big{]}_{i,j}^{p,p} &= -\frac{\partial^2 L(\boldsymbol{\theta})}{\partial \theta_i \partial \theta_j}
\end{align} for a vector parameter $\boldsymbol{\theta}=( \theta_1, \cdots, \theta_p)$.
Thus the NR update equation is
\begin{equation}
	\theta_{k+1} = \theta_{k} + I^{-1}(\theta)  S(\theta)\;.
\label{eq:NR-MLE}
\end{equation}
The \emph{Fisher scoring method} uses the expected Fisher information $\mathcal{I(\theta)}=\E_{Y|\theta}[I(\theta)]$ instead of the observed Fisher information $I(\theta)$ in (\ref{eq:NR-MLE}).

Analytic and numerical optimization methods work well for MLE on simple data models. But they often falter when the data model is complex or when the data has random missing information artifacts.

\section{The Expectation Maximization Algorithm}\label{sec:EM} 

Many statistical applications require complicated data models to account for experimental effects like data corruption, missing samples, sample grouping, censorship, truncation, and additive measurement errors. The appropriate likelihood functions for these data models can be very complicated. The Expectation Maximization (EM) algorithm \parencite{dempster-laird-rubin1977, mclachlan-krishnan2007} is an extension of the MLE methods for such complicated data models.

\subsection{Motivation \& Historical Overview}

The EM algorithm is an iterative ML technique that compensates for missing data by taking conditional expectations over missing information given the observed data \parencite{dempster-laird-rubin1977, mclachlan-krishnan2007}. The basic idea behind the EM algorithm is to treat the complex model for the observed data $Y$ as an incomplete model. The EM algorithm augments the observed data random variable $Y$ with a hidden random (latent) random variable $Z$. The aim of this augmentation is to complete the data model and obtain a simpler likelihood function for estimation. But the resulting complete log-likelihood function $L_c(\theta|Y,Z)$ needs to fit the observed incomplete data $Y=y$. The EM algorithm addresses this by sequentially updating its best guess for the complete log-likelihood function $L_c(\theta|Y,Z)$. The EM algorithm uses the conditional expectation of the complete log-likelihood $\E_Z[L_c(\theta| y,Z)|Y =y,\theta_k]$ given the observed data\footnote{
	Some researchers incorrectly assume that this conditional expectation just fills in the missing data $Z$ with the current estimate $\E[Z|Y,\theta_k]$. The process of substituting estimates for missing random variables is \emph{imputation} \parencite{little-rubin2002, schafer-graham2002}\nocite{rubin1976}. EM is equivalent to direct imputation only when the data model has a log-linear likelihood~\parencite{wei-tanner1990,ng-mclachlan2004} like the exponential or gamma distributions (see \Sec\ref{subsec:Exp-EM}). But EM does not impute estimates for $Z$ directly when the data model has a nonlinear log-likelihood.
}
as its best guess for the compatible complete log-likelihood function on the $k^{th}$ iteration.

\begin{align}
\text{\bf E-Step:}& ~~Q \left( \theta |\theta_k \right) = \E_{Z|y,\theta_k}  \left[L_c(\theta| y,Z)\right] \label{eq:E-step} \\
\text{\bf M-Step:}& ~~\theta_{k+1} = \argmax{\theta} \left\{ Q\left( \theta |\theta_k \right) \right\}  \label{eq:M-Step}
\end{align}

The EM algorithm is a synthesis of ML statistical methods for dealing with complexity due to missing information. Orchard and Woodbury's \emph{Missing Information Principle} \parencite{orchard-woodbury1972} was the first coherent formulation of the idea that augmenting a data model with missing information can simplify statistical analysis. There had been decades of statistical work weaving this idea into ad hoc solutions for missing information problems in different areas. Hartley used this principle to address truncation and censoring effects in discrete data \parencite{hartley1958}. The Baum-Welch algorithm \parencite{baum-et-al1970, welch2003} uses the same missing information principle to estimate ML parameters for hidden Markov models. Sundberg \parencite{sundberg1974, sundberg1976} identified a similar theme in the analysis of incomplete generalized-exponential family models. Beale and Little \parencite{beale-little1975} applied the missing information idea to estimate standard errors for multivariate regression coefficients on incomplete samples. 

The Dempster et al. paper (\cite*{dempster-laird-rubin1977}) distilled out the unifying principle behind these ad hoc solutions and extended the principle to a general formulation for solving other missing information MLE problems. The EM formulation was such a powerful synthesis that Efron argued $17$ years later \parencite{efron1994} that the field of missing data analysis was an outgrowth of  the EM algorithm\footnote{A claim which Rubin countered quite effectively in his rejoinder to Efron's paper \parencite{rubin1994}}.
The general EM formulation now applies to parameter estimation on a wide array of data models including Gaussian mixtures, finite mixtures, censored, grouped, or truncated models, mixtures of censored models, and multivariate $t$-distribution models. %\todo{refs for each model + more models}
%\parencite{redner-walker1984, FMM-text}, \parencite{chauveau1995}

The EM algorithm is not so much a single algorithm as it is an algorithm schema or family of similar MLE techniques for handling incomplete data. The conceptual simplicity and the mature theoretical foundations of the EM schema are powerful incentives for recasting disparate estimation algorithms as EM algorithms. These algorithms include the Baum-Welch algorithm~\parencite{welch2003}, Iteratively Re-weighted Least Squares (IRLS)~\parencite{dempster-laird-rubin1980}, older parameter estimation methods for normal mixtures~\parencite{redner-walker1984}, Iterative Conditional Estimation (ICE)~\parencite{delmas1997}, Gaussian Mean-Shift~\parencite{carreira2005}, $k$-means clustering~\parencite{celeux-govaert1992, osoba-kosko2013}, backpropagation~\parencite{audhkhasi-osoba-kosko-BP2013} etc. %, Image Space Reconstruction Algorithm (ISRA) \parencite{depierro1993}

This dissertation extends the EM algorithm schema via the use of noise injection. Noise injection causes the algorithm to explore more of the log-likelihood surface. Intelligent log-likelihood exploration can lead to faster average convergence speed for the EM algorithm. Figure \ref{fg:NEM-Demo} is a demonstration of an EM algorithm speed-boost due to noise injection. The figure shows the evolution of an EM and a noisy EM (NEM) algorithm on the same log-likelihood surface. The noise-enhanced algorithm converges to the same solution $30\%$ faster than the EM algorithm. Subsequent chapters discuss the details of noise-enhanced EM algorithms like the one in this figure.

\begin{figure}[bt]
	\centerline{\bf{Log-likelihood Comparison of EM and Noise-enhanced EM}}
	\centerline{\includegraphics[width=0.95\textwidth]{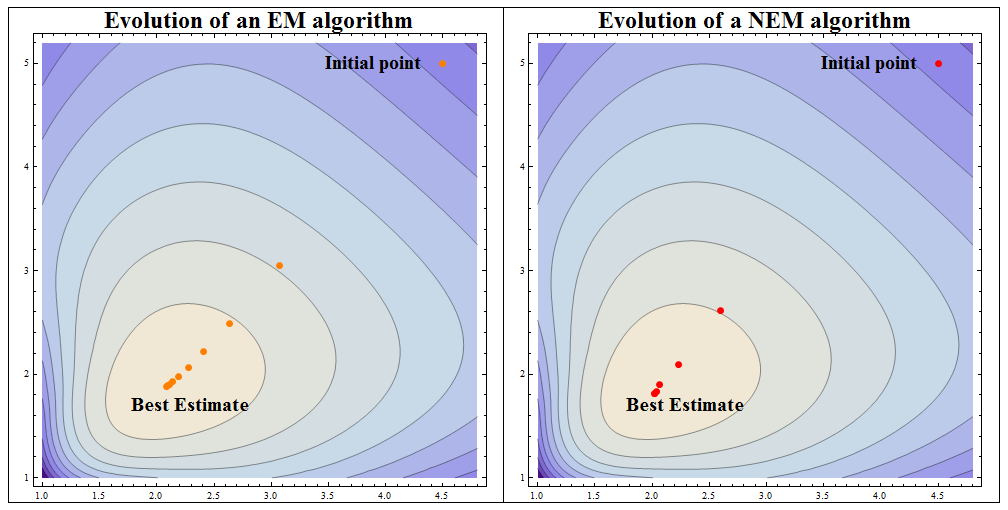}}
	\caption[Demonstration: noise injection can speed up the convergence of the EM algorithm]{
		Plot showing that noise injection can speed up the convergence of the EM algorithm. The plot shows the evolution of an EM algorithm on a log-likelihood surface with and without noise injection. Both algorithms estimate a $2$-dimensional ML parameter on the same data. Both algorithms also start at the same initial estimate and converge to the same point on the log-likelihood surface. The EM algorithm converges in $10$ iterations while the noise-enhanced algorithm converges in $7$ iterations--$30\%$ faster than the EM algorithm. 
	}
\label{fg:NEM-Demo}
\end{figure}

\subsection{Formulation of the EM Algorithm}
We now formulate the EM algorithm using the following notation:
\begin{align*}
\theta \in \Theta \subseteq \mathcal{R}^d &: \text{unknown $d$-dimensional pdf parameter}\\
\mathcal{X} &: \text{Complete data space}\\
\mathcal{Y} &: \text{Observed data space}\\
Y &: \text{observed data random variable with pdf } g(y|\theta) \\
Z &:\text{unobserved latent random variable with pdf } f(z|y, \theta) \\
X = (Y,Z) &: \text{complete data random variable with pdf} f(x|\theta) \\
f(x|\theta) =f(y,z|\theta) &:\text{joint pdf of } Z \text{ and }Y\\
L_c(\theta|x) = \ln f(x|\theta) &: \text{complete data log-likelihood} \\
L(\theta|y) = \ln g(y|\theta) &: \text{observed data log-likelihood} \\
\hat{\theta} &: \text{MLE for } \theta \\ 
\theta_t &: \text{an estimate for } \hat{\theta} \text{ at iteration } k
\end{align*} 
We use the full log-likelihood notation $L(\theta| y)$ instead of just $L(\theta)$ to avoid confusion between log-likelihoods for the different random variables.

The goal of the EM algorithm is to find the estimate $\hat{\theta}$ that maximizes $g(y|\theta)$ given observed samples of $Y$ i.e.
\begin{equation}
\hat{\theta} = \argmax{\theta} \ln g(y|\theta) = \argmax{\theta} L(\theta|y) \;.
\end{equation}
The algorithm makes essential use of the complete data pdf $f(x|\theta)$ to achieve this goal.

The EM scheme applies when we observe an incomplete data random variable $Y=r(X)$ instead of the complete data random variable $X$. The function $r: \mathcal{X} \rightarrow \mathcal{Y}$ models data corruption or information loss. $X=(Y,Z)$ can often denote the complete data $X$ where $Z$ is a latent or missing random variable. $Z$ represents any statistical information lost during the observation mapping $r(X)$. This corruption makes the observed data log-likelihood $L(\theta|y)$ complicated and difficult to optimize directly in ({\ref{eq:MLE}}).

The EM algorithm addresses this difficulty by using the simpler complete likelihood $L_c(\theta|y,z)$ to derive a surrogate log-likelihood $Q(\theta|\theta_t)$ as a replacement for $L(\theta|y)$. $Q(\theta|\theta_t)$ is the average of $L_c(\theta|y,z)$ over all possible values of the unobserved latent variable $Z$ given the observation $Y=y$ and the current parameter estimate $\theta_t$:
\begin{align}
	Q(\theta|\theta_t) &= \E_{Z} \left[ L_c(\theta|y,Z) \big{|} Y=y, \theta_t \right] \nonumber \\
	&= \int_{\mathcal{Z}} L_c(\theta|y,z)f(z|y,\theta_t) ~\mathrm{d}z.
\end{align} 

Dempster et al. \cite{dempster-laird-rubin1977} (DLR) first showed that any $\theta$ that increases $Q(\theta|\theta_t)$ cannot reduce the likelihood difference $\ell(\theta|y)-\ell(\theta_t|y)$. This ``ascent property'' led to an iterative method that performs gradient ascent on the likelihood $\ell(\theta|y)$. This result underpins the EM algorithm and its many variants \parencite{fessler-hero1994, hudson-larkin1994, meng-rubin1993, liu-rubin1994, liu-rubin-wu1998,meng-vandyk1997}. % SEM (celeux-diebolt1985) does not have a deterministic ascent property...

A standard EM algorithm performs the following two steps iteratively on a vector $\mathbf{y} = (y_1,\ldots,y_M)$ of observed random samples of $Y$: (assume $\theta_0$ is a suitable random initialization)

\vspace{6pt}

%%%%%%%%%% From: Noiseless EM Section %%%%%%%%%
\begin{algorithm}[H]
\DontPrintSemicolon
\SetKwInOut{Input}{Input}
\SetKwInOut{Output}{Output}
\Input{$\mathbf{y} = \{y_i\}_i$ : vector of observed incomplete data}
\Output{ $\hat{\theta}_{EM}$ : EM estimate of parameter $\theta$}
\BlankLine
\While{($\|\theta_t-\theta_{t-1}\| \geq 10^{-tol}$)}{
{\bf E-Step:} $Q \left( \theta |\theta_t \right) \leftarrow\E_{\mathbf{Z}|\mathbf{y},\theta_t}  \left[L_c(\theta|\mathbf{y},\mathbf{Z}) \right] $\;
{\bf M-Step:} $\theta_{t+1} \leftarrow \argmax{\theta} \left\{ Q\left( \theta |\theta_t \right) \right\}$\;
$t \leftarrow t+1$\;
}
$\hat{\theta}_{EM} \leftarrow \theta_t$
\caption{The Expectation Maximization Algorithm}\label{algo:EM}
\end{algorithm}

\vspace{6pt}

\noindent $L(\theta_t|y)$ increases or remains constant with each EM iteration. Thus the algorithm performs a gradient ascent procedure on the likelihood surface $L(\theta|y)$. The algorithm stops when successive estimates differ by less than a given tolerance $\|\theta_t-\theta_{t-1}\| < 10^{-tol}$ or when $\|L(\theta_t|y)-L(\theta_{t-1}|y)\| < \epsilon$ \parencite{gupta-chen2010}. The algorithm converges when the estimate is close to a local optimum \parencite{mclachlan-krishnan2007,gupta-chen2010,wu1983, boyles1983}.

\subsection{Complete Data Spaces for Incomplete Data Models}
EM algorithms optimize observed data log-likelihoods $L(\theta|Y)$ via the complete log-likelihood $L_c(\theta|Y)$ on a complete data space (CDS). The CDS specifies the complete data random variable $X$ and its likelihood function $L_c(\theta|X)$. The complete data random variable $X$ and the associated latent variable $Z$ depend crucially on the data model. The selection of $Z$ or $X$ determines the E-step function $Q(\theta|\theta_t)$. Careful CDS selection can improve the analytic, computational, and convergence properties of the EM algorithm \parencite{fessler-hero1993, fessler-hero1994}. 

So far we have assumed that the complete data random variable $X$ augments the observed data $Y$ with a latent variable $Z$ via a direct product operation $X=(Y,Z)$. This direct product complete data space identifies the complete data random variable as $X=(Y,Z)$. Then the data corruption function $r(X)$ is a projection onto the observed data space $\mathcal{Y}$. And the observed data pdf is just a marginal pdf of the complete data pdf:
\begin{align}
X&=(Y,Z) \\
Y&= r(X) = proj_\mathcal{Y}(X) \\
g(y|\theta) &= \int_{r^{-1}(y)} f(x|\theta) ~\mathrm{d}x = \int_{Z} f(y,z|\theta) ~\mathrm{d}z \;.
\end{align}
This approach is standard for estimating parameters of mixture models. Mixture model data comes from an unknown convex mixture of sub-populations. The latent random variable $Z$ identifies the sub-population to which a sample $Y$ belongs. 

The direct product CDS models are illustrative for exposition purposes. General observation functions $r(.)$ are often re-expressible as functions over the observed data $Y$ and some ancillary variable $Z$. But direct products are not the only general way to model CDSs. Many observation functions $r(\cdot)$ do not identify the latent variable $Z$ explicitly.

Another approach is to use the complete random variable $X$ as the latent random variable $Z$. All the action is in the observation transformation $r(.)$ between the complete and the observed data spaces. The right-censored gamma model gives such an example \parencite{chauveau1995, tan-tian-ng2010}. Censored models are common in survival and lifetime data analysis. The latent random variable $Z$ is the unobserved uncensored gamma random variable. 
\begin{align}
X &= Z\\
Y &= r(X) = \min\{X, T\}\\
g(y|\theta) &= f_X\left(y|X \in [0,T] \right) \;.%\Big{|}_{x=y}
\end{align}
Suppose the data $Y$ represents the lifetime of subjects after a medical procedure. Right censorship occurs if the survey experiment keeps track of lifetime data up until time-index $T$. Then the latent random variable $Z$ is the true unobserved lifetime of a subject which may exceed the censorship time-index $T$.

These CDSs are special cases. Identifying $r(X)$ can be difficult for general data models. But a functional description of $r(X)$ is not necessary. The main requirement for an EM algorithm is a description of how $r(X)$ changes the complete likelihood $L_c$ into the observed likelihood $L$:
\begin{align}
Y &= r(X) \\
g(y|\theta) &= \int_{r^{-1}(y)} f(x|\theta) \mathrm{d}x \\
\text{Where:  } r^{-1}(y) &= \{ x \in \mathcal{X}| r(x)=y  \} \;.
\end{align}

Fessler and Hero (\cite*{fessler-hero1993, fessler-hero1994}) introduced a further generalization to the idea of a CDS. They defined an \emph{admissible complete data space} $\mathcal{X}$ for the observed density $g(y|\theta)$ as a CDS whose joint density of $X \in \mathcal{X}$ and $Y \in \mathcal{Y}$ satisfies 
\begin{equation}
f(y,x|\theta) = f(y|x) f(x|\theta) \label{eq:CDS-Defn}
\end{equation}
with $f(y|x)$ independent of $\theta$. The classical EM setup assumes that $x$ selects $y$ via the deterministic assignment $y=r(x)$. The admissible CDS reduces to the classical deterministic CDS when the conditional pdf $f(y|x)$ is just a delta function:
\begin{equation}
f(y,x|\theta) = \delta(y-r(x)) f(x|\theta). \label{eq:CDS-Subsume}
\end{equation}
The admissible CDS definition in equation (\ref{eq:CDS-Defn}) allows for a more general case where the corrupting process $r(.)$ is a random transformation. A CDS is admissible under this definition only if the transformation $y=r(x)$ is independent of $\theta$. The use of admissible complete data spaces adds another level of flexibility to EM algorithms. The cascade \parencite{segal-weinstein1988} and SAGE variants \parencite{fessler-hero1994} of the EM algorithm manipulate admissible CDSs to speed-up EM convergence for data models with large parameter spaces $\boldsymbol{\Theta}$.

\section{EM Convergence Properties}
The ascent property of the EM algorithm~\parencite{dempster-laird-rubin1977} is its main strength. It ensures the stability of the EM algorithm. And it is the first step in the proof that the EM algorithm converges to desirable parameter estimates. Below is a statement and proof this property:

\begin{prop}{\bf{[The Ascent Property of the EM Algorithm]:}} \label{thm:ascent}\\
\begin{equation}
	L(\theta|y)-L(\theta_t|y) \geq Q(\theta|\theta_t) - Q(\theta_t|\theta_t) \label{eq:Ascent}
\end{equation}
\end{prop}

\begin{proof}
\noindent By Bayes theorem:
\begin{align}
	f(z|y,\theta)&= \frac{f(z,y|\theta)}{g(y|\theta)} \\
	\ln f(z|y,\theta)&= \ln f(z,y|\theta) - L(\theta|y) \\
	L(\theta|y)&= \ln f(z,y|\theta) - \ln f(z|y,\theta)\;.
\end{align}
Average out the latent random variable $Z$ conditioned on the observed data $y$ and the current parameter estimate $\theta_t$ by applying the conditional expectation operator $E_{Z|y,\theta_t}[\cdot]$:
\begin{equation}
	E_{Z|y,\theta_t}[L(\theta|y)] = E_{Z|y,\theta_t}[\ln f(Z,y|\theta)] - E_{Z|y,\theta_t}[\ln f(Z|y,\theta) ] \label{eq:condExp}
\end{equation}
\noindent Therefore
\begin{equation}
	L(\theta|y) = Q(\theta|\theta_t) - H(\theta|\theta_t) 
\label{eq:LQH-intro}
\end{equation}
\noindent Where:
\begin{align}
	Q(\theta|\theta_t) \triangleq& \E_{Z|y,\theta_t}[\ln f(Z,y|\theta)] \\
	H(\theta|\theta_t) \triangleq& \E_{Z|y,\theta_t}[\ln f(Z|y,\theta)] 
\end{align}

$L(\theta|y)$ has no dependence on $Z$ in (\ref{eq:condExp}). Thus the expectation leaves it unchanged. The goal of MLE is to maximize $L(\theta|y)$. So our immediate goal is to force $L(\theta|y)$ upwards. Compare $L(\theta|y)$ to its current value $L(\theta_t|y)$ using (\ref{eq:LQH-intro}):
\begin{align}
	L(\theta|y)-L(\theta_t|y) &=  Q(\theta|\theta_t) - H(\theta|\theta_t) - Q(\theta_t|\theta_t) + H(\theta_t|\theta_t)  \\
	&=Q(\theta|\theta_t) - Q(\theta_t|\theta_t) + \left[H(\theta_t|\theta_t) - H(\theta|\theta_t)\right]  
\label{eq:pre-EndGame}
\end{align}
Examine the terms in square brackets:
\begin{align}
	H(\theta_t|\theta_t) - H(\theta|\theta_t) &= \int_X \left[\ln f(z|y,\theta_t)-\ln f(z|y,\theta)\right] f(z|y,\theta_t) \mathrm{d}x \\
	&=  \int_X \ln \left[\frac{f(z|y,\theta_t)}{f(z|y,\theta)}\right] f(z|y,\theta_t) \mathrm{d}x
\end{align}
\noindent The expectation of a logarithm of a pdf ratio in the last line
\begin{equation}
	\int_X \ln \left[\frac{f(z|y,\theta_t)}{f(z|y,\theta)}\right] f(z|y,\theta_t) \mathrm{d}x = D\left(f(z|y,\theta_t) \Vert f(z|y,\theta)\right)
\end{equation}
is the Kullback-Leibler divergence or the relative entropy $D(\cdot \Vert \cdot)$ \parencite{cover-thomas91, kullback-leibler1951}. 
$D(\cdot \Vert \cdot)$ is always non-negative
\begin{equation}
	D\left(f(z|y,\theta_t) \Vert f(z|y,\theta)\right) \geq 0. 
\end{equation} 
Hence \begin{equation}H(\theta_t|\theta_t) - H(\theta|\theta_t) \geq 0. \label{eq:KL-positivity} \end{equation} Thus Equations (\ref{eq:pre-EndGame}) and (\ref{eq:KL-positivity}) give the ascent property inequality:
\begin{equation}
	L(\theta|y)-L(\theta_t|y) \geq Q(\theta|\theta_t) - Q(\theta_t|\theta_t) %\label{eq:Ascent}
\end{equation}
\end{proof}

This implies that any value of $\theta$ that increases $Q$ relative to its current value $Q(\theta_t|\theta_t)$ will also increase $L(\theta|y)$ relative to $L(\theta_t|y)$. The standard EM algorithm chooses the value of $\theta$ that gives a maximum increase in $Q$. i.e. 
\begin{equation}
\theta_{t+1} = \argmax{\theta}~Q(\theta|\theta_t).
\end{equation}
The \emph{Generalized Expectation-Maximization} (GEM) relaxes the M-step by requiring only improved intermediate estimates $Q(\theta_{t+1}|\theta_t) \geq  Q\left( \theta_t |\theta_t \right)$ instead of maximized intermediate estimates. 

\vspace{6pt}

\begin{algorithm}[H]
{\bf M-Step:} $\theta_{t+1} \leftarrow \tilde{\theta}$  such that $Q(\tilde{\theta}|\theta_t) \geq  Q\left( \theta_t |\theta_t \right) $
\caption{Modified M--Step for Generalized EM.}\label{algo:GEM}
\end{algorithm}

\vspace{6pt}

The ascent property implies that the EM algorithm produces a sequence of estimates $\{ \theta_t \}_{t=0}^\infty$ such that
\begin{equation}
L(\theta_{t+1}|y) \geq L(\theta_{t}|y).
\end{equation}
If we assume that the likelihood is bounded from above, then we get a finite limit for the log-likelihood sequence:
\begin{equation}
\lim_{t \rightarrow \infty} L(\theta_{t}|y) = L^* < \infty \;.
\end{equation}
The existence of this limit does not mean that $L^*$ is a global maximum. $L^*$ can be a stationary point i.e. a saddle point, a local maximum, or a global maximum. This limit also does not mean the sequence of EM iterates $\{ \theta_t \}_{t=0}$ necessarily converges to a single point. DLR purported to prove this erroneous result. But Boyles \parencite*{boyles1983} presented a counter-example of a GEM that fails to converge. The GEM in the counter-example produced non-convergent estimates stuck in an uncountable, compact, connected set.

\subsection{Convergence of the (G)EM Algorithm}
GEM estimates do not converge to the global maximizer (the ML estimate) in general. The only guarantee is that GEM estimates $\theta_t$ converge to a point $\theta_*$ in the set of stationary points $\mathcal{S}$ for the log-likelihood $L$. This section presents a proof of this claim. The proof is a direct application of Zangwill's \emph{Convergence Theorem A} \parencite[][p. 91]{zangwill1969}. The proof applies  to the GEM algorithm and therefore  to the EM algorithm by subsumption.

We first define some terms. A \emph{point-to-set} $M: \Theta \rightarrow 2^\Theta$ is a function that the maps points $\theta \in \Theta$ to subsets of $V$ i.e. $M: \theta \rightarrow M(\theta) \subset \Theta$. A point-to-set map $M$ is \emph{closed} on $J \subset \Theta$ if
\begin{align}
\text{for all points} \quad \theta_\infty &\in J \\
\text{such that} \quad \theta_t &\rightarrow \theta_\infty \\
\text{and} \quad \phi_t &\in M(\theta_t) \\
\text{and} \quad \phi_t &\rightarrow \phi_\infty ~~ \text{for some}~ \phi_\infty
\end{align}
then
\begin{equation}
\phi_\infty \in M(\theta_\infty).
\end{equation} This is a generalization of function continuity to point-to-set maps. A function $f$ is continuous if 
\begin{equation}
	\lim_{t \rightarrow \infty} f(x_t) = f(\lim_{t \rightarrow \infty} x_t)\;.
\end{equation}
The closedness property for the map $M$ is similar. The last equation is exactly analogous to the closed map if we replaces equality $=$ with set-elementhood $\in$ and identify $\phi_\infty = \lim_t M(\theta_t)$.

Zangwill \parencite*{zangwill1969} defines an \emph{algorithm} for non-linear optimization as a sequence of point-to-set maps $M_t(\theta)$ that generate a sequence of estimates $\{ \theta_t \}$ via the recursion:
\begin{equation}
\theta_{t+1} \in M(\theta_t).
\end{equation}
Algorithms solve optimization problems by generating a sequence which converges to elements in the problem's solution set $\mathcal{S} \subset \Theta$. The GEM algorithm is an example. The next theorem by Zangwill gives conditions under which general optimization algorithms converge. A sketch of Zangwill's proof follows the theorem.

\begin{thm}{\bf{[Zangwill's Convergence Theorem A]:}} \label{thm:zangwill}

\noindent Suppose the algorithm $M: \Theta \rightarrow 2^\Theta$ generates a sequence $\{ \theta_t \in M(\theta_{t-1})\}_t$. Let $\mathcal{S}$ be the set of solution points for a continuous objective function $L(\theta)$. Suppose also that
\begin{enumerate}
\item $\theta_t \in J \subset V$ for all $t$ where $J$ is compact.
\item $M$ is closed over $\mathcal{S}^c$.
\item if $\theta \in \mathcal{S}^c$ then $L(\psi) > L(\theta)$ for all $\psi \in M(\theta)$
\item if $\theta \in \mathcal{S}$ then $L(\psi) \geq L(\theta)$ for all $\psi \in M(\theta)$.
\end{enumerate}

Then the limit $\theta_*$ of the sequence $\{ \theta_t \}$ or any of its convergent subsequences is in the solution set $\mathcal{S}$.
\end{thm}

\begin{proof}
See Zangwill \parencite*[][pp. 91--94]{zangwill1969}.

The algorithm either stops at a solution or generates an infinite sequence. Suppose $\{ \theta_t \}$ is an infinite sequence on a compact set by condition (1). So $\{ \theta_t \}$ must have a convergent subsequence $\{ \theta_{\tau} \}$ with limit $\theta_*$. $L$ is a continuous function. So $\lim_\tau L(\theta_\tau) = L(\theta_*)$. The sequence $\{ L(\theta_t) \}_t$ is monotone increasing by conditions (3) and (4). Thus the original sequence $\{ L(\theta_t) \}_t$ and the subsequence $\{ L(\theta_\tau) \}_\tau$ converge to the same value $L(\theta_*)$
\begin{equation}
	\lim_t L(\theta_t) = \lim_\tau L(\theta_\tau) = L(\theta_*) \;. \label{eq:lim-L}
\end{equation}
This holds for any subsequence of $\{ L(\theta_t) \}_t$

Suppose $\theta_*$ is not a solution for the objective function $L$. Define a derived subsequence $\{ \theta_{k} \}_{k}$ such that
\begin{equation}
	\theta_{k} \in M(\theta_\tau)
\end{equation}
for all $\tau$. This derived sequence is the original subsequence $\{ \theta_{\tau} \}$ after a single step advance. The derived sequence will also have a convergent subsequence $\{ \theta_{\tilde k} \}$ with limit $\lim_{\tilde k} \theta_{\tilde k} = \theta_{*+1}$. Then $\theta_{*+1} \in M(\theta_*)$ by the closure condition in (2). $\theta_*$ is not a solution by assumption. So 
\begin{equation}
L(\theta_{*+1}) > L(\theta_*) \label{eq:L-contra}
\end{equation} by condition (3). A similar argument for (\ref{eq:lim-L}) means that
\begin{equation}
	L(\theta_{*+1}) = L(\theta_*)\;.
\end{equation}
This contradicts (\ref{eq:L-contra}). Thus $\theta_*$ must be a solution.
\end{proof}

\begin{thm}{\bf{[Wu's (G)EM Convergence Theorem]:}} \label{thm:wu-converge}\\
Suppose the log-likelihood function $L(\theta|y)$ is continuous and bounded from above. Suppose also that the set $J$
\begin{equation}
J = \{ \theta \in \Theta |  L(\theta|y) \geq L(\theta_{0}|y) \}
\end{equation} is compact for all $\theta_{0}$ and $Q(\psi|\phi)$ is continuous in both $\psi$ and $\phi$. Then the limit points of a (G)EM sequence $\{ \theta_t \}$ are stationary points of the $L(\theta|y)$.
\end{thm}

\begin{proof}
The theorem applies directly to GEM algorithms under the following identifications and assumptions~\parencite{wu1983}. The objective function $L(\theta)$ is the observed data log-likelihood $L(\theta|y)$. The solution set $\mathcal{S}$ is the set of interior stationary points of $L$. 
\begin{equation}
\mathcal{S} = \{ \theta \in int(\Theta) ~~ |~~~  L'(\theta|y) = 0\}. \label{eq:GEM-solutionset}
\end{equation}
The \emph{point-to-set} algorithm map $\theta \rightarrow M(\theta_t)$ for the GEM algorithm is
\begin{equation}
M(\theta_t) = \{ \theta \in \Theta ~~ | ~~~ Q(\theta|\theta_t) \geq Q(\theta_t|\theta_t)\}. \label{eq:GEM-pointset-map}
\end{equation}

The current estimate $\theta_t$ is an element of the set $J_t$ 
\begin{equation}
\theta_t \in J_t = \{ \theta \in \Theta |  L(\theta|y) \geq L(\theta_{t-1}|y) \} \;.
\end{equation}
If $\theta_0$ is the initial estimate then the largest $J_t$ set is 
\begin{equation}
J = J_0 = \{ \theta \in \Theta |  L(\theta|y) \geq L(\theta_{0}|y) \}
\end{equation}
and $J_t \subseteq J $ for all $t$ by the set definitions and the GEM ascent property. The log-likelihood $L(\theta|y)$ is bounded from above. The set $J$ is compact by assumption and $\theta_t \in J$ for all $t$, thus satisfying Zangwill's condition (1). This is a generalization of the case where the entire parameter $\Theta$ space is compact.

It is hard to verify that arbitrary $M$-map algorithms are closed over $\mathcal{S}^c$ in general \parencite{zangwill1969}. But $M$-maps for (G)EM algorithm are closed whenever $Q(\psi|\theta)$ is continuous in both $\psi$ and $\theta$ \parencite{wu1983}, thus satisfying Zangwill's condition (2). Zangwill's conditions (3) and (4) follow because of the ascent property and because the log-likelihood has an upper bound. %\todo{Not quite Wu's assertion in GEM case... I think he's full of it. Possible to do continuity in just $\theta$?}

These four conditions imply that the Zangwill's convergence theorem holds. Thus the (G)EM iterates $\{\theta_t\}_t$ converge to the solution set $\mathcal{S}$ of stationary points of the log-likelihood $L$.
\end{proof}

The final (G)EM estimate is a fixed-point of the GEM point-to-set map $M$. This is because it converges to estimates $\theta_*$ such that
\begin{align}
	\theta_* \in M(\theta_*) \;.
\end{align}
{Point-to-set maps or \emph{correspondences} feature prominently in the theory of Nash equilibria in game theory. Kakutani's \parencite*{kakutani1941} fixed-point theorem gives conditions under which these maps have fixed-points. That theorem applies to (G)EM algorithms only if $M(\theta_t)$ is a convex set at each iteration. This is not always the case for EM algorithms.}

In summary, (G)EM guarantees that the estimates $\{ \theta_t \}$ converge into a set of stationary points for the log-likelihood when the $M$ map is closed. The $M$ map is closed when $Q(\psi|\phi)$ is continuous in both $\psi$ and $\phi$. This continuity condition holds for incomplete data models based on exponential-family pdfs~\parencite{sundberg1976,wu1983}. These exponential family models account for a large number of EM applications (e.g. censored gamma, Gaussian mixtures, convolutions of exponential pdfs). 
Stronger conditions on $Q(\psi|\phi)$ and its first-order partial derivatives can replace the closed $M$-map condition for convergence~\parencite[cf. Theorem 8.2]{little-rubin2002}. Further stronger conditions on $Q(\psi|\phi)$ may restrict that solution set to just local maxima instead of stationary points.

The (G)EM sequence $\{ \theta_t \}$ may not converge point-wise ($\theta_t \rightarrow \theta_*$) if the sequence gets trapped in a compact, connected non-singleton set over which the likelihood is flat (like in Boyles'~\parencite*{boyles1983} example). The (G)EM sequence must be a Cauchy sequence or the likelihood must have unique stationary points to avoid this kind of pathology \parencite{boyles1983,wu1983,hartley-hocking1971}. But this kind of pathology is usually easy to detect in practice. And the convergence of the log-likelihood sequence $\{ L(\theta_t|y) \}_t$ to stationary points is more important than the point-wise convergence of the estimate sequence $\{ \theta_t \}_t$ for MLE problems.

\subsection{Fisher Information in EM Algorithms} \label{sec:Fisher-EM}
The Fisher information
\begin{equation}
	\mathcal{I}(\theta)=\E_{Y|\theta} \left[ \left(\frac{\partial L(\theta)}{\partial \theta} \right)^2 \right] = -\E_{Y|\theta} \left[\frac{\partial^2 L(\theta)}{\partial \theta^2} \right] 
\label{eq:Fisher-Info}
\end{equation}
plays an important part in MLE methods. The Fisher information measures the average amount of information each sample carries about the parameter. It is a measure of precision for ML estimates. This is because the ML estimator $\hat{\theta}_n$ is asymptotically normal with variance equal the Fisher information inverse $[\mathcal{I}_n(\theta)]^{-1}$. So the estimator $\hat{\theta}_n$ for large $n$ has standard error approximately equal to $(\mathcal{I}_n(\theta))^{-0.5}$. Both the observed Fisher information $I(\theta)=-L''(\theta)$ and the expected Fisher information $\mathcal{I}(\theta)$ are useful for measuring the standard errors of EM estimates. But Efron and Hinkley~\parencite*{efron-hinkley1978} showed that the observed Fisher information is the preferred measure of standard errors.

The Fisher information has a further descriptive role for EM algorithms. It quantifies how much missing information the complete data random variable has to impute to complete the observed data. We can use equation (\ref{eq:LQH-intro}) to analyze the relative information in the observed and complete data samples:
\begin{align}
	L(\theta|y) &= Q(\theta|\theta_t) - H(\theta|\theta_t) \\
%\therefore \frac{\partial^2 L(\theta|y)}{\partial \theta^2} &= \frac{\partial^2 Q(\theta|\theta_t)}{\partial \theta^2} - \frac{\partial^2 H(\theta|\theta_t) }{\partial \theta^2} \\
	\therefore \frac{-\partial^2 L(\theta|y)}{\partial \theta^2} &= \frac{-\partial^2 Q(\theta|\theta_t)}{\partial \theta^2} - \frac{-\partial^2 H(\theta|\theta_t) }{\partial \theta^2}
\end{align}
\begin{equation}
	\therefore I_\textsc{observed} = I_\textsc{complete} - I_\textsc{missing} \;. \label{eq:EM-Obs-Fisher}
\end{equation}
These second derivatives are the observed Fisher information for the observed data, the complete data, and the missing data respectively. The expectation over the population $\E[I]$ gives the Fisher information $\mathcal{I} = \E \left[ I \right]$. The expectations are over $Y$:
\begin{align}
	\E_{Y|\theta} \left[\frac{-\partial^2 L(\theta|y)}{\partial \theta^2}\right] &= \E_{Y|\theta} \left[\frac{-\partial^2 Q(\theta|\theta_t)}{\partial \theta^2}\right] - \E_{Y|\theta} \left[\frac{-\partial^2 H(\theta|\theta_t) }{\partial \theta^2} \right]
\end{align}
\begin{equation}
	\therefore \mathcal{I}_\textsc{observed}(\theta) =\mathcal{I}_\textsc{complete}(\theta) - \mathcal{I}_\textsc{missing}(\theta)\;. \label{eq:EM-Fisher}
\end{equation}

These observed and expected Fisher information decompositions in (\ref{eq:EM-Obs-Fisher}) and (\ref{eq:EM-Fisher}) are examples of Orchard and Woodbury's Missing Information Principle~\parencite{orchard-woodbury1972} in the EM scheme~\parencite{louis1982}. Equation (\ref{eq:EM-Fisher}) shows that missing information in an EM data model diminishes the information each sample of $Y$ holds about the parameter $\theta$. The more missing information the EM algorithm has to estimate, the less informative the observed samples are. Iterative MLE methods converge slowly to stated standard error levels when the samples are less informative. 

Dempster et al.~\parencite*{dempster-laird-rubin1977} showed that the observed Fisher information controls the rate of EM convergence when the current estimate $\theta_t$ is \emph{close} to the fixed-point (the local maximum) $\theta_*$. The EM algorithm convergence rate $\lambda_{EM}$ is proportional to the ratio of missing to complete information~\parencite{tanner1996}: 
\begin{equation}
	\lambda_{EM} \propto \frac{I_\textsc{missing} }{ I_\textsc{complete} }
	\label{eq:EM-convrate}
\end{equation}
in a \emph{small} neighborhood of $\theta_*$. Thus higher missing information (higher $I_\textsc{missing}$) leads to slower EM convergence in keeping with the Missing Information Principle.

\section{Variations on the EM Theme}
\subsection{M--step Variations}
There are a number of variations of the EM algorithm. EM variants usually modify either the M-step or the E-step. The M-step modifications are more straight-forward. The first M-step variant is the GEM algorithm \parencite{dempster-laird-rubin1977}. GEM is useful if the roots of the derivatives of $Q$ are hard to find. 

Another M-step variant is very useful for estimating multidimensional parameters $\boldsymbol{\theta} = \left\{\theta^i\right\}_i^d$. This variant replaces the M-step's single $d$-dimensional optimization for $\boldsymbol{\theta}$ with a series of 1-dimensional conditional maximizations (CM) for $\theta^i$ with all other $\left\{\theta^j\right\}_{j\neq i}$ fixed. The M-step may also conditionally maximize sub-selections of the parameter vector $\theta$ instead of single parameter dimensions. These conditional maximizations are typically easier than performing a single multidimensional maximization. This variant is the \emph{Expectation Conditional Maximization} (ECM) algorithm\parencite{meng-rubin1993}. 

Fessler and Hero developed the idea of iteratively optimizing subsets of the EM parameter vector. Fessler and Hero were working from a medical image reconstruction point of view. They based their approach on the relation in (\ref{eq:EM-convrate}) which states that EM convergence rates $\lambda_{EM}$ are inversely proportional to the amount of complete information embedded in the complete data space $\lambda_{EM} \propto ( I_\textsc{complete} )^{-1}$. Subspaces of the full parameter space correspond to smaller CDSs that are less informative than the full CDS~\parencite{fessler-hero1994}. Therefore iterative optimization on parameter subspaces can lead to faster EM convergence. The algorithm may alternate between different subspaces at each EM iteration. These switches change the complete data space for the E-step. So they called this EM-variant \emph{Space Alternating Generalized Expectation-maximization} (SAGE). 

SAGE, ECM, and other parameter subspace methods are very appealing for applications with large parameter spaces. Image reconstruction and tomography applications \parencite{shepp-vardi1982, vardi-shepp-kaufman1985,hudson-larkin1994} make heavy use of these EM methods. The parameter spaces in these applications are very large; every image pixel or voxel corresponds to a parameter.

Liu and Rubin (\cite*{liu-rubin1994}) described a simple extra variation on ECM. ECM (and EM in general) maximize the $Q$-function with the goal of improving the log-likelihood function $L(\theta|y)$. Liu and Rubin proposed replacing some of the CM step of ECM with steps that conditionally maximize the log-likelihood $L(\theta|y)$ directly. They called this the \emph{Expectation Conditional Maximization Either} (ECME) where the ``either'' implies that the CM-steps either conditionally maximize the $Q$-function or the log-likelihood $L(\theta|y)$. Meng and {van Dyk} later combined SAGE and ECME into the more general \emph{Alternating Expectation Conditional Maximization} (AECM) algorithm \parencite{meng-vandyk1997}.

\subsection{E--step Variations}
These E--step variants are no longer EM algorithm in the strict sense. They do not always satisfy the EM ascent property. But they share the EM theme and they may satisfy the ascent property with high probability for some cases.

E--Step variants try to simplify the computation of the conditional expectation and make the $Q$-function simpler to calculate. The $Q$-function is an expectation with respect to  a conditional distribution
\begin{align}
Q(\theta|\theta_t) &= \int_Z \ln f(y,Z|\theta) ~df(z|Y=y,\theta_t)\\
&= \E_{Z|Y=y, \theta_t}[\ln f(y,Z|\theta)].
\end{align}
A Monte Carlo approximation to this integral uses $M$ samples of $Z$ drawn from the current estimate of the conditional pdf $f(z|Y=y,\theta_t)$. Then the approximation is
\begin{equation}
\tilde{Q}(\theta|\theta_t) = \frac{1}{M} \sum_{m=1}^M \ln f(y,z_m|\theta). \label{eq:MCEM-Q}
\end{equation}
This is the sample mean of the random variable $\ln f(y,Z|\theta)$ for samples of $Z$ with pdf $f(Z|y, \theta_t)$. Thus the Strong Law of Large Numbers \parencite{durrett2010} implies that
\begin{equation}
\tilde{Q}(\theta|\theta_t)  \longrightarrow Q(\theta|\theta_t) \quad \text{with probability--one as }~~ M\rightarrow \infty.
\end{equation}
Replacing $Q$ with $\tilde{Q}$ gives the \emph{Monte Carlo Expectation Maximization} (MCEM) algorithm \parencite{wei-tanner1990}. MCEM is similar to multiple imputation where multiple conditional sample draws replace each missing point in a data set \parencite{rubin1996, little-rubin2002}. The MCEM iterations become increasingly similar to EM iterations as the number of Monte Carlo samples $M$ increases. MCEM may need to run multiple secondary iterations of Markov chain Monte Carlo (MCMC) \parencite{robert-casella2004} to get samples from the appropriate distribution on each EM iteration.

Stochastic Expectation Maximization (SEM) \parencite{celeux-diebolt1985} also uses Monte Carlo methods to approximate the EM approach. SEM ``completes'' each observed data $y$ with a single random sample of $z$ drawn from $f(z|Y=y, \theta_t)$. The M--step maximizes the complete log-likelihood $L_c(\theta|Y=y, Z=z)$. The SEM algorithm is most useful for data models where the complete data space has a simple direct product structure (e.g. mixture models). SEM draws a single Monte Carlo sample $z$ per observed data sample $y$ per iteration. While MCEM draws $M$ Monte Carlo samples per observed data sample $y$ per iteration. The difference between SEM and MCEM mirrors the difference between single imputation and multiple imputation as methods for handling missing data \parencite{schafer-graham2002,little-rubin2002}.

\subsection{MAP EM for Bayesian Parameter Estimation}
Maximum A Posteriori (MAP) estimation picks parameters $\theta$ that maximize the posterior distribution $f(\theta|y)$ instead of the likelihood $g(y|\theta)$. Bayes theorem together with a prior pdf $h(\theta)$ specify the posterior:
\begin{equation}
f(\theta|y) \propto h(\theta) g(y|\theta)
\end{equation}
The MAP estimate is thus
\begin{align}
\hat{\theta}_{MAP} &= \argmax{\theta} f(\theta|y) \\
\text{or} \quad \hat{\theta}_{MAP} &= \argmax{\theta} \ln f(\theta|y) \\
\hat{\theta}_{MAP} &= \argmax{\theta} \left(\ln g(y|\theta) + \ln h(\theta) \right)
\end{align}
The MAP estimate is the mode of the posterior pdf $f(\theta|y)$. MAP estimation reduces to MLE when the prior pdf is a constant function. The log-prior term adds a penalty for poor choices of $\theta$. This penalty effect makes MAP estimation an example of \emph{penalized} maximum likelihood estimation methods \parencite{green:JRSS1990}. However the penalty term for general penalized MLE methods may be arbitrary, motivated by optimization desiderata rather than specific Bayesian model considerations. The log prior term usually makes the objective function more concave and thus makes the optimization problem easier to solve \parencite{boyd-vandenberghe2004}. %also mclachlan-krishnan2007

MAP estimation for missing information problems may use a modified variant of the EM algorithm. The MAP variant modifies the $Q$-function by adding a log prior term $P(\theta) = \ln h(\theta)$:
\begin{align}
Q(\theta|\theta_t) = \E_{Z|Y, \theta_t}[L_c(\theta|y,Z)] + P(\theta) \;.
\end{align}
The MAP-EM framework is very prominent in the field of medical image analysis \parencite{hebert-leahy1989,hebert-leahy1992,green:TMI1990}. The use of prior knowledge can help smoothen the image and prevents reconstruction artifacts \parencite{depierro1995}.

\section{Examples of EM Algorithms}
We now outline EM algorithms for some popular data models. The E-step changes with the data models. The M-step is independent of the model.

\subsection{EM for Curved-Exponential Family Models} \label{subsec:Exp-EM}
EM algorithms for curved-exponential family models apply to observed data with complete data random variables from exponential family distributions. An exponential-family pdf with vector parameter $\boldsymbol{\theta}$ has the form
\begin{equation}
f(x|\boldsymbol{\theta}) = B(x)~\exp \big[ 
\mathbf{T}(x) \cdot \boldsymbol{\eta}(\boldsymbol{\theta}) - 
A(\boldsymbol{\theta}) 
\big]
\label{eq:pdf-Exp-Form}
\end{equation}
and log-likelihood
\begin{equation}
L(\boldsymbol{\theta}|x) = \mathbf{T}(x) \cdot \boldsymbol{\eta}(\boldsymbol{\theta}) - A(\boldsymbol{\theta}) + \ln B(x) \;.
\label{eq:LL-Exp-Form}
\end{equation}
Examples include the exponential ($exp(\theta)$), gamma ($\gamma(\alpha, \theta)$),Poisson ($p(\lambda)$) and normal ($N(\mu, \sigma^2)$) pdfs: 
\begin{align}
L_{exp}(\theta|x) &= - x~\frac{1}{\theta} -\ln(\theta) \;, \\
L_{\gamma}(\alpha,\theta|x) &=  x\left(- \frac{1}{\theta} + (\alpha-1) \right) -\ln(\Gamma(\alpha)) - \alpha\ln\theta \;, \\
L_{P}(\lambda|x) &= x~\ln\lambda - \lambda - \ln(x!) \;, \\
L_{N}(\mu, \sigma|x) &= (x, x^2) \cdot (\mu, -0.5)\frac{1}{\sigma^2}  - \frac{\mu^2}{2\sigma^2} - 0.5\ln(2\pi\sigma) \;.
\end{align}
EM algorithms on these models are the simplest~\parencite{sundberg1974,sundberg1976,dempster-laird-rubin1977}, and have the best possible convergence properties~\parencite{sundberg1974,wu1983} for moderate amounts of missing information. Some of the $Q$-functions are simple polynomials or even linear functions in the case of exponential or geometric pdfs.

One simple example is the EM algorithm on right-censored exponential data. The complete data random variable is $X \sim exp(\theta)$. Right-censorship occurs when the observer only records data up to a certain value $C$. This implies that the observed data $Y$ is the minimum of the complete data $X$ and the censorship point $C$. Censorship is common in time-limited trials (e.g. medical testing) and product reliability analyses. The $Q$-function is 
\begin{align}
Q(\theta|\theta_t) &= \E_{X|Y=y,\theta_t} \big[ L(\theta|X) \big] \\
&= -\ln(\theta) - \frac{1}{\theta} ~\E_{X|Y,\theta_t} \big[ X \big] \;. \label{eq:exp-Estep}
\end{align}
The log-linear $exp(\theta)$ likelihood means that the EM algorithm just replaces the latent random variable$X$ with the current best estimate for the latent random variable $\E_{X|Y,\theta_t}[X]$. The surrogate likelihood is easy to maximize
\begin{align}
Q'(\theta|\theta_t) &= -\frac{1}{\theta} + \frac{1}{\theta^2} ~\E_{X|Y,\theta_t} \big[ X \big] \;. \\
Q'(\theta|\theta_t)\big{|}_{\theta=\theta_{t+1}} &= 0 \\ 
\therefore \theta_{t+1} &= \E_{X|Y,\theta_t} \big[ X \big] \;. \label{eq:exp-Mstep}
\end{align}
Equations (\ref{eq:exp-Mstep}) is the EM-update for any incomplete data model on the exponential complete data space. The conditional expectation in the right-censorship case depends on the relation $Y=\min\{X,C\}$. Thus
\begin{align}
\E_{X|Y=y,\theta_t} \big[ X \big] &= 
\begin{cases}
y & \text{if } y < C \\
\E \big[X| \{X \geq C\},\theta_t \big]& \text{if } y = C
\end{cases}\\
\therefore \E_{X|Y=y,\theta_t} \big[ X \big] &= \begin{cases}
y & \text{if } y < C \\ 
C+\theta_t & \text{if } y = C
\end{cases} \;.
\end{align}

The gamma EM algorithm generalizes exponential EM algorithm using the following $Q$-function for the vector parameter $\boldsymbol{\varphi} = (\alpha, \theta)$:
\begin{equation}
Q(\boldsymbol{\varphi}|\boldsymbol{\varphi}_t) = \E_{X|Y=y,\boldsymbol{\varphi}_t} \big[
X(- \frac{1}{\theta} + (\alpha-1))
- \ln(\Gamma(\alpha)) - \alpha\ln\theta 
\big] 
\end{equation}
Taking the conditional expectation gives
\begin{align}
Q(\boldsymbol{\varphi}|\boldsymbol{\varphi}_t) &= \E_{X|y,\boldsymbol{\varphi}_t}[X] 
\big( -\frac{1}{\theta} + (\alpha-1) \big)
- \ln(\Gamma(\alpha)) - \alpha\ln\theta 
\label{eq:gamma-Estep}\\
\text{and}\quad \nabla_{\boldsymbol{\varphi}} Q(\boldsymbol{\varphi}|\boldsymbol{\varphi}_t) &=
\Big{\{}
-\frac{\partial \ln(\Gamma(\alpha))}{\partial \alpha} -\ln\theta, 
\frac{\E_{X|y,\boldsymbol{\varphi}_t}[X]}{\theta^2}- \frac{\alpha}{\theta} 
\Big{\}}\\
\nabla_{\boldsymbol{\varphi}} Q (\boldsymbol{\varphi}|\boldsymbol{\varphi}_t) \big{|}_{\boldsymbol{\varphi}=\boldsymbol{\varphi}_{t+1}} &= \mathbf{0} \;.
\end{align} 
The transcendental nature of the $\Gamma$-function and its derivative results in intractable closed-forms for the M-step for general values of $\alpha$. The M-step can use numerical estimates here. GEM algorithms on this model would have a simpler numerical M-step. The M-step just does a local search for an estimate that increases the $Q$-function. An ECM algorithm is also useful here because the $Q$-function is complicated only in the $\alpha$ coordinate. Users can decompose the M-step into an analytic conditional maximization in the $\theta$ coordinate and a numerical conditional maximization in the $\alpha$ coordinate.

\subsection{EM for Finite Mixture Models} \label{subsec:FMM-EM}

A finite mixture model~\cite{redner-walker1984,mclachlan-peel2004} is a convex combination of a finite set of sub-populations. The sub-population pdfs come from the same parametric family. Mixture models are useful for modeling mixed populations for statistical applications such as clustering and pattern recognition~\parencite{xu-wunsch2005}. We use the following notation for mixture models. $Y$ is the observed mixed random variable. $K$ is the number of sub-populations. $Z \in \left\{1,\ldots,K\right\}$ is the hidden sub-population index random variable. The convex population mixing proportions $\alpha_1,\ldots, \alpha_K$ form a discrete pdf for $Z$: $P(Z=j) = \alpha_j$. The pdf $f(y|Z=j,\theta_j)$ is the pdf of the $j^{th}$  sub-population where $\theta_1,\ldots, \theta_K $ are the pdf parameters for each sub-population. $\Theta$ is the vector of all model parameters $\Theta = \left\{\alpha_1, \ldots , \alpha_K, \theta_1, \ldots , \theta_K\right\}$. The joint pdf $f(y, z|\Theta)$ is
\begin{equation}
	f(y,z|\Theta) =\sum_{j=1}^K \alpha_j ~f(y|j,\theta_j) ~\delta[z-j] \;.
\end{equation}
The marginal pdf for $Y$ and the conditional pdf for $Z$ given $y$ are
\begin{align}
	f(y|\Theta)=&\sum_{j} \alpha_j f(y|j,\theta_j) \\
	\textrm{and} \qquad p_Z(j|y,\Theta) =& \frac{\alpha_j f(y|Z=j,\theta_j)}{f(y|\Theta)}
\label{eq:Z-condpdf}
\end{align}
by Bayes theorem.

We rewrite the joint pdf in exponential form for ease of analysis:
\begin{align}
	f(y,z|\Theta) = &\exp \left[ \sum_{j} \left[ \ln( \alpha_j) + \ln f(y|j,\theta_j) \right] \delta[ z-j]  \right] \;, \label{eq:FMM-fyz-ExpForm}\\
	\ln f(y,z|\Theta) =& \sum_j \delta[z-j] \ln[\alpha_j f(y|j,\theta_j)] \;. \label{eq:FMM-LL-ExpForm}
\end{align}

EM algorithms for finite mixture models estimate $\Theta$ using the sub-population index $Z$ as the latent variable. An EM algorithm on a finite mixture model uses (\ref{eq:Z-condpdf}) to derive the $Q$-function
\begin{align}
	Q(\Theta|\Theta_t) =& \E_{Z|y,\Theta_t}[\ln f(y,Z|\Theta)]  \\
	=& \sum_{z} \left(\sum_j \delta[z-j] \ln[\alpha_j f(y|j,\theta_j)]\right)\times p_Z(z|y,\Theta_t) \\
	=& \sum_{j} \ln[\alpha_j f(y|j,\theta_j)] p_Z(j|y,\Theta_t). 
\label{eq:Q_Mixture}
\end{align}
This mixture model is versatile because we can replace the sub-populations by changing the function $f(y|j,\theta_j)$. We can also use different parametric pdfs for the sub-populations. The versatility of FMMs makes them very popular in areas such as data clustering, genomics, proteomics, image segmentation, speech recognition, and speaker identification.

Mixture Models are not identifiable \parencite{teicher1961,teicher1963}. This means there may not be a one-to-one mapping between distinct mixture model distributions and distinct parameter vectors in the FMM parameter space. So there may be no unique maximum likelihood parameter estimate that specifies an FMM.

\begin{figure}[ht!]
\centerline{ \includegraphics[width=\columnwidth]{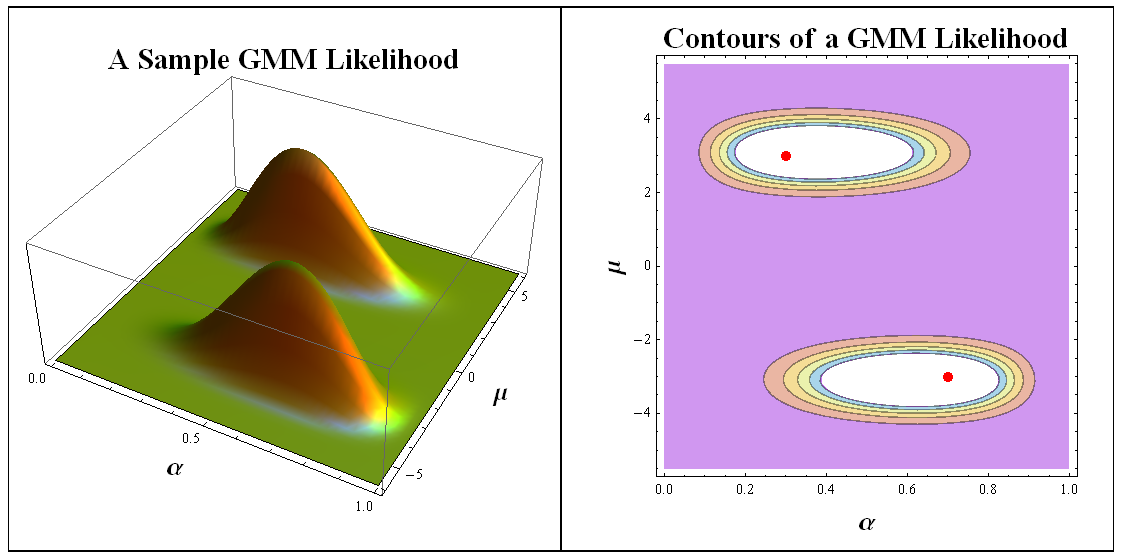} }
\caption[Demonstration: Finite mixture models may not be uniquely identifiable]{
	Demonstration: finite mixture models (FMMs) may not be uniquely identifiable. The figure shows a likelihood function  for $8$ samples of a mixture of two Gaussian distributions with equal variance sub-populations. Points on the likelihood space coordinate $\theta = (\alpha, \mu)$ represent all possible FMM distributions of the form $\mathcal{F} = \alpha \mathcal{N}(\mu, 1.5) + (1-\alpha) \mathcal{N}(-\mu, 1.5)$. The $8$ sample points generating the likelihood in the figure come from an FMM with true parameter values $(\alpha_*, \mu_*) = (0.3, 3)$. But the sample likelihood function has two maxima on the coordinate space at approximately $\theta_1= (\alpha_*, \mu_*)$ and $\theta_2 = (1-\alpha_*, -\mu_*)$. These are the dots in the contour plot. Both maxima represent alternate parameterizations of the same distribution.
%(\alpha_*, \mu_*) = (0.3, -3)?
}
\label{fg:FMM-nonID}
\end{figure}%a = 0.3; m = 3; s = 1.5;

\subsection{EM for Positron Emission Tomography}\label{subsec:PET-EM}

Positron Emission Tomography (PET) constructs a radioisotope emission profile for a biological organ. The emission profile image represents levels of metabolic activity, blood flow, or chemical activity within the organ. 

A PET scan starts with an injection of a radioisotope into the tissue of interest. An array of detectors around the tissue captures the radioisotope emissions (electrons in PET scans) from within the tissue. The model assumes a ``fine enough'' grid over the tissue so that emission statistics are uniform within each grid square. PET scans usually generate a series of 2-D images. So each grid square is a picture element or pixel. Each pixel can emit particles. The emissions are discrete events and the detectors simply count the number of emissions along multiple lines-of-sight. The geometry of the detector array determines the mode of reconstruction. It also sets up a latent random variable structure that admits an EM-model which Shepp and Vardi developed \parencite{shepp-vardi1982, vardi-shepp-kaufman1985}. The PET-EM model is a geometry-dependent finite mixture model where the sub-populations have a stable discrete distribution.

The probabilistic model for PET image reconstruction via EM assumes that emissions from each pixel are Poisson-distributed. A detector can only observe the sum of all pixel emissions along the detector's line-of-sight. So the detector records a geometry-dependent sum of Poisson random variables. The detector response is the observed random variable. The individual pixel emissions are the complete unobserved random variables. And the pixel intensities (Poisson parameters for each pixel) are the parameters we want to optimize.

The model starts with a 2-D array of tissue pixels
\begin{align}
\mathbf{X} = \{X_i\}_i^n : X_i \sim Poisson(\lambda_i)
\end{align}
where $n$ is the image size. The detector array is
\begin{align}
\mathbf{Y} = \{Y_j\}_j^d : Y_j \sim Poisson(\mu_j)
\end{align}
where $d$ is the number of detectors. An emission from $X_i$ can go to one of many detectors $Y_j$. The probability of emitting at $X_i$ and detecting at $Y_j$ is $p_{ij}$. This gives an array of parameters
\begin{align}
\mathbf{P} &= ((p_{ij}))_{i,j}^{n,d} \\
\text{where} \quad p_{ij} &= P\big( \text{Emit at }X_i, \text{Detect at }Y_j \big).
\end{align}
$\mathbf{P}$ depends on the geometry of the detector array. If we define the portion of $X_i$ emissions that go to $Y_j$ as $Z_{ij}$. Then
\begin{align}
Z_{ij} &= p_{ij} X_i\;. \\
\text{So} \quad Y_j &= \sum_i^n Z_{ij} \\
\text{and} \quad Y_j &\sim Poisson \Big(\sum_i^n p_{ij} \lambda_i \Big) \;.
\end{align}
The estimate for $\Lambda = \{ \lambda_ i\}_i^n$ gives the average emission intensities at each pixel. \emph{This is the PET image reconstruction}.

We use $\mathbf{Z} = \{Z_{ij}\}_{i,j}$ as the complete data random variable for the EM model. This gives the complete data pdf under an independent pixel assumption
\begin{align}
f(\mathbf{Z}| \Lambda) &= \prod_j \prod_i f(z_{ij}|\Lambda)\\
&= \prod_j \prod_i \exp(-p_{ij} \lambda_i) \frac{(p_{ij} \lambda_i)^{z_{ij}}}{z_{ij}!} \;.
\end{align}
Thus
\begin{equation}
\ln f(\mathbf{Z}| \Lambda) = \sum_j \sum_i -p_{ij} \lambda_i + z_{ij} \ln (p_{ij} \lambda_i) - \ln(z_{ij}!) 
\end{equation}
The $Q$ function for the EM model is
\begin{align}
Q(\Lambda | \Lambda(t)) &= \E_{\mathbf{Z}| \mathbf{Y,} \Lambda(t)} \Big[ \ln f(\mathbf{Z}| \Lambda) \Big] \\
Q(\Lambda | \Lambda(t)) &= \sum_j \sum_i -p_{ij} \lambda_i + \E_{\mathbf{Z}| \mathbf{Y,} \Lambda(t)}[Z_{ij}] \ln (p_{ij} \lambda_i) - \E_{\mathbf{Z}| \mathbf{Y,} \Lambda(t)}[\ln(z_{ij}!)].  \label{eq:Q-PET}
\end{align}
The Poisson pdf is an exponential-family pdf. So the $Q$-maximization has the closed-form:
\begin{align}
\lambda_i(t+1) = \lambda_i(t) \sum_j^d \Big( \frac{Y_j p_{ij}}{\sum_k^n \lambda_k(t) p_{kd} } \Big)
\end{align}

There has been a lot of work done on PET with EM. Some EM innovations originated from this problem space. These include the use of penalty terms or Gibbs priors to avoid singularities in the objective function \parencite{hebert-leahy1989,fessler-hero1995}, block iterative approaches to estimation \parencite{hudson-larkin1994}, and methods for convexifying the reconstruction objective function \parencite{depierro1995}. The basic structure of the model remains the same.

\section{MM: A Generalization for EM}
The EM algorithm relies on an incomplete data structure for which we supply a complete data space. Some estimation problems do not fit into this structure (e.g. estimating logistic regression coefficients). This invalidates the E-step. \emph{Minorization-Maximization} or (\emph{Majorization-Minimization}) (MM) \parencite{jacobson-fessler2007,becker-yang-lange1997} algorithms generalize the idea of a surrogate optimization function (the $Q$-function in EM). An MM algorithm specifies a \emph{tangent minorizing function}. A function $q$ is a tangent minorizer for $L(\theta|y)$ if $q(\theta|\theta_t)$ is a $\theta$-parametrized function which depends on the current estimate $\theta_t$ such that \parencite{wu-lange2010}
\begin{align}
q(\theta|\theta_t) &\leq L(\theta|y) \quad \forall \; \theta \in \Theta \label{eq:minorizeQ} \\
\text{and} \quad q(\theta_t|\theta_t) &= L(\theta_t|y). \label{eq:tangentQ}
\end{align}
MM algorithms maximize the surrogate $q$-function instead of $L(\theta_t|y)$ just like the EM algorithm:
\begin{equation}
\theta_{t+1} = \argmax{\theta} q(\theta|\theta_t).
\end{equation}
The difference of equations (\ref{eq:minorizeQ}) and (\ref{eq:tangentQ}) 
\begin{equation}
q(\theta|\theta_t) - q(\theta_t|\theta_t) \leq L(\theta|y) - L(\theta_t|y)
\end{equation}
establishes an analogous ascent property for MM algorithms. 

The MM algorithm transfers the optimization from the original objective function to a minorizing (or majorizing) function. Ideally the minorizing function is easier to optimize like the $Q$-function in the EM algorithm. Lange et al.~\cite*{lange-hunter-yang2000} use the term \emph{optimization transfer} instead of ``MM'' to highlight this transference behavior.

The EM algorithm fits into the MM scheme with a simple modification: EM's $Q$-function needs a constant level-shift $L(\theta_t|y) - Q(\theta_t|\theta_t)$ to satisfy the MM conditions. Level-shifts do not change the optimization. So MM subsumes the EM algorithm. This subsumption depends on the decomposition $L=Q-H$ (\ref{eq:LQH-intro}) and the ascent property (Prop. \ref{thm:ascent}). The minorizer for the EM algorithm is
\begin{equation}
q_{EM}(\theta|\theta_t) = Q(\theta|\theta_t) + [L(\theta_t|y) - Q(\theta_t|\theta_t)] \;.
\end{equation}
The ascent property implies the minorizing condition in  (\ref{eq:minorizeQ}):
\begin{align}
L(\theta|y) - L(\theta_t|y)&\geq Q(\theta|\theta_t) - Q(\theta_t|\theta_t) \\
\therefore L(\theta|y) &\geq Q(\theta|\theta_t) + [L(\theta_t|y) - Q(\theta_t|\theta_t)]\;.
\end{align}
The tangency condition in (\ref{eq:minorizeQ}) holds because 
\begin{align*}
q_{EM}(\theta_t|\theta_t) &= Q(\theta_t|\theta_t) - Q(\theta_t|\theta_t) + L(\theta_t|y)\\
&= L(\theta_t|y)\;.
\end{align*}
The M-step is the same (or may be one of the aforementioned M--step variants). This establishes EM as a special case of MM.

The EM subsumption argument shows that a primary property of EM and MM algorithms is that the gradient vectors of the objective and surrogate functions are parallel at the current estimate $\theta_t$. This ensures that the the surrogate function inherits the objective function's direction of steepest ascent/descent at $\theta_t$. The MM tangency property is just an indirect means to this goal. The tangency condition implies this parallel gradient vector property because curves that are tangent (but not crossing) at a point have parallel gradient vectors at that point of tangency. The EM algorithm also satisfies this parallel gradient property without requiring tangency since
\begin{equation}
\nabla_{\theta} L(\theta|y)\Big{|}_{\theta=\theta_t}  = \nabla_{\theta} Q(\theta|\theta_t) \Big{|}_{\theta=\theta_t}.
\end{equation}
This is one of Oakes' results on the derivatives of the observed likelihood in EM algorithms \parencite{oakes1999, meng2000}. The result holds because the residual $H$ in the $L=Q-H$ decomposition also has a stationary point precisely at $\theta=\theta_t$.

The Quadratic Lower-Bound (QLB) algorithm \parencite{bohning-lindsay1988} is another example of an MM algorithm. The QLB $q$-function comes from local convexity considerations instead of EM's missing information considerations. For example, a QLB algorithm on an MLE problem would use the following minorizer
\begin{equation}
q(\theta|\theta_t)=L(\theta_t|y)+(\theta-\theta_t)^T \nabla L(\theta_t|y) - \frac{1}{2} (\theta-\theta_t)^T \mathbf{B} (\theta-\theta_t) \label{eq:QLB}
\end{equation}
where $\mathbf{B}$ is any positive definite matrix that satisfies the inequality 
\begin{equation}
\nabla^2 L(\theta_t|y)+\mathbf{B} \geq 0.
\end{equation}
The QLB update equation is thus:
\begin{equation}
\theta_{t+1} = \theta_t + \mathbf{B}^{-1} \nabla L(\theta_t|y).
\end{equation}
These QLB constraints effectively create local quadratic approximations for $L$ at every step $t$. Thus QLB produces \emph{convex} minorizing functions at every QLB iteration. This is not necessarily true for the EM-algorithm in general.%depierro cite? \todo{Is the MM minorizing condition a local or global condition?} 

MM methods bypass the need for a CDS specification. But the user must specify custom tangent-minorizers for each estimation problem. Proponents of MM methods argue that designing minorizer functions is less difficult than designing complete data spaces for an EM model. And MM methods provide an extra level of flexibility: there are often many MM approaches to the same problem. However there is no published proof of convergence for general MM algorithms \parencite{jacobson-fessler2007}. I present a \emph{new} convergence proof for a restricted class of MM algorithms. It mirrors Wu's EM convergence theorem (Theorem~\ref{thm:wu-converge}).

\begin{thm}{\bf{[MM Convergence Theorem]:}} \label{thm:MM-converge}\\
An minorization-maximization algorithm for optimizing a continuous, upper-bounded objective function $L(\theta|y)$ converges to the set of stationary points of $L(\theta|y)$ if the following conditions hold:
\begin{itemize}
\item the set $J$
\begin{equation}
J = \{ \theta \in \Theta |  L(\theta|y) \geq L(\theta_{0}|y) \}
\end{equation} is compact for all $\theta_{0}$,
\item the MM point-to-set algorithm map $\theta \rightarrow M(\theta_t)$ is closed over $\mathcal{S}_{MM}^c$. This MM map is
\begin{equation}
M(\theta_t) = \{ \theta \in \Theta |  q(\theta|\theta_t) \geq q(\theta_t|\theta_t)\}. \label{eq:MM-pointset-map}
\end{equation}
% \item the minorizing function $q(\psi|\phi)$ is continuous in both $\psi$ and $\phi$.
\end{itemize}
\end{thm}
\begin{proof}

Zangwill's Global Convergence Theorem A (Theorem~\ref{thm:zangwill}) applies directly to MM algorithms under the following identifications and assumptions. The objective function $L(\theta)$ is the MM objective function $L(\theta|y)$. The solution set $\mathcal{S}_{MM}$ is the set of interior stationary points of $L$. 
\begin{equation}
\mathcal{S}_{MM} = \{ \theta \in int(\Theta) |  L'(\theta|y) = 0\}. \label{eq:MM-solutionset}
\end{equation}

The compactness assumption implies that Zangwill's conditions (1) holds just like in the (G)EM case. The closure assumption fulfills Zangwill's condition (2). Conditions (3) and (4) in Zangwill's theorem follow by ascent property and boundedness of the objective function $L$. Thus the iterates converge to limit points in the solution set of stationary points.

Majorization-minimization is equivalent to the minorization-maximization of the negative of the objective function. So the proof applies to both types of MM algorithms.

\end{proof}

Continuity of the minorizing function $q(\psi|\phi)$ in both $\psi$ and $\phi$ implies that the $M$ map is closed over $\mathcal{S}^c$. EM and QLB algorithms are examples of MM algorithms that satisfy this closure condition. But the class of MM algorithms may be broad enough to include algorithms which violate the closure condition.

\section{Conclusion}
The EM algorithm is a versatile tool for analyzing incomplete data. But the EM algorithm has some drawbacks. It may converge slowly for high-dimensional parameter spaces or when the algorithm needs to estimate large amounts of missing information (see \textbf{\S}\ref{sec:Fisher-EM}, \parencite{mclachlan-krishnan2007,tanner1996}). EM implementations can also get very complicated for some models. And there is also no guarantee that the EM algorithm converges to the global maximum. 

%The EM algorithm still enjoys wide usage despite its computational shortcomings. 

The next few chapters develop and demonstrate a new EM scheme that addresses the EM algorithm's slow convergence. It improves average EM convergence rates via a novel application of a noise--benefit or ``stochastic resonance''. Chapter \ref{ch:NEM} presents the theory behind this approach. Chapters \ref{ch:NEM-CNBT}, \ref{ch:NEM-HMM}, \ref{ch:NEM-BP} demonstrate the speed improvement in key EM applications.

\clearpage

%\subfile{./Chapters/chapNEM}

\def\lgcx{$13.3\%$}
\def\Tlgcx{$4.72$}
\def\GMMpct{$27.2\%$}

\chapter{Noisy Expectation--Maximization (NEM)} \label{ch:NEM}

This chapter introduces the idea of a \emph{noise benefit} as a way to improve the average convergence speed of the EM algorithm. The result is a noise-injected version of the Expectation-Maximization (EM) algorithm: the Noisy Expectation Maximization (NEM) algorithm. The NEM algorithm (Algorithm~\ref{algo:NEM}) uses noise to speed up the convergence of the noiseless EM algorithm. The NEM theorem (Theorem~\ref{thm:NEM}) shows that additive noise speeds up the average convergence of the EM algorithm to a local maximum of the likelihood surface if a positivity condition holds. Corollary results give special cases when noise improves the EM algorithm. We demonstrate these noise benefits on EM algorithms for three data models: the Gaussian mixture model (GMM), the Cauchy mixture model (CMM), and the censored log-convex gamma model. The NEM positivity condition simplifies to a quadratic inequality in the GMM and CMM cases. A final theorem shows that the noise--benefit for independent identically distributed additive noise models decreases with sample size in mixture models. This theorem implies that the noise benefit is most pronounced if the data is sparse.

The next section (\Sec\ref{sec:NB-SR}) reviews the concept of a noise benefit or \emph{stochastic resonance} (SR). We then formulate the idea of a noise benefit for EM algorithms and discuss some intuition behind noise benefits in this context. A Noisy EM algorithm is any EM algorithm that makes use of noise benefits to improve the performance of the EM algorithm. \Sec\ref{sec:NEM} introduces the theorem and corollaries that underpin the NEM algorithm. \Sec\ref{sec:NEM-Alg} presents the NEM algorithm and some of its variants. \Sec\ref{sec:samp-sz-effects} presents a theorem that describes how sample size affects the NEM algorithm for mixture models when the noise is independent and identically distributed (i.i.d.). \Sec\ref{sec:samp-sz-effects} also shows how the NEM positivity condition arises from the central limit theorem and the law of large numbers.

\section{Noise Benefits and Stochastic Resonance}\label{sec:NB-SR}
A noise benefit or \emph{stochastic resonance} (SR)~\parencite{bulsara-gammaitoni96,gammaitoni-hanggi-jung-marchesoni1998,kosko2006noise} occurs when noise improves a signal system's performance: small amounts of noise improve the performance while too much noise degrades it. Examples of such improvements include improvements in signal-to-noise ratio~\parencite{bulsara-zador96}, Fisher information~\parencite{chapeau-blondeau-rousseau2002,chapeau-blondeau-blanchard-rousseau2008}, cross-correlation~\parencite{collins-chow-capela-imhoff96,collins-chow-imhoff95b}, or mutual information~\parencite{kosko-mitaimNN2003} between the input and output signal.

Much early work on noise benefits involved natural systems in physics~\cite{brey-prados96}, chemistry~\cite{kramers40, forster-merget-schneider96}, and biology~\cite{moss-bulsara-shlesinger93, cordo-inglis-vershueren-collins-merfeld-rosenblum-buckley-moss96, adair-astumian-weaver1998, hanggi2002}. Early descriptions of noise benefits include Benzi's model for ice age periodicity~\parencite{benzi-sutera-vulpiani81, benzi-parisi-sutera-vulpiani83}, noise-induced SNR improvements in bidirectional ring lasers~\parencite{mcnamara-wiesenfeld-roy88}, and Kramers' model of Brownian motion-assisted particle escape from a chemical potential wells~\parencite{kramers40}. 

These early works inspired the search for noise benefits in nonlinear signal processing and statistical estimation.~\parencite{bulsara-zador96, chapeau-blondeau-rousseau2004,mcdonnell-stocks-pearce-abbott2008,chen-varshney-kay-michels2009,patel-kosko-TSP2011,franzke-kosko2011}. Some of these statistical signal processing works describe screening conditions under which a system exhibits noise benefits\parencite{patel-thesis2009}. Forbidden interval theorems~\parencite{kosko2004robust}, for example, give necessary and sufficient conditions for SR in threshold detection of weak binary signals. These signal processing screening theorems also help explain some previously observed noise benefits in natural systems~\parencite{kosko2009applications}.

Figure \ref{fg:SR-Demo} shows a typical example of a noise benefit. The underlying grayscale image is barely perceptible without any additive noise (leftmost panel). Additive pixel noise makes the image more pronounced (middle panels). The image becomes worse as the noise power increases (rightmost panel). This example demonstrates how noise can improve the detection of subthreshold signals.
\begin{figure}[ht!]
\centerline{ \includegraphics[width=0.85\textwidth]{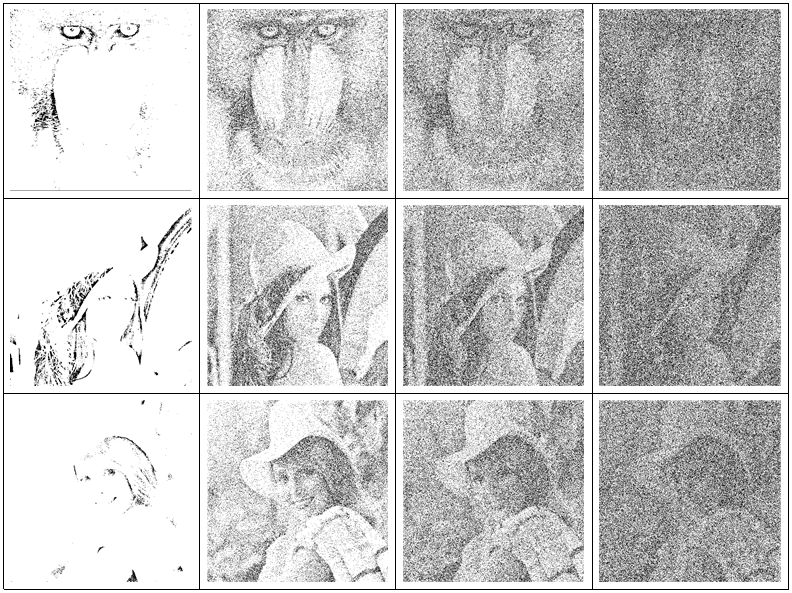} }
\caption[Stochastic Resonance on faint images using white Gaussian pixel noise]{
	Stochastic Resonance on faint images (mandrill, Lena, and Elaine test images) using white Gaussian pixel noise. The images are faint because the gray-scale pixels pass through a suboptimal binary thresholder. The faint images (leftmost panels) become clearer as the power $\sigma$ of the additive white Gaussian pixel noise increases. But increasing the noise power too much degrades the image beyond recognition (rightmost panels).
}
\label{fg:SR-Demo}
\end{figure}
Noise benefits exhibit as non-monotonic curves of the system's figure-of-merit as a function of noise power. Noise improves the figure-of-merit only up to a point. Further noise degrades the figure-of-merit. Figure \ref{fg:GaussNEM} below is an example of such a non-monotonic SR signature curve.

This chapter presents a new type of noise benefit on the popular EM algorithm~(\Sec\ref{ch:EM}). A key weakness of the EM algorithm is its slow convergence speed in many applications~\parencite{mclachlan-krishnan2007}. Careful noise injection can increase the average convergence speed of the EM algorithm. Theorem~\ref{thm:NEM} states a general sufficient condition for this EM noise benefit. The EM noise benefit does not involve a signal threshold unlike almost all SR noise benefits~\parencite{gammaitoni-hanggi-jung-marchesoni1998}. We apply this general sufficient condition and demonstrate the EM noise benefit on three data models: the ubiquitous Gaussian mixture model (Figure~{\ref{fg:GaussNEM}}), the Cauchy mixture model (Figure~{\ref{fg:CauchyNEM}}), and the censored gamma model (Figure~{\ref{fg:LogConvex}}). The simulations in Figure~{\ref{fg:NEMCompare-varyM}} and Figure~{\ref{fg:NEMCompare-SingleM}} also show that the noise benefit is faint or absent if the system simply injects blind noise that ignores the sufficient condition. This suggests that the noise benefit sufficient condition may also be a necessary condition for some data models. The last results of this chapter show that the noise benefit occurs most sharply in sparse data sets.

\subsection{Noise Benefits in the EM Algorithm}
Theorem~\ref{thm:NEM} below states a general sufficient condition for a noise benefit in the average convergence time of the EM algorithm. Figure \ref{fg:GaussNEM} shows a simulation instance of this theorem for the important EM case of Gaussian mixture densities.  Small values of the noise variance reduce convergence time while larger values increase it.  This U-shaped noise benefit is the non-monotonic signature of stochastic resonance. The optimal noise speeds convergence by \GMMpct. Other simulations on multidimensional GMMs have shown speed increases of up to $40\%$.

\begin{figure}[ht!]
\centerline{ \includegraphics[width=0.75\textwidth]{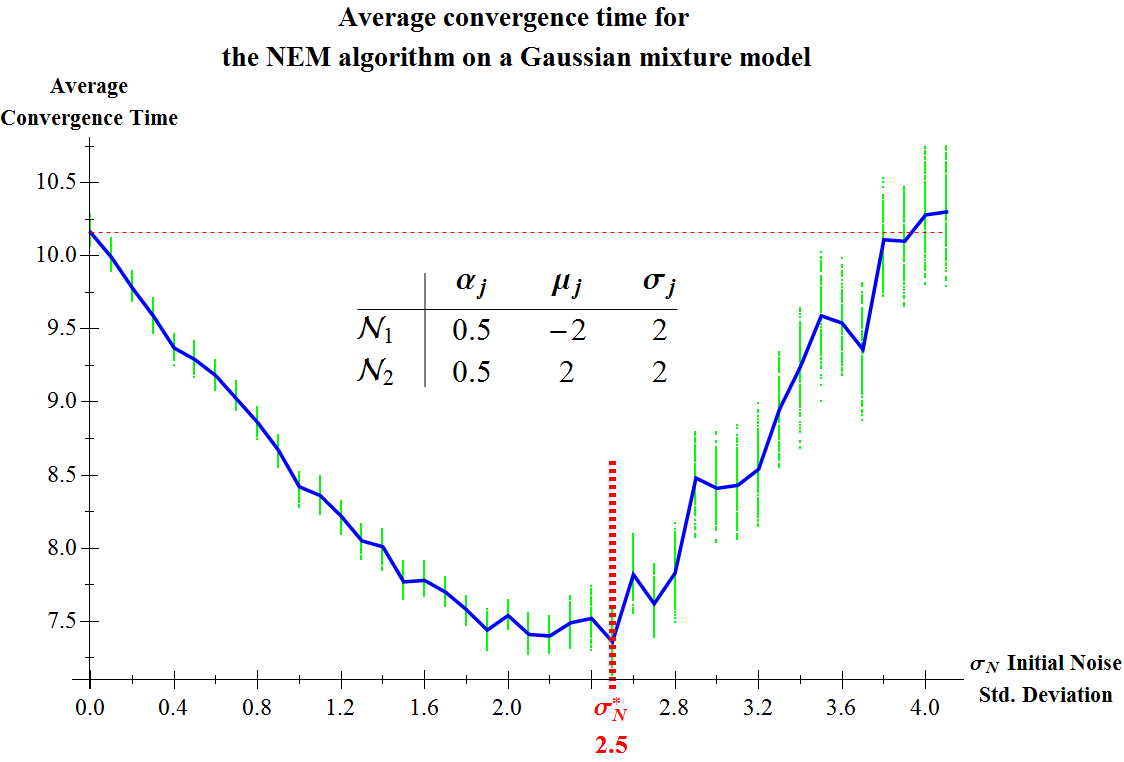} }
\caption[EM noise benefit for a Gaussian mixture model]{
	EM noise benefit for a Gaussian mixture model. The plot uses the noise-annealed NEM algorithm. Low intensity starting noise decreases convergence time while higher intensity starting noise increases it.  The optimal initial noise level has standard deviation $\sigma^*_N = 2.5$. The average optimal NEM speed-up over the noiseless EM algorithm is $27.2\%$. This NEM procedure adds noise with a cooling schedule. The noise cools at an inverse-square rate. The Gaussian mixture density is a convex combination of two normal sub-populations $\mathcal{N}_1$ and $\mathcal{N}_2$. The simulation uses $200$ samples of the mixture normal distribution to estimate the standard deviations of the two sub-populations. The additive noise uses samples of zero-mean normal noise with standard deviation $\sigma_N$ screened through the GMM--NEM condition in ({\ref{eq:FinMixtureCond}}). Each sampled point on the curve is the average of $100$ trials. The vertical bars are $95\%$ bootstrap confidence intervals for the mean convergence time at each noise level.
}
\label{fg:GaussNEM}
\end{figure}

The EM noise benefit differs from almost all SR noise benefits because it does not involve the use of a signal threshold \parencite{gammaitoni-hanggi-jung-marchesoni1998}. The EM noise benefit also differs from most SR noise benefits because the additive noise can depend on the signal. Independent noise can lead to weaker noise benefits than dependent noise in EM algorithms. This also happens with enhanced convergence in noise-injected Markov chains \parencite{franzke-kosko2011}. Figure \ref{fg:NEMCompare-varyM} shows that the proper dependent noise outperforms independent noise at all tested sample sizes for a Gaussian mixture model. The dependent noise model converges up to $14.5\%$ faster than the independent noise model.

\subsection{Intuition on EM Noise Benefits}

%%%%% Likelihood intuition
We can gain some intuition into the EM noise benefit by examining the space of plausible likelihood functions for the observed data. The likelihood function $\ell(\theta|y)$ is a statistic of the observed data $y$. Its functional behavior often overshadows its statistical behavior. We use the likelihood function's statistical behavior to produce a noise benefit.

Suppose an experiment produces a data sample $y$\footnote{If the data model has a sufficient statistic $T(y_1,\cdots,y_n)$ then we can apply the subsequent analysis to the sampling pdf of $T(y_1,\cdots,y_n)$ instead.}.
Suppose also that the sampling pdf $f(y|\theta)$ is model-appropriate and sufficiently smooth. Then there is a $\Delta n$-ball $\mathcal{B}_{\Delta n}(y)$ of samples around $Y=y$ that the experiment could have emitted with almost equal probability. The size of this ball depends on the steepness and continuity of the sampling pdf at $y$. This understanding is consistent with the probabilistic model for the experiment but easy to ignore when working with sample realizations $y$. This $\Delta n$-ball of alternate representative samples specifies a set of likelihoods 
\begin{equation}
	\mathcal{L}(\theta|y) = \{y' \in \mathcal{B}_{\Delta n}(y)| \; \ell(\theta|y') \; \}
\end{equation}
that can also model $\theta$ for the experiment. Some of the alternate samples $y' \in \mathcal{B}_{\Delta n}(y)$ assign a higher likelihood $\ell(\theta_*|y')$ to the true $\theta_*$ than the current sample's likelihood $\ell(\theta_*|y)$ does. So it is useful to pay attention to $\mathcal{L}(\theta|y)$ while seeking estimates for $\theta_*$.

Analytic MLE finds the best estimate for $\theta_*$ in a single iteration. So there is little point in analyzing estimates from the set of alternative likelihoods $\mathcal{L}(\theta|y)$. The set $\mathcal{L}(\theta|y)$ provides enough information to calculate bootstrap estimates~\parencite{efron-tibshirani1991} for the standard error of $\hat{\theta}$ for example. But the analytic Fisher information already summarizes that information since standard error is the inverse square root of the Fisher information.

EM and other iterative MLE techniques can benefit from the information $\mathcal{L}(\theta|y)$ contains. The NEM algorithm uses noise to pick favorable likelihoods in $\mathcal{L}(\theta|y)$ that result in higher likelihood intermediate estimates $\theta_k$. This leads to faster EM convergence. The NEM condition in Theorem~\ref{thm:NEM} describes how to screen and pick favorable likelihoods at the current EM iteration.
%%%%%

The idea behind the NEM condition is that sometimes \emph{small} sample noise $n$ can increase the likelihood of a parameter $\theta$. This occurs at the local level when
\begin{equation} 
	f(y+n|\theta)~>~f(y|\theta) \label{eq:MLE-DomCondn} 
\end{equation} 
for probability density function (pdf) $f$, realization $y$ of random variable $Y$, realization $n$ of random noise $N$, and parameter $\theta$. This condition holds if and only if the logarithm of the pdf ratio is positive:
\begin{equation} 
	\ln \left( \frac{f(y + n|\theta)}{f(y|\theta)} \right) ~>~0 \;. \label{eq:LogCondn} 
\end{equation}

The logarithmic condition (\ref{eq:LogCondn}) in turn occurs much more generally if it holds only on average with respect to all the pdfs involved in the EM algorithm:
\begin{equation}
	\E_{Y,Z,N|\theta_*} \left[ \ln \frac{f(Y + N, Z|\theta_k)}{f(Y, Z|\theta_k)}  \right] ~\geq~0 
\label{eq:E-LogCondn} 
\end{equation}
\noindent where random variable $Z$ represents missing data in the EM algorithm and where $\theta_*$ is the limit of the EM estimates $\theta_k$: $\theta_k \longrightarrow \theta_*$. The positivity condition (\ref{eq:E-LogCondn}) is precisely the sufficient condition for a noise benefit in Theorem~\ref{thm:NEM} below.
%%%%%

\section{Noisy Expectation Maximization Theorems} \label{sec:NEM}
The EM noise benefit requires the use of the modified surrogate log-likelihood function 
\begin{equation}
	Q_N\left( \theta |\theta_k \right) = \E_{Z|y,\theta_k}   \left[ \ln f(y+N,Z| \theta)  \right]
\end{equation} and its maximizer 
\begin{equation*}
	\theta_{k+1,N} = \argmax{\theta} \left\{ Q_N\left( \theta |\theta_{k} \right) \right\} \;.
\end{equation*}
The modified surrogate log-likelihood $Q_N\left( \theta |\theta_{k} \right)$ equals the regular surrogate log-likelihood $Q\left( \theta |\theta_{k} \right)$ when $N = 0$. $Q(\theta|\theta_*)$ is the final surrogate log-likelihood given the optimal EM estimate $\theta_*$. So $\theta_*$ maximizes $Q(\theta|\theta_*)$. Thus
\begin{align}
	Q(\theta_*|\theta_*) & \geq Q(\theta|\theta_*)  \quad \text{for all } \theta \;. \label{eq:QD-NoiseBenefit}
\end{align}

An EM noise benefit occurs when the noisy surrogate log-likelihood $Q_N(\theta_k|\theta_*)$ is closer to the optimal value $Q(\theta_*|\theta_*)$ than the regular surrogate log-likelihood $Q(\theta_k|\theta_*)$ is. This holds when
\begin{align}
	Q_N(\theta_k|\theta_*) &\geq Q(\theta_k|\theta_*) \label{eq:QN-Ineq-NoiseBenefit} \\
	\text{or } \Big( Q(\theta_*|\theta_*) - Q(\theta_{k}|\theta_*) \Big)  &\geq \Big( Q(\theta_*|\theta_*) - Q_N(\theta_{k}|\theta_*) \Big). \label{eq:Q-NoiseBenefit}
\end{align}
So the noisy perturbation $Q_N(\theta|\theta_k)$ of the current surrogate log-likelihood $Q(\theta|\theta_k)$ is a better log-likelihood for the data than $Q$ is itself. An average noise benefit results when we take expectations of both sides of inequality (\ref{eq:Q-NoiseBenefit}):
\begin{align}
	\E_N \Big[ Q(\theta_*|\theta_*) - Q(\theta_{k}|\theta_*) \Big] &\geq \E_N \Big[ Q(\theta_*|\theta_*) - Q_N(\theta_{k}|\theta_*) \Big].  \label{eq:EQ-NoiseBenefit}
\end{align}

This average noise benefit takes on a more familiar form if we re-express (\ref{eq:EQ-NoiseBenefit}) in terms of the relative entropy pseudo-distance. The relative entropy (or Kullback-Leibler divergence)\parencite{cover-thomas91, kullback-leibler1951} between pdfs $f_0(x)$ and $f_1(x)$ is
\begin{align}
	D\left(f_0(x) \Vert f_1(x)\right) &= \int_X \ln \left[\frac{f_0(x)}{f_1(x)}\right] f_0(x) ~\mathrm{d}x \\
	&= \E_{f_0}\left[ \ln f_0(x) -\ln f_1(x) \right] \;.
\end{align}
when the pdfs have the same support. It quantifies the average information--loss incurred by using the pdf $f_1(x)$ instead of $f_0(x)$ to describe random samples of $X$~\parencite{bernardo-smith2009}. It is an example of a Bayes risk~\parencite{carlin-louis2009, bernardo-smith2009}.

Let $f_0(x)$ be the final EM complete pdf and $f_1(x)$ be the current EM or NEM complete pdf. Then the relative-entropy pseudo-distances in the noisy EM context are
\begin{align}
	c_k\left(N\right) =& D\left(f(y,z|\theta_*)\Vert f(y+N,z|\theta_k)\right)
	\label{eq:Noisy-c}\\
	\textrm{and} \quad \quad c_k =& D\left(f(y,z|\theta_*)\Vert f(y,z|\theta_k)\right) = c_k\left(0\right) \;.
	\label{eq:noiseless-c}
\end{align}

The average noise benefit (\ref{eq:EQ-NoiseBenefit}) occurs when the final EM pdf $f(y, z| \theta_*)$ is closer in relative-entropy to the noisy pdf $f(y + N,z|\theta_k)$ than it is to the noiseless pdf $f(y, z|\theta_k)$. So noise benefits the EM algorithm when
\begin{equation}
	c_k  \geq c_k\left(N\right) \;.
\label{eq:NoiseBenefit}
\end{equation}
This means that the noisy pdf is a better information-theoretic description of the complete data than the regular pdf. The noisy pdf incurs lower average information--loss or Bayes risk than the regular pdf.

The proof of the NEM theorem in the next section shows that the relative-entropy and expected $Q$-difference formulations of the average noise benefit are equivalent. Manipulating the relative-entropy formulation leads to the NEM sufficient condition that guarantees the average EM noise benefit.

\subsection{NEM Theorem}
The Noisy Expectation Maximization (NEM) Theorem below uses the following notation. The noise random variable $N$ has pdf $f(n|y)$. So the noise $N$ can depend on the data $Y$. Independence implies that the noise pdf becomes $f(n|y)=f_N(n)$. $\left\{\theta_k\right\}$  is  a sequence of EM estimates for $\theta$. $\theta_*= \lim_{k \to \infty} \theta_k$ is the converged EM estimate for $\theta$. Assume that the differential entropy of all random variables is finite. Assume also that the additive noise keeps the data in the likelihood function's support.

%%%%%%%%%%%%%%%%%%%
\begin{thm}{\bf{[Noisy Expectation Maximization (NEM)]:}} \label{thm:NEM}\\
An EM iteration noise benefit
\begin{equation}
	\Big( Q(\theta_*|\theta_*) - Q(\theta_{k}|\theta_*) \Big)  \geq \Big( Q(\theta_*|\theta_*) - Q_N(\theta_{k}|\theta_*) \Big) 
	\end{equation}
	occurs on average if
	\begin{equation}
		\E_{Y,Z,N|\theta_*} \left[ \ln\left( \frac{f(Y+N,Z|\theta_k)}{f(Y,Z|\theta_k)} \right) \right] \geq 0\;. \label{eq:Goal}
	\end{equation}
\end{thm}

\begin{proof}
First show that each expectation of $Q$-function differences in (\ref{eq:EQ-NoiseBenefit}) is a distance pseudo-metric.
\noindent Rewrite $Q$ as an integral:
\begin{equation}
\int_Z \ln [f(y,z | \theta)] f( z| y, \theta_k) ~dz \;.
\end{equation}

\noindent $c_k=D(f(y,z|\theta_*)\Vert f(y,z|\theta_k))$ is the expectation over $Y$ because
\begin{align}
c_k & = \iint [\ln(f(y,z|\theta_*)) - \ln(y,z|\theta_k)] f(y,z|\theta_*) ~dz ~dy \\
& = \iint [\ln(f(y,z|\theta_*)) - \ln(y,z|\theta_k)] f(z|y, \theta_*)  f(y|\theta_*) ~dz ~dy \\
& = \E_{Y | \theta_k}  \Big[  Q \left( \theta_* |\theta_* \right)  -  Q \left( \theta_k |\theta_* \right) \Big] \;.
\end{align}

\noindent $c_k\left(N\right)$ is likewise the expectation over $Y$:
\begin{align}
c_k\left(N\right)  &= \iint [\ln(f(y,z|\theta_*)) - \ln(y+N,z|\theta_k)] f(y,z|\theta_*) ~dz ~dy \\
& = \iint [\ln(f(y,z|\theta_*)) - \ln(y+N,z|\theta_k)]  f(z|y, \theta_*)  f(y|\theta_*) ~dz ~dy \\
& = \E_{Y| \theta_k}  \Big[  Q \left( \theta_* |\theta_* \right)  -  Q_N \left( \theta_k |\theta_* \right) \Big] \;.
\end{align}
\noindent Take the noise expectation of $c_k$ and $c_k\left(N\right)$:
\begin{align}
\E_N \big[ c_k \big]  &= c_k \\
\E_N \big[ c_k\left(N\right) \big]  &= \E_N \big[ c_k\left(N\right) \big] \;.
\end{align}
\noindent So the distance inequality
\begin{align}
c_k \geq E_N\left[c_k\left(N\right)\right] \;. \label{eq:c-Dom}
\end{align}
guarantees that noise benefits occur on average:
\begin{align}
\E_{N,Y| \theta_k} \Big[ Q \left( \theta_* |\theta_* \right)  -  Q \left( \theta_k |\theta_* \right) \Big] \geq \E_{N,Y| \theta_k} \Big[ Q \left( \theta_* |\theta_* \right)  -  Q_N \left( \theta_k |\theta_* \right) \Big]
\end{align}

We use the inequality condition (\ref{eq:c-Dom}) to derive a more useful sufficient condition for a noise benefit.  Expand the difference of relative entropy terms $c_k-c_k\left(N\right)$:
\begin{multline}
c_k-c_k\left(N\right)  \\
=\iint_{Y,Z} \left( \ln\left[ \frac{f(y,z|\theta^*)}{f(y,z|\theta_k)} \right] - \ln\left[ \frac{f(y,z|\theta^*)}{f(y+N,z|\theta_k)} \right] \right) f(y,z|\theta^*)~dy~dz
\end{multline}
\begin{align}
&= \iint_{Y,Z} \left( \ln \left[\frac{f(y,z|\theta^*)}{f(y,z|\theta_k)} \right] + \ln\left[ \frac{f(y+N,z|\theta_k)}{f(y,z|\theta^*)} \right] \right) f(y,z|\theta^*)~dy~dz \\
&= \iint_{Y,Z}  \ln \left[ \frac{f(y,z|\theta^*) f(y+N,z|\theta_k)}{f(y,z|\theta_k) f(y,z|\theta^*)} \right] f(y,z|\theta^*)~dy~dz\\
&= \iint_{Y,Z} \ln \left[ \frac{f(y+N,z|\theta_k)}{f(y,z|\theta_k)} \right]  f(y,z|\theta^*)~dy~dz \;.
\end{align}
Take the expectation with respect to the noise term N:
\begin{align}
E_{N} &\left[ c_k-c_k\left(N\right) \right] = c_k - E_{N} \left[c_k\left(N\right) \right] \\
&=\int_{N} \iint_{Y,Z} \ln \left[ \frac{f(y+n,z|\theta_k)}{f(y,z|\theta_k)} \right] f(y,z|\theta^*) f(n|y)~dy~dz~dn \\
&= \iint_{Y,Z}  \int_{N} \ln \left[ \frac{f(y+n,z|\theta_k)}{f(y,z|\theta_k)} \right]  f(n|y) f(y,z|\theta^*) ~dn~dy~dz \\
&= \E_{Y,Z,N|\theta^*} \left[ \ln \frac{f(Y+N,Z|\theta_k)}{f(Y,Z|\theta_k)} \right] \;.
\label{eq:Equality}
\end{align}
\noindent The assumption of finite differential entropy for $Y$ and $Z$ ensures that  \begin{equation}\ln f(y,z|\theta) f(y,z|\theta_*)\end{equation} is integrable. Thus the integrand is integrable. So Fubini's theorem \parencite{folland1999} permits the change in the order of integration in (\ref{eq:Equality}):
\begin{align}
c_k \geq E_N\left[c_k\left(N\right)\right] \quad \text{iff} \quad \E_{Y,Z,N|\theta_*} \left[ \ln\left(\frac{f(Y+N,Z|\theta_k)}{f(Y,Z|\theta_k)}\right)\right] \geq 0 
\end{align}
\noindent Then an EM noise benefit occurs on average if 
\begin{equation}\E_{Y,Z,N|\theta_*} \left[ \ln\left(\frac{f(Y+N,Z|\theta_k)}{f(Y,Z|\theta_k)}\right)\right] \geq 0 \;.\end{equation} 
\end{proof} 

%%%%%%%%%%%%%%%%%%%

The NEM theorem also applies to data models that use the complete data as their latent random variable. The proof for these cases is the same. The NEM positivity condition in these models changes to
\begin{equation}
\E_{X,Y,N|\theta_*} \left[ \ln\left( \frac{f(X+N|\theta_k)}{f(X|\theta_k)} \right) \right] \geq 0\;. \label{eq:X-Goal}
\end{equation}

The relative entropy definition of the noise benefit also allows for a more general version of NEM condition in users replace noise addition $y+N$ with other methods of noise injection $\phi(y,N)$ such as \emph{multiplicative noise injection} $y.N$. The NEM condition for generalized noise injection is
\begin{equation}
	\E_{Y,Z,N|\theta_*} \left[ \ln\left( \frac{f(\phi(Y, N),Z|\theta_k)}{f(Y,Z|\theta_k)} \right) \right] \geq 0 
	\label{eq:genNEM-Goal}
\end{equation}
The proof of this general condition is exactly the same as in the additive noise case with $f_N(y,z|\theta_k) = f(\phi(y, N),z|\theta_k)$ replacing $f(y+N,z|\theta_k)$.

%%%%%%%%%%%%%%%%%%%%%%%%%%%%%%%%
%% Addendum to NEM Convergence %%%
The NEM Theorem implies that each iteration of a suitably noisy EM algorithm moves closer on average towards the EM estimate $\theta_*$  than does the corresponding noiseless EM algorithm \parencite{osoba-mitaim-kosko2011}.  This holds because the positivity condition (\ref{eq:Goal}) implies that $E_N\left[c_k\left(N\right)\right] \leq c_k$ at each step $k$ since $c_k$ does not depend on $N$ from (\ref{eq:noiseless-c}). The NEM algorithm uses larger overall steps on average than does the noiseless EM algorithm for any number $k$ of steps. 

The NEM theorem's stepwise noise benefit leads to a noise benefit at any point in the sequence of NEM estimates. This is because we get the following inequalities when the expected value of inequality({\ref{eq:QD-NoiseBenefit}}) takes the form
\begin{align}
Q(\theta_{k}|\theta_*) &\leq \E [Q_{N}(\theta_{k}|\theta_*) ] \quad \textrm{for any } k. {\label{eq:EQv2-NoiseBenefit}} 
\end{align}
Thus 
\begin{align}
Q(\theta_*|\theta_*) - Q(\theta_{k}|\theta_*) &\geq   Q(\theta_*|\theta_*) - \E [Q_{N}(\theta_{k}|\theta_*) ]  \quad \textrm{for any } k. {\label{eq:EQD-NoiseBenefit}} 
\end{align}
The EM (NEM) sequence converges when the left (right) side of inequality ({\ref{eq:EQD-NoiseBenefit}}) equals zero. Inequality ({\ref{eq:EQD-NoiseBenefit}}) implies that the difference on the right side is closer to zero at any step $k$. 
%Thus any estimate in the NEM sequence is on average higher on the final surrogate log-likelihood $Q(\theta|\theta_*)$ than the same step estimate in the EM sequence.

NEM sequence convergence is even stronger if the noise $N_k$ decays to zero. This noise annealing implies $N_k \rightarrow 0$ with probability one. Continuity of $Q$ as a function of $Y$ implies that $Q_{N_k}(\theta|\theta_k) \rightarrow Q(\theta|\theta_k)$ as $N_k \rightarrow 0$. This holds because $Q(\theta|\theta_k) = \E_{Z|y,\theta_k}[\ln f(y,Z|\theta)]$ and and because the continuity of $Q$ implies that
\begin{align}
\lim_{N \rightarrow 0} Q_N(\theta|\theta_k) &= \E_{Z|y,\theta_k}[\ln f(\lim_{N \rightarrow 0}(y+N) ,Z|\theta)] \\
&= \E_{Z|y,\theta_k}[\ln f(y ,Z|\theta)] \\
&= Q(\theta|\theta_k) \;.
\end{align}
The evolution of EM algorithms guarantees that $\lim_k Q(\theta_k|\theta_*) = Q(\theta_*|\theta_*)$. This gives the probability-one limit
\begin{align}
\lim_k Q_{N_k}(\theta_k|\theta_*) &= Q(\theta_*|\theta_*). \label{eq:QNk-aslimit}
\end{align}
So for any $\epsilon > 0$ there exists a $k_0$ such that for all $k>k_0$:
\begin{align}
\big|Q(\theta_{k}|\theta_*) - Q(\theta_*|\theta_*) \big| &< \epsilon \quad \text{and} 
\\ \big| Q_{N_k}(\theta_{k}|\theta_*) - Q(\theta_*|\theta_*) \big| &< \epsilon 
\;\; \textrm{with probability one.} \label{eq:epsilon-Q}
\end{align}
Inequalities ({\ref{eq:EQv2-NoiseBenefit}}) and ({\ref{eq:epsilon-Q}}) imply that $Q(\theta_{k}|\theta_*)$ is $\epsilon$-close to its upper limit $Q(\theta_*|\theta_*)$ and
\begin{align}
\E [Q_{N_k}(\theta_{k}|\theta_*) ] &\geq Q(\theta_{k}|\theta_*)
\quad \text{and} \quad
Q(\theta_*|\theta_*) \geq  Q(\theta_{k}|\theta_*) \;. \label{eq:QOrder-Ineq}
\end{align}
So the NEM and EM algorithms converge to the same fixed-point by ({\ref{eq:QNk-aslimit}}). And the inequalities ({\ref{eq:QOrder-Ineq}}) imply that NEM estimates are closer on average to optimal than EM estimates are at any step $k$.
%%%%%%%%%%%%%%%%%%%%%%%%%%%%%%%%

\subsection{NEM for Finite Mixture Models}
The first corollary of Theorem \ref{thm:NEM} gives a dominated-density condition that satisfies the positivity condition (\ref{eq:Goal}) in the NEM Theorem.  This strong point-wise condition is a direct extension of the pdf inequality in (\ref{eq:MLE-DomCondn}).

%%%%%%%%%%%%%%%%%

\begin{cor}{\bf{[Dominance Condition for NEM]:}}\label{cor:DomPDF-NEM}
	The NEM positivity condition
	\begin{equation}
		\E_{Y,Z,N|\theta^*} \left[ \ln \frac{f(Y+N,Z|\theta)}{f(Y,Z|\theta)} \right] \geq 0
	\end{equation}
	holds if
	\begin{equation}
		f(y+n,z|\theta) \geq f(y,z|\theta) \label{eq:Dominance}
	\end{equation} for almost all $n,y,z$.
\end{cor}

\begin{proof}
	The following inequalities need hold only for almost all $y,z$, and $n$:
	\begin{align}
	 f(y+n,z|\theta) &\geq  f(y,z|\theta) \\
	\text{iff} \quad \ln\left[f(y+n,z|\theta)\right] &\geq \ln\left[f(y,z|\theta)\right] \\
	\text{iff} \quad \ln\left[f(y+n,z|\theta)\right] - \ln\left[f(y,z|\theta)\right] &\geq 0 \\
	\text{iff} \quad \ln\left[ \frac{f(y+n,z|\theta)}{f(y,z|\theta)} \right] &\geq 0 \; .
	\end{align}
	Thus
	\begin{equation} \E_{Y,Z,N|\theta^*} \left[\ln \frac{f(Y+N,Z|\theta)}{f(Y,Z|\theta)} \right] \geq 0 \;. \end{equation} 
\end{proof} 
%%%%%%%%%%%%%%%

Most practical densities do not satisfy the condition (\ref{eq:Dominance}) in Corollary \ref{cor:DomPDF-NEM}\footnote{
	PDFs for which condition (\ref{eq:Dominance}) holds are monotonically increasing or non-decreasing like ramp pdfs. The condition holds only for positive noise samples $n$ even in these cases.
}. But Corollary~\ref{cor:DomPDF-NEM} is still useful for setting conditions on the noise $N$ to produce the effect of the condition (\ref{eq:Dominance}).

We use Corollary~\ref{cor:DomPDF-NEM} to derive conditions on the noise $N$ that produce NEM noise benefits for mixture models. NEM mixture models use two special cases of Corollary \ref{cor:DomPDF-NEM}.  We state these special cases as Corollaries \ref{cor:GMM-NEM} and \ref{cor:CMM-NEM} below. The corollaries use the finite mixture model notation in \Sec\ref{subsec:FMM-EM}. Recall that the joint pdf of $Y$ and $Z$ is
\begin{equation}
	f(y,z|\theta)= {\sum}_j \alpha_j ~f(y|j,\theta) ~\delta[z-j] \;.
	\label{eq:gmm-compl-pdf}
\end{equation}
Define the population-wise noise likelihood difference as 
\begin{equation}
	\Delta f_j(y,n) = f(y+n|j,\theta)-f(y|j,\theta)\;. 
\end{equation}
Corollary~\ref{cor:DomPDF-NEM} implies that noise benefits the mixture model estimation if the dominated-density condition holds:
\begin{align}
	f(y+n,z|\theta) &\geq f(y,z|\theta) \;. 
\end{align}
This occurs if
\begin{equation}%\alpha_j 
	\Delta f_j(y,n) \geq 0 \quad \text{for all } j\;. \label{eq:DeltaCond}
\end{equation}

The Gaussian mixture model (GMM) uses normal pdfs for the sub-population pdfs \parencite{hasselblad1966, redner-walker1984}. Corollary~\ref{cor:GMM-NEM} states a simple quadratic condition that ensures that the noisy sub-population pdf $f(y+n|Z=j,\theta)$ dominates the noiseless sub-population pdf $f(y|Z=j,\theta)$ for GMMs. The additive noise samples $n$ depend on the data samples $y$.

%%%%%%%%%%%%%
\begin{cor}{\bf{[NEM Condition for GMMs]:}}\label{cor:GMM-NEM} \\
	Suppose $Y|_{Z=j} \sim {\cal N}(\mu_j,\sigma^2_j)$ and thus $f(y|j,\theta)$ is a normal pdf. Then 
	\begin{align}
		\Delta f_j(y,n) &\geq 0 
	\end{align}
	holds if 
	\begin{align}
		n^2 &\leq 2n\left(\mu_j-y\right) \;.  \label{eq:GaussMLECondn}
	\end{align} 
\end{cor}

\begin{proof}
	The proof compares the noisy and noiseless normal pdfs. The normal pdf is
	\begin{align}
		f(y|\theta) &= \frac{1}{\sigma_j \sqrt{2\pi}} \exp \left[-\frac{\left(y-\mu_j \right)^2 }{2\sigma_j^2}\right] \;. 
	\end{align}
	So $f(y+n|\theta) \geq f(y|\theta)$
	\begin{align}
		%f(y+n|\theta) &\geq& f(y|\theta) \\
		\text{iff} \; \exp \left[-\frac{\left(y+n-\mu_j \right)^2 }{2\sigma_j^2}\right]  &~~\geq ~~ \exp \left[-\frac{\left(y-\mu_j \right)^2 }{2\sigma_j^2}\right] \\
		\text{iff}\quad -\left( \frac{y+n-\mu_j}{\sigma_j} \right)^2 &~~\geq ~~  -\left( \frac{y-\mu_j}{\sigma_j} \right)^2\\
		\text{iff}\quad -\left( y-\mu_j+n \right)^2 &~~\geq ~~  -\left(y-\mu_j \right)^2\;.
		\label{eq:Stpid-Line1}
	\end{align}
	Inequality (\ref{eq:Stpid-Line1}) holds because $\sigma_j$ is strictly positive. Expand the left-hand side to get (\ref{eq:GaussMLECondn}):
	\begin{align}
		\left(y-\mu_j\right)^2 + n^2 +2n\left(y-\mu_j\right) &~~\leq ~~ \left(y-\mu_j \right)^2\\
		\text{iff} \quad n^2 +2n(y-\mu_j) &~~\leq ~~ 0\\
		\text{iff} \quad n^2 &~~\leq ~~ - 2n\left(y-\mu_j\right)\\
		\text{iff} \quad n^2 &~~\leq ~~ 2n\left(\mu_j-y\right) \;.
	\end{align}
	This proves (\ref{eq:GaussMLECondn}). 
\end{proof}
%%%%%%%%%%%%%%%

Now apply the quadratic condition (\ref{eq:GaussMLECondn}) to (\ref{eq:DeltaCond}). Then 
\begin{equation}
	f(y+n,z|\theta) \geq f(y,z|\theta)
\end{equation}
holds when 
\begin{align}
	n^2 &\leq 2n\left(\mu_j-y\right) \textrm{ for all } j \;.
	\label{eq:FinMixtureCond}
\end{align}

The inequality (\ref{eq:FinMixtureCond}) gives a condition under which NEM estimates standard deviations $\sigma_j$ faster than EM. This can benefit Expectation-Conditional-Maximization (ECM) \parencite{meng-rubin1993} methods.

Corollary~\ref{cor:CMM-NEM} gives a similar quadratic condition for the Cauchy mixture model. The noise samples also depend on the data samples.

%%%%%%%
\begin{cor}{\bf{[NEM Condition for CMMs]:}}\label{cor:CMM-NEM} \\
	Suppose $Y|_{Z=j} \sim {\cal C}(m_j, d_j)$ and thus $f(y|j,\theta)$ is a Cauchy pdf. Then 
	\begin{align}
		\Delta f_j(y,n) &\geq 0 
	\end{align} 
	holds if 
	\begin{align}
		n^2 &\leq 2n\left(m_j-y\right) \;. \label{eq:CauchyMLECond}
	\end{align}
\end{cor}

\begin{proof}
	The proof compares the noisy and noiseless Cauchy pdfs. The Cauchy pdf is
	\begin{equation}
		f(y|\theta) = \frac{1}{ \pi d_j\left[1+\left( \frac{y-m_j}{d_j} \right)^2 \right]}
	\end{equation}
	Then $f(y+n|\theta) \geq f(y|\theta) $
	\begin{align}
		%f(y+n|\theta) &\geq& f(y|\theta) \\
		\text{iff} \quad  \frac{\frac{1}{\pi d_j}}{ \left[1+\left( \frac{y+n-m_j}{d_j} \right)^2 \right]}  &\geq \frac{\frac{1}{\pi d_j}}{ \left[1+\left( \frac{y-m_j}{d_j} \right)^2 \right]}\\
		\text{iff} \quad  \left[1 + \left(\frac{y-m_j}{d_j}\right)^2\right] &\geq  \left[1 + \left(\frac{y+n-m_j}{d_j}\right)^2 \right] \\
		\text{iff} \quad  \left(\frac{y-m_j}{d_j}\right)^2 &\geq \left(\frac{y+n-m_j}{d_j}\right)^2 \;.
	\end{align}
	\noindent Proceed as in the last part of the Gaussian case:
	\begin{align}
		\left( \frac{y-m_j}{d_j} \right)^2 &\geq  \left( \frac{y-m_j+n}{d_j} \right)^2\\
		\text{iff} \quad  \left( y-m_j \right)^2 &\geq  \left( y-m_j+n \right)^2\\
		\text{iff}  \quad \left( y-m_j \right)^2 &\geq  \left( y-m_j \right)^2 + n^2 + 2n\left(y-m_j\right)\\
		\text{iff} \quad 0 &\geq  n^2 + 2n\left(y-m_j\right)\\
		\text{iff} \quad n^2 &\leq 2n\left(m_j-y\right) \;. 
\end{align}
This proves (\ref{eq:CauchyMLECond}).
\end{proof}
\noindent The conditions in Corollaries~\ref{cor:GMM-NEM} and~\ref{cor:CMM-NEM} simplify to $n \leq 2\left(\mu_j-y\right)$ when $N>0$.

%%%%%%%
Again apply the quadratic condition (\ref{eq:CauchyMLECond}) to (\ref{eq:DeltaCond}). Then
\begin{equation}
	f(y+n,z|\theta) \geq f(y,z|\theta)
\end{equation}
holds when 
\begin{align}
	n^2 &\leq 2n\left(m_j-y\right) \textrm{ for all } j \;. 
	\label{eq:FinCauchyMixtureCond}
\end{align}

Figure \ref{fg:GaussNEM} shows a simulation instance of noise benefits for EM estimation on a GMM.  Figure \ref{fg:CauchyNEM} also shows that EM estimation on a CMM exhibits similar pronounced noise benefits. The GMM simulation estimates the sub-population standard deviations $\sigma_1$ and $\sigma_2$ from $200$ samples of a Gaussian mixture of two $1$-D sub-populations with known means $\mu_1 = -2$ and $\mu_2 = 2$ and mixing proportions $\alpha_1=0.5$ and $\alpha_2=0.5$. The true standard deviations are $\sigma_1^*=2$ and $\sigma_2^*=2$. Each EM and NEM procedure starts at the same initial point\footnote{
	The use of the fixed initial points and bootstrap confidence intervals for the figure-of-merit is in keeping with Shilane's~\parencite{shilane-et-al2008} framework for the statistical comparison of evolutionary or iterative computation algorithms with random trajectories.
} with $\sigma_1(0)=4.5$ and  $\sigma_2(0)=5$. The simulation runs NEM on $100$ GMM data sets for each noise level $\sigma_N$ and counts the number of iterations before convergence for each instance. The average of these iteration counts is the \emph{average convergence time} at that noise level $\sigma_N$. The CMM simulation uses a similar setup for estimating sub-population dispersions $(d_1^*, d_2^*)$ with its own set of true values $(\alpha^*, m_1^*, d_1^*, m_2^*, d_2^*)$ (see inset table in Figure~\ref{fg:CauchyNEM}). The noise in the CMM-NEM simulation satisfies the condition in (\ref{eq:FinCauchyMixtureCond}). Simulations also confirm that non-Gaussian noise distributions produce similar noise benefits.

\begin{figure}[ht!]
\centerline{  \includegraphics[width=0.75\textwidth]{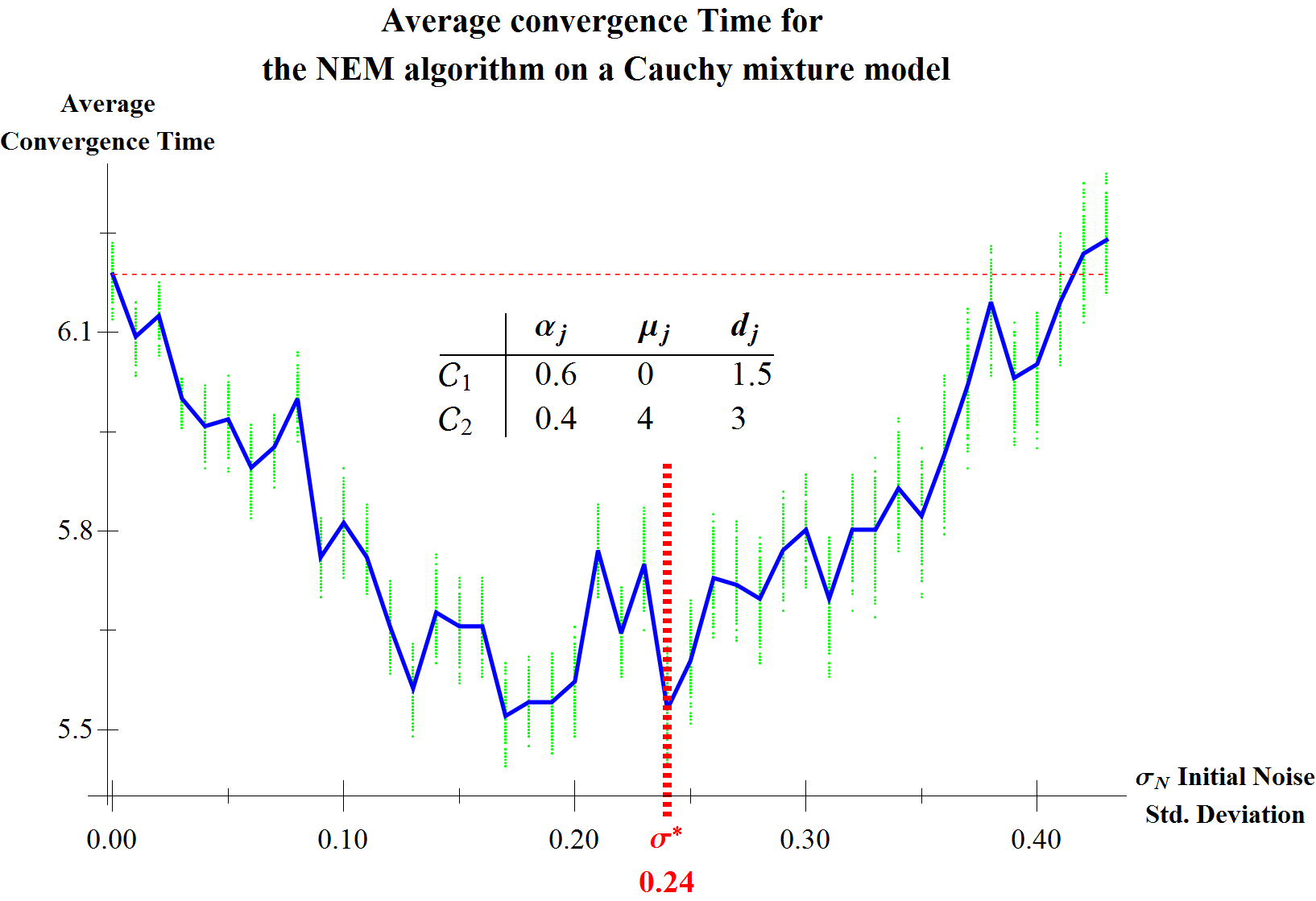} }
\caption[EM noise benefit for a Cauchy mixture model]{
	EM noise benefit for a Cauchy mixture model. This plot uses the noise-annealing variant of the NEM algorithm. Low intensity starting noise decreases convergence time while higher intensity starting noise increases it.  The optimal initial noise level has standard deviation $\sigma^* = 0.24$ which gave a speed improvement of about $11\%$. This NEM procedure adds independent noise with a cooling schedule. The noise cools at an inverse square rate. The data model is the Cauchy mixture model.  The mixture Cauchy distribution is a convex combination of two Cauchy sub-populations $\mathbb{C}_1$ and $\mathbb{C}_2$. The simulation uses $200$ samples of the mixture Cauchy distribution to estimate the dispersions of the two sub-populations. This NEM implementation adds noise to the data only when the noise satisfies the condition in ({\ref{eq:FinCauchyMixtureCond}}) . The additive noise is zero-mean normal with standard deviation $\sigma$. Each sampled point on the curve is the mean of $100$ trials. The vertical bars are the $95\%$ bootstrap confidence intervals for the convergence time mean at each noise level.
}
\label{fg:CauchyNEM}
\end{figure}

The mixture model NEM conditions predict noise benefits for the estimation of specific distribution parameters -- variance and dispersion parameters. The noise benefit also applies to the full EM estimation of all distribution parameters since the EM update equations decouple for the different parameters. So we can apply the NEM condition to just the update equations for the variance/dispersion parameters. And use the regular EM update equations for all other parameters. The NEM condition still leads to a noise--benefit in this more general estimation procedure. Figure~\ref{fg:GaussNEM-2D} shows a simulation instance of this procedure. The simulation for this figure estimates all distribution parameters for a GMM in $2$-dimensions.

\begin{figure}[ht!]
\centerline{ \includegraphics[width=0.75\textwidth]{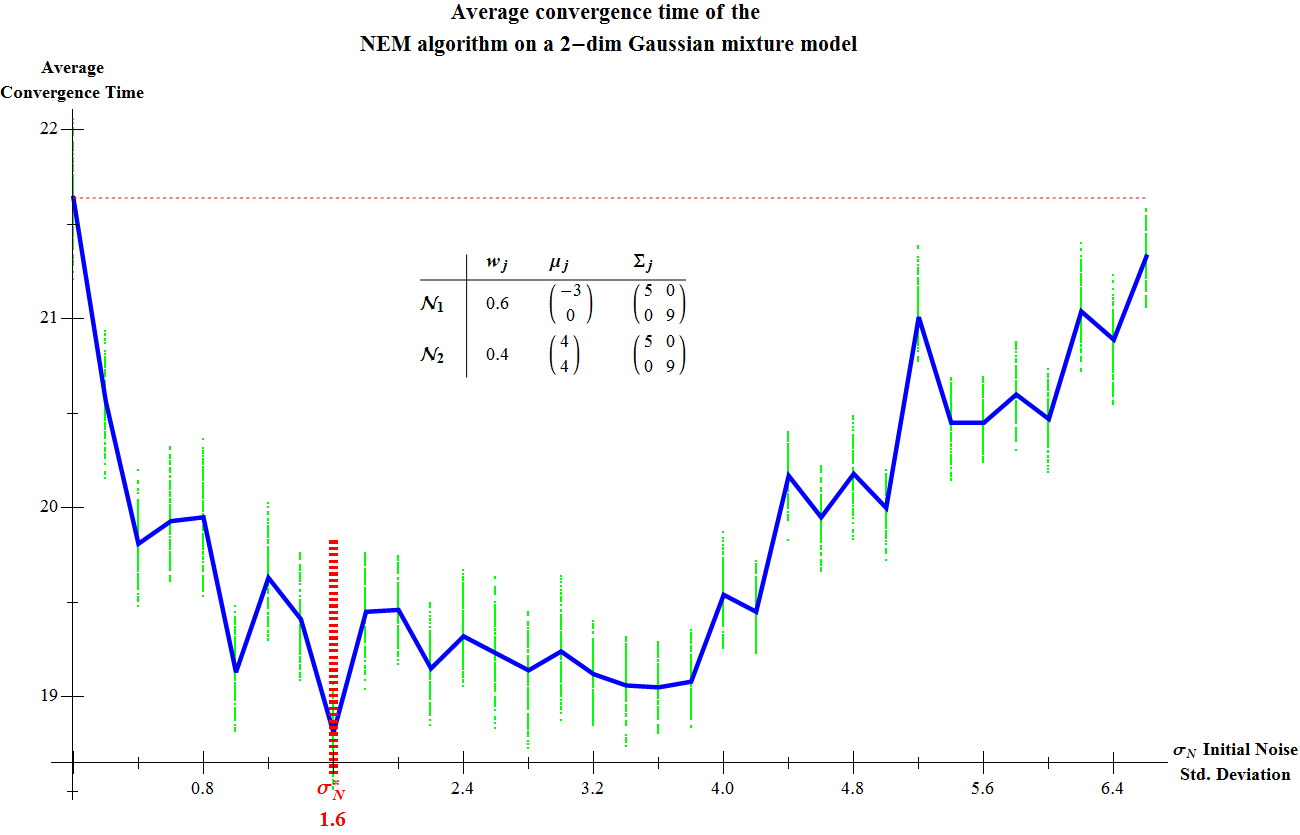} }
\caption[EM noise benefit for a $2$-D Gaussian mixture model]{
	EM noise benefit for a $2$-D Gaussian mixture model. The plot shows that noise injection can improve EM convergence time. The simulation is similar to Figure~\ref{fg:GaussNEM} except we estimate all distribution parameters $(\alpha, \mu, \sigma)$ instead of just $\sigma$. And the GMM is for $2$-dim samples instead of $1$-dim samples.	The optimal initial noise level has standard deviation $\sigma^* = 1.6$. The average optimal NEM speed-up over the noiseless EM algorithm is $14\%$. The simulation uses $225$ samples of the mixture normal distribution to estimate all distribution parameters for a $2$-D $2$-cluster GMM. The additive noise uses samples of zero-mean normal noise with standard deviation $\sigma_N$ screened through the GMM--NEM condition in ({\ref{eq:FinMixtureCond}}).
}
\label{fg:GaussNEM-2D}
\end{figure}

\subsection{The Geometry of the GMM-- \& CMM--NEM Condition}\label{subsec:NEM-Geom}

Both quadratic NEM inequalities in (\ref{eq:FinMixtureCond}) and (\ref{eq:FinCauchyMixtureCond}) reduce to 
\begin{align}
	n\left[n - 2\left(\mu_j-y\right) \right] \leq 0 \quad \forall \;  j \;. \label{eq:QuadForm-NEMCondn}
\end{align}
So the noise $n$ must fall in the interval where the \emph{parabola} $n^2  - 2n\left(\mu_j-y\right)$ is negative for all $j$. We call the set of such admissible noise samples the \emph{NEM set}. This set forms the support of the NEM noise pdf.

There are two possible solution sets for (\ref{eq:QuadForm-NEMCondn}) depending on the values of $\mu_j$ and $y$. These solution sets are
\begin{align}
	N_j^{-}(y) & =  [2\left(\mu_j-y\right), 0] \label{eq:neg-nem-set}\\
	N_j^{+}(y) & =  [0,2\left(\mu_j-y\right)] \label{eq:pos-nem-set}\;.
\end{align}
The goal is to find the set $N(y)$ of $n$ values that satisfy the inequality in (\ref{eq:FinMixtureCond}) for \emph{all} $j$:
\begin{align}
	N(y) = {\bigcap}_j N_j(y) 
\label{eq:intersect-nem-sets} %\label{eq:Ny-support}
\end{align}
where $N_j(y) = N_j^+(y)$ or $N_j(y) = N_j^-(y)$. Thus $N(y) \neq \{0\}$ holds only when the sample $y$ lies on one side of all sub-population means (or location parameters) $\mu_j$. This holds for 
\begin{align}
	\mu_j >y \textrm{ for all } j \quad \textrm{or} \quad   &\mu_j < y \textrm{ for all } j \;.
\end{align}

The NEM noise $N$ takes values in $\bigcap_j N_j^{-}$ if the data sample $y$ falls to the right of all sub-population means ($y>\mu_j$ for all $j$). The NEM noise $N$ takes values in $\bigcap_j N_j^{+}$ if the data sample $y$ falls to the left of all sub-population means ($y<\mu_j$ for all $j$). And $N=0$ is the only valid value for $N$ when $y$ falls between sub-populations means. Thus the noise $N$ tends to pull the data sample $y$ away from the tails and towards the cluster of sub-population means (or locations).

The GMM--NEM condition extends component-wise to $d$-dimensional GMM models with \emph{diagonal} covariance matrices. Suppose $\mathbf{y} = \{y_\iota\}_\iota^d$ is one such $d$-D GMM sample. Then the vector noise sample $\mathbf{n}$ satisfies the NEM condition if
\begin{equation}
	n_\iota^2 \leq 2n_\iota\left(\mu_{j,\iota}-y_\iota\right) \quad \forall \; j, \forall \; \iota \label{eq:d-Dim-IndepCondition}
\end{equation}
where $j$ denotes sub-populations and $\iota$ denotes dimensions. The condition (\ref{eq:d-Dim-IndepCondition}) holds because 
\begin{equation}
	\prod_{\iota=1}^d f(y_\iota + n_\iota|j,\theta) \geq \prod_{\iota=1}^d f(y_\iota |j,\theta)
\end{equation} when $f(y_\iota + n_\iota|j,\theta) \geq f(y_\iota |j,\theta)$ for each dimension $\iota$. So $d$-D GMM--NEM for vector sub-populations with diagonal covariance matrices is equivalent to running $d$ separate GMM--NEM estimations.

The multidimensional quadratic positivity condition (\ref{eq:d-Dim-IndepCondition}) has a simple geometric description. Suppose there exist noise samples $\mathbf{n}$ that satisfy the quadratic NEM condition. Then the valid noise samples $\mathbf{n} = \{n_\iota\}_\iota^d$ for the data $\mathbf{y}$ must fall in the interior or on the boundary of the minimal hyper-rectangle:
\begin{equation}
	\mathbf{n} \in \prod_\iota^d \{n_\iota \mid \; n_\iota^2 \leq 2n_\iota\left(\mu_{min,\iota}-y_\iota\right) \} \;. \label{eq:NEMGeom-eq}
\end{equation}
$\boldsymbol{\mu}_{min}$ is the centroid of the NEM set. The differences $\{(\mu_{j,\iota}-y_\iota)\}$ determine the location of $\boldsymbol{\mu}_{min}$:
\begin{align}
	\boldsymbol{\mu}_{min} &= (\mu_{min,1}, \ldots, \mu_{min,d}) \\
	\textrm{where} \quad \mu_{min, \iota} &= \min_{j=1,\ldots,K}\{(\mu_{j,\iota}-y_\iota) \}
\end{align}
for a GMM with $K$ sub-populations. While the differences $\{2(\mu_{j,\iota}-y_\iota)\}$ determine the bounds of the set (see Figure~\ref{fg:NEM-Geom}) just like in the $1$-D case (\ref{eq:neg-nem-set}--\ref{eq:intersect-nem-sets}). This multidimensional NEM set is the product of $1$-D GMM--NEM sets $N_j(y_\iota)$ from each dimension $\iota$. It is a hyper-rectangle because the $1$-D GMM--NEM sets are either intervals or the singleton set $\{0\}$. The NEM set always has one vertex anchored at the origin\footnote{
	Thus the EM algorithm is always a special case of the NEM algorithm. This also means that contractions of convex NEM sets still satisfy the NEM condition. This property is important for implementing noise cooling in the NEM algorithm.
} since the origin is always in the NEM set (\ref{eq:NEMGeom-eq}).

Figure~\ref{fg:NEM-Geom} illustrates the geometry of a sample $2$-D GMM--NEM set for a $2$-D GMM sample $\mathbf{y}$. GMM--NEM sets are lines for $1$-D samples, rectangles for $2$-D samples, cuboids for $3$-D samples, and hyper-rectangles for $(d>3)$-D samples.
\begin{figure}[ht!]
\centerline{  \includegraphics[width=0.9\textwidth]{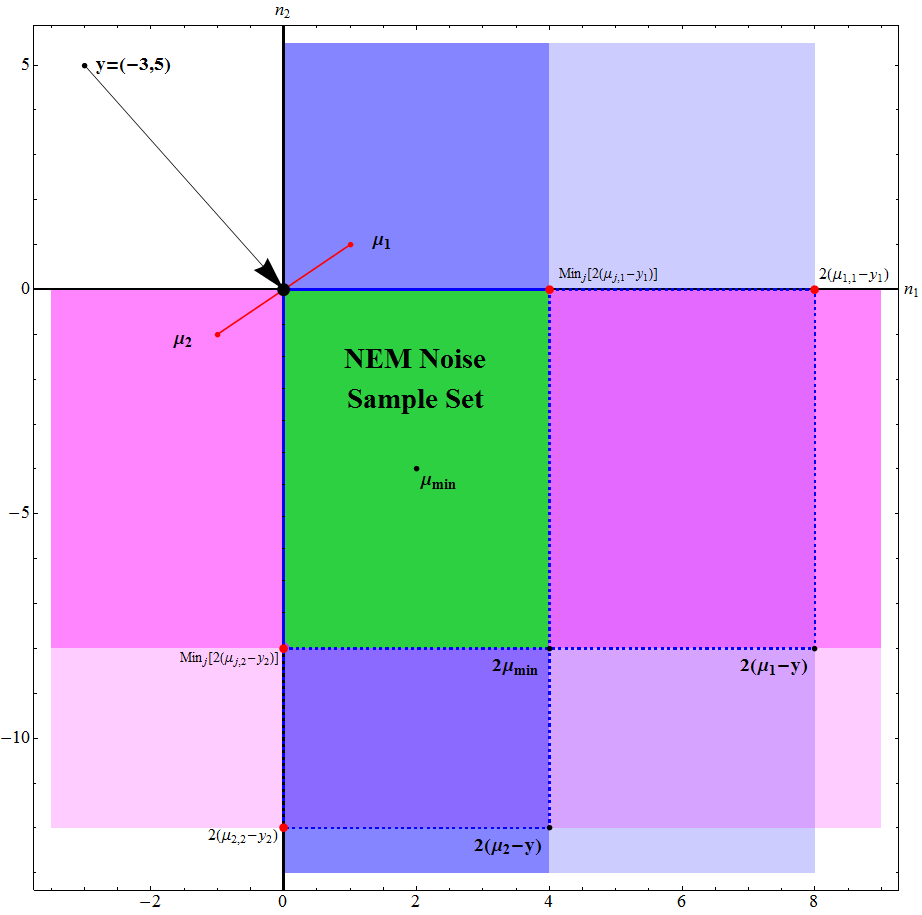} }
\caption[The Geometry of a NEM set for GMM-- and CMM--NEM]{
	An illustration of the geometry of the GMM--NEM set in $2$-dimensions. The $2$-D NEM set is the green rectangle with centroid $\boldsymbol{\mu}_{min}=(2,-4)$ and a vertex at the origin. The sample $\mathbf{y}=(-3,5)$ comes from a $2$-cluster GMM. The GMM sub-population means are $\boldsymbol{\mu}_{1}=(1,1)$ and $\boldsymbol{\mu}_{2}=(-1,-1)$. The differences $(\boldsymbol{\mu}_{j,\iota}-y_\iota)$ determine the location of the centroid $\boldsymbol{\mu}_{min}$ and $2\boldsymbol{\mu}_{min}$ is the diagonal vertex of the NEM set. The $1$-D NEM intervals in the $n_1$ (intersection of the blues sets) and $n_2$ (intersection of the magenta sets) dimensions specify the side-lengths for the $2$-D NEM-set. Vectors in the NEM set draw the sample $\mathbf{y}$ towards the midpoint of $\boldsymbol{\mu}_1$ and $\boldsymbol{\mu}_2$. 
}
\label{fg:NEM-Geom}
\end{figure}
The geometry of the NEM set implies that NEM noise tends to pulls the sample $y$ towards the centroid of the set of sub-population means $\{\boldsymbol{\mu}_1, \ldots, \boldsymbol{\mu}_K\}$. This centralizing action suggests that NEM converges faster for GMMs and CMMs because the injective noise makes the EM estimation more robust to outlying samples.

\subsection{NEM for Mixtures of Jointly Gaussian Populations}\label{subsec:JGMM-NEM}

The $d$-dimensional NEM condition in equation~(\ref{eq:NEMGeom-eq}) is valid when the sub-population covariance matrices are diagonal. More general jointly Gaussian (JG) sub-populations have correlations between dimensions and thus non-zero off-diagonal terms in their covariance matrices. NEM for such jointly Gaussian mixture models(JG-MMs) sub-populations requires a more general sufficient condition. Corollary~\ref{cor:JGMM-NEM} states this condition. The corollary uses the following notation for the quadratic forms in the JG pdf:
\begin{align}
	\minnprod{\mathbf{u}}{\mathbf{v}} &= \mathbf{u}^{T} \, \Sigma_j^{-1} \, \mathbf{v} \\
	\mnorm{\mathbf{u}} &= \mathbf{u}^{T} \, \Sigma_j^{-1} \, \mathbf{u} \;.
\end{align}
These are the inner product and the norm based on the inverse of the non-degenerate symmetric positive definite covariance matrix $\Sigma_j^{-1}$.

%%%%%%%%%
\begin{cor}{\bf{[NEM Condition for JG-MMs]:}}\label{cor:JGMM-NEM}\\
	Suppose $\mathbf{Y}|_{Z=j} \sim {\cal N}(\boldsymbol{\mu}_j,\Sigma_j)$ for $d$-dimensional random vectors $\mathbf{Y}$. And thus $f(\mathbf{y}|j,\theta)$ is a jointly Gaussian  pdf. Then the NEM sufficient condition
	\begin{equation}
		\E_{\mathbf{Y},Z,N|\theta^*} \left[ \ln \frac{f(\mathbf{Y}+\mathbf{N},Z|\theta)}{f(\mathbf{Y},Z|\theta)} \right] \geq 0
	\end{equation}
	holds if 
	\begin{equation}
		\mnorm{\mathbf{n}} 
		+ 2\minnprod{(\mathbf{y}-\boldsymbol{\mu}_j)}{\mathbf{n}}
		\leq  0
		\quad \forall \;\; j \;.
		\label{eq:JG-MM-NEM}
	\end{equation}
\end{cor}

\begin{proof}
By Corollary~\ref{cor:DomPDF-NEM}, the NEM condition
\begin{align}
	\E_{\mathbf{Y},Z,N|\Theta_*} \left[ \ln\left( \frac{f(\mathbf{Y}+\mathbf{N},Z|\Theta)}{f(\mathbf{Y},Z|\Theta)} \right) \right] \geq 0 \;.
\end{align}
holds if
\begin{equation}
	f(\mathbf{y}+\mathbf{n},z|\theta) - f(\mathbf{y},z|\theta) \geq 0
\end{equation} for almost all $\mathbf{n}, \mathbf{y}, z$.
This dominance condition is equivalent to 
\begin{equation}
	\ln f(\mathbf{y}+\mathbf{n},z|\theta) - \ln f(\mathbf{y},z|\theta) \geq 0
\end{equation}
over the support of the pdf.
The complete JG-MM log-likelihood $\ln f(\mathbf{y},z|\theta)$ is
\begin{align}
	\ln f(\mathbf{y},z|\theta) = \sum_j \delta[z-j]  \ln \Big[\alpha_j ~f(\mathbf{y}|j,\theta_j) \Big]
\end{align}
using the exponential form of the finite mixture model pdf in (\ref{eq:FMM-fyz-ExpForm}). So \linebreak $\ln f(\mathbf{y}+\mathbf{n},z|\theta) - \ln f(\mathbf{y},z|\theta) \geq 0$ iff 
\begin{align}
	\sum_j \delta[z-j] \Big[\ln \left(\alpha_j ~f(\mathbf{y} + \mathbf{n}|j,\theta_j) \right) - \ln \left(\alpha_j ~f(\mathbf{y}|j,\theta_j) \right) \Big] &\geq  0 \\
	\iff \quad \sum_j \delta[z-j] 	 \Big[ \ln f(\mathbf{y} + \mathbf{n}|j,\theta_j) - \ln f(\mathbf{y}|j,\theta_j) \Big] &\geq  0 \;.
\end{align}
JG sub-population pdf and log-likelihood are 
\begin{align}
	f(\mathbf{y}|j,\theta_j) &= \frac{1}{\sqrt{ (2\pi)^d |\Sigma_j|} } \exp \left(- 0.5\left(\mathbf{y}-\boldsymbol{\mu}_j \right)^{T} \Sigma_j^{-1} \left(\mathbf{y}-\boldsymbol{\mu}_j  \right) \right) \;,
\label{eq:JG-sub-pdf}\\
	\ln f(\mathbf{y}|j,\theta) &= -0.5 d \ln(2\pi) -0.5\ln |\Sigma_j|  - 0.5\left(\mathbf{y}-\boldsymbol{\mu_j} \right)^{T} \Sigma_j^{-1} \left(\mathbf{y}-\boldsymbol{\mu_j}  \right)
\label{eq:JG-sub-LL}
\end{align}
where $ |\Sigma_j|$ is the determinant of the covariance matrix $\Sigma_j$. Define $\mathbf{w}_j$ as 
\begin{equation}
	\mathbf{w}_j = \mathbf{y}-\boldsymbol{\mu_j} \;.
\end{equation}
This simplifies the JG sub-population log-likelihood to
\begin{align}
	\ln f(\mathbf{y}|j,\theta) &= -0.5(d\ln(2\pi) + \ln |\Sigma_j|) - 0.5\ln |\Sigma_j| - 0.5\mnorm{\mathbf{w}_j} \;.
\end{align}
Thus
\begin{multline}
	\sum_j \delta[z-j] \Big[ 
		 -0.5(d \ln(2\pi) + \ln |\Sigma_j|) - 0.5\mnorm{\mathbf{w}_j + \mathbf{n}} \\
		 +0.5(d \ln(2\pi) + \ln |\Sigma_j|) + 0.5\mnorm{\mathbf{w}_j}
	\Big] \geq  0
\end{multline}
\begin{align}
	\iff \sum_j \delta[z-j] \Big[ 
		- \mnorm{\mathbf{w}_j + \mathbf{n}} 
		+ \mnorm{\mathbf{w}_j}
	\Big] &\geq  0 \\
	\iff \sum_j \delta[z-j] \Big[ 
		- \mnorm{\mathbf{n}} 
		- 2\minnprod{\mathbf{w}_j}{\mathbf{n}}
	\Big] &\geq  0\;.
\end{align}
This gives the $d$-D JG-MM NEM condition 
\begin{equation}
	\sum_j \delta[z-j] \Big[ \mnorm{\mathbf{n}} + 2\minnprod{\mathbf{w}_j}{\mathbf{n}} \Big] \leq  0 \;.
	\label{eq:d-Dim-JGMMNEM}
\end{equation}
A sub-case of (\ref{eq:d-Dim-JGMMNEM}) satisfies the condition by  ensuring that each summand is negative $[ \mnorm{\mathbf{n}} + 2\minnprod{\mathbf{w}_j}{\mathbf{n}} ] \leq  0$. This sub-case eliminates the need for the summation in (\ref{eq:d-Dim-JGMMNEM}) and gives the more conservative JG-MM NEM condition:
\begin{equation}
	\mnorm{\mathbf{n}} 
	+ 2\minnprod{(\mathbf{y}-\boldsymbol{\mu}_j)}{\mathbf{n}}
	\leq  0
	\quad \forall \;\; j
\end{equation}
since $\mathbf{w}_j=(\mathbf{y}-\boldsymbol{\mu}_j)$.
\end{proof} 
%%%%%%%%%

This is a direct analogue and generalization of the 1-D condition with one important difference: all estimated parameters, $(\mu_j, \Sigma_j)$, occur in the condition. So the condition is not as useful as the specialized GMM--NEM condition for diagonal covariance matrices. Diagonal covariance matrices allow us to eliminate variances from the NEM condition and thus get \emph{exact} NEM noise benefit conditions for variance estimation. The condition in (\ref{eq:JG-MM-NEM}) provides a basis for \emph{approximate} NEM noise benefit conditions at best.

Figure~\ref{fg:GMM-Geom-ellipse} shows a sample NEM set for a 2-dimensional jointly Gaussian mixture with two clusters. The NEM sets for JG-MMs are intersections of ellipsoids. The component-wise $d$-Dim NEM condition (\ref{eq:NEMGeom-eq}) is a sub-case of (\ref{eq:JG-MM-NEM}) when the covariance matrices are diagonal\footnote{
	The inner products reduce to sums of $1$-D quadratic terms which we bound separately to get the component-wise NEM condition in (\ref{eq:NEMGeom-eq}).
}. Thus the NEM set for (\ref{eq:NEMGeom-eq}) (the green box) is a subset of the more general JG-MM NEM set (the blue overlapping region) for the same sample $\mathbf{y}$ and data model. The general theme for JG and Gaussian mixture models is that the roots of second-order polynomials determine the boundaries of NEM sets.
\begin{figure}[h!]
	\centerline{ \includegraphics[width=0.9\textwidth]{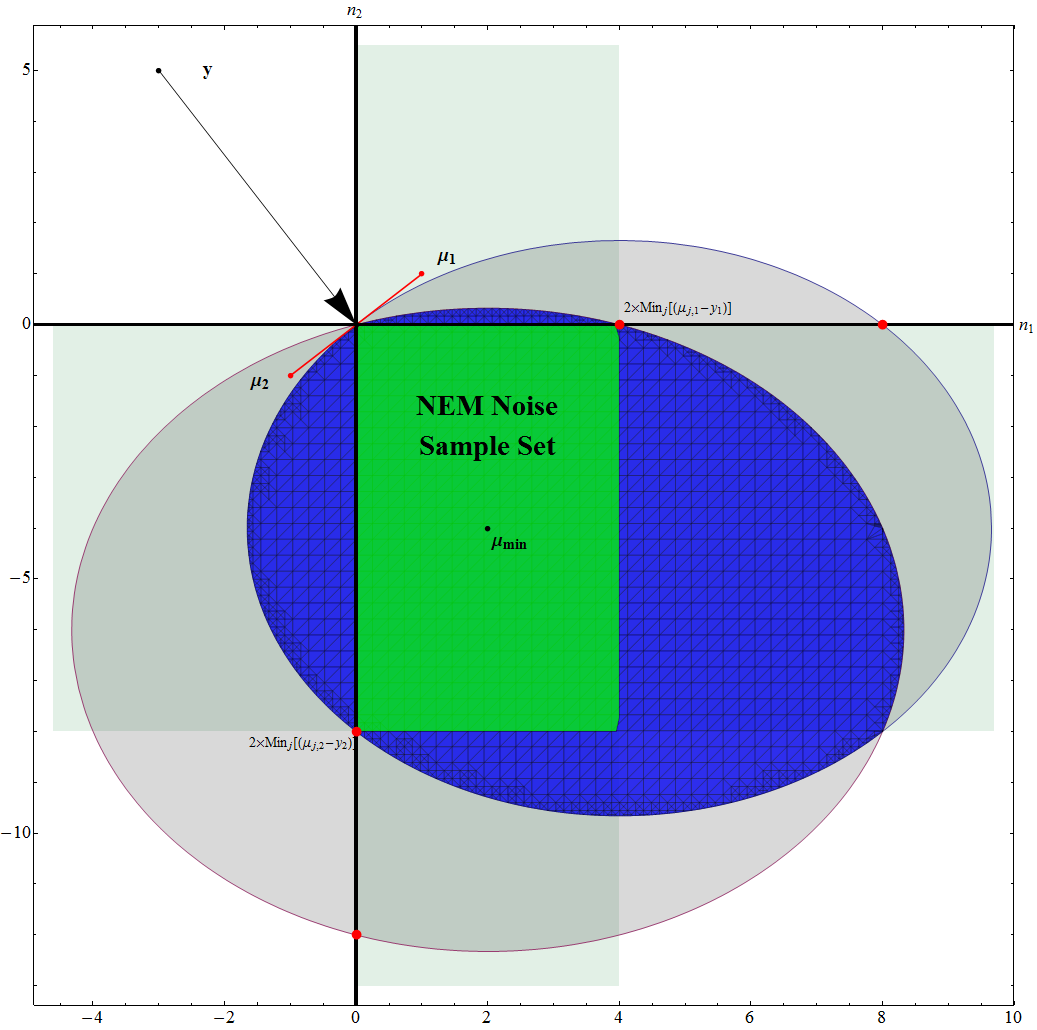} }
	\caption[An illustration of a sample NEM set for a mixture of two $2$-D jointly Gaussian populations]{
		An illustration of a sample NEM set for a mixture of two $2$-D jointly Gaussian populations. NEM noise samples must come from the overlapping region in blue. This region is the intersection of two ellipses. This is the JG-MM NEM set for the same sample and data model in Figure~\ref{fg:NEM-Geom}. The JG-MM NEM set is a superset of the product GMM--NEM set in Figure~\ref{fg:NEM-Geom} (the green box here) that comes from manipulating each dimension separately.
	}
\label{fg:GMM-Geom-ellipse}
\end{figure}

Intersections of ellipsoids do not factor into component dimensions the way products of intervals do. So noise sampling from the JG-MM NEM set requires the use of complicated joint distributions. This emphasizes the need for identifying simpler interior NEM set (like the green box in Figure~\ref{fg:GMM-Geom-ellipse}) for noise sampling. A study of the geometry of NEM sets may lead to more intelligent noise sampling schemes for NEM. Ease of noise sampling becomes increasingly important as sample dimension increases.

The NEM theorem produces better average performance than EM algorithm for any EM estimation problem. This may seem to violate Wolpert and Macready's ``No Free Lunch'' (NFL) theorems~\parencite{wolpert-macready1997} which asserts that this sort of superior algorithmic performance is not possible over the whole space of all optimization problems. The apparent violation is an illusion. The NFL theorems apply to algorithms with static objective functions. The objective functions in NEM algorithms are dynamic; noise injection perturbs the likelihood/objective function at each iteration. And the NEM theorem specifically steers the likelihood function evolution towards favorable solutions.

But the use of evolving objective functions comes at a price: intelligent objective function evolution raises the algorithm's computational complexity. The noise sampling step accounts for the increased computational complexity in the NEM algorithm. Noise sampling complexity can sometimes lead to higher raw computation time even when the NEM algorithm converges in fewer iterations than the EM algorithm. This problem highlights the need for efficient noise sampling routines. Simple NEM conditions (e.g. GMM--NEM and CMM--NEM) can keep noise sampling cost down. But the ideal case would inject unscreened noise and still produce noise benefits. \Sec\ref{sec:samp-sz-effects} below discusses this ideal case.

\subsection{NEM for Models with Log-Convex Densities}
EM algorithms can satisfy the positivity condition (\ref{eq:Goal}) if they use the proper noise $N$. They can also satisfy the condition if the data model has an amenable complete data pdf $f(x|\theta)$. Inequalities (\ref{eq:FinMixtureCond}) and (\ref{eq:FinCauchyMixtureCond}) can sculpt the noise $N$ to satisfy (\ref{eq:Goal}) for Gaussian and Cauchy mixture models. The next corollary shows how the complete data pdf can induce a noise benefit. The corollary states that a log-convex complete pdf satisfies (\ref{eq:Goal}) when the noise is zero-mean. The corollary applies to data models with more general complete random variables $X$. These include models whose complete random variables $X$ do not decompose into the direct product $X=(Y,Z)$. Examples include censored models that use the unobserved complete random variable as the latent random variable $Z=X$ \parencite{gupta-chen2010, chauveau1995, tan-tian-ng2010}.

%%%%%%%%%%%%%%%%%%%%
\begin{cor}{\bf{[NEM Condition for Log-Convex pdfs]:}}\label{cor:lgcnvx-NEM} \\
Suppose that $f(x|\theta)$ is log-convex in $x$  and {$N$ is independent of $X$}. Suppose also that $E_N\left[N\right]=0$.
Then
\begin{align}
	\E_{X,N|\theta^*} \left[ \ln \frac{f(X+N|\theta_k)}{ f(X|\theta_k)} \right] &\geq 0 \;. \label{eq:condin2}
\end{align}
\end{cor}

\begin{proof}
$f(y,z|\theta)$ is log-convex in $y$ and $E_N\left[y+N\right]=y$. So
\begin{align}
	E_N\left[\ln f(y+N,z|\theta_k)\right] \geq \ln f(E_N\left[y+N\right],z|\theta_k)\;.
\end{align}
The right-hand side becomes
\begin{align}
	\ln f(E_N\left[y+N\right],z|\theta_k) &~=~  \ln f(y+E_N\left[N\right],z|\theta_k) \\
	&~=~ \ln f(y,z|\theta_k)
\end{align}
because $E[N]=0$. So
\begin{align}
	E_N\left[\ln f(y+N,z|\theta_k)\right] &~\geq~ \ln f(y,z|\theta_k) \\
	\text{iff} \quad \left( E_N\left[\ln f(y+N,z|\theta_k)\right] - \ln f(y,z|\theta_k) \right) &~\geq~ 0 \\
	\text{iff} \quad \left( E_N\left[\ln f(y+N,z|\theta_k) - \ln f(y,z|\theta_k) \right) \right] &~\geq~ 0 \\
	\text{iff} \quad \E_{Y,Z|\theta^*} \left[ \E_{N} \left[ \ln f(Y+N,Z|\theta_k) - \ln f(Y,Z|\theta_k) \right] \right] &~\geq ~0 
\end{align}
\begin{equation}
	\text{iff} ~~
	\quad \E_{Y,Z,N|\theta^*} \left[ \ln \frac{f(Y+N,Z|\theta_k)}{ f(Y,Z|\theta_k)} \right] ~\geq~ 0 \;. \label{eq:End}
\end{equation}
Inequality (\ref{eq:End}) follows because $N$ is independent of $\theta^*$. 
\end{proof}
%%%%%%%%%%%%%%%%%%%%

The right-censored gamma data model gives a log-convex data model when the $\alpha$-parameter of its complete pdf lies in the interval $(0,1)$. This holds because the gamma pdf is log-convex when $0<\alpha<1$.  Log-convex densities often model data with decreasing hazard rates in survival analysis applications \cite{proschan1963,dahiya-gurland1972,bagnoli-bergstrom2005}. \Sec\ref{subsec:Exp-EM} describes the gamma data model and EM algorithm. Figure~\ref{fg:LogConvex} shows a simulation instance of noise benefits for a log-convex model. The simulation estimates the $\theta$ parameter from right-censored samples of a $\gamma(0.65,4)$ pdf. Samples are censored to values below a threshold of \Tlgcx. The average optimal NEM speed-up over the noiseless EM algorithm is about \lgcx.

\begin{figure}[ht!]
\centerline{\includegraphics[width=0.75\textwidth]{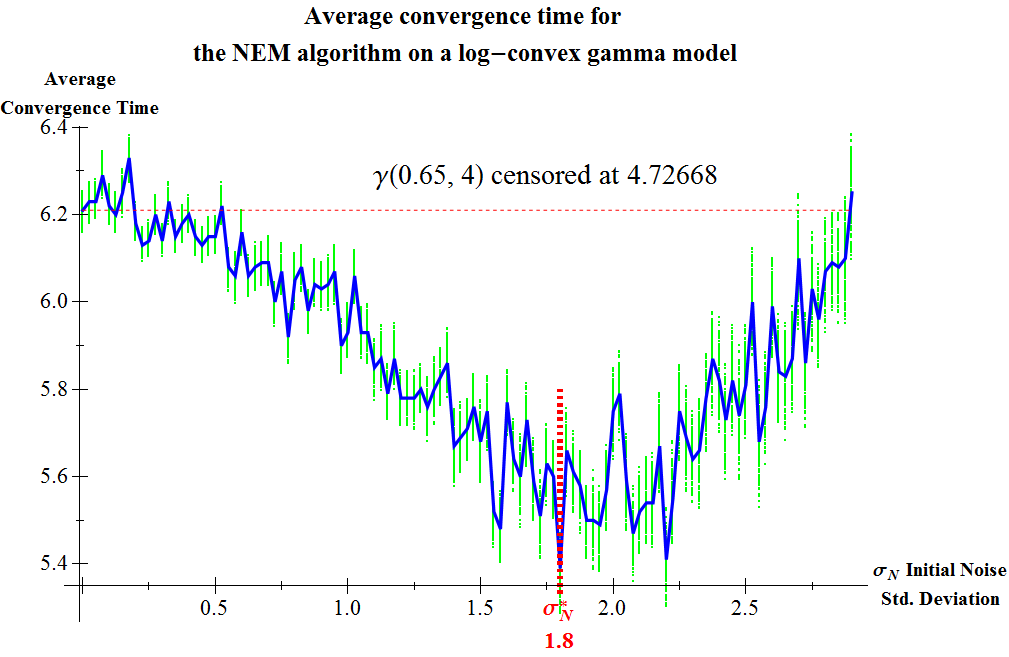} }
\caption[EM noise benefit for log-convex censored gamma model]{
	EM noise benefit for log-convex censored gamma model. This plot uses the annealed-noise NEM algorithm. The average optimal NEM speed-up over the noiseless EM algorithm is about \lgcx. Low intensity initial noise decreases convergence time while higher intensity starting noise increases it. This NEM procedure adds cooled i.i.d. normal noise that is \emph{independent} of the data. The noise cools at an inverse-square rate. The log-convex gamma distribution is a $\gamma(\alpha, \theta)$ distribution with $\alpha < 1$. The censored gamma EM estimates the $\theta$ parameter. The model uses $375$ censored gamma samples. Each sampled point on the curve is the mean of $100$ trials. The vertical bars are $95\%$ bootstrap confidence intervals for the mean convergence time at each noise level.
}
\label{fg:LogConvex}
\end{figure}

A modification of Corollary~\ref{cor:lgcnvx-NEM} predicts a similar noise benefit if we replace the zero-mean additive noise with unit-mean multiplicative noise. The noise is also independent of the data.

\section{The Noisy Expectation-Maximization Algorithm} \label{sec:NEM-Alg}

The NEM Theorem and its corollaries give a general method for modifying the noiseless EM algorithm. The NEM Theorem also implies that on average these NEM variants outperform the noiseless EM algorithm. Algorithm \ref{algo:NEM} below gives the Noisy Expectation-Maximization algorithm schema. The operation \textsc{NEMNoiseSample($\mathbf{y}$) } generates noise samples that satisfy the NEM condition for the current data model. The noise sampling distribution depends on the data vector $\mathbf{y}$ in the Gaussian and Cauchy mixture models.

\begin{algorithm}[H]
\DontPrintSemicolon
\SetKwInOut{Input}{Input}
\SetKwInOut{Output}{Output}
\SetKwFunction{NEMNoiseSample}{NEMNoiseSample}
\Input{$\mathbf{y} = \left( y_1,\ldots, y_M \right)$ : vector of observed incomplete data}
\Output{ $\hat{\theta}_{NEM}$ : NEM estimate of parameter $\theta$}
\BlankLine
\While{($\|\theta_k-\theta_{k-1}\| \geq 10^{-tol}$)}{
$\mathbf{N_S}${\bf -Step:} $\mathbf{n} \leftarrow k^{-\tau}  \times$ \NEMNoiseSample{$\mathbf{y}$} \;
$\mathbf{N_A}${\bf -Step:} $\mathbf{y}_\dagger \leftarrow \mathbf{y} + \mathbf{n}$ \;
{\bf E-Step:} $Q \left( \theta |\theta_k \right) \leftarrow\E_{\mathbf{Z}|\mathbf{y},\theta_k}  \left[ \ln f(\mathbf{y}_\dagger , \mathbf{Z}|\theta) \right] $ \;
{\bf M-Step:} $\theta_{k+1} \leftarrow \argmax{\theta} \left\{ Q\left( \theta |\theta_k \right) \right\}$\;
$k \leftarrow k+1$\;
}
$\hat{\theta}_{NEM} \leftarrow \theta_k$
\caption{The NEM Algorithm}\label{algo:NEM}
\end{algorithm}

The Maximum a Posteriori (MAP) variant of the NEM algorithm applies MAP-EM modification to the NEM algorithm: it replaces the EM $Q$-function with the NEM $Q_N$-function:
\begin{align}
	Q_N(\theta|\theta_k) = \E_{Z|y, \theta_k}[\ln f(y+N,Z|\theta)] + P(\theta) \;.
\end{align}
\begin{algorithm}[H]
{\bf E-Step:} $Q_N \left( \theta |\theta_k \right) \leftarrow\E_{\mathbf{Z}|\mathbf{y},\theta_k}  \left[ \ln f(\mathbf{y}_\dagger , \mathbf{Z}|\theta) \right]  + P(\theta)$ \;
\caption{Modified {\bf E-Step:} for Noisy MAP-EM}\label{algo:NMAP-EM}
\end{algorithm}

\noindent The E-Step in both cases takes the conditional expectation of a function of the noisy data $\mathbf{y}_\dagger$ given the noiseless data $\mathbf{y}$.

A deterministic decay factor $k^{-\tau}$ scales the noise on the $k^{th}$ iteration. $\tau$ is the noise decay rate. The decay factor $ k^{-\tau} $ reduces the noise at each new iteration. This factor drives the noise $N_k$ to zero as the iteration step $k$ increases. The simulations in this paper use $\tau = 2$ for demonstration. {Values between $\tau = 1$ and $\tau = 3$ also work}. $N_k$ still needs to satisfy the NEM condition for the data model. The cooling factor $k^{-\tau}$ must not cause the noise samples to violate the NEM condition. This usually means that $0 < k^{-\tau}  \leq 1$ and that the NEM condition solution set is closed with respect to contractions.

The decay factor reduces the NEM estimator's jitter around its final value. This is important because the EM algorithm converges to fixed-points. So excessive estimator jitter prolongs convergence time even when the jitter occurs near the final solution. The simulations in this paper use polynomial decay factors instead of logarithmic cooling schedules found in annealing applications \parencite{kirkpatrick-gelatt-vecchi1983, cerny1985, geman-hwang1986, hajek1988, kosko-nnfs}.

The NEM algorithm inherits some variants from the classical EM algorithm schema. A NEM adaptation to the Generalized Expectation Maximization (GEM) algorithm is one of the simpler variations. The GEM algorithm replaces the EM maximization step with a gradient ascent step. The Noisy Generalized Expectation Maximization (NGEM) algorithm uses the same M-step:

\vspace{6pt}

\begin{algorithm}[H]
{\bf M-Step:} $\theta_{k+1} \leftarrow \tilde{\theta}$  such that $Q(\tilde{\theta}|\theta_k) \geq  Q\left( \theta_k |\theta_k \right) $
\caption{Modified {\bf M-Step:} for NGEM:}\label{algo:NGEM}
\end{algorithm}

\vspace{6pt}

The NEM algorithm schema also allows for some variations outside the scope of the EM algorithm. These involve modifications to the noise sampling step $\mathbf{N_S}${\bf -Step} or to the noise addition step $\mathbf{N_A}${\bf -Step}

One such modification does not require an additive noise term $n_i$ for each $y_i$. This is useful when the NEM condition is stringent because then noise sampling can be time intensive. This variant changes the $\mathbf{N_S}${\bf -Step} by picking a random or deterministic sub-selection of $\mathbf{y}$ to modify. Then it samples the noise subject to the NEM condition for those sub-selected samples. This is the Partial Noise Addition NEM (PNA-NEM).

\vspace{6pt}

\begin{algorithm}[H]
\DontPrintSemicolon
$\mathcal{I} \leftarrow \left\{1 \dots  M\right\}$\;
$\mathcal{J} \leftarrow$ SubSelection($\mathcal{I}$)\;
\For{$j \in \mathcal{J}$}{
$n_\iota \leftarrow k^{-\tau}  \times$ NEMNoiseSample($y_\iota$) 
}
\caption{Modified  $\mathbf{N_S}${\bf -Step} for PNA-NEM}
\end{algorithm} 

\vspace{6pt}

{The NEM algorithm and its variants need a NEM noise generating procedure \textsc{NEMNoiseSample($\mathbf{y}$)}. The procedure returns a NEM--compliant noise sample $n_j$ at the desired noise level $\sigma_N$ for the current data sample. This procedure will change with the EM data model. The noise generating procedure for the GMM and CMM models derive from Corollaries $2$ and $3$. The $1$-D noise-generating procedure for the GMM simulations is as follows:}

\vspace{6pt}

\begin{algorithm}[H]
\DontPrintSemicolon
\SetKwInOut{Input}{Input}
\SetKwInOut{Output}{Output}
\Input{$y$ and $\sigma_N$ : current data sample and noise level}
\Output{$n$ : noise sample satisfying NEM condition}
\BlankLine
$N(y) \leftarrow \bigcap_j N_j(y)$ \;
$n \leftarrow $ a sample from distribution $TN(0,\sigma_N | N(y))$\;
\caption{ NEMNoiseSample for GMM-- and CMM--NEM}
\end{algorithm} 

\vspace{6pt}

\noindent { where $TN(0,\sigma_N | N(y))$ is the normal distribution $N(0,\sigma_N)$ truncated to the support set $N(y)$. The set $N(y)$ is the interval intersection from ({\ref{eq:intersect-nem-sets}}). Multi-dimensional versions of the generator can apply the procedure component-wise.}

\subsection{NEM via Deterministic Interference}

The original formulation of the NEM theorem and the NEM conditions address the use of random perturbations in the observed data. But the NEM conditions do not preclude the use of deterministic perturbations in the data. Such nonrandom perturbation or \emph{deterministic interference} falls under the ambit of the NEM theorem since any (deterministic) constant is also a random variable with a degenerate distribution.

This interpretation of the NEM theorem leads to another notable variant of the NEM algorithm(Algorithm~\ref{algo:DIEM}): the \emph{deterministic interference EM} (DIEM) algorithm. This algorithm adds deterministic sample-dependent perturbations to the data using only perturbations that satisfy the NEM condition for the data model. DIEM also applies a cooling schedule like NEM. 

\vspace{4pt}

\begin{algorithm}[H]
\DontPrintSemicolon
\SetKwInOut{Input}{Input}
\SetKwInOut{Output}{Output}
\SetKwFunction{SamplePertubation}{SamplePertubation}
\SetKwFunction{NEMSet}{NEMSet}
\Input{$\mathbf{y} = \left( y_1,\ldots, y_M \right)$ : vector of observed incomplete data}
\Output{ $\hat{\theta}_{DIEM}$ : DIEM estimate of parameter $\theta$}
\BlankLine
\While{($\|\theta_k-\theta_{k-1}\| \geq 10^{-tol}$)}{
$\mathbf{N_S}${\bf -Step:} $\mathbf{n} \leftarrow k^{-\tau}  \times$ \SamplePertubation{$\mathbf{y}$} \;
where \SamplePertubation{$\mathbf{y}$} $\in$ $N_j(\mathbf{y})$ \;
$\mathbf{N_A}${\bf -Step:} $\mathbf{y}_\dagger \leftarrow \mathbf{y} + \mathbf{n}$ \;
{\bf E-Step:} $Q \left( \theta |\theta_k \right) \leftarrow\E_{\mathbf{Z}|\mathbf{y},\theta_k}  \left[ \ln f(\mathbf{y}_\dagger , \mathbf{Z}|\theta) \right] $ \;
{\bf M-Step:} $\theta_{k+1} \leftarrow \argmax{\theta} \left\{ Q\left( \theta |\theta_k \right) \right\}$\;
$k \leftarrow k+1$\;
}
$\hat{\theta}_{DIEM} \leftarrow \theta_k$
\caption{The Deterministic Interference EM Algorithm}\label{algo:DIEM}
\end{algorithm}

\vspace{6pt}

\noindent The only difference is in the $\mathbf{N_S}$-Step, the ``noise'' sampling step.

\vspace{4pt}

\begin{algorithm}[H]
\DontPrintSemicolon
\SetKwInOut{Input}{Input}
\SetKwInOut{Output}{Output}
\Input{$y$ and $s_N$ : current data sample and perturbation scale $(s_N\in[0,1])$}
\Output{$n$ : sample perturbation in the NEM set for $y$}
\BlankLine
$N(y) \leftarrow \bigcap_j N_j(y)$ \;
$n \leftarrow s_N\times \textrm{Centroid}[N(y)]$\;
\caption{ SamplePertubation for DIEM}
\end{algorithm} 

\vspace{4pt}

Figure~\ref{fg:GaussDIEM} below shows how the DIEM algorithm performs on the same EM estimation problem in Figure~\ref{fg:GaussNEM}. The confidence bars here measure variability caused by averaging convergence times across multiple instances of the data model. This simulation added deterministic perturbations starting at a value located somewhere on the line between the origin and the center of the NEM set for the current data sample. The initial perturbation factor $s_N$ controls how far from the origin the perturbations start. These perturbations then decay to zero at an inverse-squared cooling rate as iteration count increases.

\begin{figure}[ht!]
\centerline{  \includegraphics[width=0.75\textwidth]{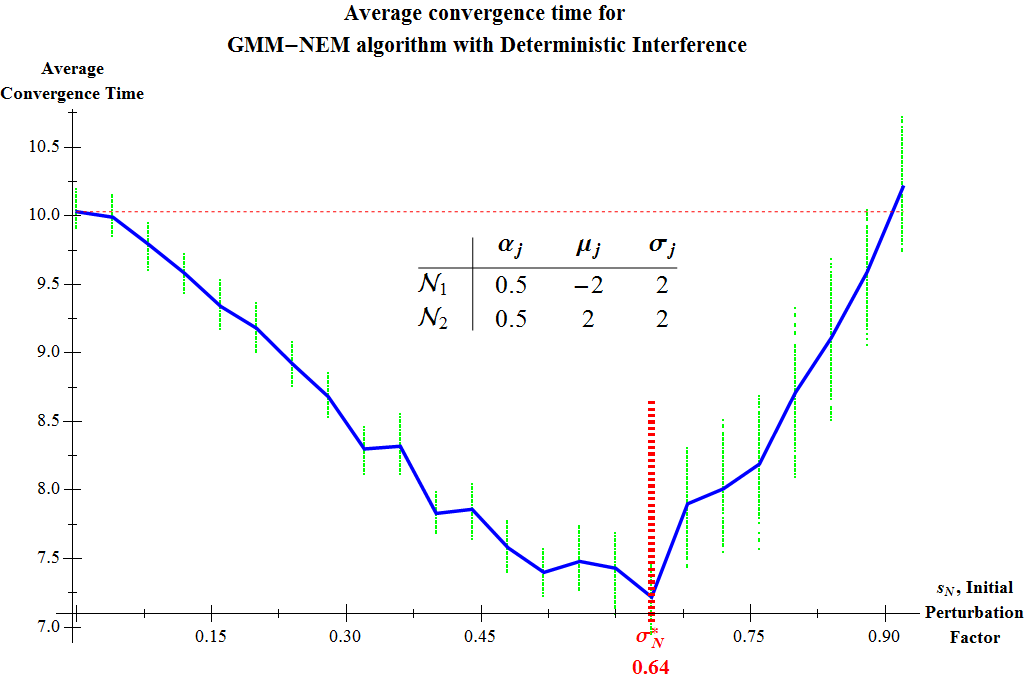} }
\caption[GMM-NEM using deterministic interference on the data samples instead of random noise]{
	Plot showing effects of GMM-NEM when using deterministic interference on the data samples instead of random noise. The data model and experimental setup is the same as in Figure~\ref{fg:GaussNEM}. This DIEM algorithm perturbs the data samples with deterministic values in valid NEM sets. The scale of the perturbations decreases with iteration count just like in the NEM case. This GMM-DIEM algorithm converges faster than the regular EM algorithm on average. The improvement is roughly equivalent to the improvement in the random-noise GMM-NEM---about a $28\%$ improvement over the baseline EM convergence time.
}
\label{fg:GaussDIEM}
\end{figure}

A chaotic dynamical system~\parencite{ott2002, alligood-sauer-yorke1997} can also provide the deterministic interference for the DIEM algorithm. This is the Chaotic EM (CEM) algorithm. Figure~\ref{fg:ChaosNEM} shows an example of a GMM-CEM using deterministic interference from the chaotic logistic map
\begin{equation}
	z_{t+1} = 4z_t(1-z_t)
	\label{eq:logisticMap}
\end{equation}
with initial value $z_0=0.123456789$~\parencite{mitaim-kosko1998SR,ippen-lindner-ditto93}. This logistic dynamical system starting at $z_0$ produces seemingly random samples from the unit interval $[0,1]$. The GMM-CEM uses the scale parameter $A_N$ to fit logistic map samples to the NEM sets for each data sample. The effective injective noise has the form 
\begin{equation}
	n_t = A_N z_t\;.
\end{equation}
$A_N$ cools at an inverse-square rate.
\begin{figure}[ht!]
\centerline{  \includegraphics[width=0.75\textwidth]{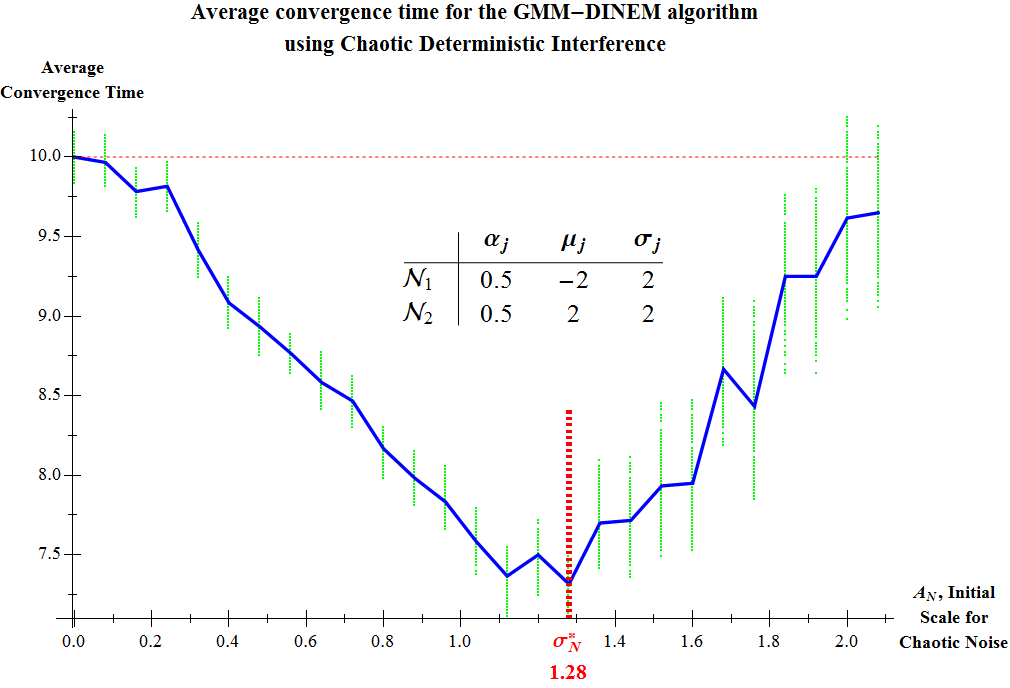} }
\caption[GMM-CEM using chaotic deterministic interference on the data samples instead of random noise]{
	Plot showing effects of GMM-CEM using chaotic deterministic interference on the data samples instead of random noise. The data model and experimental setup is the same as in Figure~\ref{fg:GaussNEM}. This CEM algorithm perturbs the data samples with samples from  a chaotic logistic map scaled to fit in valid NEM sets. The scale of the chaotic perturbations decreases with iteration. This GMM-CEM algorithm converges faster than the regular EM algorithm on average. The improvement is roughly equivalent to the improvement in the random-noise GMM-NEM---about a $26\%$ improvement over the baseline EM convergence time.
}
\label{fg:ChaosNEM}
\end{figure}

\section{Sample Size Effects in the NEM algorithm} \label{sec:samp-sz-effects}

The noise-benefit effect depends on the size of the GMM data set. Analysis of this effect depends on the probabilistic event that the noise satisfies the GMM--NEM condition for the entire sample set. This analysis also applies to the Cauchy mixture model because its NEM condition is the same as the GMM's. Define $A_k$ as the event that the noise $N$ satisfies the GMM--NEM condition for the $k^{th}$ data sample:
\begin{equation}
	A_k = \{N^2 \leq 2 N(\mu_j - y_k)|\; \forall j \} \;. \label{eq:sub-A_M-Event}
\end{equation}
Then define the event $A_M$ that noise random variable $N$ satisfies the GMM--NEM condition for each data sample as
\begin{align}
	A_M =& \bigcap_k^{M} A_k \\
	= & \left\{ N^2 \leq 2N(\mu_j -y_k)|\; \forall j \text{ and } \forall k \right\} \;. \label{eq:AM-Event}
\end{align}
This construction is useful for analyzing NEM when we use \emph{independent and identically distributed (i.i.d.)} noise $N_k \stackrel{d}{=} N$ for all $y_k$ while still enforcing the NEM condition.
%This requires that the noise for each data sample $y_k$ share a common support $A_M$. A special case is the i.i.d. noise model.

\subsection{Large Sample Size Effects}
The next theorem shows that the set $A_M$ shrinks to the singleton set $\{0\}$ as the number $M$ of samples in the data set grows. So the probability of satisfying the NEM condition for i.i.d. noise samples goes to zero as $M \rightarrow \infty$ with probability one.

%%%%%%%%%%
\begin{thm}{\bf{[Large Sample GMM-- and CMM--NEM]:}}\label{thm:lg-sampCGMM}\\
Assume that the noise random variables are i.i.d. The set i.i.d. noise values 
\begin{equation}
	A_M = \left\{ N^2 \leq 2N(\mu_j -y_k)|\; \forall j \text{ and } \forall k \right\}
\end{equation}
that satisfy the Gaussian (Cauchy) NEM condition for all data samples $y_k$ decreases with probability one to the set $\{0\}$ as $M \rightarrow \infty$:
\begin{equation}
	P\left(\lim_{M \rightarrow \infty}A_M  =  \{0\} \right) = 1\;.
\end{equation}
\end{thm}

\begin{proof}
Define the NEM-condition event $A_k$ for a single sample $y_k$ as
\begin{equation}
	A_k = \{N^2 \leq 2 N(\mu_j - y_k)|\; \forall j \} \;.
\end{equation}
Then $N^2 \leq 2N (\mu_j - y_k)$ for all $j$ if $N$ satisfies the NEM condition. So
\begin{align}
	N^2 - 2N (\mu_j -y_k) &\leq 0  \quad \text{for all} \; j \\
	\textrm{and} \quad N(N - 2 (\mu_j -y_k)) &\leq 0  \quad \text{for all} \; j \; .
\end{align}
This quadratic inequality's solution set $(a_j, b_j)$ for $j$ is
\begin{equation}
	I_j = [a_j,b_j] = 
	\begin{cases} 
		\left[ 0,2 (\mu_j -y_k) \right] & \text{if} \quad y_k<\mu_j\\
		\left[ 2 (\mu_j -y_k), 0 \right] & \text{if} \quad y_k>\mu_j \\
		\{0\} & \text{if} \quad y_k \in \left[\min \mu_j, \max \mu_j \right]
	\end{cases} \; .
\end{equation}
Define $b^+_k$ and $b^-_k$ as $b^+_k = 2\min_j (\mu_j - y_k)$ and $b^-_k = 2\max_j (\mu_j - y_k)$. Then the maximal solution set $A_k = [a,b]$ over all $j$ is

\begin{align}
	A_k = \bigcap_j^J I_j  =
	\begin{cases} 
		\left[ 0,b^+_k \right] & \text{if} \quad y_k<\mu_j \; \forall j \\
		\left[ b^-_k, 0 \right] & \text{if} \quad y_k>\mu_j \; \forall j \\
		\{0\} & \text{if} \quad y_k \in \left[\min \mu_j, \max \mu_j \right]
	\end{cases} \label{eq:Ak-Expand}
\end{align}
where $J$ is the number of sub-populations in the mixture density. There is a sorting such that the $I_j$ are nested for each sub-case in (\ref{eq:Ak-Expand}). So the nested interval theorem \parencite{spivak2006} (or Cantor's intersection theorem \parencite{rudin1976}) implies that $A_k$ is not empty because it is the intersection of nested bounded closed intervals.

$A_k = \{0\}$ holds if the NEM condition fails for that value of $y_k$ This happens when some $I_j$ sets are positive and other $I_j$ sets are negative. The positive and negative $I_j$ sets intersect only at zero. No non-zero value of $N$ will produce a positive average noise benefit. The additive noise $N$ must be zero. 

Write $A_M$ as the intersection of the $A_k$ sub-events:
\begin{align}
	A_M &= \left\{ N^2 \leq 2N(\mu_j -y_k)|\; \forall j \text{ and } \forall k \right\} \\
	 &= \bigcap_k^M A_k \\
	 &=
	\begin{cases} 
		\left[ 0, \min_k b^+_{k} \right] & \text{if} \quad y_k<\mu_j \; \forall j,k \\
		\left[ \max_k b^-_{k}, 0 \right] & \text{if} \quad y_k>\mu_j \; \forall j,k \\
		\{0\} & \text{if} \quad \exists k: y_k \in \left[\min \mu_j, \max \mu_j \right] 
	\end{cases} \; . \label{eq:Set-Spec}
\end{align}
Thus a second application of the nested interval property implies that $A_M$ is not empty.

We now characterize the asymptotic behavior of the set $A_M$. $A_M$ depends on the locations of the samples $y_k$ relative to the sub-population means $\mu_j$. 
Then  $A_M = \{0\}$ if there exists some $k_0$ such that $\min \mu_j \leq y_{k_0} \leq \max \mu_j$. Define the set $S = [\min \mu_j, \max \mu_j]$. Then by Lemma \ref{lem:borel-lln} below $\lim_{M \rightarrow \infty} \#_M(Y_k \in S) > 0$ holds with probability one. So there exists with probability one a $k_0 \in \{1 \dots M\} $ such that $y_{k_0} \in S$ as $M \rightarrow \infty$. Then $A_{k_0}=\{0\}$ by equation (\ref{eq:Set-Spec}). Then with probability one:
\begin{align} 
	\lim_{M \rightarrow \infty} A_M & = A_{k_0} \cap \lim_{M \rightarrow \infty} \bigcap^M_{k \neq k_0} A_k \\
	&= \{0\} \cap \lim_{M \rightarrow \infty}  \bigcap^M_{k \neq k_0} A_k  \; .
\end{align}
\noindent So
\begin{equation}
	\lim_{M \rightarrow \infty} A_M = \{0\} \quad \text{with probability one}
\end{equation}
since $A_M$ is not empty by the nested intervals property and since $0 \in A_k$ for all $k$.
\end{proof}

\begin{lem}{\bf{[Borel's Law of Large Numbers]:}}\label{lem:borel-lln} \\
\noindent Suppose that $S \subset \mathbb{R}$ is Borel-measurable and that $\mathbb{R}$ is the support of the pdf of the random variable $Y$. Let $M$ be the number of random samples of $Y$. Then as $M \rightarrow \infty$:
\begin{equation}
	\frac{\#_M(Y_k \in S)}{M} {\rightarrow}  P(Y \in S) \; \text{with probability one.}
\end{equation}
where $\#_M(Y_k \in S)$ is of the random samples $y_1, \ldots,y_M$ of $Y$ that fall in $S$.
\end{lem}

\begin{proof} 
Define the indicator function random variable $\mathbb{I}_S(Y)$ as
\begin{align}
\mathbb{I}_S(Y) = 
\begin{cases}
1 \quad Y \in S \\
0 \quad Y \notin S
\end{cases}\;.
\end{align}
The strong law of large numbers implies that the sample mean $\overline{\mathbb{I}}_S $
\begin{align}
\overline{\mathbb{I}}_S = \frac{\sum_k^M \mathbb{I}_S(Y_k) }{M} = \frac{\#_M(Y_k \in S)}{M} 
\end{align}
converges to $\E [\mathbb{I}_S] $ with probability one. Here $\#_M(Y_k \in S)$ is the number of random samples $ Y_1, \ldots,Y_M$ that fall in the set $S$. But $\E [\mathbb{I}_S] =  P(Y \in S)$. So with probability one:
\begin{equation}
\frac{\#_M(Y_k \in S)}{M} {\rightarrow}  P(Y \in S)
\end{equation}
as claimed.

\noindent Then $P(S) > 0$ implies that
\begin{equation}
\lim_{M \rightarrow \infty} \frac{\#_M(Y_k \in S)}{M} > 0 
\end{equation}
and $\lim_{M \rightarrow \infty} {\#_M(Y_k \in S)} > 0$ with probability one since $M>0$.
\end{proof} 
%%%%%%%%%%%

The proof shows that larger sample sizes $M$ place tighter bounds on the size of $A_M$ with probability one. The bounds shrink $A_M$ all the way down to the singleton set $\{0\}$ as $M \rightarrow \infty$. $A_M$ is the set of values that identically distributed noise $N$ can take to satisfy the NEM condition for all $y_k$. $A_M = \{0\}$ means that $N_k$ must be zero for all $k$ because the $N_k$ are \emph{identically} distributed. This corresponds to cases where the NEM Theorem cannot guarantee improvement over the regular EM using just i.i.d. noise. So identically distributed noise has limited use in the GMM- and CMM-NEM framework.

Theorem 2 is an ``probability-one'' result. But it also implies the following convergence-in-probability result. Suppose $\tilde N$ is an arbitrary \emph{continuous} random variable. Then the probability $P(\tilde N \in A_M)$ that $\tilde N$ satisfies the NEM condition for all samples falls to $P(\tilde N \in \{0\}) = 0$ as $M \rightarrow \infty$. Figure \ref{fg:NEM-SetOccupation} shows a Monte Carlo simulation of how $P(\tilde N \in A_M)$ varies with $M$.

\begin{figure}[ht!]
\centerline{ 
\includegraphics[width=0.75\textwidth]{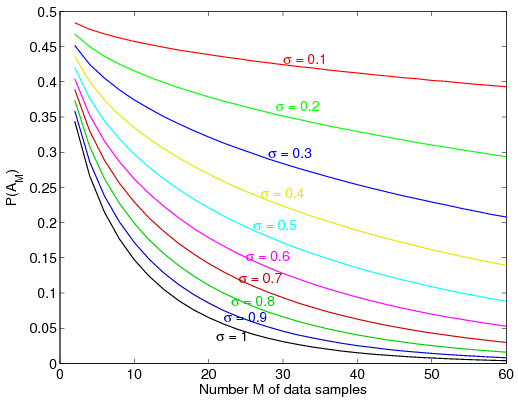} }
\caption[Probability of satisfying the NEM sufficient condition with different sample sizes $M$ and at different noise standard deviations $\sigma_N$]{
	Probability of satisfying the NEM sufficient condition with different sample sizes $M$ and at different noise standard deviations $\sigma_N$. The Gaussian mixture density has mean $\mu= [0, 1]$, standard deviations $\sigma_N = [1, 1]$, and weights $\alpha = [0.5, 0.5]$. The number $M$ of data samples varies from $M = 1$ to $M = 60$. Noise standard deviation varies from $\sigma_N = 0.1$ (top curve) to $\sigma_N = 1.0$ (bottom curve) at $0.1$ incremental step. Monte Carlo simulation computed the probability $P(A_M)$ in equation (\ref{eq:AM-Event}) from $10^6$ samples.
}
\label{fg:NEM-SetOccupation}
\end{figure}

Using non-identically distributed noise $N_k$ avoids the reduction in the probability of satisfying the NEM-condition for large $M$. The NEM condition still holds when $N_k \in A_k$ for each $k$ even if $N_k \notin A_M = \bigcap_k A_k$. This noise sampling model adapts the $k^{th}$ noise random variable $N_k$ to the $k^{th}$ data sample $y_k$. This is the general \emph{NEM noise model}. Figures \ref{fg:GaussNEM} and \ref{fg:LogConvex} use the NEM noise model. This model is equivalent to defining the global NEM event $\tilde A_M$ as a Cartesian product of sub-events $\tilde A_M = \prod_k^M A_k $ instead of the intersection of sub-events $A_M = \bigcap_k A_k$. Thus the bounds of $\tilde A_M$ and its coordinate projections no longer depend on sample size $M$.

Figures \ref{fg:GaussNEM} and \ref{fg:LogConvex} use the Cartesian product noise event. This is standard for NEM algorithms.

Figure~{\ref{fg:NEMCompare-varyM}} and~{\ref{fg:NEMCompare-SingleM}} compare the performance of the NEM algorithm with a pseudo--variant of simulated annealing on the EM algorithm. This version of EM adds annealed i.i.d. noise to data samples $y$ without screening the noise through the NEM condition. Thus we call it \emph{blind} noise injection\footnote{
	Blind noise injection differs from both standard simulated annealing and NEM. Standard simulated annealing (SA) injects independent noise into the optimization variable $\theta$. While blind noise injection (and NEM) injects noise into the data $\mathbf{y}$. Blind noise injection also differs from NEM because NEM noise depends on the data via the NEM condition. While blind noise injection uses noise that is independent of the data (hence the term ``blind''). Thus while blind noise injection borrows characteristics from both NEM and SA, it is very different from both noise injection methods.
}.
Figure {\ref{fg:NEMCompare-varyM}} shows that the NEM outperforms blind noise injection at all tested sample sizes $M$. The average convergence time is about $15\%$ lower for the NEM noise model than for the blind noise model at large values of $M$. The two methods are close in performance only at small sample sizes. This is a corollary effect of Theorem 2 from \Sec\ref{subsec:small-effects}. Figure {\ref{fg:NEMCompare-SingleM}} shows that NEM outperforms blind noise injection at a single sample size $M=225$. But it also shows that blind noise injection may fail to give \emph{any} benefit even when NEM achieves faster average EM convergence for the same set of samples. Thus blind noise injection (or simple simulated annealing) performs worse than NEM and sometimes performs worse than EM itself.

\begin{figure}[ht!]
\centerline{ 
\includegraphics[width=0.75\textwidth]{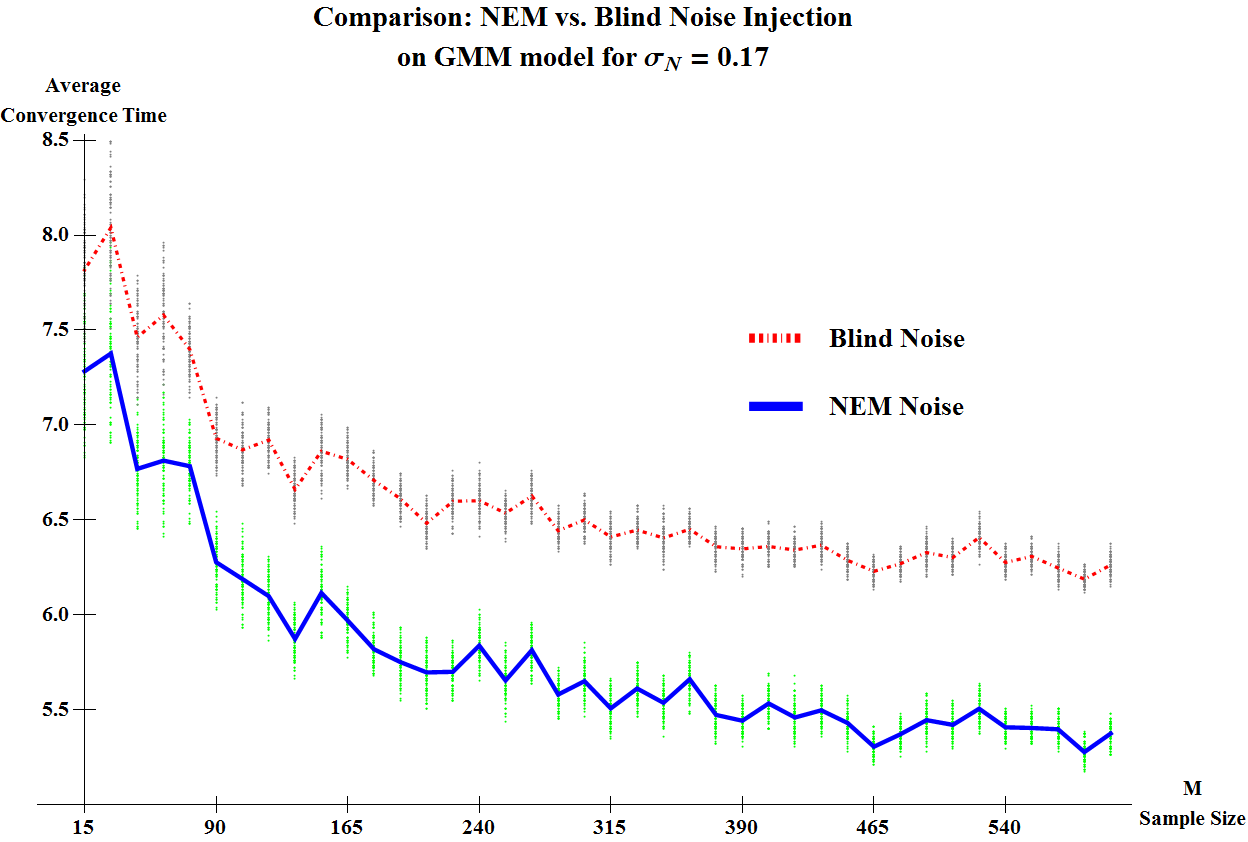} }
\caption[Comparing the effect of the NEM noise sampling model on GMM--EM at different sample sizes $M$]{
	Comparing the effect of the NEM noise sampling model on GMM--EM at different sample sizes $M$. The dependent noise model uses the NEM condition. The independent noise model does not check the NEM condition. So independent noise model has a lower probability of satisfying the NEM condition for all values of $M$. The plot shows that the dependent noise model outperforms the independent noise model at all sample sizes $M$. The dependent noise model converges in about $15\%$ fewer steps than the independent noise model for large $M$. This Gaussian mixture density has sub-population means $\mu= \{0, 1\}$, standard deviations $\sigma = \{1, 1\}$, and weights $\alpha =\{0.5, 0.5\}$. The NEM procedure uses the annealed Gaussian noise with initial noise power at $\sigma_N = 0.17$.
}
\label{fg:NEMCompare-varyM}
\end{figure}

\begin{figure}[ht!]
\centerline{ 
\includegraphics[width=0.75\textwidth]{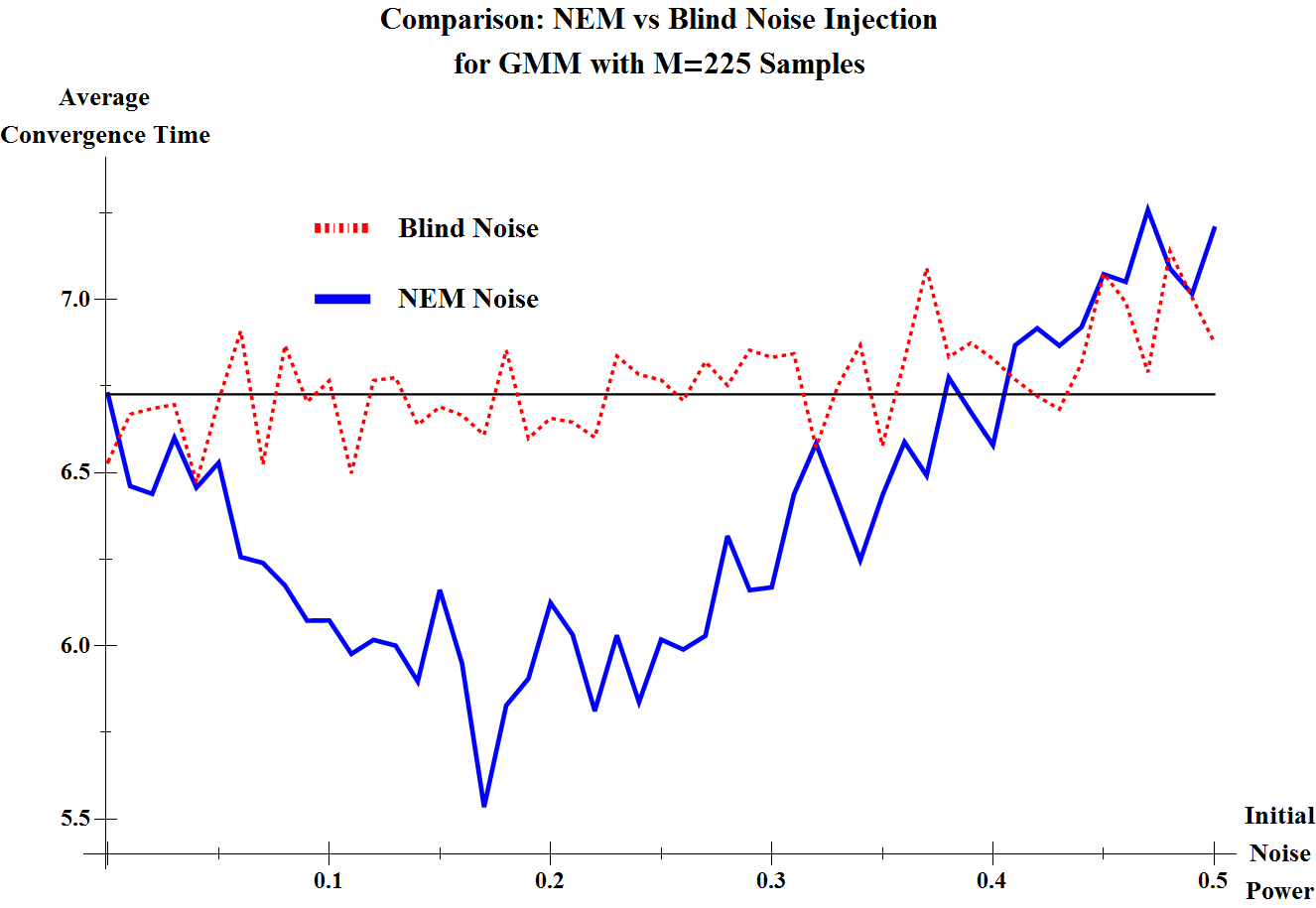} }
\caption[Comparing the effects of noise injection via simulated annealing vs. noise injection via NEM on GMM--EM]{
	Comparing the effects of noise injection via simulated annealing vs. noise injection via NEM on GMM--EM for a single sample size of $M=225$. The dependent noise model uses the NEM condition. The independent noise model adds annealed noise without checking the NEM condition. So independent noise model has a lower probability of satisfying the NEM condition by Theorem \ref{thm:lg-sampCGMM}. The plot shows that NEM noise injection outperforms the simulated annealing noise injection. NEM converges up to about $20\%$ faster than the simulated annealing for this model. And simulated annealing shows no reduction in average convergence time.The Gaussian mixture density has mean $\mu= \{0, 1\}$, standard deviations $\sigma_N = \{1, 1\}$, and weights $\alpha = \{0.5, 0.5\}$ with $M = 225$ samples.
}
\label{fg:NEMCompare-SingleM}
\end{figure}

\subsection{Small Sample Size: Sparsity Effect}\label{subsec:small-effects}

The i.i.d noise model in Theorem~\ref{thm:lg-sampCGMM} has an important corollary effect for sparse data sets. The size of $A_M$ decreases monotonically with $M$ because $A_M = \bigcap^M_k A_k$. Then for $M_0 < M_1$:
\begin{align} 
	P(N \in A_{M_0}) &\geq P(N \in A_{M_1})
\end{align}
since $M_0 < M_1$ implies that $A_{M_1} \subset A_{M_0}$.  Thus arbitrary noise $N$ (i.i.d \emph{and} independent of $Y_k$) is more likely to satisfy the NEM condition and produce a noise benefit for smaller samples sizes $M_0$ than for larger samples sizes $M_1$. The probability that $N \in A_M$ falls to zero as $M \rightarrow \infty$. So the strength of the i.i.d. noise benefit falls as $M \rightarrow \infty$. {Figure \ref{fg:SparseNEM}} shows this sparsity effect. The improvement of relative entropy $D(f*||f_{NEM})$ decreases as the number of samples increases: the noise-benefit effect is more pronounced when the data is sparse.

\begin{figure}[ht!]
\centerline{ \includegraphics[width=0.75\textwidth]{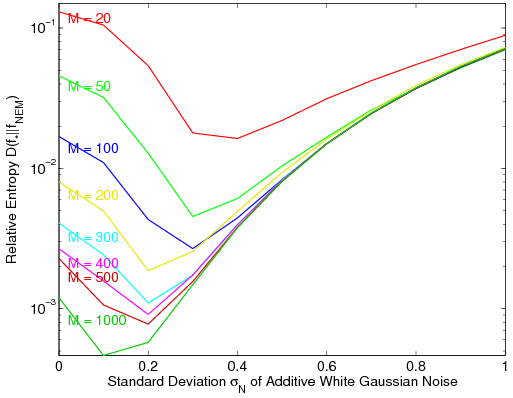} }
\caption[Noise benefits and sparsity effects in the Gaussian mixture NEM at different sample sizes $M$]{Noise benefits and sparsity effects in the Gaussian mixture NEM at different sample sizes $M$. The Gaussian mixture density consists of two sub-populations with  mean $\mu_1 = 0$ and $\mu_2 = 1$ and standard deviations $\sigma_1 = \sigma_2 = 1$. The number $M$ of data samples varies from $M = 20$ (top curve) to $M = 1000$ (bottom curve). The noise standard deviation $\sigma$ varies from $\sigma_N = 0$ (no noise or standard EM) to $\sigma_N = 1$ at 0.1 incremental steps. The plot shows the average relative entropy $D(f_*||f_{\rm NEM})$ over 50 trials for each noise standard deviation $\sigma_N$. $f_* = f(x|\theta)$ is the true pdf and $f_{\rm NEM} = f(x|\theta_{\rm NEM})$ is the pdf of NEM-estimated parameters. }
\label{fg:SparseNEM}
\end{figure}

\subsection{Asymptotic NEM Analysis}
We show last how the NEM noise benefit arises by way of the strong law of large numbers and the central limit theorem.  This asymptotic analysis uses the sample mean $\overline{W}_M$:
\begin{equation}
	\overline{W}_M = \frac{1}{M} \sum_{k=1}^M W_k\;. \label{eq:WBar-Defn}
\end{equation}
The $M$ i.i.d. terms $W_k$ have the logarithmic form
\begin{equation}
	W_k = \ln \frac{f(Y_k+N_k,Z_k|\theta_t)}{ f(Y_k,Z_k|\theta_t)}\;. \label{eq:Wk-Defn}
\end{equation}

The $W_k$ terms are independent because functions of independent random variables are independent.  The random sampling framework of the EM algorithm just means that the underlying random variables are themselves i.i.d.  Each $W_k$ term gives a sampling version of the left-hand side of (\ref{eq:LogCondn}) and thus of the condition that the added noise makes the signal value more probable.

We observe first that either the strong or weak law of large numbers \parencite{billingsley1995} applies to the sample mean $\overline{W}_M$.  The i.i.d. terms $W_k$ have population mean $\mu_W = \E[W]$ and finite population variance $\sigma^2_W = V[W]$.  Then the strong (weak) law of large numbers states that the sample mean $\overline{W}_M$ converges to the population mean $\mu_W$: \vspace{-0.1in}
\begin{equation}
	\overline{W}_M \rightarrow \mu_W
\end{equation}
with probability one (in probability) \parencite{fellerII,billingsley1995,durrett2010}.

The population mean $\mu_W$ differs from $\mu_W^*$ in general for a given $t$ because $\theta_t$ need not equal $\theta_*$ until convergence.  This difference arises because the expectation $\mu_W$ integrates against the pdf $ f(y,z,n|\theta_t)$ while the expectation $\mu_W^*$ integrates against the pdf $f(y,z,n|\theta_*)$.  But $\mu_W \rightarrow \mu_W^*$ as $\theta_t \rightarrow \theta_* $.  So the law of large numbers implies that \vspace{-0.06in}
\begin{equation} \overline{W}_M \rightarrow \mu_W^* \end{equation} 
with probability one (in probability). So the sample mean converges to the expectation in the positivity condition (\ref{eq:Goal}).

The central limit theorem (CLT) applies to the sample mean $\overline{W}_M$ for large sample size $M$.  The CLT states that the standardized sample mean of i.i.d. random variables with finite variance converges in distribution to a standard normal random variable $Z \sim N(0, 1)$ \parencite{billingsley1995}.  A noise benefit occurs when the noise makes the signal more probable and thus when $\overline{W}_M>0$. Then standardizing $\overline{W}_M$ gives the following approximation for large sample size $M$: 
\begin{align}
	P(\overline{W}_M>0) &= P\left(\frac{\overline{W}_M - \mu_W}{\sigma_W/\sqrt{M}} >  -\frac{\mu_W}{\sigma_W/\sqrt{M}} \right) \\
	& \approx P\left(Z > -\frac{\sqrt{M} \mu_W}{\sigma_W}\right) \textrm{ by the CLT} \\
	&=   \Phi\left( \frac{\sqrt{M} \mu_W}{\sigma_W}\right) \label{eq:PhiW}
\end{align}
where $\Phi$ is the cumulative distribution function of the standard normal random variable $Z$.  So $P(\overline{W}_M>0) >\frac{1}{2}$ if $\mu_W>0$ and $P(\overline{W}_M>0) <\frac{1}{2}$ if $\mu_W<0$. Suppose the positivity condition (\ref{eq:E-LogCondn}) holds  such that $\mu_W^* > 0$. Then this probability $P(\overline{W}_M>0)$ goes to one as the sample size $M$ goes to infinity and as $\theta_k$ converges to $\theta_*$:
\begin{equation}
	\lim_{M \rightarrow \infty} P(\overline{W}_M>0) =1 \;.     
\end{equation}
The same argument and (\ref{eq:PhiW}) show that
\begin{equation}
	\lim_{M \rightarrow \infty} P(\overline{W}_M>0) =0      
\end{equation}
if the positivity condition (\ref{eq:E-LogCondn}) fails such that $\mu_W^* < 0$.

This analysis suggests a sample--mean version of the NEM condition in (\ref{eq:Goal}):
\begin{equation}
	\overline{W}_M \geq 0 \label{eq:SampMean-NEM}
\end{equation}
where 
\begin{equation}
	\overline{W}_M = \frac{1}{M} \sum_{k=1}^M \ln \frac{f(Y_k+N_k,Z_k|\theta_t)}{ f(Y_k,Z_k|\theta_t)} \;.
\end{equation}
The sample mean NEM condition (\ref{eq:SampMean-NEM}) matches the NEM condition (\ref{eq:Goal}) asymptotically as the number of samples $M \longrightarrow 0$. So the sample mean NEM condition can be a surrogate NEM sufficient condition for ``large'' sample-size data sets. This result stands in counterpoint to the discussion on NEM in the small--sample regime~(\Sec\ref{subsec:small-effects}): independent additive noise can produce EM noise-benefits for small data sets. While the sample--mean NEM condition can produce noise--benefits for large data sets.

The sample mean version may also be useful when there is no analytic NEM condition available for the data model. Such scenarios may occur with analytically intractable likelihoods or with empirical or nonparametric likelihoods~\parencite{owen1988, owen2010}.

\section{Conclusion}
Careful noise injection can speed up the average convergence time of the EM algorithm. The various sufficient conditions for such a noise benefit involve a direct or average effect where the noise makes the signal data more probable.  Special cases include mixture density models and log-convex probability density models. Noise injection for the Gaussian and Cauchy mixture models improves the average EM convergence speed when the noise satisfies a simple quadratic condition. Even blind noise injection can benefit these systems when the data set is sparse. But NEM noise injection still outperforms blind noise injection in all data models tested. An asymptotic argument also shows that the sample-mean version of the EM noise benefit condition obeys a similar positivity condition.  

Future research should address theories and methods for finding optimal noise levels for NEM algorithms. The current NEM implementations use a rudimentary search to find good noise levels. The noise-benefit argument depends on the behavior of the log-likelihood. The NEM condition depends on local variability in the observed log-likelihood surface. This suggests that the Fisher information, a normalized measure of the likelihood's average rate of change, controls the size and location of favorable noise injection regimes. 

Another open research question concerns the optimal \emph{shape} of the additive noise distributions for the NEM algorithm. Most simulations in this dissertation sample scalar noise from a single distribution with varying noise power. Other simulations show that different noise distributions families also cause faster average EM convergence speed as long as the noise satisfies the NEM condition. But there has been no exploration of the relative performance of different noise distributions for noisy EM. An untested conjecture suggests that impulsive noise distributions like alpha-stable distributions~\parencite{zolotarev1986, samorodnitsky-taqqu1994, nikias-shao1995} may help NEM find global ML estimates more easily. Such noise distribution-dependent benefits provide a level of flexibility that is unavailable to the regular EM algorithm.

Future research may also study the comparative effects of random noise versus deterministic interference for EM algorithms. This includes the study of chaotic systems for deterministic interference.

The NEM theorem and algorithms are general. They apply to many data models. The next three chapters demonstrate NEM benefits in three important incomplete data models: mixture models for clustering, hidden Markov models, and feedforward neural networks with hidden layers. The EM tuning algorithms for these models are very popular algorithms: $k$-means clustering, the Baum-Welch algorithm, and backpropagation respectively. We show that NEM produces in speed improvements for these algorithms.

\clearpage

%\subfile{./Chapters/chapNEM-CNBT}

\LinesNotNumbered % Algorithm2e option set for rest of chapter

\chapter{NEM Application: Clustering and Competitive Learning Algorithms} \label{ch:NEM-CNBT}

Clustering algorithms feature prominently in large-scale commercial recommendation systems like Google News (news articles)~\parencite{Das-et-Google2007}, Netflix (movies)~\parencite{koren-netflix2008,koren-bell-volinsky2009}, and Amazon (products)~\parencite{linden-et-amazon2003}. Such recommendations often rely on centroid-based clustering algorithms to classify costumers and produce relevant recommendations. These clustering algorithms tend to be computationally expensive and slow~\parencite{linden-et-amazon2003}.

This chapter shows that noise can provably speed up convergence in many centroid-based clustering algorithms. This includes the popular $k$--means clustering algorithm. The clustering noise benefit follows is a direct consequence of the general noise benefit for the EM algorithm (Theorem \ref{thm:NEM}) because many clustering algorithms (including the $k$-means algorithm) are special cases of the EM algorithm\parencite{celeux-govaert1992, xu-wunsch2005}.

The noise benefit for clustering algorithms is a classification accuracy improvement for NEM-based classifiers compared to EM-based classifiers at the same number of training iterations. EM-based clustering algorithms use EM to train underlying data models and then use the optimized data models to classify samples into their best clusters. The data model is usually a mixture model. Both NEM and EM algorithms find locally optimal model parameters in the convergence limit. But faster NEM convergence means that the NEM algorithm gives higher likelihood pre-converged parameters on average than the EM algorithm for fixed iteration limits. The \emph{Clustering Noise Benefit Theorem} (Theorem \ref{thm:cnbt}) below formalizes this classification accuracy noise benefit. 

Figure \ref{fig:Misclass-NEMC} shows a simulation instance of the corollary clustering noise benefit of the NEM Theorem for a two-dimensional GMM with three Gaussian data clusters. Theorem \ref{thm:cnbt} below states that such a noise benefit will occur. Each point on the curve reports how much two classifiers disagree on the same data set. The first classifier is the EM-classifier with fully converged EM-parameters. This is the reference classifier. The second classifier is the same EM-classifier with only partially converged EM-parameters. The two classifiers agree eventually if we let the second classifier's EM-parameters converge. But the figure shows that they agree faster with some noise than with no noise. 

We call the normalized number of disagreements the \emph{misclassification rate}. The misclassification rate falls as the Gaussian noise power increases from zero.  It reaches a minimum for additive white noise with standard deviation $0.3$.  More energetic noise does not reduce misclassification rates beyond this point.  The optimal noise reduces misclassification by almost $30\%$.

\begin{figure}[!ht]
%\centerline{ \includegraphics[width=\textwidth]{Figures/CNBT/EMClass-Plot.eps} }
\centerline{ \includegraphics[width=\textwidth]{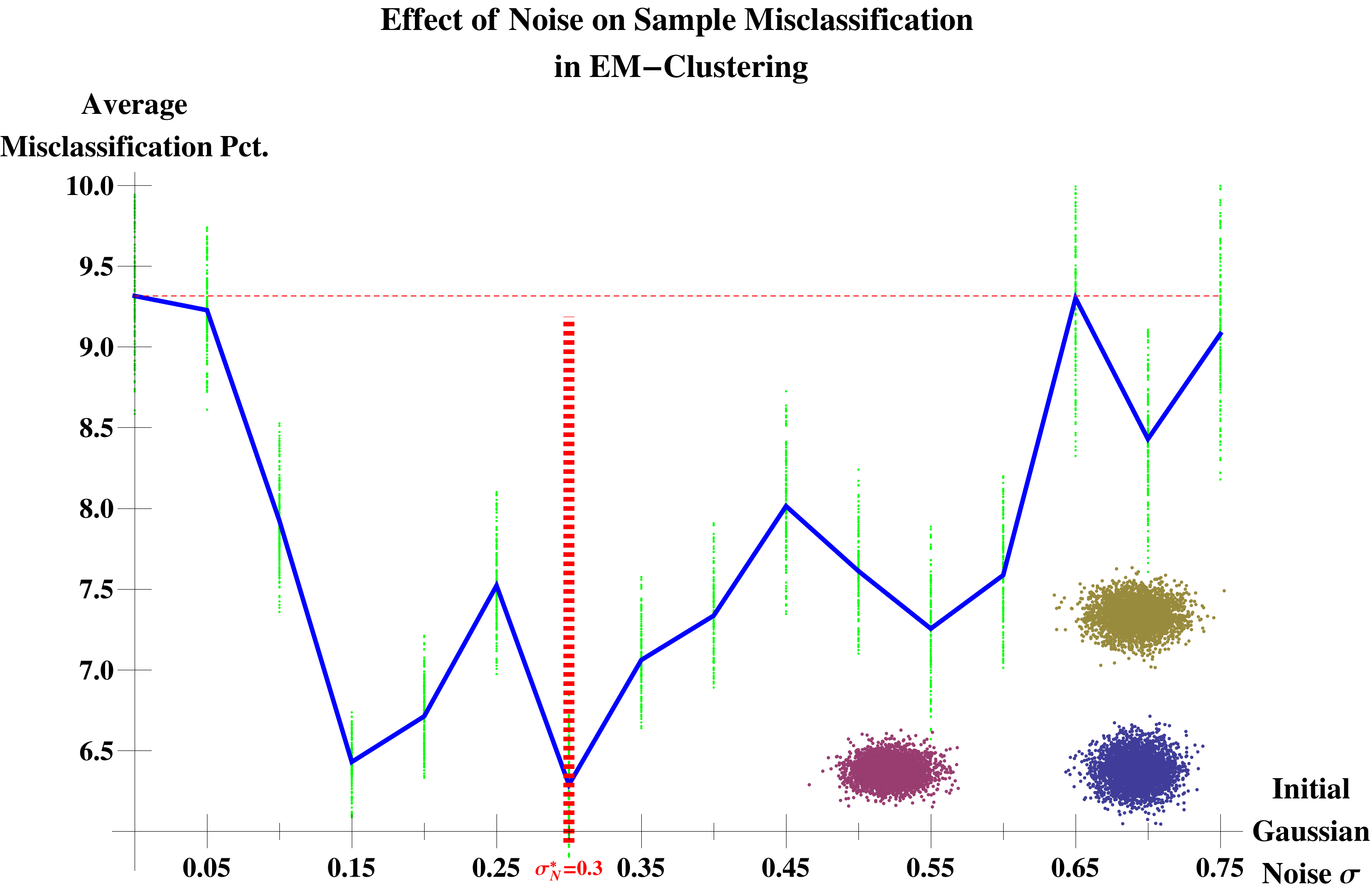} }
\caption[Noise Benefit for Classification Accuracy on a GMM-EM Model]{
	Noise benefit based on the misclassification rate for the Noisy Expectation--Maximization (NEM) clustering procedure on a $2$-D Gaussian mixture model with three Gaussian data clusters (inset) where each has a different covariance matrix. The plot shows that the misclassification rate falls as the additive noise power increases. The classification error rises if the noise power increases too much. The misclassification rate measures the mismatch between a NEM classifier with unconverged parameters $\Theta_k$ and the optimal NEM classifier with converged parameters $\Theta_*$. The unconverged NEM classifier's NEM procedure stops a quarter of the way to convergence.The dashed horizontal line indicates the misclassification rate for regular EM classification without noise. The red dashed vertical line shows the optimum noise standard deviation for NEM classification. The optimum noise has a standard deviation of $0.3$.}
\label{fig:Misclass-NEMC}
\end{figure}

Figure \ref{fig:KmeansFig} shows a similar noise benefit in the simpler $k$--means clustering algorithm on 3-dimensional Gaussian mixture data. The $k$--means algorithm is a special case of the EM algorithm as we show below in Theorem 2. So the EM noise benefit extends to the k-means algorithm. The figure plots the average convergence time for noise--injected $k$--means routines at different initial noise levels. The figure shows an instance where decaying noise helps the algorithm converge about $22\%$ faster than without noise.

\begin{figure}[!ht]
\centerline{ \includegraphics[width=\textwidth]{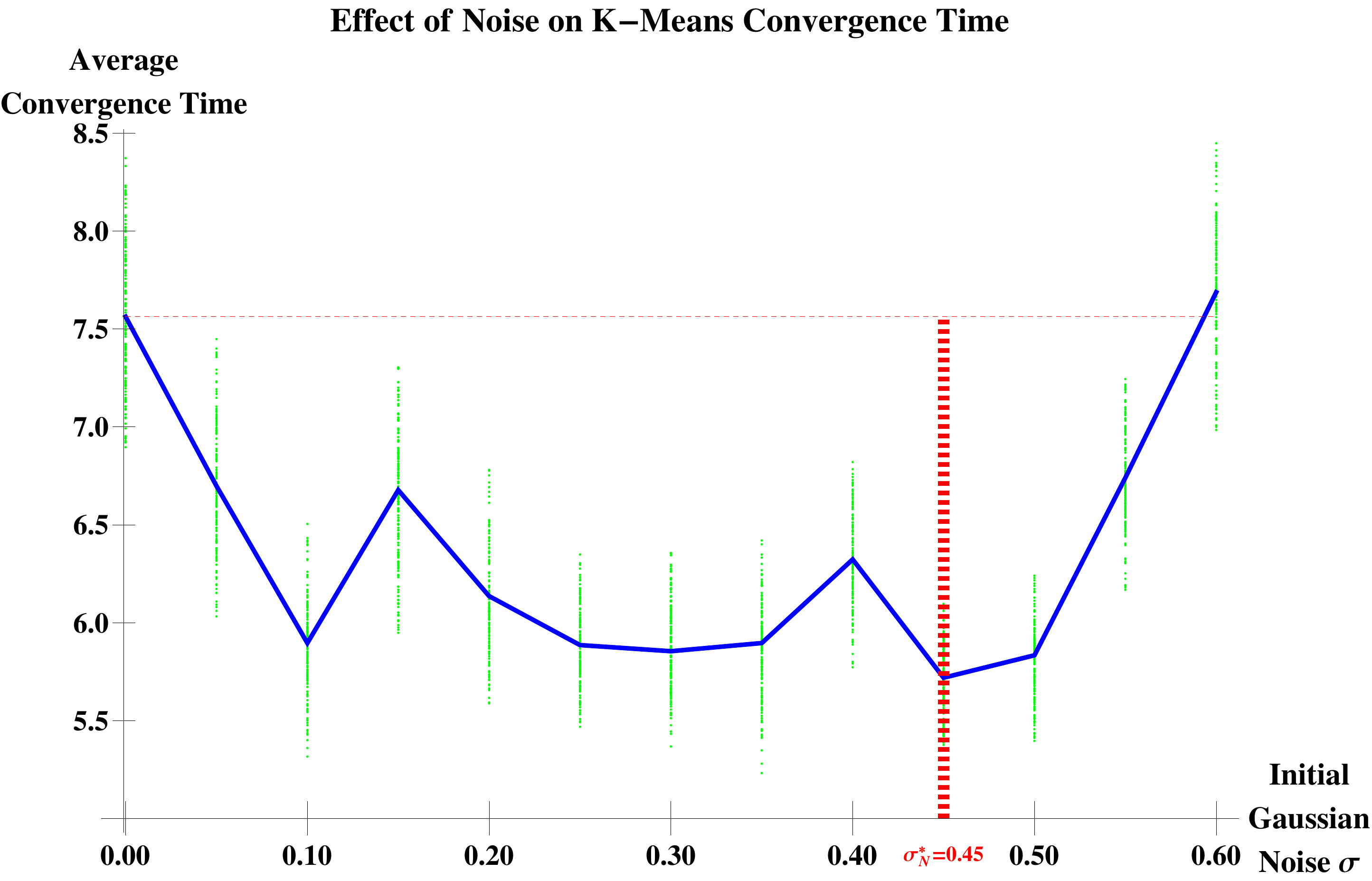} }
\caption[Noise Benefit for the Convergence Speed of a $k$-clustering Procedure]{Noise benefit in $k$--means clustering procedure on $2500$ samples of a $3$-D Gaussian mixture model with four clusters. The plot shows that the convergence time falls as additive white Gaussian noise power increases. The noise decays at an inverse square rate with each iteration. Convergence time rises if the noise power increases too much. The dashed horizontal line indicates the convergence time for regular $k$--means clustering without noise. The red dashed vertical line shows the optimum noise standard deviation for noisy $k$--means clustering. The optimum noise has a standard deviation of $0.45$: the convergence time falls by about $22\%$.}
\label{fig:KmeansFig}
\end{figure}

Simulations also show that noise also speeds up convergence in stochastic unsupervised competitive learning (UCL), supervised competitive learning (SCL), and differential competitive learning (DCL). These competitive learning (CL) algorithms are not fully within the ambit of the NEM theorem. But they are generalized neural-network versions of the $k$-means algorithm. Thus these simulations hint that a related noise benefit principle may apply to CL algorithms.

\section{Clustering}

Clustering algorithms divide data sets into clusters based on similarity measures between samples~\parencite{duda-hart-stork2001, jain2009, xu-wunsch2005, xu-wunsch2009}. The similarity measure attempts to quantify how samples differ statistically. Many algorithms use the Euclidean distance or Mahalanobis similarity measure. Clustering algorithms assign similar samples to the same cluster. Centroid-based clustering algorithms assign samples to the cluster with the closest centroid $\mu_1,...,\mu_k$. 

This clustering framework is an attempt to solve an NP-hard optimization problem. The algorithms define data clusters that minimize the total \emph{within-cluster} deviation from the centroids.  Suppose $y_i$ are samples of a data set on a sample space $D$. Centroid-based clustering  partitions $D$ into the $k$ decision classes $D_1,...,D_k$ of $D$. The algorithms look for optimal cluster parameters that minimize an objective function. The $k$--means clustering method \parencite{macqueen1967} minimizes the total sum of squared Euclidean within-cluster distances \parencite{xu-wunsch2005, celeux-govaert1992}:
\begin{equation}
\sum_{j=1}^K \sum_{i=1}^N  \| y_i - \mu_j \|^2 \mathbb{I}_{D_j}(y_i)
\end{equation} where $\mathbb{I}_{D_j}$ is the indicator function that indicates the presence or absence of pattern $y$ in ${D_j}$:
\begin{equation}
\mathbb{I}_{D_j}(y) = \begin{cases}
1 \quad \text{ if } y \in {D_j}\\
0 \quad \text{ if } y \notin {D_j}\;.
\end{cases}
\end{equation}

There are many approaches to clustering~\parencite{duda-hart-stork2001, xu-wunsch2005, jain2009}. Cluster algorithms come from fields that include nonlinear optimization, probabilistic clustering, neural networks-based clustering \parencite{kosko-nnfs}, fuzzy clustering \parencite{hopner-klawonn-kruse1999, bezdek-ehrlich-full1984}, graph-theoretic clustering \parencite{hartuv-shamir2000,cherng-lo2001}, agglomerative clustering \parencite{zhang-ramakrishnan-livny1996}, and bio-mimetic clustering \parencite{hall-ozyurt-bezdek1999,fogel1994}.

\subsection{Noisy Expectation-Maximization for Clustering}
Probabilistic clustering algorithms may model the true data distribution as a mixture of sub-populations. This mixture model assumption converts the clustering problem into a two-fold problem: density estimation for the underlying mixture model and data-classification based on the estimated mixture model. The naive Bayes classifier is one simple approach for discriminating between sub-populations. There are other more involved classifiers that may have better statistical properties e.g. ensemble classifiers or boosted classifiers. The EM algorithm is a standard method for estimating parametric mixture densities.

The parametric density estimation part of the clustering procedure can benefit from noise injection. This noise benefit derives from the application of the Noisy Expectation Maximization (NEM) theorem to EM in the clustering framework. A common mixture model in EM clustering methods is the Gaussian mixture model (GMM). We can apply the NEM Theorem to clustering algorithms that assume Gaussian sub-populations in the mixture. The NEM positivity condition
\begin{equation}
\E_{Y,Z,N|\theta^*} \left[ \ln\left( \frac{f(Y+N,Z|\theta_k)}{f(Y,Z|\theta_k)} \right) \right] \geq 0 
\label{eq:cnbt-Goal}
\end{equation}
gives a sufficient condition under which noise speeds up the EM algorithm's convergence to local optima. This condition implies that suitably noisy EM algorithm estimates the EM estimate $\theta_*$  in fewer steps on average than does the corresponding noiseless EM algorithm. 

The positivity condition reduces to a much simpler algebraic condition for GMMs. The model satisfies the positivity condition (\ref{eq:cnbt-Goal}) when the additive noise samples $n = (n_1,...n_d)$ satisfy the following algebraic condition~\parencite{osoba-mitaim-kosko2011, osoba-mitaim-kosko2012} (misstated in \cite{osoba-kosko2013} but corrected in \cite{osoba-kosko2013-Err}):
\begin{align}
n_i \left[n_i - 2\left(\mu_{j_i}-y_i\right) \right] \leq 0 \quad \textrm{ for all } j \;. \label{eq:gmm-nemcond-cnbt}
\end{align}
This condition applies to the variance update in the EM algorithm. It needs the current estimate of the centroids $\mu_j$. The NEM algorithm also anneals the additive noise by multiplying the noise power $\sigma_N$ by constants that decay with the iteration count. We found that the best application of the algorithm uses inverse--square decaying constants $k^{-2}$ to scale the noise $N$ \parencite{osoba-mitaim-kosko2012}:

\subsection{GMM-EM for Clustering}
We use the finite mixture model notation from \textbf{\S}\ref{subsec:FMM-EM}. The $Q$-function for EM on a finite mixture model with $K$ sub-populations is:
\begin{align}
Q(\Theta|\Theta(t)) &= \sum_{j=1}^K \ln[\alpha_j f(y|j,\theta_j)] ~p_Z(j|y,\Theta(t))
\label{eq:Q-Mixture}
\end{align}
where
\begin{equation}
p_Z(j|y,\Theta(t)) = \frac{\alpha_j ~f(y|Z=j,\theta_j(t))}{f(y|\Theta(t))}\;.
\end{equation}
Equation (\ref{eq:Q-Mixture}) gives the E-step for the mixture model. The GMM uses the above $Q$-function with Gaussian pdfs for $f(y|j,\theta_j)$.

Suppose there are $N$ data samples of the GMM distributions. The EM algorithm estimates the mixing probabilities $\alpha_j$, the sub-population means $\mu_j$, and the sub-population covariance $\Sigma_j$. The current estimate of the GMM parameters is $\Theta(t) = \{ \alpha_1(t),\cdots,\alpha_K(t), \mu_1(t),\cdots,\mu_K(t), \Sigma_1(t),\cdots,\Sigma_K(t) \}$. The iterations of the GMM--EM reduce to the following update equations:
%\begin{align}
%\alpha_j(t+1) &= \frac{1}{N} \sum_{i=1}^N p_Z(j|y_i,\Theta(t)) \label{eq:alpha-Update}\\
%\mu_j(t+1) &= \frac{\sum_{i=1}^N p_Z(j|y_i,\Theta(t)) y_i}{\sum_{i=1}^N p_Z(j|y_i,\Theta(t))} \label{eq:mu-Update} \\
%\Sigma_j(t+1) &= \frac{\sum_{i=1}^N p_Z(j|y_i,\Theta(t)) (y_i-\mu_j(t))(y_i-\mu_j(t))^T}{\sum_{i=1}^N p_Z(j|y_i,\Theta(t))}\;. \label{eq:Sigma-Update}
%\end{align}
\begin{align}
\alpha_j(t+1) &= \frac{1}{N} \sum_{i=1}^N p_Z(j|y_i,\Theta(t)) \label{eq:alpha-Update}\\
\mu_j(t+1) &= \sum_{i=1}^N y_i ~ \beta_j(t|y_i)  \label{eq:mu-Update} \\
\Sigma_j(t+1) &= \sum_{i=1}^N \beta_j(t|y_i) ~(y_i-\mu_j(t))(y_i-\mu_j(t))^T  \;. \label{eq:Sigma-Update}
\end{align}
where
\begin{equation}
\beta_j(t|y_i) = \frac{p_Z(j|y_i,\Theta(t))}{\sum_{i=1}^N p_Z(j|y_i,\Theta(t))} \;.
\label{eq:beta-defn} 
\end{equation}

These equations update the parameters $\alpha_j$, $\mu_j$, and $\Sigma_j$ with coordinate values that maximize the $Q$ function in (\ref{eq:Q-Mixture}) \parencite{xu-jordan1996,duda-hart-stork2001}. The updates combine both the E--steps and  M--steps of the EM procedure.

\subsection{Naive Bayes Classifier on GMMs}
GMM-EM clustering uses the membership probability density function $ p_Z(j|y,\Theta_{EM})$ as a maximum a posteriori classifier for each sample $y$. The classifier assigns $y$ to the $j^{th}$ cluster if $ p_Z(j|y,\Theta_{EM}) \geq p_Z(k|y,\Theta_{EM})$ for all $k \neq j$. Thus
\begin{equation}
EMclass(y) = \argmax{j} \;\, p_Z(j|y,\Theta_{EM}) \;.
\end{equation}
This is the \emph{naive Bayes classifier} \parencite{domingos-pazzani1997, rish2001} based on the EM-optimal GMM parameters for the data. NEM clustering uses the same classifier but with the NEM-optimal GMM parameters for the data:
\begin{equation}
NEMclass(y) = \argmax{j} \;\, p_Z(j|y,\Theta_{NEM}) \;.
\end{equation}

\subsection{The Clustering Noise Benefit Theorem}
The next theorem shows that the noise--benefit of the NEM Theorem extend to EM-clustering. The noise benefit occurs in misclassification relative to the EM-optimal classifier. The theorem uses the following notation: 
\begin{itemize}
\item $class_{opt}(Y)= \argmax \;\;\, p_Z(j|Y,\Theta_*)$:{ EM-optimal classifier. It uses the optimal model parameters $\Theta_*$}

\vspace{4pt}

\item $P_M[k] = P\left(EMclass_k(Y)\neq class_{opt}(Y) \right)$:{ Probability of EM-clustering misclassification relative to $class_{opt}$ using $k^{th}$ iteration parameters}

\vspace{4pt}

\item $P_{M_N}[k] = P\left(NEMclass_k(Y)\neq class_{opt}(Y) \right)$:{ Probability of NEM-clustering misclassification relative to $class_{opt}$ using $k^{th}$ iteration parameters}
\end{itemize}

\begin{thm}{\bf [Clustering Noise Benefit Theorem (CNBT)]} \label{thm:cnbt} \\
	Consider the NEM and EM iterations at the $k^{th}$ step. Then the NEM  misclassification probability $P_{M_N}[k]$ is less than the noise-free EM misclassification probability $P_M[k]$:
	\begin{equation} 
		P_{M_N}[k] \leq P_M[k]
	\end{equation} when the additive noise $N$ in the NEM-clustering procedure satisfies the NEM Theorem condition from (\ref{eq:cnbt-Goal}):
	\begin{equation}
		\E_{Y,Z,N|\theta^*} \left[ \ln\left( \frac{f(Y+N,Z|\theta_k)}{f(Y,Z|\theta_k)} \right) \right] \geq 0\;. \label{eq:nemcond}
	\end{equation}
	This positivity condition (\ref{eq:nemcond}) in the GMM-NEM model reduces to the simple algebraic condition (\ref{eq:gmm-nemcond-cnbt}) \parencite{osoba-mitaim-kosko2011, osoba-mitaim-kosko2012} for each coordinate $i$:
	\begin{align}
		n_i \left[n_i - 2\left(\mu_{j_i}-y_i\right) \right] \leq 0 ~~ \textrm{ for all } j \;. \nonumber
	\end{align}
\end{thm}

\begin{proof}
\noindent Misclassification is a mismatch in argument maximizations:
\begin{multline}
EMclass_k(Y) \neq class_{opt}(Y) \quad \text{if and only if  }  \\ 
 \argmax \;\;\, p_Z(j|Y,\Theta_{EM}[k])  \neq \argmax \;\;\, p_Z(j|Y,\Theta_*) \;.
\end{multline}
This mismatch disappears as $\Theta_{EM}$ converges to $\Theta_*$. Thus
\begin{equation}
\argmax \;\;\, p_Z(j|Y,\Theta_{EM}[k]) \text{ converges to } \argmax \;\;\, p_Z(j|Y,\Theta_*)  \nonumber
\end{equation}
since
\begin{equation}
\lim_{k \rightarrow \infty} \| \Theta_{EM}[k] - \Theta_*\| = 0 
\end{equation}
by definition of the EM algorithm. So the argument maximization mismatch decreases as the EM estimates get closer to the optimum parameter $\Theta_*$. But the NEM condition (\ref{eq:nemcond}) implies that the following inequality holds on average at the $k^{th}$ iteration:
\begin{equation}
\|\Theta_{NEM}[k] - \Theta_*\| \leq  \| \Theta_{EM}[k] - \Theta_*\| \;.
\end{equation} 
Thus for a fixed iteration count $k$: 
\begin{align} P\left(NEMclass_k(Y) \neq class_{opt}(Y)\right) \leq P\left(EMclass_k(Y) \neq class_{opt}(Y)\right) \end{align} on average. 
So
\begin{equation}
P_{M_N}[k] \leq P_M[k]
\end{equation} on average.
Thus noise reduces the probability of EM clustering misclassification relative to the EM-optimal classifier on average when the noise satisfies the NEM condition. This means that an unconverged NEM--classifier performs closer to the fully converged classifier than does an unconverged noise-less EM--classifier on average.
\end{proof}

We next state the noise-enhanced EM GMM algorithm in $1$-D.

\vspace{6pt}

\begin{algorithm}[H]
\DontPrintSemicolon
\SetKwInOut{Input}{Input}
\SetKwInOut{Output}{Output}
\SetKwFunction{NEMNoiseSample}{NEMNoiseSample}

\Input{$y_1, \ldots, y_N$ GMM data samples}
\Output{ $\hat{\theta}_{NEM}$ : NEM estimate of parameter $\theta$}
\BlankLine
\While{($\|\theta_k-\theta_{k-1}\| \geq 10^{-tol}$)}{
{\bf N-Step:} $z_i = y_i + n_i$\\ where $n_i$ is a sample of the truncated Gaussian $\sim N(0,\frac{\sigma_N}{k^2})$ such that $n_i \left[n_i - 2\left(\mu_{j_i}-y_i\right) \right] \leq 0 \quad \textrm{ for all } i,j$ \;
{\bf E-Step:} $Q \left( \theta |\theta_k \right) \leftarrow \sum_{i=1}^N \sum_{j=1}^K \ln[\alpha_j f(z_i|j,\theta_j)] p_Z(j|y,\Theta(t)) $ \;
{\bf M-Step:} $\theta_{k+1} \leftarrow \argmax{\theta} \left\{ Q\left( \theta |\theta_k \right) \right\}$\;
$k \leftarrow k+1$\;
}
$\hat{\theta}_{NEM} \leftarrow \theta_k$
\caption{{\bf Noisy GMM--EM Algorithm (1-D)}}\label{algo:cnbt-NEM}
\end{algorithm}

\vspace{8pt}

\noindent The $D$-dimensional GMM-EM algorithm runs the N--Step component-wise for each data dimension.

Figure \ref{fig:Misclass-NEMC} shows a simulation instance of the predicted GMM noise benefit for $2$-$D$ cluster--parameter estimation. The figure shows that the optimum noise reduces GMM--cluster misclassification by almost $30\%$.

\section{The \emph{k}-Means Clustering Algorithm}
$k$--means clustering is a non-parametric procedure for partitioning data samples into clusters \parencite{macqueen1967, xu-wunsch2005}. Suppose the data space $\mathbb{R}^d$ has $K$ centroids $\mu_1,...,\mu_K$. The procedure tries to find $K$ partitions $D_1,...,D_K$ with centroids $\mu_1,...,\mu_K$ that minimize the within-cluster Euclidean distance from the cluster centroids:
\begin{equation}
\argmin{D_1,...D_K} \sum_{j=1}^K \sum_{i=1}^N  \| y_i - \mu_j \|^2 \mathbb{I}_{D_j}(y_i)
\end{equation}
for $N$ pattern samples $y_1, ..., y_N$. The class indicator functions $\mathbb{I}_{D_1},...,\mathbb{I}_{D_K}$ arise from the nearest--neighbor classification in (\ref{eq:k-Assign}) below. Each indicator function $\mathbb{I}_{D_j}$ indicates the presence or absence of pattern $y$ in ${D_j}$:
\begin{equation}
\mathbb{I}_{D_j}(y) = \begin{cases}
1 \quad \text{ if } y \in {D_j}\\
0 \quad \text{ if } y \notin {D_j}\;.
\end{cases}
\end{equation}
The $k$--means procedure finds local optima for this objective function. $k$--means clustering works in the following two steps:

\vspace{6pt}

\begin{algorithm}[H]
\DontPrintSemicolon
{\bf Assign Samples to Partitions:}\;
\begin{equation}
y_i \in D_j(t) ~~\text{ if }~~ \| y_i - \mu_j(t) \| \leq \| y_i - \mu_k(t) \| \quad k \neq j \label{eq:k-Assign}
\end{equation} \;
\vspace{-9pt}
{\bf Update Centroids:}\;
\vspace{-12pt}
\begin{equation}
\mu_j(t+1) = \frac{1}{|D_j(t)|} \sum_{i=1}^N y_i \mathbb{I}_{D_j(t)}(y_i)\;. \label{eq:k-Update}
\end{equation} \;
\vspace{-9pt}
\caption{{\bf K-Means Clustering Algorithm}}\label{algo:k-means}
\end{algorithm}

\vspace{6pt}

\subsection{\emph{k}-Means Clustering as a GMM--EM Procedure}
$k$--means clustering is a special case of the GMM--EM model \parencite{hathaway1986, celeux-govaert1992}. The key to this subsumption is the ``degree of membership'' function or ``cluster-membership measure'' $m(j|y)$ \parencite{xu-wunsch2010, hamerly-elkan2002}. It is a fuzzy measure of how much the sample $y_i$ belongs to the $j^{th}$ sub-population or cluster. The GMM--EM model uses Bayes theorem to derive a soft cluster-membership function:
\begin{equation}
m(j|y) =  p_Z(j|y,\Theta) = \frac{\alpha_j f(y|Z=j,\theta_j)}{f(y|\Theta)}\;.
\end{equation}
$k$--means clustering assumes a hard cluster-membership \parencite{xu-wunsch2010, hamerly-elkan2002, kearns-mansour-ng1997}:
\begin{equation}
m(j|y) = \mathbb{I}_{D_j}(y)
\end{equation} where $D_j$ is the partition region whose centroid is closest to $y$. The $k$--means assignment step redefines the cluster regions $D_j$ to modify this membership function. The procedure does not estimate the covariance matrices in the GMM--EM formulation. 

Noise also benefits the $k$--means procedure as Figure \ref{fig:KmeansFig} shows since $k$--means is an EM-procedure. 

\begin{thm}{\bf{[$k$-means is a sub-case of the EM algorithm]:}}\label{thm:kmeans-EM}\\

Suppose that the sub-populations have known spherical covariance matrices $\Sigma_j$ and known mixing proportions $\alpha_j$.  Suppose further that the cluster-membership function is hard: 
\begin{equation}
m(j|y) = \mathbb{I}_{D_j}(y) \;.
\end{equation}
Then GMM--EM reduces to $k$-Means clustering:
\begin{equation}
\frac{\sum_{i=1}^N p_Z(j|y_i,\Theta(t)) y_i}{\sum_{i=1}^N p_Z(j|y_i,\Theta(t))} = \frac{1}{|D_j(t)|} \sum_{i=1}^N y_i \mathbb{I}_{D_j(t)}(y_i)\;.
\end{equation}
\end{thm}

\begin{proof}
The covariance matrices $\Sigma_j$ and mixing proportions $\alpha_j$ are constant. So the update equations (\ref{eq:alpha-Update}) and (\ref{eq:Sigma-Update}) do not apply in the GMM--EM procedure. The mean (or centroid) update equation in the GMM--EM procedure becomes
\begin{align}
\mu_j(t+1) &= \frac{\sum_{i=1}^N p_Z(j|y_i,\Theta(t)) y_i}{\sum_{i=1}^N p_Z(j|y_i,\Theta(t))}\;.
\end{align} 
The hard cluster-membership function
\begin{equation}
m_t(j|y) = \mathbb{I}_{D_j(t)}(y)
\end{equation}
changes the $t^{th}$ iteration's mean update to 
\begin{align}
\mu_j(t+1) &= \frac{\sum_{i=1}^N y_i m_t(j|y_i) }{\sum_{i=1}^N m_t(j|y_i)}\;.
\end{align}
The sum of the hard cluster-membership function reduces to
\begin{equation}
\sum_{i=1}^N m_t(j|y_i) = N_j = |D_j(t)|
\end{equation}
where $N_j$ is the number of samples in the $j^{th}$ partition.
Thus the mean update is
\begin{equation}
\mu_j(t+1) = \frac{1}{|D_j(t)|} \sum_{i=1}^N y_i \mathbb{I}_{D_j(t)}(y_i) \;.
\end{equation}
Then the EM mean update equals the $k$--means centroid update:
\begin{equation}
\frac{\sum_{i=1}^N p_Z(j|y_i,\Theta(t)) y_i}{\sum_{i=1}^N p_Z(j|y_i,\Theta(t))} = \frac{1}{|D_j(t)|} \sum_{i=1}^N y_i \mathbb{I}_{D_j(t)}(y_i) \;.
\end{equation}
\end{proof}

The known diagonal covariance matrices $\Sigma_j$ and mixing proportions $\alpha_j$ can arise from prior knowledge or previous optimizations. Estimates of the mixing proportions (\ref{eq:alpha-Update}) get collateral updates as learning changes the size of the clusters. 

Approximately hard cluster membership can occur in the regular EM algorithm when the sub-populations are well separated. An EM--optimal parameter estimate $\Theta^*$ will result in very low posterior probabilities $p_Z(j|y,\Theta^*)$ if $y$ is not in the $j^{th}$ cluster. The posterior probability is close to one for the correct cluster. Celeux and Govaert proved a similar result by showing an equivalence between the objective functions for EM and $k$--means clustering \parencite{celeux-govaert1992, xu-wunsch2005}.  Noise injection simulations confirmed the predicted noise benefit in the $k$--means clustering algorithm.

\subsection{\emph{k}--Means Clustering and Adaptive Resonance Theory}

$k$--means clustering resembles Adaptive Resonance Theory (ART) \parencite{carpenter-grossberg1987, xu-wunsch2005, kosko-nnfs}. And so ART should also benefit from noise. $k$--means clustering learns clusters from input data without supervision. ART performs similar unsupervised learning on input data using neural circuits.

ART uses interactions between two fields of neurons: the comparison neuron field (or bottom--up activation) and the recognition neuron field (or  top--down activation). The comparison field matches against the input data. The recognition field forms internal representations of learned categories. ART uses bidirectional ``resonance'' as a substitute for supervision. Resonance refers to the coherence between recognition and comparison neuron fields. The system is stable when the input signals match the recognition field categories. But the ART system can learn a new pattern or update an existing category if the input signal fails to match any recognition category to within a specified  level of ``vigilance'' or degree of match.

%% Check weird hyphen
ART systems are more flexible than regular $k$--means systems because ART systems do not need a pre--specified cluster count $k$ to learn the data clusters. ART systems can also update the cluster count on the fly if the input data characteristics change. Extensions to the ART framework include ARTMAP \parencite{carpenter-grossberg-reynolds1991} for supervised classification learning and Fuzzy ART for fuzzy clustering \parencite{carpenter-grossberg-rosen1991}. An open research question is whether NEM-like noise injection will provably benefit ART systems.

\section{Competitive Learning Algorithms}
Competitive learning algorithms learn centroidal patterns from streams of input data by adjusting the weights of only those units that win a distance-based competition or comparison \parencite{kosko-nnfs,kosko-SCL1991, kohonen2001,grossberg1987}. Stochastic competitive learning behaves as a form of \emph{adaptive quantization} because the trained synaptic fan-in vectors (centroids) tend to distribute themselves in the pattern space so as to minimize the mean-squared-error of vector quantization \parencite{kosko-nnfs}. Such a quantization vector also converges with probability one to the centroid of its nearest-neighbor class \parencite{kosko-SCL1991}. We will show that most competitive learning systems benefit from noise. This further suggests that a noise benefit holds for ART systems because they use competitive learning to form learned pattern categories.

Unsupervised competitive learning (UCL) is a blind clustering algorithm that tends to cluster like patterns together.  It uses the implied topology of a two-layer neural network.  The first layer is just the data layer for the input patterns $y$ of dimension $d$.  There are $K$-many competing neurons in the second layer.  The synaptic fan-in vectors to these neurons define the local centroids or quantization vectors $\mu_1,...,\mu_K$.  Simple distance matching approximates the complex nonlinear dynamics of the second-layer neurons competing for activation in an on-center/off-surround winner-take-all connection topology \parencite{kosko-nnfs} as in a an ART system.  Each incoming pattern stimulates a new competition.  The winning $j^{th}$ neuron modifies its fan-in of synapses while the losing neurons do not change their synaptic fan-ins. Nearest-neighbor matching picks the winning neuron by finding the synaptic fan-in vector closest to the current input pattern.  Then the UCL learning law moves the winner's synaptic fan-in centroid or quantizing vector a little closer to the incoming pattern.

We first write the UCL algorithm as a two-step process of distance-based ``winning'' and synaptic-vector update. The first step is the same as the assignment step in $k$--means clustering. This equivalence alone argues for a noise benefit. But the second step differs in the learning increment. So UCL differs from $k$--means clustering despite their similarity.  This difference prevents a direct subsumption of UCL from the EM algorithm.  It thus prevents a direct proof of a UCL noise--benefit based on the NEM Theorem. 

We also assume in all simulations that the initial $K$ centroid or quantization vectors equal the first $K$ \emph{random} pattern samples:  \linebreak $\mu_1(1) = y(1),..., \mu_K(K) = y(K)$. Other initialization schemes could identify the first $K$ quantizing vectors with any $K$ other pattern samples so long as they are random samples.  Setting all initial quantizing vectors to the same value can distort the learning process.  All competitive learning simulations used linearly decaying learning coefficients $c_j(t) = 0.3(1-t/1500)$.

\vspace{6pt}

\begin{algorithm}[H]
\DontPrintSemicolon
{\bf Pick the Winner:} \;
The $j^{th}$ neuron wins at $t$ if \;
\begin{equation}
\| y(t) - \mu_j(t) \| \leq \| y(t) - \mu_k(t) \| \quad k \neq j \;. \label{eq:ucl-PickWin}
\end{equation}\;
{\bf Update the Winning Quantization Vector:}\;
\begin{align}
\mu_j(t+1) &= \mu_j(t) + c_t  \left[y(t)-\mu_j(t)\right] \label{eq:ucl-Update}
\end{align} for a decreasing sequence of learning coefficients $\{c_t\}$.
\caption{{\bf Unsupervised Competitive Learning (UCL) Algorithm}}\label{alg:ucl}
\end{algorithm}

\vspace{6pt}

A similar stochastic difference equation can update the covariance matrix $\Sigma_j$ of the winning quantization vector:
\begin{align}
\Sigma_j(t + 1)  \,  =  \,   \Sigma_j(t)  \, + 
c_t \left[ (y(t) - \mu_j(t))^T (y(t) - \mu_j(t))  -  \Sigma_j(t) \right] \;.
\end{align}
A modified version can update the pseudo-covariations of alpha-stable random vectors that have no higher-order moments \parencite{kim-kosko1996}.  The simulations in this paper do not adapt the covariance matrix.

We can rewrite the two UCL steps (\ref{eq:ucl-PickWin}) and (\ref{eq:ucl-Update}) into a single stochastic difference equation.  This rewrite requires that the distance-based indicator function $\mathbb{I}_{D_j}$ replace the pick-the-winner step (\ref{eq:ucl-PickWin}) just as it does for the assign-samples step (\ref{eq:k-Assign}) of $k$--means clustering:
\begin{equation}
\mu_j(t+1) = \mu_j(t) + c_t\; \mathbb{I}_{D_j}(y(t)) \;\left[y(t)-\mu_j(t) \right] \;. \label{eq:ucl-short}
\end{equation}

The one-equation version of UCL in (\ref{eq:ucl-short}) more closely resembles Grossberg's original deterministic differential-equation form of competitive learning in neural modeling \parencite{grossberg1982, kosko-nnfs}:
\begin{equation}
\dot{m}_{i j} =   S_j(y_j) \left[ S_i(x_i) - {m}_{i j} \right] \label{eq:GrossbergDiffEq}
\end{equation}
where $m_{i j}$ is the synaptic memory trace from the $i^{th}$ neuron in the input field to the $j^{th}$ neuron in the output or competitive field.  The $i^{th}$ input neuron has a real-valued activation $x_i$ that feeds into a bounded nonlinear signal function (often a sigmoid) $S_i$.  The $j^{th}$ competitive neuron likewise has a real-valued scalar activation $y_j$ that feeds into a bounded nonlinear signal function $S_j$.  But competition requires that the output signal function $S_j$ approximate a zero-one decision function. This gives rise to the approximation $S_j \approx \mathbb{I}_{D_j}$.

The two-step UCL algorithm is the same as Kohonen's ``self-organizing map'' algorithm \parencite{kohonen1990, kohonen2001} if the self-organizing map updates only a single winner.  Both algorithms can update direct or graded subsets of neurons near the winner.  These near-neighbor beneficiaries can result from an implied connection topology of competing neurons if the square $K$-by-$K$ connection matrix has a positive diagonal band with other entries negative.

Supervised competitive learning (SCL) punishes the winner for misclassifications.  This requires a teacher or supervisor who knows the class membership $D_j$ of each input pattern $y$ and who knows the classes that the other synaptic fan-in vectors represent.  The SCL algorithm moves the winner's synaptic fan-in vector $\mu_j$ away from the current input pattern $y$ if the pattern $y$ does not belong to the winner's class $D_j$.  So the learning increment gets a minus sign rather than the plus sign that UCL would use.  This process amounts to inserting a \emph{reinforcement function} $r$ into the winner's learning increment as follows:
\begin{align}
\mu_j(t+1) \; =& \; \mu_j(t) + c_t r_j(y) \left[y-\mu_j(t)\right] \\
r_j(y) \;  =& \; \mathbb{I}_{D_j}(y) - \sum_{i \neq j} \mathbb{I}_{D_i}(y) \;.
\end{align}
Russian learning theorist Ya Tsypkin appears the first to have arrived at the SCL algorithm.  He did so in 1973 in the context of an adaptive Bayesian classifier \parencite{tsypkin1973}.

Differential Competitive Learning (DCL) is a hybrid learning algorithm \parencite{kong-kosko1991, kosko-nnfs}.  It replaces the win-lose competitive learning term $S_j$ in (\ref{eq:GrossbergDiffEq}) with the \emph{rate} of winning $\dot{S}_j$.  The rate or differential structure comes from the differential Hebbian law \parencite{kosko-DHL1986}:
\begin{equation}
\dot{m}_{i j} =   -m_{i j}   +  \dot{S}_i  \dot{S}_j
\end{equation}
using the above notation for synapses $m_{i j}$ and signal functions $S_i$ and $S_j$. The traditional Hebbian learning law just correlates neuron activations rather than their velocities. The result is the DCL differential equation:
\begin{equation} 
\dot{m}_{i j} = \dot{S}_j(y_j) \left[ S_i(x_i)  -  m_{i j} \right]    \;.
\end{equation}
Then the synapse learns only if the $j^{th}$ competitive neuron changes its win-loss status. The synapse learns in competitive learning only if the $j^{th}$ neuron itself wins the competition for activation.  The time derivative in DCL allows for both positive and negative reinforcement of the learning increment. This polarity resembles the plus-minus reinforcement of SCL even though DCL is a blind or unsupervised learning law.  Unsupervised DCL compares favorably with SCL in some simulation tests \parencite{kong-kosko1991}.

We simulate DCL with the following stochastic difference equation:
\begin{align}
\mu_j(t+1)&= \mu_j(t) + c_t \Delta S_j(z_j) \left[S(y)-\mu_j(t)\right]\\
\mu_i(t+1) &= \mu_i(t) \quad \text{if } i\neq j \;.
\end{align}when the $j^{th}$ synaptic vector wins the metrical competition as in UCL. $\Delta S_j(z_j)$ is the time-derivative of the $j^{th}$ output neuron activation. We approximate it as the signum function of time difference of the training sample $z$ \parencite{kong-kosko1991,dickerson-kosko1994}:
\begin{align}
\Delta S_j(z_j) = sgn\left[ z_j(t+1) - z_j(t) \right] \;.
\end{align}

The competitive learning simulations in Figure \ref{fig:UCL} used noisy versions of the competitive learning algorithms just as the clustering simulations used noisy versions.  The noise was additive white Gaussian vector noise $n$ with decreasing variance (annealed noise).  We added the noise $n$ to the pattern data $y$ to produce the training sample $z$: $z  =  y + n$  where $n \sim N(\mathbf{0}, \Sigma_\sigma(t))$. The noise covariance matrix $\Sigma_\sigma(t)$ was just the scaled identity matrix $(t^{-2}\sigma) I$  for standard deviation or noise level  $\sigma > 0$.  This allows the scalar $\sigma$ to control the noise intensity for the entire vector learning process. We annealed or decreased the variance as $\Sigma_\sigma(t) = (t^{-2}\sigma) I$ as in \parencite{osoba-mitaim-kosko2011,osoba-mitaim-kosko2012}.  So the noise vector random sequence $n(1), n(2),...$ is an independent (white) sequence of similarly distributed Gaussian random vectors. We state for completeness the three-step noisy UCL algorithm. 

\vspace{6pt}

\begin{algorithm}[H]
\DontPrintSemicolon
{\bf Noise Injection:} \;

Define \begin{equation}z(t)   =   y(t)   +  n(t)  \label{eq:nucl-n} \end{equation}
For $n(t)  \sim  N(\mathbf{0}, \Sigma_\sigma(t)) $ and annealing schedule $\Sigma_\sigma(t) = \frac{\sigma}{t^{2}} I $.\;
\BlankLine
{\bf Pick the Noisy Winner:} \;
The $j^{th}$ neuron wins at $t$ if \;
\begin{equation}
 \| z(t) - \mu_j(t) \| \leq \| z(t) - \mu_k(t) \| \quad k \neq j \;. \label{eq:nucl-PickWin}
\end{equation}\;

{\bf Update the Winning Quantization Vector:}\;
\begin{align}
\mu_j(t+1) &= \mu_j(t) + c_t  \left[z(t)-\mu_j(t)\right] \label{eq:nucl-Update}
\end{align} for a decreasing sequence of learning coefficients $\{c_t\}$
\caption{\bf Noisy UCL Algorithm}
\end{algorithm}

\vspace{6pt}

We can define similar noise-perturbed versions the SCL and DCL algorithms.

\vspace{6pt}

\begin{algorithm}[H]
\DontPrintSemicolon
{\bf Noise Injection:} \;
Define \begin{equation}z(t)   =   y(t)   +  n(t)  \label{eq:nscl-n} \end{equation}
For $n(t)  \sim  N(\mathbf{0}, \Sigma_\sigma(t)) $ and annealing schedule $\Sigma_\sigma(t) = \frac{\sigma}{t^{2}} I $.\;
\BlankLine
{\bf Pick the Noisy Winner:} \;
The $j^{th}$ neuron wins at $t$ if 
\begin{align}
 \| z(t) - \mu_j(t) \| \leq \| z(t) - \mu_k(t) \| \quad k \neq j \;. \label{eq:nscl-PickWin}
\end{align}\;
{\bf Update the Winning Quantization Vector:}
\begin{align}
\mu_j(t+1) &= \mu_j(t) + c_t ~ r_j(z) \left[z(t)-\mu_j(t)\right] \label{eq:nscl-Update}
\end{align}
\vspace{-16pt}
where
\begin{align}
r_j(z) ~~  =& ~~ \mathbb{I}_{D_j}(z) - \sum_{i \neq j} \mathbb{I}_{D_i}(z)
\end{align}
and $\{c_t\}$ is a decreasing sequence of learning coefficients.
\caption{\bf Noisy SCL Algorithm}
\end{algorithm}

\vspace{6pt}

%%% Noisy DCL Algorithm Statement
\begin{algorithm}[H]
\DontPrintSemicolon
{\bf Noise Injection:} \;
Define \begin{equation}z(t)   =   y(t)   +  n(t)  \label{eq:ndcl-n} \end{equation}
For $n(t)  \sim  N(\mathbf{0}, \Sigma_\sigma(t)) $ and annealing schedule $\Sigma_\sigma(t) = \frac{\sigma}{t^{2}} I $.\;
\BlankLine
{\bf Pick the Noisy Winner:} \;
The $j^{th}$ neuron wins at $t$ if 
\begin{align}
 \| z(t) - \mu_j(t) \| \leq \| z(t) - \mu_k(t) \| \quad k \neq j \;. \label{eq:ndcl-PickWin}
\end{align}\;
{\bf Update the Winning Quantization Vector:}
\begin{align}
\mu_j(t+1)&= \mu_j(t) + c_t ~ \Delta S_j(z_j) ~ \left[S(x)-\mu_j(t)\right] \label{eq:ndcl-Update}
\end{align}
where
\vspace{-16pt}
\begin{align}
\Delta S_j(z_j) = sgn\left[ z_j(t+1) - z_j(t) \right] 
\end{align}
and $\{c_t\}$ is a decreasing sequence of learning coefficients.
\caption{\bf Noisy DCL Algorithm}
\end{algorithm}

\vspace{8pt}
 
Figure \ref{fig:UCL} shows that noise injection sped up UCL convergence by about $25\%$. The simulation used the four-cluster Gaussian data model shown in the inset of figure \ref{fig:UCL}. Noise perturbation gave the biggest relative improvement for the UCL algorithm in most of our simulations.

\begin{figure}[!ht]
\centerline{ \includegraphics[width=0.9\textwidth]{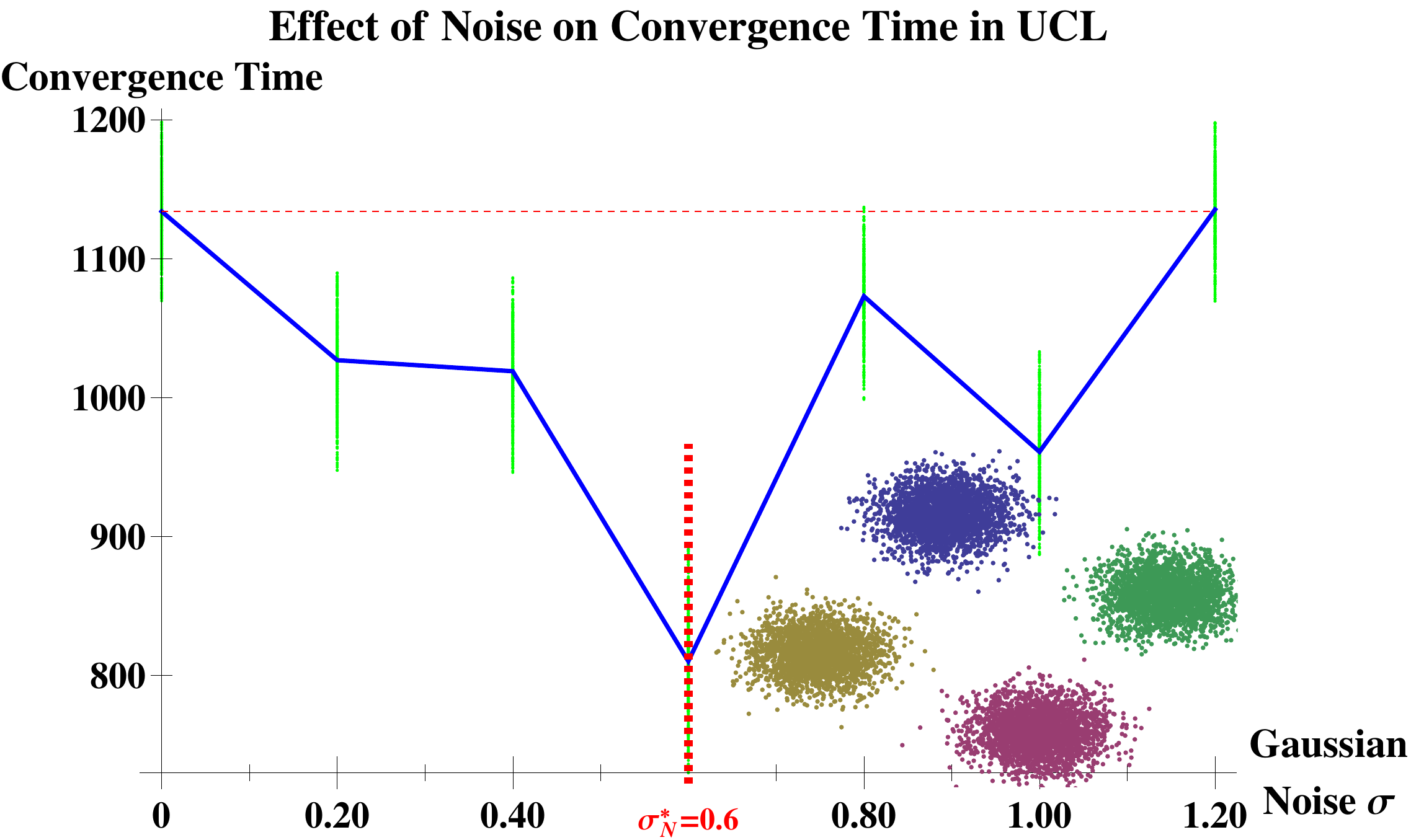} } 

\caption[Noise--benefit in the convergence time of Unsupervised Competitive Learning (UCL)]{Noise--benefit in the convergence time of Unsupervised Competitive Learning (UCL). The inset shows the four Gaussian data clusters with the same covariance matrix. The convergence time is the number of learning iterations before the synaptic weights stayed within $25\%$ of the final converged synaptic weights. The dashed horizontal line shows the convergence time for UCL without additive noise. The figure shows that a small amount of noise can reduce convergence time by about $25\%$. The procedure adapts to noisy samples from a Gaussian mixture of four sub-populations. The sub-populations have centroids on the vertices of the rotated square of side--length $24$ centered at the origin as the inset figure shows. The additive noise is zero-mean Gaussian.
}
\label{fig:UCL}
\end{figure}

% Algorithm2e Option reset
\LinesNumbered

\section{Conclusion}
Noise can speed convergence in expectation-maximization clustering and in some types of competitive learning algorithms under fairly general conditions. This suggests that noise should improve the performance of other clustering and competitive-learning algorithms.  An open research question is whether the Noisy Expectation-Maximization Theorem or some other mathematical result can guarantee the observed noise benefits for competitive learning and for similar clustering algorithms.

Future work may also extend EM noise benefits to co-clustering methods~\parencite{hartigan1972,mirkin1996,dhillon-mallela-modha2003, jain2009} (also known as biclustering or two-mode clustering). Co-clustering methods cluster both data samples and data features simultaneously. For example, a co-clustering method for a movie recommendation system would simultaneously cluster the users \emph{and} the movies in the database. An important model for co-clustering algorithms is the Probabilistic Latent Semantic Analysis (PLSA)~\parencite{hofmann1999,hofmann2001,hofmann2004} used in collaborative filtering~\parencite{Das-et-Google2007,hofmann2004} and document analysis~\parencite{hofmann1999}. PLSA makes use of an EM algorithm to train a hierarchical multinomial model. EM noise benefits may apply here. EM noise benefits may also apply to other co-clustering methods based on generalized k-means or generalized mixture-model approaches.

\clearpage

%\subfile{./Chapters/chapNEM-HMM}

\LinesNotNumbered % Algorithm2e option set for rest of chapter

\chapter{NEM Application: Baum-Welch Algorithm for Training Hidden Markov Models} \label{ch:NEM-HMM}

\blfootnote{
	This chapter features work done in collaboration with Kartik Audhkhasi and first published in \parencite{audhkhasi-osoba-kosko-HMM2013}.
}

This chapter shows that careful noise injection can speed up iterative ML parameter estimation for hidden Markov models (HMMs). The proper noise appears to help the training process explore less probable regions of the parameter space. This new \emph{noisy HMM} (NHMM)~\parencite{audhkhasi-osoba-kosko-HMM2013} is a special case of NEM~\parencite{osoba-mitaim-kosko2011, osoba-mitaim-kosko2012}. The NEM algorithm gives rise to the NHMM because the Baum-Welch algorithm \parencite{baum-petrie1966, baum-eagon1967, baum-et-al1970} that trains the HMM parameters is itself a special case of the EM algorithm~\parencite{welch2003}. The NEM theorem gives a sufficient positivity condition for an average noise boost in an EM algorithm. The positivity condition~(\ref{eq:nhmm-Goal}) below states the corresponding sufficient condition for a noise boost in the Baum-Welch algorithm when the HMM uses a Gaussian mixture model at each state. 

Figure~\ref{fig:nhmm_block_diag} describes the NHMM architecture based on the NEM algorithm. Simulations show that a noisy HMM converges faster than a noiseless HMM on the TIMIT data set. Figure~\ref{fig:per_red_iter} shows that noise produces a $37\%$ reduction in the number of iterations that it takes to converge to the maximum-likelihood estimate.

The simulations below confirm the theoretical prediction that proper injection of noise can improve speech recognition. This appears to be the first deliberate use of noise injection in the speech data itself.  Earlier efforts~\parencite{krogh1994hidden, eddy1995multiple} used annealed noise to perturb the model parameters and to pick an alignment path between HMM states and the observed speech data. These earlier efforts neither added noise to the speech data nor found any theoretical guarantee of a noise--benefit.

The next section (\Sec\ref{sec:hmm_rev}) reviews HMMs and the Baum-Welch algorithm that tunes them. The section shows that the Baum-Welch algorithm is an EM algorithm as Welch~\parencite*{welch2003} pointed out. \Sec\ref{sec:single_nhmm} presents the sufficient condition for a noise boost in HMMs. \Sec\ref{sec:hmm-sim} tests the new NHMM algorithm for training monophone models on the TIMIT corpus.

\begin{figure}[ht]
%\centering \includegraphics[width=\textwidth]{Figures/Speech/nhmm_block_diag.eps}
\centering \includegraphics[width=\textwidth]{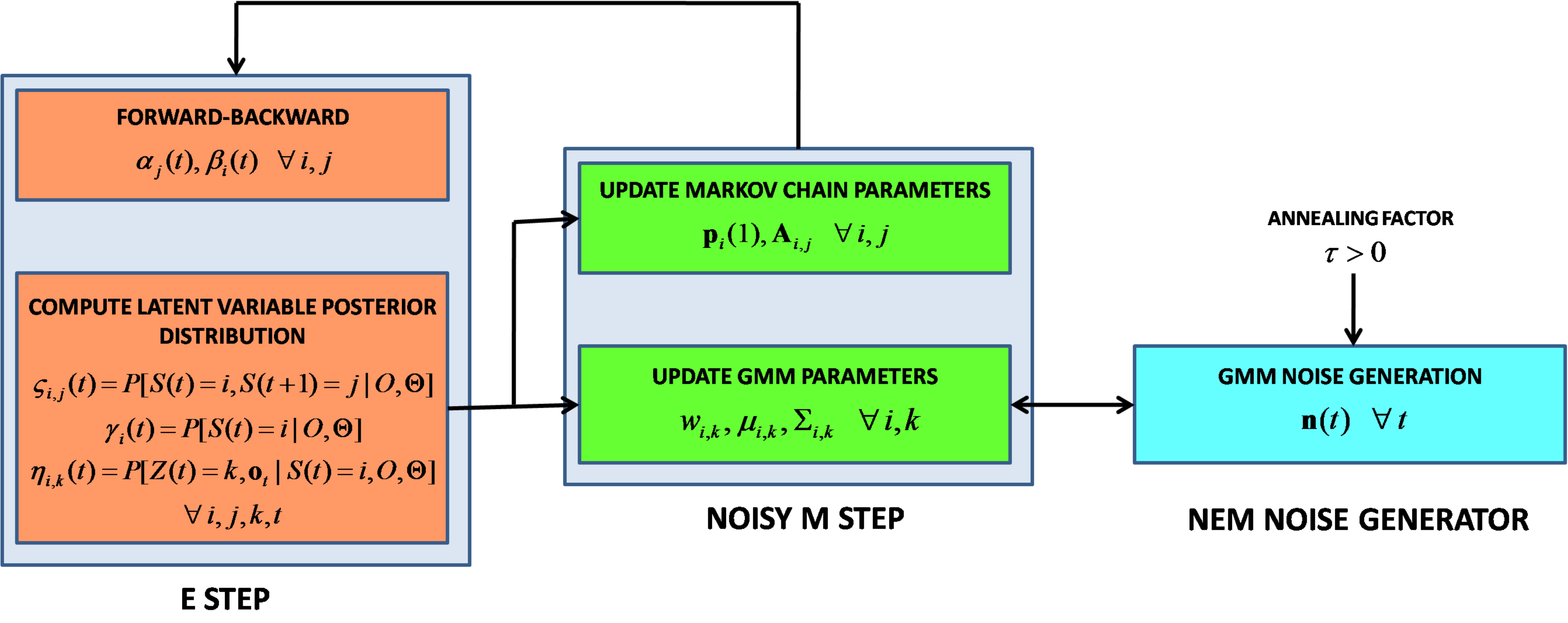}
\caption[Training a Noisy Hidden Markov Model]{The training process of the NHMM: The NHMM algorithm adds annealed noise to the observations during the M-step in the EM algorithm if the noise satisfies the NEM positivity condition. This noise changes the GMM covariance estimate in the M-step.}
\label{fig:nhmm_block_diag}
\end{figure}

\section{Hidden Markov Models}\label{sec:hmm_rev}
An HMM~\parencite{rabiner1989} is a popular probabilistic latent variable models for multivariate time series data. Its many applications include speech recognition~\parencite{rabiner1989, levinson1986continuously, wilpon1990automatic}, computational biology~\parencite{eddy1998profile, karplus1998hidden, krogh1994hidden}, computer vision~\parencite{yamato1992recognizing, brand1997coupled}, wavelet-based signal processing~\parencite{crouse1998wavelet}, and control theory~\parencite{elliott1994hidden}. HMMs are especially widespread in speech processing and recognition. All popular speech recognition toolkits use HMMs:  Hidden Markov Model Toolkit (HTK)~\parencite{young2002htk}, Sphinx~\parencite{walker2004sphinx}, SONIC~\parencite{pellom2003recent}, RASR~\parencite{rybach2009rwth}, Kaldi~\parencite{povey2011kaldi}, Attila~\parencite{soltau2010ibm}, BYBLOS~\parencite{chow1987byblos}, and Watson~\parencite{goffin2005t}.

An HMM consists of a time-homogeneous Markov chain with $M$ states and a single-step transition matrix $\mathbf{A}$. Let $S:\mathbb{Z}^+ \rightarrow \mathbb{Z}_M$ denote a function that maps time to state indices. Then
\begin{align}
 \mathbf{A}_{i,j} &= P[S(t+1) = j | S(t) = i]
\end{align}
$\text{ for }\forall t \in \mathbb{Z}^+ \text{ and } \forall i,j \in \mathbb{Z}_M$. Each state contains a probability density function (pdf) of the multivariate observations. A GMM is a common choice for this purpose~\parencite{rabiner1993fundamentals}. The pdf $f_i$ of an observation $\mathbf{o} \in \mathbb{R}^D$ at state $i$ is
\begin{align}\label{eq:gmm}
 f_i(\mathbf{o}) &= \sum_{k=1}^K w_{i,k} \mathcal{N}(\mathbf{o};\boldsymbol{\mu}_{i,k},\boldsymbol{\Sigma}_{i,k})
\end{align}
where $w_{i,1},\ldots,w_{i,K}$ are convex coefficients and $\mathcal{N}(\mathbf{o};\boldsymbol{\mu}_{i,k},\boldsymbol{\Sigma}_{i,k})$ denotes a multivariate Gaussian pdf with population mean $\boldsymbol{\mu}_{i,k}$ and covariance matrix $\boldsymbol{\Sigma}_{i,k}$.

\subsection{The Baum-Welch Algorithm for HMM Parameter Estimation}
The Baum-Welch algorithm~\parencite{baum-et-al1970} is an EM algorithm for tuning HMM parameters. Let $\mathcal{O} = (\mathbf{o}_1,\ldots,\mathbf{o}_T)$ denote a multivariate time series of length $T$. Let $\mathcal{S} = (S(1),\ldots,S(T))$ and $\mathcal{Z} = (Z(1),\ldots,Z(T))$ be the respective latent state and Gaussian index sequences. Then the ML estimate $\Theta^*$ of the HMM parameters is
\begin{align}
	\Theta^* &= \argmax{\Theta} \log \sum_{\mathcal{S},\mathcal{Z}} P[\mathcal{O},\mathcal{S},\mathcal{Z}|\Theta] \;.
	\label{eq:ml_est}
\end{align}
The sum over latent variables makes it difficult to directly maximize the objective function (\ref{eq:ml_est}). EM uses Jensen's inequality~\parencite{boyd-vandenberghe2004} and the concavity of the logarithm to obtain the following lower-bound on the observed data log-likelihood $\log P[\mathcal{O}|\Theta]$ at the current parameter estimate $\Theta^{(n)}$:
\begin{align}
 \log P[\mathcal{O}|\Theta] &\geq \mathbb{E}_{P[\mathcal{S},\mathcal{Z} | \mathcal{O}, \Theta^{(n)}]} \log P[\mathcal{O},\mathcal{S},\mathcal{Z}|\Theta] \nonumber \\
 &= Q(\Theta|\Theta^{(n)}) 
\end{align}
The complete data log-likelihood for an HMM factors is
\begin{multline}
\log P[\mathcal{O},\mathcal{S},\mathcal{Z}|\Theta] = \sum_{i=1}^M I(S(1) = i)\log \mathbf{p}_i(1) ~ +  \\
  \sum_{t=1}^T\sum_{i=1}^M\sum_{k=1}^K I(S(t) = i, Z(t) = k) \Big{\{} \log w_{i,k} + %\\
\log\mathcal{N}(\mathbf{o}_t|\boldsymbol{\mu}_{i,k},\boldsymbol{\Sigma}_{i,k}) \Big{\}} + \\
 \sum_{t=1}^{T-1}\sum_{i=1}^M\sum_{j=1}^M I(S(t+1) = j,S(t) = i) \log\mathbf{A}_{i,j}
\end{multline}
where $I(.)$ is an indicator function and $\mathbf{p}_i(1) = P[S(1) = i]$. The $Q$-function requires computing the following sets of variables:
\begin{align}
 \gamma^{(n)}_i(1) &= P[S(1) = i | \mathcal{O},\Theta^{(n)}] \\
 \eta^{(n)}_{i,k}(t) &= P[S(t) = i, Z(t) = k | \mathcal{O},\Theta^{(n)}] \\
 \zeta^{(n)}_{i,j}(t) &= P[S(t+1) = j, S(t) = i|\mathcal{O},\Theta^{(n)}]
\end{align}
for $\forall t \in \{1,\ldots,T\}$, $i,j \in \{1,\ldots,M\}$, and $k \in \{1,\ldots,K\}$. The Forward-Backward algorithm is a dynamic programming approach that efficiently computes these variables~\parencite{rabiner1989}. The resulting $Q$-function is
\begin{multline}
 Q(\Theta|\Theta^{(n)}) = \sum_{i=1}^M \gamma^{(n)}_i(1) \log\mathbf{p}_i(1) + \\
 \sum_{t=1}^T\sum_{i=1}^M\sum_{k=1}^K \eta^{(n)}_{i,k}(t) \Big{\{} \log w_{i,k}  
  + \log\mathcal{N}(\mathbf{o}_t|\boldsymbol{\mu}_{i,k},\boldsymbol{\Sigma}_{i,k}) \Big{\}} + \\
 \sum_{t=1}^{T-1}\sum_{i=1}^M \sum_{j=1}^M \zeta^{(n)}_{i,j}(t) \log\mathbf{A}_{i,j} .
\end{multline}
Maximizing the auxiliary function $Q(\Theta|\Theta^{(n)})$ with respect to the parameters $\Theta$ subject to sum-to-one constraints leads to the re-estimation equations for the M-step at iteration $n$:
\begin{align}
 \mathbf{p}^{(n)}_i(1) &= \gamma^{(n)}_i(1) \\
 \mathbf{A}^{(n)}_{i,j} &= \frac{\sum_{t=1}^{T-1}\zeta^{(n)}_{i,j}(t)}{\sum_{t=1}^{T-1}\gamma^{(n)}_{i}(t)} \\
 w^{(n)}_{i,k} &= \frac{\sum_{t=1}^{T}\eta^{(n)}_{i,k}(t)}{\sum_{t=1}^{T}\gamma^{(n)}_{i}(t)} \\
 \boldsymbol{\mu}^{(n)}_{i,k} &= \frac{\sum_{t=1}^{T}\eta^{(n)}_{i,k}(t)\mathbf{o}_t}{\sum_{t=1}^{T}\gamma^{(n)}_{i}(t)} \label{eq:BW-mu}\\
 \boldsymbol{\Sigma}^{(n)}_{i,k} &= \frac{\sum_{t=1}^{T}\eta^{(n)}_{i,k}(t)(\mathbf{o}_t-\boldsymbol{\mu}^{(n)}_{i,k})(\mathbf{o}_t-\boldsymbol{\mu}^{(n)}_{i,k})^T}{\sum_{t=1}^{T}\gamma^{(n)}_{i}(t)}. \label{eq:BW-covar}
\end{align}

The next section restates the NEM theorem and algorithm in the HMM context.

\section{NEM for HMMs: The Noise-Enhanced HMM (NHMM)}\label{sec:single_nhmm}

A restatement of the NEM condition \parencite{osoba-mitaim-kosko2011,osoba-mitaim-kosko2012} in the HMM context requires the following definitions: The noise random variable $\mathbf{N}$ has pdf  $f(\mathbf{n}|\mathbf{o})$. So the noise $\mathbf{N}$ can depend on the observed data $\mathcal{O}$. $\mathcal{L}$ are the latent variables in the model. $\{\Theta^{(n)}\}$  is  a sequence of EM estimates for $\Theta$. $\Theta^*$ is the converged EM estimate for $\Theta$: $\Theta^*= \lim_{n \to \infty} \Theta^{(n)}$. Define the noisy $Q_N$ function $ Q_\mathbf{N}( \Theta |\Theta^{(n)}) = \E_{\mathcal{L}|\mathcal{O},\Theta^{(n)}}   \left[ \ln f(\mathbf{o}+\mathbf{N},\mathcal{L}| \Theta)  \right]$. Then the NEM positivity condition for Baum-Welch training algorithm is:
\begin{align}
\E_{\mathcal{O,L},\mathbf{N}|\Theta^*} &\left[ \ln\left( \frac{f(\mathcal{O}+\mathbf{N},\mathcal{L}|\Theta^{(n)})}{f(\mathcal{O,L}|\Theta^{(n)})} \right) \right] \geq 0\;. \label{eq:nhmm-Goal}
\end{align}

The HMM uses GMMs at each state. So the HMM satisfies the NEM condition when we enforce the simplified GMM-NEM positivity condition for each GMM. The HMM-EM positivity condition (\ref{eq:nhmm-Goal}) holds when the additive noise sample $\mathbf{N}= (N_1,...,N_D)$ for each observation vector $\mathbf{o}=(o_1,...,o_D)$ satisfies the following quadratic constraint:
\begin{align}
N_d \left[N_d - 2\left(\mu_{i,k,d}-o_d\right) \right] &\leq 0 \quad \textrm{ for all } k \;. \label{eq:gmm-nemcond}
\end{align}

The state sequence $\mathcal{S}$ and the Gaussian index $\mathcal{Z}$ are the latent variables $\mathcal{L}$ for an HMM. The noisy $Q$-function for the NHMM is
\begin{align}\label{eq:hmm_noisyQ}
 Q_N(\Theta|\Theta^{(n)}) = \sum_{i=1}^M \gamma^{(n)}_i(1) \log\mathbf{p}_i(1) &+ 
 \sum_{t=1}^T\sum_{i=1}^M\sum_{k=1}^K \eta^{(n)}_{i,k}(t) \nonumber \\ 
  \Big{\{} \log w_{i,k}  + \log\mathcal{N}(\mathbf{o}_t &+ \mathbf{n}_t|\boldsymbol{\mu}_{i,k},\boldsymbol{\Sigma}_{i,k}) \Big{\}} \nonumber \\
+ \sum_{t=1}^{T-1}\sum_{i=1}^M\sum_{j=1}^M& \zeta^{(n)}_{i,j}(t) \log\mathbf{A}_{i,j}
\end{align}
where $\mathbf{n}_t \in \mathbf{R}^D$ is the noise vector for the observation $\mathbf{o}_t$. Then the $d^{th}$ element $n_{t,d}$ of this noise vector satisfies the following positivity constraint:
\begin{align}\label{eq:gmm-nemcond-hmm}
 n_{t,d} [n_{t,d} -2( \mu^{(n-1)}_{i,k,d} - o_{t,d})] &\leq 0 \quad \textrm{ for all } k
\end{align}
where $\boldsymbol{\mu}^{(n-1)}_{i,k}$ is the mean estimate at iteration $n-1$.

Maximizing the noisy $Q$-function (\ref{eq:hmm_noisyQ}) gives the update equations for the M-step. Only the GMM mean and covariance update equations differ from the noiseless EM because the noise enters the noisy $Q$-function (\ref{eq:hmm_noisyQ}) only through the Gaussian pdf. But the NEM algorithm requires modifying only the covariance update equation (\ref{eq:BW-covar}) because it uses the noiseless mean estimates (\ref{eq:BW-mu}) to check the positivity condition (\ref{eq:gmm-nemcond-hmm}). Then the NEM covariance estimate is
\begin{align}\label{eq:cov_update_nhmm}
\boldsymbol{\Sigma}^{(n)}_{i,k} &= \frac{\sum_{t=1}^{T}\eta^{(n)}_{i,k}(t)(\mathbf{o}_t + \mathbf{n}_t-\boldsymbol{\mu}^{(n)}_{i,k})(\mathbf{o}_t + \mathbf{n}_t-\boldsymbol{\mu}^{(n)}_{i,k})^T}{\sum_{t=1}^{T}\gamma^{(n)}_{i}(t)}.
\end{align}

\begin{figure}[ht!]
 \centering
\includegraphics[width=0.85\textwidth]{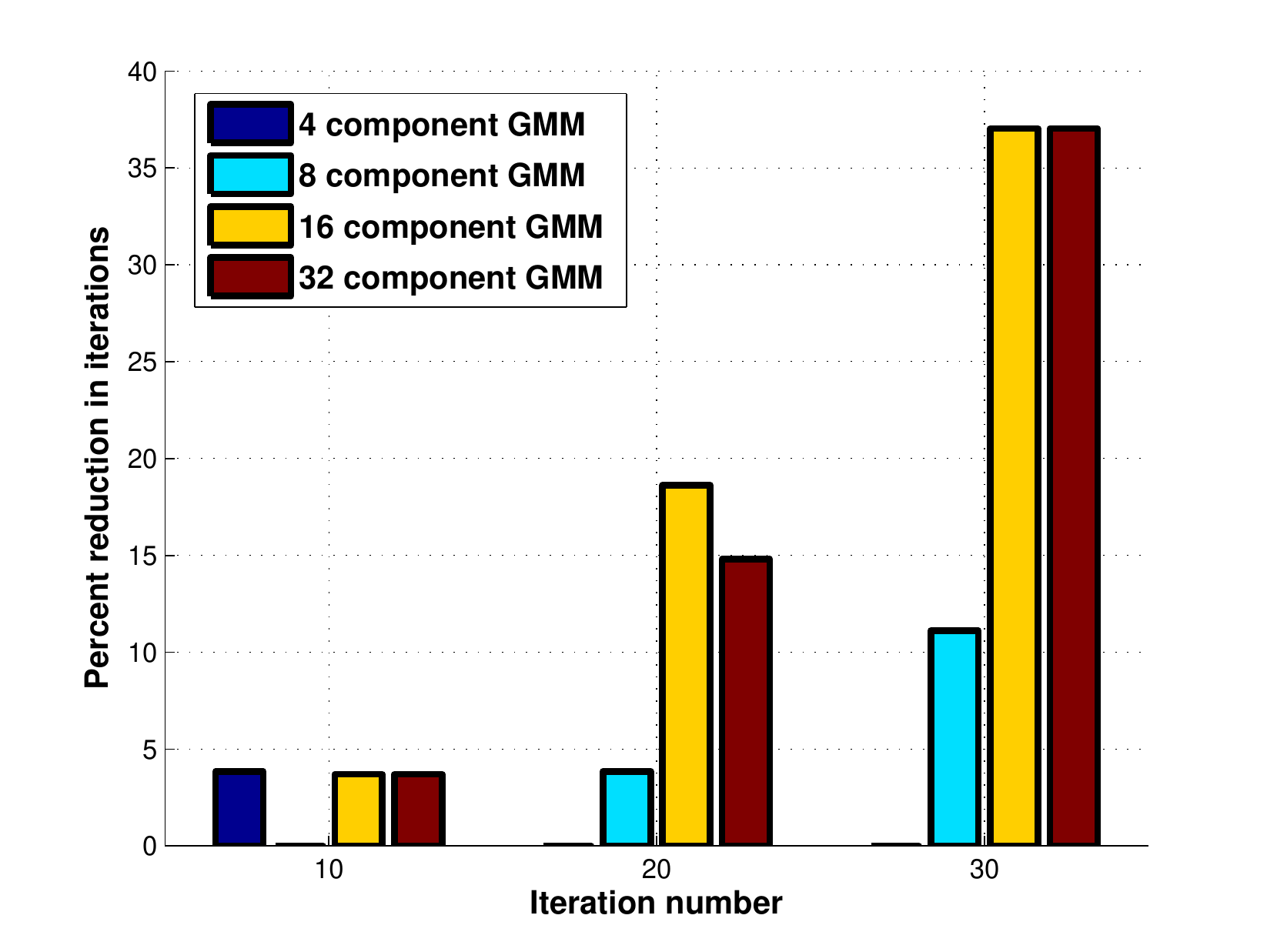}
 \caption[Noisy Hidden Markov Model training converges in fewer iterations than regular hidden Markov Model training]{
	NHMM training converges in fewer iterations than regular HMM training: The bar graph shows the percent reduction in the number of Baum-Welch iterations with respect to the HMM log-likelihood at iterations $10$, $20$, and $30$. Noise significantly reduces the number of iterations for $8$-, $16$-, and $32$-component GMMs. Noise also produces greater reduction for iterations $20$ and $30$ due to the compounding effect of the log-likelihood improvement for the NHMM at each iteration. Noise produces only a marginal reduction for the $4$-component GMM case at $10$ iterations and no improvement for $20$ and $30$ iterations. This pattern of decreasing noise benefits comports with the data sparsity analysis in~\parencite{osoba-mitaim-kosko2012}. The probability of satisfying the NEM sufficient condition increases with fewer data samples for ML estimation.
}
 \label{fig:per_red_iter}
\end{figure}

%\SetKwProg{Fn}{Function}{ is}{end}

\begin{algorithm}[H]
\DontPrintSemicolon
\SetKwInOut{Input}{Input}
\SetKwInOut{Output}{Output}
\SetKwFunction{GenerateNoise}{GenerateNoise}

{\bf Initialize parameters:} $\Theta^{(1)} \leftarrow \Theta_\text{init}$
\For{$n = 1 \to n_\text{max}$}{
	{\bf E-Step($\mathcal{O},\Theta^{(n)}$)}{
	 	\For{$t = 1 \to T$, $i,j = 1 \to M$, and $k = 1 \to K$}{
			$\gamma^{(n)}_i(1) \leftarrow P[S(1) = i | \mathcal{O},\Theta^{(n)}]$ \;
 			$\eta^{(n)}_{i,k}(t) \leftarrow P[S(t) = i, Z(t) = k | \mathcal{O},\Theta^{(n)}]$ \;
    			$\zeta^{(n)}_{i,j}(t) \leftarrow P[S(t+1) = j, S(t) = i|\mathcal{O},\Theta^{(n)}]$ \;
		}
	}
	{\bf M-Step($\mathcal{O},\gamma,\eta,\zeta,\tau$)}{
		\For{$i,j = 1 \to M$ and $k = 1 \to K$}{
			$\mathbf{p}^{(n)}_i(1) \gets \gamma^{(n)}_i(1)$\;
			$\mathbf{A}^{(n)}_{i,j} \gets \frac{\sum_{t=1}^{T-1}\zeta^{(n)}_{i,j}(t)}{\sum_{t=1}^{T-1}\gamma^{(n)}_{i}(t)}$\;
			$w^{(n)}_{i,k} \gets \frac{\sum_{t=1}^{T}\eta^{(n)}_{i,k}(t)}{\sum_{t=1}^{T}\gamma^{(n)}_{i}(t)}$\; 
			$\boldsymbol{\mu}^{(n)}_{i,k} \gets \frac{\sum_{t=1}^{T}\eta^{(n)}_{i,k}(t)\mathbf{o}_t}{\sum_{t=1}^{T}\gamma^{(n)}_{i}(t)}$\;
      			$\mathbf{n}_t \gets $ \GenerateNoise{$\boldsymbol{\mu}^{(n)}_{i,k},\mathbf{o}_t,n^{-\tau}\sigma_N^2$}\;
      			$\boldsymbol{\Sigma}^{(n)}_{i,k} = \frac{\sum_{t=1}^{T}\eta^{(n)}_{i,k}(t)(\mathbf{o}_t + \mathbf{n}_t-\boldsymbol{\mu}^{(n)}_{i,k})(\mathbf{o}_t + \mathbf{n}_t-\boldsymbol{\mu}^{(n)}_{i,k})^T}{\sum_{t=1}^{T}\gamma^{(n)}_{i}(t)}$\;
    		}
	}
	{\bf \GenerateNoise{$\boldsymbol{\mu}^{(n)}_{i,k},\mathbf{o}_t,\sigma^2$} }{
		$\mathbf{n}_t \gets \mathcal{N}(0,\sigma^2)$\;
		\For{$d = 1 \to D$}{
			\If{$n_{t,d} [n_{t,d} -2( \mu^{(n-1)}_{i,k,d} - o_{t,d})] > 0 \textrm{ for some } k$}{
				 $n_{t,d} = 0$\;
			}
		}
		{\bf Return} $\mathbf{n}_t$\;
	}
}
\caption{Noise Injection Algorithm for Training NHMMs}\label{alg:nhmm_train}
\end{algorithm}

\section{Simulation Results}\label{sec:hmm-sim}
We modified the Hidden Markov Model Toolkit (HTK)~\parencite{young2002htk} to train the NHMM. HTK provides a tool called ``HERest'' that performs embedded Baum-Welch training for an HMM. This tool first creates a large HMM for each training speech utterance. It concatenates the HMMs for the sub-word units. The Baum-Welch algorithm tunes the parameters of this large HMM.

The NHMM algorithm used (\ref{eq:cov_update_nhmm}) to modify covariance matrices in HERest. We sampled from a suitably truncated Gaussian pdf to produce noise that satisfied the NEM positivity condition (\ref{eq:gmm-nemcond-hmm}). We used noise variances in $\{0.001, 0.01, 0.1, 1\}$. A deterministic annealing factor $n^{-\tau}$ scaled the noise variance at iteration $n$. The noise decay rate was $\tau > 0$. We used $\tau \in \{1,\ldots,10\}$. We then added the noise vector to the observations during the update of the covariance matrices (\ref{eq:cov_update_nhmm}).

The simulations used the TIMIT speech dataset~\parencite{garofolo1993timit} with the standard setup in~\parencite{halberstadt1997heterogeneous}. We parameterized the speech signal with $12$ Mel-Frequency Cepstral Coefficients (MFCC) computed over $20$-msec Hamming windows with a $10$-msec shift. We also appended the first- and second-order finite differences of the MFCC vector with the energies of all three vectors. We used $3$-state left-to-right HMMs to model each phoneme with a $K$-component GMM at each state. We varied $K$ over $\{1,4,8,16,32\}$ for the experiments and used two performance metrics to compare NHMM with HMM.  The first metric was the percent reduction in EM iterations for the NHMM to achieve the same per-frame log-likelihood as does the noiseless HMM at iterations $10, 20$, and $30$. The second metric was the median improvement in per-frame log-likelihood over $30$ training iterations.

Figure~\ref{fig:per_red_iter} shows the percent reduction in the number of training iterations for the NHMM compared to the HMM log-likelihood at iterations $10$, $20$, and $30$. Noise substantially reduced the number of iterations for $16$- and $32$-component GMMs. But it only marginally improved the other cases. This holds because the noise is more likely to satisfy the NEM positivity condition when the number of data samples is small relative to the number of parameters~\parencite{osoba-mitaim-kosko2012}.

% Algorithm2e Option reset
\LinesNumbered

\section{Conclusion}\label{sec:hmm-concl}
Careful addition of noise can speed the average convergence of iterative ML estimation for HMMs. The NEM theorem gives a sufficient condition for generating such noise. This condition reduces to a simple quadratic constraint in the case of HMMs with a GMM at each state. Experiments on the TIMIT data set show a significant improvement in per-frame log-likelihood and in time to convergence for the NHMM as compared with the HMM. Future work should develop algorithms to find the optimal noise variance and annealing decay factor. It should also explore noise benefits at other stages of EM training in an HMM.

\clearpage

%\subfile{./Chapters/chapNEM-BP}
\chapter{NEM Application: Backpropagation for Training Feedforward Neural Networks}\label{ch:NEM-BP}

\blfootnote{
	This chapter features work done in collaboration with Kartik Audhkhasi and first published in \parencite{audhkhasi-osoba-kosko-BP2013}.
}

An artificial neural network is a biomimetic mathematical model consisting of a network of artificial neurons. Artificial neurons are simplified models of biological neurons. These neurons transform input signals using an activation function (usually a sigmoid or squashing function). The connecting edges of the neural network (NN) simulate biological synapses between biological neurons. They amplify or inhibit signal transmission between neurons. Figure~\ref{fig:FF-ANN} is an example of an artificial neural network with three layers or \emph{fields} of neurons: an input, hidden, and an output layer. It is an example of a multilayer feedforward neural network. The term ``feedforward'' refers to the absence of backward or self connections (or feedback) between neurons. 

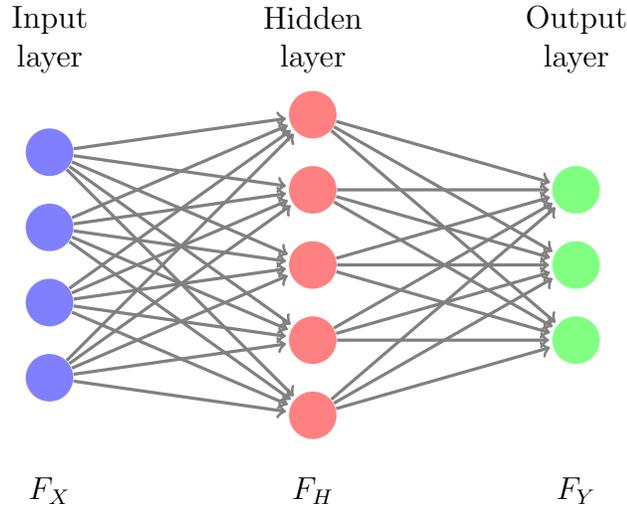
\begin{figure}[ht!]
	\centering
	\begin{tikzpicture}[shorten >=1pt,->,draw=black!50, node distance=\layersep]
		\tikzstyle{every pin edge}=[<-,shorten <=2pt]
		\tikzstyle{neuron}=[circle,fill=black!25,minimum size=18pt,inner sep=0pt]
		\tikzstyle{output neuron}=[neuron, fill=green!50];
		\tikzstyle{hidden neuron}=[neuron, fill=red!50];
		\tikzstyle{input neuron}=[neuron, fill=blue!50];
		\tikzstyle{annot} = [text width=4em, text centered]
		\foreach \name / \y in {1,...,4}
			\node[input neuron] (I-\name) at (0,-\y) {};
		\foreach \name / \y in {1,...,5}
			\path[yshift=0.5cm]
				node[hidden neuron] (H-\name) at (\layersep,-\y cm) {};
		\foreach \name / \y in {2,...,4}
			\node[output neuron, right of=H-\y] (O-\name) {};
		\foreach \source in {1,...,4}
			\foreach \dest in {1,...,5}
				\draw[->, very thick] (I-\source) edge (H-\dest);
		\foreach \source in {1,...,5}
			\foreach \dest in {2,...,4}
					\draw[->, very thick] (H-\source) edge (O-\dest);
		\node[annot,above of=H-1, node distance=1cm] (hl) {Hidden layer};
		\node[annot,left of=hl] {Input layer};
		\node[annot,right of=hl] {Output layer};
		\node[annot,below of=H-5, node distance=1cm] (hll) {$F_H$};
		\node[annot,left of=hll] {$F_X$};
		\node[annot,right of=hll] {$F_Y$};
	\end{tikzpicture}
	\caption[A $3$-Layer Feedforward Artificial Neural Network]{
		A Feedforward Neural Network with three layers of neurons. The nodes represent artificial neurons. The edges represent synaptic connections between neurons. The backpropagation algorithm trains the network by tuning the strength of these synapses. Feedforward neural networks have no feedback or recurrent synaptic connections.}
	\label{fig:FF-ANN}
\end{figure}

Feedforward neural networks are a popular computational model for many pattern recognition and signal processing problems. Their applications include speech recognition~\parencite{mohamed2009deep,mohamed2010investigation,mohamed2012acoustic,seide2011conversational,dahl2010phone,sainath2011making,mohamed2011deep}, machine translation of text~\parencite{deselaers2009deep}, audio processing~\parencite{hamel2010learning}, artificial intelligence~\parencite{bengio2009learning}, computer vision~\parencite{ciresan2010deep,nair20093d,susskind2008generating}, and medicine~\parencite{hu2013artificial}.

Neural networks learn to respond to input stimuli in the same way the brain learns to process sensory stimuli. Learning occurs in both artificial and biological neural networks when the network parameters change~\parencite{kosko-nnfs}. We train neural networks via a sequence of adaptations until the network produces acceptable responses to input stimuli. The process of learning the correct input-output response is the same as learning to approximate an arbitrary function. Halbert White~\parencite{hornik-white1989, hornik-white1990} proved that multilayer feedforward networks can approximate any Borel-measurable function to arbitrary accuracy.% This means that neural networks can learn any input-output response or approximate any function\footnote{Neural networks may not be able to approximate functions which are not Borel measurable~\parencite{folland1999}. Such functions are pathological and uncommon in engineering applications.} to arbitrary accuracy.

\section{Backpropagation Algorithm for NN Training}\label{subsec:bp-overview}
Backpropagation (BP)~\parencite{werbos1974,rumelhart-hinton-williams1986,werbos1990} is the standard supervised training method for multilayer feedforward neural networks~\parencite{kosko-nnfs}. The goal of the BP algorithm is to train a feedforward network (FF-NN) to approximate an input-output mapping by learning from multiple examples of said mapping. The algorithm adapts the neural network by tuning the hidden synaptic weights of the FF-NN to minimize an error function over the example data set. The error function is an aggregated cost function for approximation errors over the training set. The BP algorithm uses gradient descent to minimize the error function.

Feedforward neural networks and other ``connectionist'' learning architectures act as black-box approximators: they offer no insight about \emph{how} the units cooperate to perform the approximation task. This lack of explanatory power complicates the training process. How does a training algorithm determine how much each unit in the network contributes to overall network failure (or success)? How does a training algorithm fairly distribute blame for errors among the many hidden units in the network? This is a key problem that plagues multi-layered connectionist learning architectures. Minsky~\parencite{minsky1961, kosko-nnfs} called this problem the \emph{credit assignment problem}. 

Backpropagation's solution to the credit assignment problem is to propagate the global error signal backwards to each local hidden unit via recursive applications of the chain rule. The application of the chain rule gives a measure the rate of change of the global error with respect to changes in each hidden network parameter. Thus the training algorithm can tune hidden local parameters to perform gradient descent on the global error function.
%This provides information about how sensitive the global error is to each hidden unit. 
% minsky-papert1969 on 2-layer perceptrons

\subsection{Summary of NEM Results for Backpropagation}
The main finding of this chapter is that the backpropagation algorithm for training feedforward neural networks is a special case of the Generalized Expectation-Maximization algorithm (Theorem \ref{thm:bp_gem_equiv})~\parencite{audhkhasi-osoba-kosko-BP2013}. This subsumption is consistent with the EM algorithm's missing information theme. BP estimates hidden parameters to match observed data. While EM estimates hidden variables or functions thereof to produce a high-likelihood fits for observed data. 

The EM subsumption result means that noise can speed up the convergence of the BP algorithm according the NEM theorem. The NEM positivity condition provides the template for sufficient conditions under which injecting noise into training data speeds up BP training time of feedforward neural networks. Figures \ref{fig:bp_em_sqrerr} and \ref{fig:bp_em_crossent} show that the backpropagation error function falls faster with NEM noise than without NEM noise. Thus NEM noise injection leads to faster backpropagation convergence time. %A Monte Carlo approximation replaces the E-step in our EM-BP implementation. This makes our BP implementation a generalized Monte Carlo EM algorithm.

\begin{figure}[ht!]
	\centering
	\includegraphics[width=0.85\textwidth]{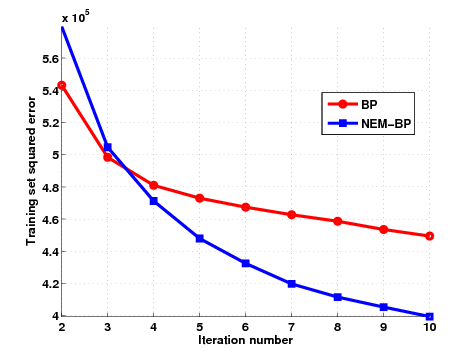}
	\caption[Comparison between backpropagation and NEM-BP using the squared-error cost function]{
		This figure shows the training-set squared error for backpropagation and NEM-backpropagation (NEM-BP) training of an auto-encoder neural network on the MNIST digit classification data set. There is a $5.3\%$ median decrease in the squared error per iteration for NEM-BP when compared with backpropagation training. We added annealed independent and identically-distributed (i.i.d.) Gaussian noise to the target variables. The noise had mean $\mathbf{a}^t-\mathbf{t}$ and a variance that decayed with the training epochs as $\{0.1, 0.1/2, 0.1/3,\ldots\}$ where $\mathbf{a}^t$ is the vector of activations of the output layer and $\mathbf{t}$ is the vector of target values. The network used three logistic (sigmoidal) hidden layers with 20 neurons each. The output layer used 784 logistic neurons.
	}
	\label{fig:bp_em_sqrerr}
\end{figure}

Matsuoka~\parencite*{matsuoka1992noise} and Bishop~\parencite*{bishop1995} hypothesized that injecting noise into the input field may act as a regularizer during BP training and improve generalization performance. G. An~\parencite*{an1996} re-examined Matsuoka's and Bishop's claims with a stochastic gradient descent analysis~\parencite{bottou1991}. He showed that the regularization claim is invalid. But input and synaptic weight noise may indeed improve the network's generalization performance. These previous works used white noise and focused on generalization performance. We instead add non-white noise using a sufficient condition which depends on the neural network parameters and output activations. And our goal is to reduce training time.

The next section (\Sec\ref{sec:backprop}) reviews the details of the backpropagation algorithm and recasts the algorithm as an MLE method. \Sec\ref{sec:bp_em} presents the backpropagation algorithm as an EM algorithm for neural network training. \Sec\ref{sec:nem_backprop} discusses the NEM sufficient conditions in the backpropagation context. \Sec\ref{sec:bp-sim} shows simulation results comparing BP to NEM for NN training.

\section{Backpropagation as Maximum Likelihood Estimation} \label{sec:backprop}

Backpropagation performs ML estimation of a neural network's parameters. We use a 3-layer neural network for notational convenience. The results in this paper extend to deeper networks with more hidden layers. $\mathbf{x}$ are the neuron values at the input layer consisting of $I$ neurons. $\mathbf{a}^h$ is the vector of hidden neuron sigmoidal activations whose $j^{th}$ element is
\begin{align}
	a^h_j & = \frac{1}{1+\exp\Big{(}-\sum_{i=1}^I w_{ji}x_i\Big{)}} = \sigma\Big{(}\sum_{i=1}^I w_{ji}x_i\Big{)} \;,\label{eq:ahj}
\end{align}
where $w_{ji}$ is the weight of the link connecting the $i^{th}$ visible and $j^{th}$ hidden neuron. $y$ represents the $K$-valued target variable and $\mathbf{t}$ is its $1$-in$K$ encoding. $t_k$ is the $k^{th}$ output neuron's value with activation
\begin{align}
	a^t_k &= \frac{\exp\Big{(}\sum_{j=1}^J u_{kj}a^h_j\Big{)}}{\sum_{k_1=1}^K \exp\Big{(}\sum_{j=1}^J u_{k_1j}a^h_j\Big{)}} \label{eq:gibbs_act}\\
	&= p(y = k | \mathbf{x},\Theta)\;,
\end{align}
where $u_{kj}$ is the weight of the link connecting the $j^{th}$ hidden and $k^{th}$ target neuron. $a^t_k$ depends on input $\mathbf{x}$ and parameter matrices $\mathbf{U}$ and $\mathbf{W}$. 

Backpropagation minimizes the following cross entropy:
\begin{align}
	E &= -\sum_{k=1}^K t_k \ln(a^t_k) \;.
\end{align}
The cross-entropy is equal to the negative conditional log-likelihood of the targets given the inputs because
\begin{align}
	E &= -\ln\Big{[}\prod_{k=1}^K (a^t_k)^{t_k} \Big{]} \\
	&= -\ln\Big{[}\prod_{k=1}^K p(y = k|\mathbf{x},\Theta)^{t_k}\Big{]} \\
	&= -\ln p(y|\mathbf{x},\Theta) = -L \;.
\end{align}
Backpropagation updates the network parameters $\Theta$ using gradient ascent to maximize the log likelihood $\ln p(y|\mathbf{x},\Theta)$. The partial derivative of this log-likelihood with respect to $u_{kj}$ is
\begin{align}\label{eq:partial_u_final}
	\frac{\partial L}{\partial u_{kj}} &= (t_k - a^t_k) a^h_j \;,
\end{align}
and with respect to $w_{ji}$ is
\begin{align}
	\frac{\partial L}{\partial w_{ji}} %&= \sum_{k=1}^K (t_k - a^t_k)u_{kj}a^h_j(1-a^h_j)x_i \\
	&= a^h_j(1-a^h_j)x_i\sum_{k=1}^K (t_k - a^t_k)u_{kj} \label{eq:partial_w_final}\;.
\end{align}
(\ref{eq:partial_u_final}) and (\ref{eq:partial_w_final}) give the partial derivatives to perform gradient ascent on the log-likelihood $L$.

Original formulations of backpropagation~\parencite{werbos1974, rumelhart-hinton-williams1986, werbos1990} used the squared error between the observed and target network output as the error function. This setup is sometimes called \emph{regression} since it is equivalent to generalized least-squares regression. Hidden layers in this configuration correspond to higher order interactions in the regression. Later theoretical developments~\parencite{solla-levin-fleisher1988, holt-semnani1990, vanooyen-nienhuis1992} in backpropagation showed that the cross-entropy error function leads to better convergence properties.

The linear activation function
\begin{align}
	a^t_k &= \sum_{j=1}^J u_{kj} a^h_j
\end{align}
often replaces the Gibbs function at the output layer for a regression NN. The target values $\mathbf{t}$ of the output neuron layer are free to assume any real values for regression. Backpropagation then minimizes the following squared error function:
\begin{align}
	E &= \frac{1}{2}\sum_{k=1}^K (t_k - a_k^t)^2 \;.
\end{align}
We assume that the estimation error $\mathbf{e} = \mathbf{t} - \mathbf{a}^t$ is Gaussian with mean $\mathbf{0}$ and identity covariance matrix $\mathbf{I}$ in keeping with standard assumptions for least-squares regression. Least-squares regression estimates optimal regression parameters $\mathbf{U}$ and $\mathbf{W}$ that minimize the least-squares error. These least-squares estimates are the maximum likelihood estimate under the Gaussian assumption.  So backpropagation also maximizes the following log-likelihood function:
\begin{align}
	L &= \log p(\mathbf{t} | \mathbf{x},\Theta) = \log \mathcal{N}(\mathbf{t};\mathbf{a}^t,\mathbf{I})
\end{align}
for
\begin{align}
	\mathcal{N}(\mathbf{t};\mathbf{a}^t,\mathbf{I}) &= \frac{1}{(2\pi)^{d/2}} \exp\Bigg{\{}-\frac{1}{2}\sum_{k=1}^K (t_k - a_k^t)^2\Bigg{\}} \;.
\end{align}
And thus the gradient partial derivatives of this log-likelihood function are the same as those for the $K$-class classification case in (\ref{eq:partial_u_final}) and (\ref{eq:partial_w_final}).

\section{Backpropagation as an EM Algorithm}\label{sec:bp_em}
Both backpropagation and the EM algorithm seek ML estimates of a neural network's parameters. The next theorem shows that backpropagation is a generalized EM algorithm. 

%%%%%%%%%%%%
\begin{thm}{\bf{[Backpropagation is a GEM Algorithm]:}} \label{thm:bp_gem_equiv}\\
The backpropagation update equation for a differentiable likelihood function $p(y|\mathbf{x},\Theta)$ at epoch $n$ is
\begin{align}
	\Theta^{n+1} &= \Theta^{n} + \eta \nabla_{\Theta} \ln p(y|\mathbf{x},\Theta)\Big{|}_{\Theta = \Theta^{n}}
\end{align}
equals the GEM update equation at epoch $n$
\begin{align}
	\Theta^{n+1} &= \Theta^{n} + \eta  \nabla_{\Theta} Q(\Theta|\Theta^{n})\Big{|}_{\Theta = \Theta^{n}} \;,
\end{align}
where the Q-function used in GEM is
\begin{align}\label{eq:qfun}
	Q(\Theta|\Theta^n) &= \mathbb{E}_{p(\mathbf{h}|\mathbf{x},y,\Theta^n)}\Big{\{}\ln p(y,\mathbf{h}|\mathbf{x},\Theta)\Big{\}} \;. 
\end{align}
\end{thm}

\begin{proof}
We know that~\parencite{bishop2006, oakes1999}
\begin{align}
	\log p(y|\mathbf{x},\Theta) &= Q(\Theta|\Theta^n) + H(\Theta|\Theta^n)
\end{align}
if $H(\Theta|\Theta^n)$ is the following cross entropy~\parencite{cover-thomas91}:
\begin{align}
	H(\Theta|\Theta^{n}) = -\int \ln p(\mathbf{h}|\mathbf{x},y,\Theta) ~\mathrm{d}p(\mathbf{h}|\mathbf{x},y,\Theta^{n}) \;.
\end{align}
Hence
\begin{align}\label{eq:ce_qfun_relation}
	 H(\Theta|\Theta^{n}) &= \log p(y|\mathbf{x},\Theta) - Q(\Theta|\Theta^{n}) \;.
\end{align}
Now expand the relative entropy:
\begin{align}
	D_\text{KL}(\Theta^n||\Theta) &= \int \ln \Bigg(\frac{p(\mathbf{h}|\mathbf{x},y,\Theta^n)}{p(\mathbf{h}|\mathbf{x},y,\Theta)}\Bigg) ~\mathrm{d}p(\mathbf{h}|\mathbf{x},y,\Theta^n) \\
	&= \int \ln p(\mathbf{h}|\mathbf{x},y,\Theta^n) ~\mathrm{d}p(\mathbf{h}|\mathbf{x},y,\Theta^n) \nonumber \\
	&- \int \ln p(\mathbf{h}|\mathbf{x},y,\Theta) ~\mathrm{d}p(\mathbf{h}|\mathbf{x},y,\Theta^n) \\
	&= -H(\Theta^n|\Theta^n) + H(\Theta|\Theta^n) \;.
\end{align}
So $H(\Theta|\Theta^n)$ $\geq$ $H(\Theta^n|\Theta^n)$ for all $\Theta$ because $D_\text{KL}(\Theta^n||\Theta)$ $\geq$ $0$. Thus $\Theta^n$ minimizes $H(\Theta|\Theta^n)$ and hence $\nabla_{\Theta} H(\Theta|\Theta^n) = 0$ at $\Theta = \Theta^n$. Putting this in (\ref{eq:ce_qfun_relation}) gives
\begin{align}
	\nabla_{\Theta} \log p(y|\mathbf{x},\Theta)\Big{|}_{\Theta = \Theta^{n}} &= \nabla_{\Theta} Q(\Theta|\Theta^{n})\Big{|}_{\Theta = \Theta^{n}} \;.
\end{align}
Hence the backpropagation and GEM update equations are identical.
\end{proof}
%%%%%%%%%%%%%%%%%%%%%%%%%%%%

This makes sense because backpropagation is a greedy algorithm for optimizing the likelihood function as the last section showed. The basic intuition is that any greedy algorithm for likelihood--function optimization is a generalized EM algorithm if there is hidden or incomplete data involved in the estimation.

The GEM algorithm involves a probabilistic description of the hidden layer neurons. We assume that the hidden layer neurons are Bernoulli random variables. Their activation is thus the following conditional probability:
\begin{align}
	a^h_j &= p(h_j = 1 | \mathbf{x},\Theta) \;.
\end{align}

We can now formulate an EM algorithm for ML estimation of a feedforward neural network's parameters. The E-step computes the Q-function in (\ref{eq:qfun}). Computing the expectation in (\ref{eq:qfun}) requires $2^J$ values of $p(\mathbf{h}|\mathbf{x},y,\Theta^n)$. This is expensive for large values of $J$. So we thus resort to Monte Carlo sampling to approximate the above $Q$-function. The strong law of large numbers ensures that this Monte Carlo approximation converges almost surely to the true Q-function. $p(\mathbf{h}|\mathbf{x},y,\Theta^n)$ becomes the following using Bayes theorem:
\begin{align}
	p(\mathbf{h}|\mathbf{x},y,\Theta^n) &= \frac{p(\mathbf{h}|\mathbf{x},\Theta^n)p(y|\mathbf{h},\Theta^n)}{\sum_\mathbf{h} p(\mathbf{h}|\mathbf{x},\Theta^n)p(y|\mathbf{h},\Theta^n)} \;.
\end{align}
$p(\mathbf{h}|\mathbf{x},\Theta^n)$ is easier to sample from because $h_j$s are independent given $\mathbf{x}$. We replace $p(\mathbf{h}|\mathbf{x},\Theta^n)$ by its Monte Carlo approximation using $M$ independent and identically-distributed (IID) samples:
\begin{align}
	p(\mathbf{h}|\mathbf{x},\Theta^n) &\approx \frac{1}{M} \sum_{m=1}^M \delta_K(\mathbf{h} - \mathbf{h}^m) \;,
\end{align}
where $\delta_K$ is the $J$-dimensional Kronecker delta function. The Monte Carlo approximation of the hidden data conditional PDF becomes
\begin{align}
	p(\mathbf{h}|\mathbf{x},y,\Theta^n) &\approx \frac{\sum_{m=1}^M \delta_K(\mathbf{h} - \mathbf{h}^m) p(y|\mathbf{h},\Theta^n)}{\sum_\mathbf{h} \sum_{m_1=1}^M \delta_K(\mathbf{h} - \mathbf{h}^{m_1})p(y|\mathbf{h},\Theta^n)} \\
	&= \frac{\sum_{m=1}^M \delta_K(\mathbf{h} - \mathbf{h}^m) p(y|\mathbf{h}^m,\Theta^n)}{\sum_{m_1=1}^M p(y|\mathbf{h}^{m_1},\Theta^n)} \\
	&= \sum_{m=1}^M \delta_K(\mathbf{h} - \mathbf{h}^m) \gamma^m \;, \label{eq:wgt_mc_approx}
\end{align}
where
\begin{align}
	\gamma^m &= \frac{p(y|\mathbf{h}^m,\Theta^n)}{\sum_{m_1=1}^M p(y|\mathbf{h}^{m_1},\Theta^n)}
\end{align}
is the ``importance'' of $\mathbf{h}^m$. (\ref{eq:wgt_mc_approx}) gives an importance-sampled approximation of $p(\mathbf{h}|\mathbf{x},y,\Theta^n)$ where each sample $\mathbf{h}^m$ is given weight $\gamma^m$. We can now approximate the Q-function as:
\begin{align}
	Q(\Theta|\Theta^n) &\approx \sum_{\mathbf{h}} \sum_{m=1}^M \gamma^m \delta_K(\mathbf{h} - \mathbf{h}^m) \ln p(y,\mathbf{h}|\mathbf{x},\Theta) \\
	&= \sum_{m=1}^M \gamma^m \ln p(y,\mathbf{h}^m|\mathbf{x},\Theta) \\
	&= \sum_{m=1}^M \gamma^m \Big{[}\ln p(\mathbf{h}^m|\mathbf{x},\Theta) + \ln p(y|\mathbf{h}^m,\Theta)\Big{]} \;,\label{eq:qfun-nn}
\end{align}
where
\begin{align}
	\ln p(\mathbf{h}^m|\mathbf{x},\Theta) &= \sum_{j=1}^J \Big{[} h^m_j \ln a^h_j + (1-h^m_j)\ln (1-a^h_j) \Big{]} \;,
\end{align}
for sigmoidal hidden layer neurons. Gibbs activation neurons at the output layer give
\begin{align}
	\ln p(y|\mathbf{h}^m,\Theta) &= \sum_{k=1}^K t_k \ln a^{mt}_k \;,
\end{align}
where $a^h_j$ is given in (\ref{eq:ahj}) and
\begin{align}
	a^{mt}_k &= \frac{\exp\Big{(}\sum_{j=1}^J u_{kj}a^{mh}_j\Big{)}}{\sum_{k_1=1}^K \exp\Big{(}\sum_{j=1}^J u_{k_1j}a^{mh}_j\Big{)}}
\end{align}
Gaussian output layer neurons give
\begin{align}
	\ln p(y|\mathbf{h}^m,\Theta) &= \frac{1}{2}\sum_{k=1}^K (t_k-a^{mt}_k)^2  \;.
\end{align}

The Q-function in (\ref{eq:qfun-nn}) is equal to a sum of log-likelihood functions for two 2-layer neural networks between the visible-hidden and hidden-output layers. The M-step maximizes this Q-function by gradient ascent. It is equivalent to two disjoint backpropagation steps performed on these two 2-layer neural networks.

\section{NEM for Backpropagation Training}\label{sec:nem_backprop}
Theorem \ref{thm:bp_gem_equiv} recasts the backpropagation training algorithm as a GEM algorithm. Thus the NEM theorem provides conditions under which BP training converges faster to high likelihood network parameters. 

We restate the NEM positivity condition \parencite{osoba-mitaim-kosko2011,osoba-mitaim-kosko2012} in the neural network training context. We use the following notation: The noise random variable $\mathbf{N}$ has pdf  $p(\mathbf{n}|\mathbf{x})$. So the noise $\mathbf{N}$ can depend on the data $\mathbf{x}$.  $\mathbf{h}$ are the latent variables in the model. $\{\Theta^{(n)}\}$  is  a sequence of EM estimates for $\Theta$. $\Theta_*= \lim_{n \to \infty} \Theta^{(n)}$ is the converged EM estimate for $\Theta$. Define the noisy $Q$ function $ Q_\mathbf{N}( \Theta |\Theta^{(n)}) =  \E_{\mathbf{h}|\mathbf{x},\Theta_k}   \left[ \ln p(\mathbf{x}+\mathbf{N},\mathbf{h}| \theta)  \right]$. Assume that the differential entropy of all random variables is finite. Assume further that the additive noise keeps the data in the likelihood function's support. Then the NEM positivity condition for neural network training algorithm is:
\begin{align}
	\E_{\mathbf{x,h},\mathbf{N}|\Theta^*} \left[ \ln\left( \frac{p(\mathbf{x}+\mathbf{N},\mathbf{h}|\Theta_k)}{p(\mathbf{x,h}|\Theta_k)} \right) \right] \geq 0\;.
\end{align}

The exact form of the noise--benefit condition depends on the activation function of the output neurons.

\subsection{NEM Conditions for Neural Network ML Estimation}
Consider adding noise $\mathbf{n}$ to the $1$-in-$K$ encoding $\mathbf{t}$ of the target variable $y$. We first present the noise--benefit sufficient condition for Gibbs activation output neurons used in $K$-class classification.
\begin{thm}{\bf{[Forbidden Hyperplane Noise--Benefit Condition]:}} \label{thm:hyp-ndnn-bp}\\
The NEM positivity condition holds for ML training of feedforward neural network with Gibbs activation output neurons if
\begin{align}
 \mathbb{E}_{\mathbf{t},\mathbf{h},\mathbf{n}|\mathbf{x},\Theta^*}& \Big{\{} \mathbf{n}^T \log(\mathbf{a}^t) \Big{\}} \geq 0 \;. \label{eq:NEM-LogisticOut}
\end{align}
\end{thm}
\begin{proof}
We add noise to the target $1$-in-$K$ encoding $\mathbf{t}$. The likelihood ratio in the NEM sufficient condition becomes
\begin{align}
	\frac{p(\mathbf{t}+\mathbf{n},\mathbf{h}|\mathbf{x},\Theta)}{p(\mathbf{t},\mathbf{h}|\mathbf{x},\Theta)} &= \frac{p(\mathbf{t}+\mathbf{n}|\mathbf{h},\Theta)p(\mathbf{h}|\mathbf{x},\Theta)}{p(\mathbf{t}|\mathbf{h},\Theta)p(\mathbf{h}|\mathbf{x},\Theta)} \\
	&= \frac{p(\mathbf{t}+\mathbf{n}|\mathbf{h},\Theta)}{p(\mathbf{t}|\mathbf{h},\Theta)} \\
	&= \prod_{k=1}^K \frac{(a_k^t)^{t_k+n_k}}{(a_k^t)^{t_k}} = \prod_{k=1}^K (a_k^t)^{n_k} \;.
\end{align}
So the NEM positivity condition becomes
\begin{align}
	\mathbb{E}_{\mathbf{t},\mathbf{h},\mathbf{n}|\mathbf{x},\Theta^*}\Big{\{}\log \prod_{k=1}^K (a_k^t)^{n_k} \Big{\}} &\geq 0 \;.
\end{align}
This condition is equivalent to
\begin{align}
	\mathbb{E}_{\mathbf{t},\mathbf{h},\mathbf{n}|\mathbf{x},\Theta^*} \Big{\{}\sum_{k=1}^K n_k \log (a_k^t)\Big{\}} &\geq 0 \;.
\end{align}
We can rewrite this positivity condition as the following matrix inequality:
\begin{align}
	\mathbb{E}_{\mathbf{t},\mathbf{h},\mathbf{n}|\mathbf{x},\Theta^*}\{\mathbf{n}^T \log(\mathbf{a}^t)\} &\geq 0
\end{align}
where $\log(\mathbf{a}^t)$ is the vector of output neuron log-activations.
\end{proof}

The above sufficient condition requires that the noise $\mathbf{n}$ lie above a hyperplane with normal $\log(\mathbf{a}^t)$. Figure~\ref{fg:Xent-BP-demo} illustrates this geometry.
\begin{figure}[ht!]
	\centerline{\bf{Geometry of NEM Noise for Cross-Entropy Backpropagation}}
	\centerline{\includegraphics[width=0.75\textwidth]{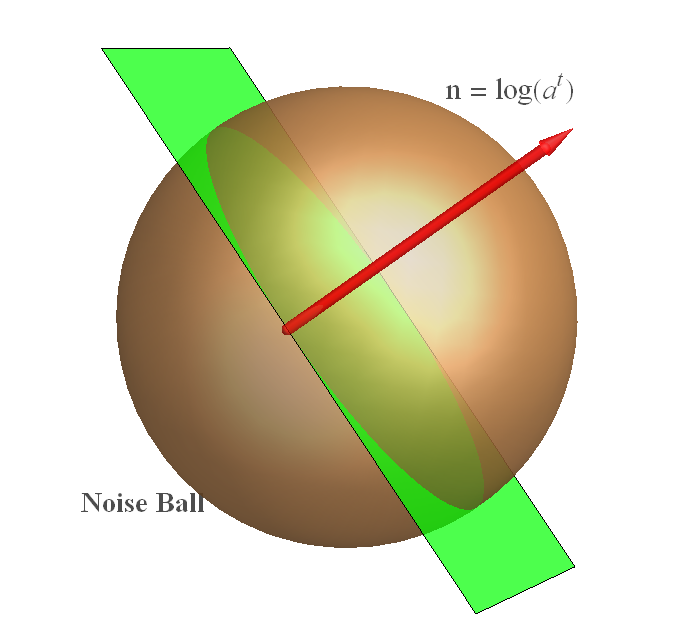}}
	\caption[Geometry of NEM Noise for Cross-Entropy Backpropagation]{
		NEM noise for faster backpropagation using logistic output neurons. NEM noise must fall above a hyperplane through the origin in the noise space. The output activation signal $\mathbf{a}^t$ controls the normal vector $\mathbf{n}$ of the slicing hyperplane. The hyperplane changes on each iteration.
	}
\label{fg:Xent-BP-demo}
\end{figure}

The next theorem gives a sufficient condition for a noise--benefit in the case of Gaussian output neurons.
\begin{thm}{\bf{[Forbidden Sphere Noise--Benefit Condition]:}}\label{thm:sph-ndnn-bp}\\
The NEM positivity condition holds for ML training of a feedforward neural network with Gaussian output neurons if
\begin{align}
 \mathbb{E}_{\mathbf{t},\mathbf{h},\mathbf{n}|,\mathbf{x},\Theta^*}& \Big{\{} \Big{|}\Big{|}\mathbf{n} - \mathbf{a}^t + \mathbf{t} \Big{|}\Big{|}^2 - \Big{|}\Big{|}\mathbf{a}^t - \mathbf{t} \Big{|}\Big{|}^2 \Big{\}} \leq 0  \label{eq:NEM-GaussOut}
\end{align}
where $||.||$ is the $L_2$ vector norm.
\end{thm}
\begin{proof}
	We add noise $\mathbf{n}$ to the $K$ output neuron values $\mathbf{t}$. The log-likelihood in the NEM sufficient condition becomes
\begin{align}
	\frac{p(\mathbf{t}+\mathbf{n},\mathbf{h}|\mathbf{x},\Theta)}{p(\mathbf{t},\mathbf{h}|\mathbf{x},\Theta)} &= \frac{p(\mathbf{t}+\mathbf{n}|\mathbf{h},\Theta)p(\mathbf{h}|\mathbf{x},\Theta)}{p(\mathbf{t}|\mathbf{h},\Theta)p(\mathbf{h}|\mathbf{x},\Theta)} \\
	&= \frac{\mathcal{N}(\mathbf{t}+\mathbf{n};\mathbf{a}^t,\mathbf{I})}{\mathcal{N}(\mathbf{t};\mathbf{a}^t,\mathbf{I})} \\
	= \exp\Big{(}\frac{1}{2}&\Big{[}\Big{|}\Big{|}\mathbf{t} - \mathbf{a}^t\Big{|}\Big{|}^2 - \Big{|}\Big{|}\mathbf{t} + \mathbf{n} - \mathbf{a}^t\Big{|}\Big{|}^2\Big{]} \Big{)} \;.
\end{align}
So the NEM sufficient condition becomes
\begin{align}
 \mathbb{E}_{\mathbf{t},\mathbf{h},\mathbf{n}|,\mathbf{x}\Theta^*}& \Big{\{} \Big{|}\Big{|}\mathbf{n} - \mathbf{a}^t + \mathbf{t} \Big{|}\Big{|}^2 - \Big{|}\Big{|}\mathbf{a}^t - \mathbf{t} \Big{|}\Big{|}^2 \Big{\}} \leq 0 \;. %\label{eq:NEM-NLRBM} 
\end{align}
\end{proof}

The above sufficient condition defines a forbidden noise region outside a sphere with center $\mathbf{t} - \mathbf{a}^t$ and radius $\|\mathbf{t} - \mathbf{a}^t\|$. All noise inside this sphere speeds convergence of ML estimation in the neural network on average. Figure~\ref{fg:LS-BP-demo} illustrates this geometry.
\begin{figure}[ht!]
	\centerline{\bf{Geometry of NEM Noise for Least-Squares Backpropagation}}
	\centerline{\includegraphics[width=0.75\textwidth]{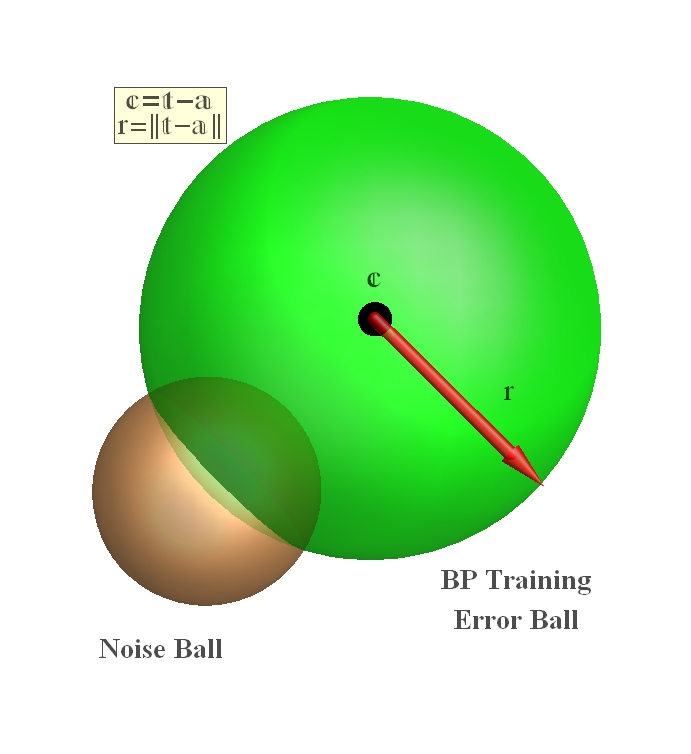}}
	\caption[Geometry of NEM Noise for Least-Squares Backpropagation]{
		NEM noise for faster backpropagation using Gaussian output neurons. The NEM noise must fall inside the backpropagation ``training mismatch'' sphere. This is the sphere with center $\mathbf{c} = \mathbf{t}-\mathbf{a}^t$ (the difference between the target output $\mathbf{t}$ and the actual output layer activation $\mathbf{a}^t$) with radius $r=\|\mathbf{c}\|$. Noise from the noise ball section that intersects with the mismatch sphere will speed up backpropagation training according to the NEM theorem. The mismatch ball changes at each training iteration.
	}
\label{fg:LS-BP-demo}
\end{figure}

The actual implementation of the BP as a GEM uses a Monte Carlo approximation for the E-step. So we are really applying the NEM condition to a \emph{Generalized Monte Carlo EM} algorithm. The quality of the result depends on the approximation quality of the Monte Carlo substitution. This approximation quality degrades quickly as the number of hidden neurons $J$ increases.

\section{Simulation Results}\label{sec:bp-sim}
We modified the Matlab code available in \parencite{hinton_matlab_tool} to inject noise during EM-backpropagation training of a neural network. We used 10,000 training instances from the training set of the MNIST digit classification data set. Each image in the data set had $28 \times 28$ pixels with each pixel value lying between $0$ and $1$. We fed each pixel into the input neuron of a neural network. We used a 5-layer neural network with 20 neurons in each of the three hidden layers and 10 neurons in the output layer for classifying the 10 digits. We also trained an auto-encoder neural network with 20 neurons in each of the three hidden layers and $784$ neurons in the output layer for estimating the pixels of a digit's image.

%\missingfigure{Cross-Entropy-BP comparison with N-MCEM-BP}
\begin{figure}[ht!]
	\centering
	\includegraphics[width=0.85\textwidth]{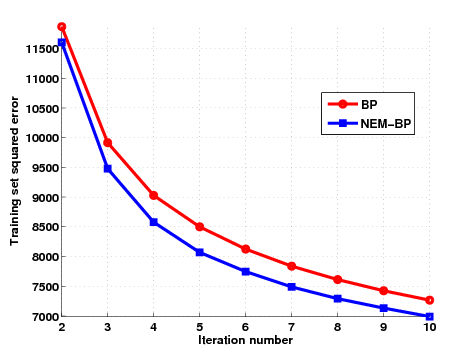}
	\caption[Comparison between backpropagation and NEM-BP using a Cross-Entropy cost function]{
		Training-set cross entropy for backpropagation and NEM-BP training of a 10-class classification neural network on the MNIST digit classification data set. There is a $4.2\%$ median decrease in the cross entropy per iteration for NEM-BP when compared with backpropagation training. We added annealed i.i.d. Gaussian noise to the target variables. The noise had mean $0$ and a variance that decayed with training epochs as $\{0.2, 0.2/2, 0.2/3,\ldots\}$. The network used three logistic (sigmoidal) hidden layers with 20 neurons each. The output layer used 10 neurons with the Gibbs activation function in (\ref{eq:gibbs_act}).
	}
	\label{fig:bp_em_crossent}
\end{figure}

The output layer used the Gibbs activation function for the 10-class classification network and logistic activation function for the auto-encoder. We used logistic activation functions in all other layers. Simulations used 10 Monte Carlo samples for approximating the Q-function in the 10-class classification network and 5 Monte Carlo samples for the auto-encoder. Figure~\ref{fig:bp_em_sqrerr} shows the training-set squared error for the auto-encoder neural network for backpropagation and NEM-backpropagation when we added annealed Gaussian noise with mean $\mathbf{a}^t - \mathbf{t}$ and variance $0.1$ epoch$^{-1}$. Figure~\ref{fig:bp_em_crossent} shows the training-set cross entropy for the two cases when we added annealed Gaussian noise with mean $0$ and variance $0.2$ epoch$^{-1}$. We used $10$ Monte Carlo samples to approximate the Q-function. We observed a $5.3\%$ median decrease in squared error and $4.2\%$ median decrease in cross entropy per iteration for the NEM-backpropagation algorithm compared to standard backpropagation.

\section{Conclusion}\label{sec:bp-concl}
This chapter showed that the backpropagation algorithm is a generalized EM algorithm. This allows us to apply the NEM theorem to develop noise--benefit sufficient conditions for speeding up convergence EM-backpropagation. Simulations on the MNIST digit recognition data set show that NEM noise injection reduces squared-error and cross-entropy in NN training by backpropagation. 

Feedforward neural networks are extremely popular in machine learning and data--mining applications. Most applications of feedforward NNs use backpropagation for training. So the backpropagation noise benefits in this chapter are available for these applications and can lead to shorter training times. Such training time reductions are important for large-scale NN applications where training may take weeks.

Stacked layers of stochastic neural networks (or ``deep'' neural networks) may also benefit from the NEM algorithm. \Sec\ref{subsec:deepNN} presents some theorems predicting noise benefits for the pre-training such deep networks.

\clearpage

%\subfile{./Chapters/chapBayes}
\chapter{Bayesian Statistics}\label{ch:Bayes}
\section{Introduction: The Bayesian \& The Frequentist }
Bayesian inference methods subsume ML estimation methods described in previous chapters. The defining feature of Bayesian inference is the use of Bayes theorem to revise prior beliefs based on new observed data. Rev. Thomas Bayes first introduced his eponymous theorem in a posthumous letter~\parencite{bayes1763} to the Royal Society of London in 1763. Bayes theorem is now at the heart of many statistical applications including spam filtering, evidence-based medicine, and semantic web search.

The ``Bayesian approach'' to statistics sees probabilities as statements of beliefs about the state of a random parameter. The opposed ``frequentist'' view sees probabilities as long-run frequencies of the outcome of experiments involving a fixed underlying parameter. The difference between the two approaches does not affect the basic Kolmogorov theory of probability. But acceptable statistical inference methods differ based on which view the statistician espouses. The Bayesian (e.g.  de Finetti, Lindley, Savage) argues that his approach takes account of all available information (prior information and the data itself) when making an inference or a decision. The frequentist (e.g. Fisher, Student, E.S. Pearson) argues that inference should use only information provided by the data and should be free of subjective input\footnote{
	Fisher even argued \parencite{fisher1934} that Rev. Bayes withheld publication of his work because he was wary of the dangers involved in injecting subjective information into inferences on data.
}. Modern statistics tends to blend both approaches to fit individual applications.

The statistical problem of point estimation highlights this Bayesian vs frequentist schism. The goal of point estimation is to find the best estimate for parameters underlying the data distribution. MLE is a very popular method for point estimation because of its simplicity, its beautiful properties, and its intuitive interpretation. It is frequentist in conception and in spirit; it uses only the data for inference. Bayesian point estimation techniques are more general but they can be complex. They may also give different point estimates depending on available prior information and the cost of choosing bad estimates. The Bayesian framework subsumes the frequentist ML estimate as a possible solution when there is no prior information. Thus Bayesian point estimation is more powerful in cases when there is authoritative prior information. 

The rest of the chapter gives a detailed introduction to Bayesian inference. This sets the foundation for the next chapter which deals with the effects of model and data corruption in Bayesian inference. \textbf{\S}\ref{sec:B-PEstimatn} shows where point estimation and MLE (including E--M) fits in the Bayesian inference framework. Thus any subsequent results also apply to previously discussed ML estimation frameworks.

\section{Bayesian Inference}

Bayesian inference models learning as computing a conditional probability based both on new evidence or data and on prior probabilistic beliefs.  It builds on the simple Bayes theorem that shows how set-theoretic evidence should update competing prior probabilistic beliefs or hypotheses.  The theorem gives the posterior conditional probability $P(H_j|E)$ that the $j^{th}$ hypothesis $H_j$ occurs given that evidence $E$ occurs. The posterior depends on all the converse conditional probabilities $P(E|H_k)$ that $E$ occurs given $H_k$ and on all the unconditional prior probabilities $P(H_k)$ of the disjoint and exhaustive hypotheses $\{H_k\}$: 
\begin{align}
P(H_j|E) \,=\, \frac{P(E|H_j)P(H_j)}{P(E)} \,=\, \frac{P(E|H_j)P(H_j)}{\sum_{k} P(E|H_k)P(H_k)}.
\label{eq:TextBayes}
\end{align}
The result follows from the definition of conditional probability $P(B|A) = P(A\cap B)/P(A)$ for $P(A) > 0$ when the set hypotheses $H_j$ partition the state space of the probability measure $P$ \parencite{ross2005,kosko2004}. $P(H_j|E)$ is a measure of the degree to which the data or evidence $E$ supports each of the competing hypothesis $H_k$. This represents the data-informed update of prior beliefs about competing hypotheses $\{H_k\}$. More accurate beliefs allow for better discrimination between competing hypothesis.

Bayesian inference usually works with a continuous version of (\ref{eq:TextBayes}). Now the parameter value $\theta$ corresponds to the hypothesis of interest and the evidence corresponds to the sample values $x$ from a random variable $X$ that depends on $\theta$:
\begin{align}
f(\theta|x) ~=~ \frac{g(x|\theta)h(\theta)}{\int g(x|u)h(u)du}~\propto~ g(x|\theta)h(\theta)
\label{eq:TextBayesPosterior}
\end{align}
where we follow convention and drop the normalizing term that does not depend on $\theta$ as we always can if $\theta$ has a sufficient statistic \parencite{hogg-mckean-craig2005,hogg-tanis2006}.  The model (\ref{eq:TextBayesPosterior}) assumes that random variable $X$ conditioned on $\theta$ admits the random sample $X_1,\ldots,X_n$ with observed realizations $x_1,\ldots,x_n$.  So again the posterior pdf $f(\theta|x)$ depends on the converse likelihood $g(x|\theta)$ and on the prior pdf $h(\theta)$.  The posterior $f(\theta|x)$ contains the complete probabilistic description of $\theta$ given observed data $x$. Its maximization is a standard optimality criterion in statistical decision making  \parencite{bickel-doksum2001,carlin-louis2009,degroot1970,duda-hart-stork2001}. %,hogg-mckean-craig2005,hogg-tanis2006}.

The Bayesian inference structure in (\ref{eq:TextBayesPosterior}) involves a radical abstraction.  The set or event hypothesis $H_j$ in (\ref{eq:TextBayes}) has become the measurable function or {\it random variable} $\Theta$ that takes on realizations  $\theta$ according to the prior pdf $h(\theta): \Theta \sim h(\theta)$.  The pdf $h(\theta)$ can make or break the accuracy of the posterior pdf $f(\theta|x)$  because it scales the data pdf $g(x|\theta)$ in (\ref{eq:TextBayesPosterior}).  Statisticians can elicit priors from an expert \parencite{kadane-wolfson1998,garthwaite-kadane-ohagan2005}. Such elicited priors are thus ``subjective'' because they are ultimately opinions or guesses.  Or the prior in ``empirical Bayes'' \parencite{carlin-louis2009,hogg-mckean-craig2005} can come from ``objective'' data or from statistical hypothesis tests such as chi-squared or Kolmogorov-Smirnov tests for a candidate pdf \parencite{hogg-tanis2006}.

\subsection{Conjugacy}
%\todo{Ref for theorem that every exp-pam pdf has conjugate prior???}
The prior pdf $h(\theta)$ is the most subjective part of the Bayesian inference framework. The application determines the sampling pdf $g(x|\theta)$. But the prior comes from preconceptions about the parameter $\theta$. These preconceptions could be in the form of information from experts or from collateral data about about $\theta$. It is not always easy to articulate these sources of information into accurate pdfs for $\Theta$. Thus most Bayesian applications resort to simplifications. They restrict themselves to a limited set of closed form pdfs for $\Theta$. Many applications limit themselves to an even smaller subset of pdfs called ``conjugate priors''. 

Conjugate priors produce not only closed-form posterior pdfs but posteriors that come from the same family as the prior \parencite{bickel-doksum2001,degroot1970,hogg-mckean-craig2005,raiffa-schlaifer2000}. The three most common conjugate priors in the literature are the beta, the gamma, and the normal.  Table \ref{tab:Conjugate-Priors} displays these three conjugacy relationships. The posterior $f(\theta|x)$ is beta if the prior $h(\theta)$ is beta and if the data or likelihood $g(x|\theta)$ is binomial or has a dichotomous Bernoulli structure.  The posterior is gamma if the prior is gamma and if the data is Poisson or has a counting structure.  The posterior is normal if the prior and data are normal.  

\begin{table}[htb]
\begin{center}
\scalebox{0.875}{
\begin{tabular}{|c|c||c|}
			\hline
			\textsc{Prior $h(\theta)$} & \textsc{Likelihood $g(x|\theta)$} & \textsc{ Posterior $f(\theta|x)$} \\
			\hline \hline	
			 Beta&  Binomial&  Beta$'$\\	
			 $B(\alpha, \beta)$ &  $\text{bin}(n,\theta)$ & $B(\alpha+x,~\beta+n-x)$\\	
  		 $~~\frac{\Gamma(\alpha+\beta)}{\Gamma(\alpha)\Gamma(\beta)} ~ \theta^{\alpha-1}(1-\theta)^{\beta-1}$  &
		   $~~\binom{n}{x} \theta^x (1 - \theta)^{n-x}$ & 
			 {\small $~~\frac{\Gamma(\alpha+\beta+n )}{\Gamma(\alpha+x)\Gamma(\beta+n-x)} ~ \theta^{\alpha+x-1}(1-\theta)^{\beta+n-x-1}$}\\
			\hline
			 Gamma  &  Poisson &  Gamma$'$\\
			 $\Gamma(\alpha, \beta)$ &  $p(\theta)$ &  $\Gamma(\alpha+x, \frac{\beta}{1+\beta})$\\
			 $\frac{\theta^{\alpha-1}\exp \left({-\theta/\beta}\right)}{\Gamma(\alpha)\beta^\alpha}$ &
			 $e^{-\theta}~\frac{\theta^x}{x!}$ & $~~\frac{(\theta+\theta\beta)^{\alpha+x}}{\theta~\Gamma(\alpha+x)~\beta^{\alpha+x}}  \exp \left(\frac{-\theta (1+\beta)}{\beta} \right)$\\
			\hline
			 Normal&  Normal$'$ &  Normal$''$\\
			 $N(\mu, \tau^2)$ &  $N(\theta| \sigma^2)$ &  $N\left(\frac{\mu \tau^2 + x \sigma^2}{\tau^2 + \sigma^2}, \frac{\tau^2 \sigma^2}{\tau^2 +\sigma^2}\right)$\\
			\hline			
		\end{tabular}
}
\end{center}
\caption{Conjugacy relationships in Bayesian inference.  A prior pdf of one type combines with its conjugate likelihood to produce a posterior pdf of the same type.}

\label{tab:Conjugate-Priors}
\end{table}

Consider first the beta prior on the unit interval:
\begin{equation}
\Theta \sim \beta(\alpha,\beta):~~ h(\theta) = \frac{\Gamma(\alpha+\beta)}{\Gamma(\alpha)\Gamma(\beta)}\theta^{\alpha-1}(1-\theta)^{\beta-1}
\label{eq:betapdf}
\end{equation}
if $0 < \theta < 1$ for parameters $\alpha > 0$  and $\beta > 0$.  Here $\Gamma$ is the gamma function $\Gamma(\alpha) = \int_{0}^\infty x^{\alpha-1}e^{-x} dx$.  Then $\Theta$ has population mean or expectation $E[\Theta] = \alpha/(\alpha+\beta)$. The beta pdf reduces to the uniform pdf if $\alpha = \beta = 1$. A beta prior is a natural choice when the unknown parameter $\theta$ is the success probability for binomial data such as coin flips or other Bernoulli trials because the beta's support is the unit interval (0, 1) and because the user can adjust the $\alpha$ and $\beta$ parameters to shape the beta pdf over the interval.

A beta prior is conjugate to binomial data with likelihood $g(x_1,\ldots,x_n|\theta)$.  This means that a beta prior $h(\theta)$  combines with binomial sample data to produce a new beta posterior:           
\begin{align}
f(\theta|x) = \frac{\Gamma(n+\alpha+\beta)}{\Gamma(\alpha+x)\Gamma(n+\beta-x)}\theta^{x+\alpha-1}(1-\theta)^{n-x+\beta-1}
\label{eq:betaposterior}
\end{align}
Here $x$ is the observed sum of $n$ Bernoulli trials and hence is an observed sufficient statistic for $\theta$ \parencite{hogg-tanis2006}.  So $g(x_1,\ldots,x_n|\theta) = g(x|\theta)$. This beta posterior $f(\theta|x)$ gives the mean-square optimal estimator as the conditional mean $E[\Theta|X=x] = (\alpha+x)/(\alpha+\beta+n)$ if the loss function is squared-error \parencite{hogg-tanis2006}.  A beta conjugate relation still holds when negative-binomial or geometric data replaces the binomial data or likelihood. The conjugacy result also extends to the vector case for the Dirichlet or multidimensional beta pdf.  A Dirichlet prior is conjugate to multinomial data \parencite{degroot1970,neapolitan2004}.

Gamma priors are conjugate to Poisson data. The gamma pdf generalizes many right-sided pdfs such as the exponential and chi-square pdfs.  The generalized (three-parameter) gamma further generalizes the Weibull and lognormal pdfs.  A gamma prior is right-sided and has the form
\begin{align}
\Theta \sim \gamma(\alpha,\beta):~~ h(\theta) = \frac{\theta^{\alpha-1}e^{-\theta/\beta}}{\Gamma(\alpha)\beta^\alpha}
\label{eq:gammapdf}
\quad \mbox{if $\theta > 0$.}
\end{align}
The gamma random variable $\Theta$ has population mean $E[\Theta] = \alpha\beta$ and variance $V[\Theta] = \alpha\beta^2$.  

The Poisson sample data $x_1,\ldots, x_n$ comes from the likelihood \begin{align}
g(x_1\ldots,x_n|\theta) = \frac{\theta^{x_1}e^{-\theta}}{x_1!}\cdots \frac{\theta^{x_n}e^{-\theta}}{x_n!} \;.
\end{align}
The observed Poisson sum $x = x_1 + \cdots + x_n$ is an observed sufficient statistic for $\theta$ because the Poisson pdf also comes from an exponential family \parencite{bickel-doksum2001,hogg-mckean-craig2005}. The gamma prior $h(\theta)$ combines with the Poisson likelihood $g(x|\theta)$ to produce a new gamma posterior $f(\theta|x)$ \parencite{hogg-tanis2006}:
\begin{align}
f(\theta|x) = 
\frac{\theta^{(\sum_{k=1}^n x_k+\alpha-1)}e^{-\theta/[\beta/(n\beta+1)]}}
{\Gamma(\sum_{k=1}^n x_k + \alpha)[\beta/(n\beta+1)]^{( \sum_{k=1}^n x_k+\alpha)}}.
\label{eq:gammaposterior}
\end{align}
So $E[\Theta| X = x] = (\alpha + x) \beta / (1 + \beta)$  and $V[\Theta| X = x] =  (\alpha + x) \beta^2 / (1+\beta)^2$.

A normal prior is self-conjugate because a normal prior is conjugate to normal data.  A normal prior pdf has the whole real line as its domain and has the form \parencite{hogg-tanis2006}
\begin{align}
\Theta \sim N(\theta_0,\sigma_0^2):~~ h(\theta) = \frac{1}{\sqrt{2\pi}\sigma_0}e^{-(\theta-\theta_0)^2/2\sigma_0^2}
\label{eq:normalpdf}
\end{align}
for known population mean $\theta_0$ and known population variance $\sigma_0^2$.  The normal prior $h(\theta)$ combines with normal sample data from $g(x|\theta) = N(\theta| \sigma^2/n)$ given an observed realization $x$ of the sample-mean sufficient statistic $\overline{X}_n$. This gives the normal posterior pdf $f(\theta|x) = N(\mu_n,\sigma_n^2)$.  Here $\mu_n$ is the weighted-sum conditional mean
\begin{equation}
E[\Theta|X=x] = \left(\frac{\sigma_0^2}{\sigma_0^2+\sigma^2/n}\right)x + \left(\frac{\sigma^2/n}{\sigma_0^2+\sigma^2/n}\right)\theta_0
\end{equation}
and 
\begin{equation}\sigma_n^2 = \left(\frac{\sigma^2/n}{\sigma_0^2+\sigma^2/n}\right)\sigma_0^2 \;.
\end{equation}
A hierarchical Bayes model \parencite{carlin-louis2009,hogg-mckean-craig2005} would write any of the these priors as a function of still other random variables and their pdfs.

Conjugate priors permit easy iterative or sequential Bayesian learning because the previous posterior pdf $f_{\rm old}(\theta|x)$ becomes the new prior pdf $h_{\rm new}(\theta)$ for the next experiment based on a fresh random sample: $h_{\rm new}(\theta) = f_{\rm old}(\theta|x)$. Such conjugacy relations greatly simplify iterative convergence schemes such as Gibbs sampling in Markov chain Monte Carlo estimation of posterior pdfs \parencite{carlin-louis2009,hogg-mckean-craig2005}

\section{Bayesian Point Estimation} \label{sec:B-PEstimatn}

Many statistical decision problems involve selecting an ``optimal'' point estimates for model parameters $\theta$. Bayesian inference solves the hard question about how to update beliefs about the data-model's parameters given new observed data. It produces the posterior pdf $f(\theta|x)$ which is a measure of the spread of probable values of the model parameter $\theta$ based on observed data. But how can we use this information about the parameter spread to select an``optimal'' parameter point estimate?

The answer to this question depends on the definition of an ``optimal estimate''.  Each parameter estimate $d(X)$ represents a decision. The Bayesian point of view argues \parencite{cox1958} that the concept of the ``optimal'' point estimate is incomplete if there is no consideration given to the \emph{losses} incurred by making the wrong decisions. Every parameter estimate $d(x)$ incurs a penalty proportional to how much the estimate $d(x)$ deviates from the parameter $\theta$. The loss function $\ell(d,\theta)$ models these losses. The parameter $\theta$ and thus $\ell(d,\theta)$ are random. The magnitude of the \emph{average} loss function can be a measure of estimate optimality -- higher average loss being less desirable. The average loss is the \emph{Bayes risk} $R(d(X), \theta)$ associated with the estimate $d(X)$. The posterior pdf enables us to calculate this risk subject to observed data
\begin{equation} R(d(X), \theta) =  \E_{\theta|X}  \left[ \ell(d(X),\theta) | X   \right] = \int_\Theta \ell(d(X),\theta) f(\theta|x) ~dx \;. \end{equation}
The Bayesian point estimation defines the optimal estimate for $\theta$ as one that minimizes the Bayes risk. This estimate is the \emph{Bayes estimate} $\hat{\theta}_{Bayes}(X)$:
\begin{equation}
\hat{\theta}_{Bayes}(X) = \argmin{d(X) \in \Theta} \, R( d(X), \theta) \;. \label{eq:BayesEst-defn}
\end{equation}

A \emph{utility function} $\mathfrak{u}(d,\theta)$ (with sign opposite the loss function's) can model rewards for good estimates. Then the Bayes estimate becomes a maximization of the expected utility.
\begin{align}
\hat{\theta}_{Bayes}(X) &= \argmax{d \in \Theta} \, \E_\theta \left[ \mathfrak{u}(d(X),\theta) | X \right].
\end{align}
The utility function formulation is more typical in classical decision theory \parencite{wald1950, savage1951}, game theory, and theories about rational economic choice. For example, a rational economic choice $d^*(X)$ is a choice that maximizes an economic agent's expected utility\parencite{neumann-morgenstern1947}.% function $\mathfrak{u}(d(X), \theta)$ 

Bayesian point estimation is a posterior-based variant of the statistical decision framework of Wald, Neyman, and Pearson\parencite{neyman-pearson1933, wald1950}. The basic theme is to treat inference problems (including point estimation and hypothesis testing) as special cases of decision problems \parencite{cox1958, bernardo-smith2009}.

\subsection{Bayes Estimates for Different Loss Functions}
The loss function determines the Bayes estimate. The loss function ideally mirrors estimation error penalties from the application domain. Some common loss functions in engineering applications are: the squared error, absolute error, and the 0-1 loss functions. 
The loss function and their corresponding Bayes estimates are:
\begin{table}[htb]
\begin{center}
    \begin{tabular}{ | c || c |c |}
    \hline
    Loss Function & $\ell(d,\theta)$ & $\hat{\theta}_{Bayes}(X)$  \\ \hline \hline
    squared error loss &  $c\, (d-\theta)$ & $\E_{\theta|X} \left[ \Theta | X \right]$ \\ \hline
    absolute error loss & $c\, |d-\theta|$ &  $\text{Median} \left( f(\theta | X) \right)$ \\ \hline
    0-1 loss &  $1 - \delta(d-\theta)$ & $\text{Mode}\left(f(\theta | X) \right)$ \\  \hline    
    \end{tabular}
\end{center}
\label{tb:BayesEstimators}
\caption{Three Loss Functions and Their Corresponding Bayes Estimates.}
\end{table} %\todo{Consider deriving these Bayes estimates}
The Bayes estimates are solutions the optimization problem in equation \ref{eq:BayesEst-defn}. The three loss functions (squared, absolute, and zero-one) are conceptually similar to $\ell_2$, $\ell_1$, and $\ell_0$ minimization respectively. Bayes estimation with the 0-1 loss function is equivalent to Maximum A Posteriori (MAP) estimation. The MAP estimate is the mode of the posterior pdf: \begin{equation} \hat{\theta}_{MAP} =  \argmax{\theta} f(\theta | x) = \argmax{\theta} g(x|\theta) h(\theta).  \end{equation} If the prior distribution is uniform (and possibly improper) then the MAP estimate is equivalent to the Maximum Likelihood (ML) estimate.
\begin{equation} 
\hat{\theta}_{ML} \equiv  \argmax{\theta} g(x | \theta ).
\end{equation} 
$\Theta$ has a uniform distribution. So the prior $h(\theta)$ is constant. Thus:
\begin{align} 
h(\theta) &= c \\
\argmax{\theta}  g(x|\theta) h(\theta) &= \argmax{\theta} \left[c \, g(x|\theta) \right]\\
\argmax{\theta}  g(x|\theta) h(\theta) &= \argmax{\theta}  g(x|\theta) 
\end{align}
since argument maximization is invariant under scalar multiplication. Therefore
\begin{equation}
\hat{\theta}_{MAP} =  \hat{\theta}_{ML} .
\end{equation}
This reduction from MAP to MLE is valid when $\theta$ takes values from a bounded subset of the parameter space. The same reduction holds for unbounded domains of $\theta$. This requires the use of improper prior pdfs i.e. prior pdfs that are not integrable \parencite{carlin-louis2009}. Thus MAP and ML estimation fit into the Bayesian estimation framework.

All the Bayes estimates above minimize risk functions. There is an alternative decision strategy that addresses worst-case scenarios: we can define an estimate $d^*(x)$ that minimizes the worst case risk 
\begin{equation} 
d* = \argmin{d} \{sup_\theta R(\theta,d) \}.
\end{equation}
This is the \emph{minimax} estimator. This is a deeply conservative estimator that is typically \emph{inadmissible}\footnote{A decision rule (point estimate) $d(X)$ is \emph{inadmissible} \parencite{degroot1970, bernardo-smith2009} in the statistical sense if there exists an alternate decision rule (point estimate) $d^*(X)$ with lower Bayes risk for all values of the parameter $\theta$ i.e.$R(d^*(X),\theta) \leq R(d(X), \theta)$ for all values of $\theta$ with strict inequality for some values of $\theta$.}
 having higher risk compared to \emph{any} admissible estimator\parencite{carlin-louis2009, bernardo-smith2009}. The minimax approach to rational decision makes the most sense in zero-sum game-theoretic scenarios \parencite{osborne-rubinstein1994}. 
Wald (\cite*[][pp. 24--27]{wald1950}) showed that statistical decision problems have the same form as zero-sum two-person games between Nature and the experimenter. The minimax estimators represent minimax strategies for the experimenter. However minimax estimators are often too conservative for Bayesian statistics applications.
%are not adversarial or game-theoretic in nature. Thus the minimax approach is of limited use in inferential statistics. %Wald1950 gives an interpretation of Stat. Dec. theory as a Zero-sum 2-player game... Rethink closing remark?

\subsection{Measures of Uncertainty for Bayes Estimates}
Point estimates need appropriate measures of uncertainty. The full posterior pdf is the most complete Bayesian description of uncertainty about the parameter $\theta$. We can also specify more succinct measures of parameter variability depending on the type of Bayes estimate in use. Such measures lack the full generality of the posterior but they are simpler to use for fixed loss functions. The conditional variance $Var_\theta[\Theta | X] $ measures variability around the conditional mean Bayes estimate $\E_\theta \left[ \Theta | X \right]$. The inter-quartile range measures variability around the median Bayes estimate $\text{Median} \left\{ f(\theta | X) \right\}$. The \emph{highest posterior density (HPD) credible interval} \parencite{bernardo-smith2009, carlin-louis2009} measures variability around the mode Bayes estimate $\text{Mode}\left\{f(\theta | X) \right\}$.

The credible interval is most akin to the more familiar confidence interval $CI(\alpha)$ in frequentist statistical inference. The credible and confidence intervals are both subsets of the parameter space $\Theta$ that highlight the characteristic spread of the point estimate $\hat{\theta}=d(X)$. The $(1-\alpha)$-level confidence interval is the \emph{random set} $CI(\alpha)$ specified by the test statistic $\hat{\theta}$ such that 
\begin{equation}
(1-\alpha) = P\left(\theta \in CI(\alpha)\right).
\end{equation} 
While a $1-\alpha$ credible intervals is a connected set  $\mathcal{C}(\alpha)$ such that 
\begin{equation}
	(1-\alpha) = P\left(\theta \in \mathcal{C}(\alpha) | x\right) = \int_{\mathcal{C}(\alpha)} f\left(\theta|x\right) \;. 
\end{equation}
The key difference is that the confidence interval measures probabilities via a distribution on random sets with $\theta$ constant but unknown. While the credible interval measures probabilities via a posterior distribution on the random parameter $\theta$. The two intervals have different interpretations and are generally not equivalent except under very special conditions \parencite{welch-peers1963, peers1965, welch1965, peers1968}.

Credible intervals are not unique. Any Bayes estimate can belong to a continuum of different credible intervals. But the HPD credible interval is optimal \parencite{box-tiao1965} in the sense that it is the minimum-volume credible interval and it always contains the posterior mode.

\section{Conclusion}
This chapter highlighted the differences between the frequentist and Bayesian approach to statistical inference and point estimation in particular. It also showed how the Bayesian framework subsumes frequentist ML point estimation in theory. The rest of this work addresses some issues with the Bayesian inference framework. The subsumption of MLE methods under Bayesian methods implies that these issues are also relevant to ML point estimation.

The exposition so far assumes that the data model functions (the priors pdfs and likelihood functions) are accurate. The Bayesian inference framework works well when this assumption is true. What happens when this assumption fails? Is Bayesian inference robust to corruption caused by incorrect model functions? The next chapter addresses these questions.

\clearpage

%\subfile{./Chapters/chapBAT}
\chapter{Bayesian Inference with Fuzzy Function Approximators}\label{ch:BAT}

\blfootnote{
	This chapter features work done in collaboration with Prof. Sanya Mitaim and first presented in \parencite{osoba-mitaim-kosko-SMC2011,osoba-mitaim-kosko2009}.
}

The key assumptions in the Bayesian inference scheme are: (1) models for the observable data are accurate and (2) the source of confusion is the randomness of the data and the model parameters. Many applications of Bayesian inference involve inaccurate data models and possibly other forms of model corruption. This raises the questions: how reliable are Bayesian statistical estimates when the analytic data model does not match the true data model? Are these estimates useful when we only have approximations of the data model? This chapter addresses these questions by analyzing the effect of approximate model functions on Bayes theorem. 

The main result of this analysis is the \emph{Bayesian Approximation Theorem} (BAT) for posterior pdfs: applying Bayes theorem to separate uniform approximators for the prior pdf and the likelihood function (the model functions) results in a uniform approximator for the posterior pdf. This theorem guarantees that good model function approximators produce good approximate posterior pdf. The BAT also applies to any type of uniform function approximator.

We demonstrate this result with fuzzy rule-based uniform approximators for the model functions.  Fuzzy approximation techniques have two main advantages over other approximation techniques for Bayesian inference. First, they allow users to express prior or likelihood descriptions in words rather than as closed-form probability density functions. Learning algorithms can tune approximators based on expert linguistic rules or just grow them from sample data. Second, they can \emph{represent} any bounded closed-form model function exactly.  Furthermore, the learning laws and fuzzy approximators have a tractable form because of the convex-sum structure of additive fuzzy systems. This convex-sum structure carries over to the fuzzy posterior approximator (see Theorem \ref{thm:Doubly-Posterior-SAM}). We also show that fuzzy approximators are robust to noise in the data (see figure \ref{fig:EmpiricalCasesCombined}).

Simulations demonstrate this fuzzy approximation scheme on the priors and posteriors for the three most common conjugate models (see Figures \ref{fig:BetaASAM}-\ref{fig:NormalASAM}): the beta-binomial, gamma-Poisson and normal-normal conjugate models. Fuzzy approximators can also approximate non-conjugate priors and likelihoods as well as approximate hyperpriors in hierarchical Bayesian inference. We later extend this approimation scheme to more general hierarchical Bayesian models in Chapter \ref{ch:ExBAT}. Most of the approximation qualities carry over to this more general case.

We use the notation from Chapter \ref{ch:Bayes}: $f(\theta|x)$ is the posterior pdf. $f(\theta|x)$ is the result of the application of Bayes theorem to the prior $h(\theta)$ and the likelihood $g(x|\theta)$:
\begin{align}
f(\theta|x) ~=~ \frac{g(x|\theta)h(\theta)}{\int g(x|u)h(u)du} ~\propto~ g(x|\theta)h(\theta)
\label{eq:BayesPosterior}
\end{align}
The probabilistic graphical model (PGM) in Figure \ref{fig:BAT-PGM} represents this data model succinctly.
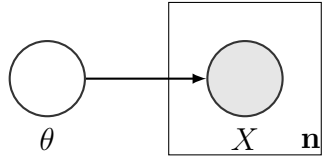
\begin{figure}
	\centering
	\begin{tikzpicture}
	\tikzstyle{main}=[circle, minimum size = 10mm, thick, draw =black!80, node distance = 16mm]
	\tikzstyle{connect}=[-latex, thick]
	\tikzstyle{box}=[rectangle, draw=black!100]
	  \node[main] (theta) [label=below:$\theta$] { };
	  \node[main, fill = black!10] (x) [right=of theta, label=below:$X$] { };
	  \path (theta) edge [connect] (x);
	  \node[rectangle, inner sep=0mm, fit=(x),label=below right:$\mathbf{n}$, xshift=0.75mm,yshift=-0.75mm] {};
	  \node[rectangle, inner sep=5mm,draw=black!100, fit= (x)] {};
	\end{tikzpicture}
	\caption[Probabilistic graphical model for Bayesian data models in Chapter \ref{ch:BAT}.]{
		Probabilistic graphical model for all Bayesian data models in this chapter. We observe $\mathbf{n}$ samples of the data $X$ which depends on a hidden random parameter $\theta$. The likelihood function $g(x|\theta)$ captures this dependence. The prior $h(\theta)$ describes the distribution of the hidden parameter $\theta$.
	}
\label{fig:BAT-PGM}
\end{figure}

\section{Bayesian Inference with Fuzzy Systems}\label{sec:1}

Additive fuzzy systems can extend Bayesian inference because they allow users to express prior or likelihood knowledge in the form of if-then rules.   Fuzzy systems can approximate prior or likelihood probability density functions (pdfs) and thereby approximate posterior pdfs.  This allows a user to describe priors with fuzzy if-then rules rather than with closed-form pdfs.  The user can also train the fuzzy system with collateral data to adaptively grow or tune the fuzzy rules and thus to approximate the prior or likelihood.  A simple two-rule system can also exactly represent a bounded prior pdf if such a closed-form pdf is available. So fuzzy rules extend the range of knowledge that prior or likelihood can capture and they do so in an expressive linguistic framework based on multivalued or fuzzy sets \parencite{zadeh1965}.

Figure 1 shows how five tuned fuzzy rules approximate the skewed beta prior pdf $\beta(8, 5)$.  Learning has sculpted the five if-part and then-part fuzzy sets so that the approximation is almost exact.  Users will not in general have access to such training data because they do not know the functional form of the prior pdf.  They can instead use any noisy sample data at hand or just state simple rules of thumb in terms of fuzzy sets and thus implicitly define a fuzzy system approximator $F$.  The following prior rules define such an implied skewed prior that maps fuzzy-set descriptions of the parameter random variable $\Theta$ to fuzzy descriptions $F(\Theta)$ of the occurrence probability:

\vspace{4pt}

\hspace{0.5in} 
Rule 1:  If $\Theta$ is {\it much smaller} than $\frac{1}{2}$   \,then  $F(\Theta)$ is {\it very small}
            
\hspace{0.5in} 
Rule 2:  If $\Theta$ is {\it smaller} than $\frac{1}{2}$  \,then $F(\Theta)$ is {\it small}
            
\hspace{0.5in} 
Rule 3:  If $\Theta$ is {\it approximately} $\frac{1}{2}$  then  $F(\Theta)$ is {\it large}

\hspace{0.5in} 
Rule 4:  If $\Theta$ is {\it larger} than $\frac{1}{2}$  then $F(\Theta)$ is {\it medium}
            
\hspace{0.5in} 
Rule 5:  If $\Theta$ is {\it much larger} than $\frac{1}{2}$  then $F(\Theta)$ is {\it small}

\vspace{4pt}

\noindent Learning shifts and scales the Cauchy bell curves that define the if-part fuzzy sets in Figure 1.  The tuned bell curve in the third rule has shifted far to the right of the equi-probable value $\frac{1}{2}$.  Different prior rules and fuzzy sets will define different priors just as will different sets of sample data.  The simulations results in Figures \ref{fig:BetaASAM}--\ref{fig:NormalASAM} show that such fuzzy rules can quickly learn an implicit prior if the fuzzy system has access to data that reflects the prior.  These simulations give probative evidence that an informed expert can use fuzzy sets to express reasonably accurate priors in Bayesian inference even when no training data is available. The uniform fuzzy approximation theorem in \parencite{kosko-fat1994,kosko-fuzeng} gives a theoretical basis for such rule-based approximations of priors or likelihoods.  Theorem 2 below further shows that such uniform fuzzy approximation of priors or likelihoods leads in general to the uniform fuzzy approximation of the corresponding Bayesian posterior.

\renewcommand{\baselinestretch}{1.0}
\begin{figure}[thb]
\centerline{
	\includegraphics[width=0.75\textwidth]{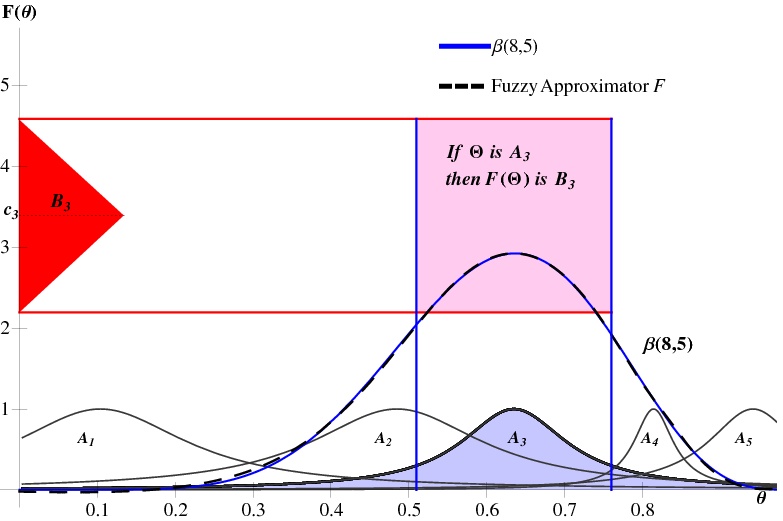}
}
\caption[Five fuzzy if-then rules approximate the beta prior $h(\theta) = \beta(8, 5)$]{
	Five fuzzy if-then rules approximate the beta prior $h(\theta) = \beta(8, 5)$.  The five if-part fuzzy sets are truncated Cauchy bell curves.  An adaptive Cauchy SAM (standard additive model) fuzzy system tuned the sets' location and dispersion parameters to give a nearly exact approximation of the beta prior.  Each fuzzy rule defines a patch or 3-D surface above the input-output planar state space.  The third rule has the form ``If $\Theta = A_3$ then $B_3$'' where then-part set $B_3$ is a fuzzy number centered at centroid $c_3$.  This rule might have the linguistic form ``If $\Theta$ is {\it approximately} $\frac{1}{2}$ then $F(\Theta)$ is {\it large}.'' The training data came from 500 uniform samples of $\beta(8, 5)$.  The adaptive fuzzy system cycled through each training sample 6,000 times.  The fuzzy approximator converged in fewer than 200 iterations.  The adaptive system also tuned the centroids and areas of all five then-part sets (not pictured).    
}
\label{fg:1}
\end{figure}

Bayesian inference itself has a key strength and a key weakness.  The key strength is that it computes the posterior pdf $f(\theta|x)$ of a parameter $\theta$ given the observed data $x$.  The posterior pdf gives all probabilistic information about the parameter given the available evidence.  The key weakness is that this process requires that the user produce a prior pdf $h(\theta)$ that describes the unknown parameter.  The prior pdf can inject ``subjective'' information into the inference process because it can be little more than a guess from the user or from some consulted expert or other source of authority.  Priors can also capture ``objective'' information from a collateral source of data.

Additive fuzzy systems use if-then rules to map inputs to outputs and thus to model priors or likelihoods.  A fuzzy system with enough rules can uniformly approximate any continuous function on a compact domain.  Statistical learning algorithms can grow rules from unsupervised clusters in the input-output data or from supervised gradient descent.  Fuzzy systems also allow users to add or delete knowledge by simply adding or deleting if-then rules.  So they can directly model prior pdfs and approximate them from sample data if it is available.  Inverse algorithms can likewise find fuzzy rules that maximize the posterior pdf or functionals based on it.  These applications of adaptive fuzzy approximators to Bayesian inference do not involve unrelated efforts to {\it fuzzify} Bayes Theorem  \parencite{kosko-fuzzyentropy,terano-asai-sugeno1987}. The use of adaptive fuzzy systems allows for more accurate prior pdf and likelihood function estimation thus improving the versatility and accuracy of {\it classical} Bayesian applications.

\textbf{\S}\ref{sec:AdaptFAT} reviews the theory behind fuzzy function approximation. We show this fuzzy approximation scheme with the three well-known conjugate priors and their corresponding posterior approximations (Figures \ref{fig:BetaASAM}, \ref{fig:GammaASAM}, and \ref{fig:NormalASAM}). The scheme works well even with non-conjugate data models (Figures \ref{fig:NonConjugatePrior} and \ref{fig:NonConjugate-Posterior-x=6}). \textbf{\S}\ref{sec:doubly-fuzzy} further extends the fuzzy approach to doubly fuzzy Bayesian inference where separate fuzzy systems approximate the prior and the likelihood. This section also states and proves what we call the Bayesian Approximation Theorem:  Uniform approximation of the prior and likelihood results in uniform approximation of the posterior.

\section{Adaptive Fuzzy Function Approximation}\label{sec:AdaptFAT}

Additive fuzzy systems can uniformly approximate continuous functions on compact sets \parencite{kosko-nnfs,kosko-fat1994,kosko-fuzeng}.  Hence the set of additive fuzzy systems is dense in the space of such functions.  A scalar fuzzy system is the map $F: R^n \to R$ that stores $m$ if-then rules and maps vector inputs $x$ to scalar outputs $F(x)$. The prior and likelihood simulations below map not $R^n$ but a compact real interval $[a,~b]$ into reals.  So these systems also satisfy the approximation theorem but at the expense of truncating the domain of pdfs such as the gamma and the normal. Truncation still leaves a proper posterior pdf through the normalization in (\ref{eq:BayesPosterior}).

\subsection{SAM Fuzzy Systems}

A standard additive model (SAM) fuzzy system computes the output $F(x)$ by taking the centroid of the sum of the ``fired'' or scaled then-part sets:  
$F(x) = Centroid(w_1 a_1(x) B_1 + \cdots + w_m a_m(x) B_m)$.  Then the SAM Theorem states that the output $F(x)$ is a simple convex-weighted sum of the then-part set centroids $c_j$ \parencite{kosko-nnfs,kosko-fat1994,kosko-fuzeng,mitaim-kosko2001ASAM}:  
\begin{align}
F(x) &= 
\frac{\sum_{j=1}^m w_j a_j(x) V_j c_j}
{\sum_{j=1}^m w_j a_j(x) V_j}
~=~ \displaystyle \sum_{j=1}^m p_j(x)c_j.
\label{eq:SAM}
\end{align}
Here $V_j$ is the finite area of then-part set $B_j$ in the rule ``If $X = A_j$ then $Y = B_j$'' and $c_j$ is the centroid of $B_j$.  The convex weights $p_1(x),\ldots,p_m(x)$ have the form $p_j(x) = \frac{w_j a_j(x) V_j}{\sum_{i=1}^m w_i a_i(x) V_i}$.  The convex coefficients $p_j(x)$ change with each input $x$.   The positive rule weights $w_j$ give the relative importance of the $j$th rule.  They drop out in our case because they are all equal.  

The scalar set function $a_j: R \to [0, 1]$ measures the degree to which input $x \in R$ belongs to the fuzzy or multivalued set $A_j$:  $a_j(x) = Degree(x \in A_j)$.  The sinc set functions below map into the augmented range $[-.217,\,1]$ and so require some care in simulations.  The fuzzy membership value $a_j(x)$ ``fires'' the rule ``If $X = A_j$ then $Y = B_j$'' in a SAM by scaling the then-part set $B_j$ to give $a_j(x) B_j$.  The if-part sets can in theory have any shape but in practice they are parametrized pdf-like sets such as those we use below: sinc, Gaussian, triangle, Cauchy, Laplace, and generalized hyperbolic tangent.  The if-part sets control the function approximation and involve the most computation in adaptation.  Extensive simulations in \parencite{mitaim-kosko2001ASAM} show that the sinc function (in 1-D and 2-D) tends to perform best among all six sets in terms of sum of squared approximation error.  Users define a fuzzy system by giving the $m$ corresponding pairs of if-part $A_j$ and then-part $B_j$ fuzzy sets.   Many fuzzy systems in practice work with simple then-part fuzzy sets such as congruent triangles or rectangles.

SAMs define ``model-free'' statistical estimators in the following sense \parencite{kosko-fuzeng,lee-kosko-anderson2005,mitaim-kosko2001ASAM}:
\begin{align}
E[Y | X = x]  &=   F(x)  =  \sum_{j=1}^m p_j(x) c_j
\label{eq:SAMcondMean}\\
V[Y | X = x]  &= \sum_{j=1}^m p_j(x)\sigma_{B_j}^2 +  \sum_{j=1}^m p_j(x)[c_j-F(x)]^2.
\label{eq:SAMcondVar}
\end{align}
The then-part set variance $\sigma_{B_j}^2$ is $\sigma_{B_j}^2 = \int_{-\infty}^\infty (y-c_j)^2p_{B_j}(y)dy$. Then $p_{B_j}(y) = b_j(y)/V_j$ is an integrable pdf if $b_j: R \to [0,~1]$ is the integrable set function of then-part set $B_j$.  The conditional variance $V[Y|X=x]$ gives a direct measure of the uncertainty in the SAM output $F(x)$ based on the inherent uncertainty in the stored then-part rules.  This defines a type of confidence surface for the fuzzy system \parencite{lee-kosko-anderson2005}.  The first term in the conditional variance (\ref{eq:SAMcondVar}) measures the inherent uncertainty in the then-part sets given the current rule firings.  The second term is an interpolation penalty because the rule ``patches'' $A_j \times B_j$  cover different regions of the input-output product space. The shape of the then-part sets affects the conditional variance of the fuzzy system but affects the output $F(x)$ only to the extent that the then-part sets $B_j$ have different centroids $c_j$ or areas $V_j$.  The adaptive function approximations below tune only these two parameters of each then-part set.  The conditional mean (\ref{eq:SAMcondMean}) and variance (\ref{eq:SAMcondVar}) depend on the realization $X = x$ and so generalize the corresponding unconditional mean and variance of mixture densities \parencite{hogg-mckean-craig2005}.

\begin{figure}[h]
\centerline{\includegraphics[width=3in]{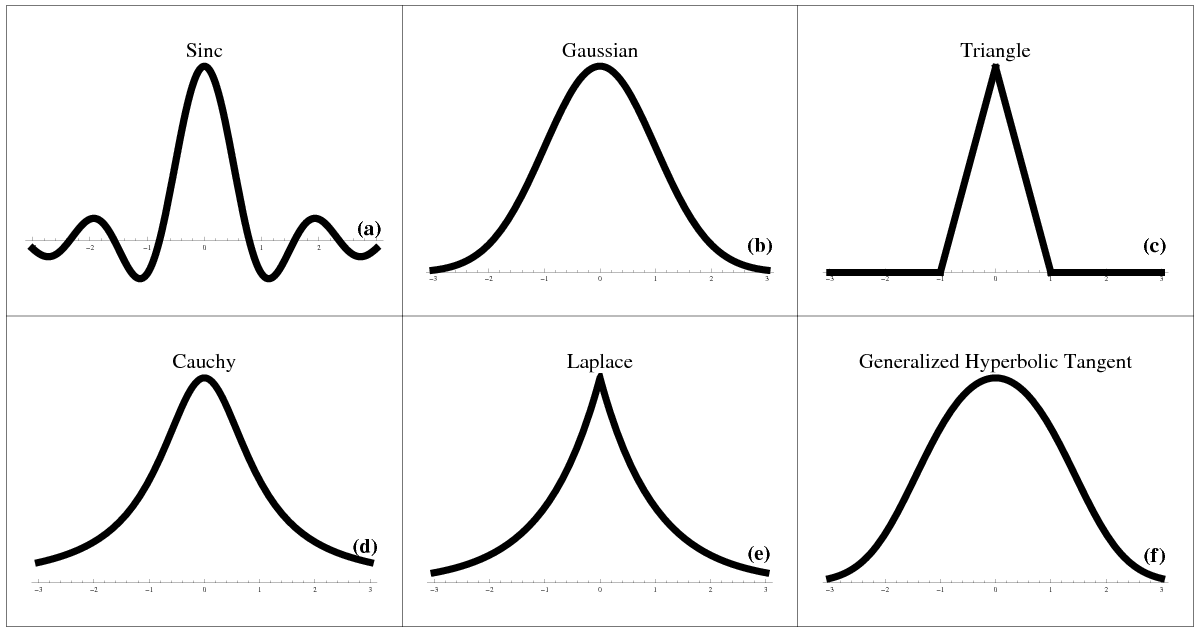} }
\caption[Six types of if-part fuzzy sets in conjugate prior approximations]{
Six types of if-part fuzzy sets in conjugate prior approximations.  Each type of set produces its own adaptive SAM learning law for tuning its location and dispersion parameters:  (a) sinc set, (b) Gaussian set, (c) triangle set, (d) Cauchy set, (e) Laplace set, and (f) a generalized hyperbolic-tangent set. The sinc shape performed best in most approximations of conjugate priors and the corresponding fuzzy-based posteriors.}
\label{fg:2}
\end{figure}

A SAM fuzzy system $F$ can always approximate a function $f$ or $F \approx f$ if the fuzzy system contains enough rules.  But multidimensional fuzzy systems $F:R^n \to R$  suffer exponential rule explosion in general \parencite{jin2000,Kosko1995opt}.  Optimal rules tend to reside at the extrema or turning points of the approximand $f$ and so optimal fuzzy rules ``patch the bumps'' \parencite{Kosko1995opt}.  Learning tends to quickly move rules to these extrema and to fill in with extra rules between the extremum-covering rules.  The supervised learning algorithms can  involve extensive computation in higher dimensions \parencite{mitaim-kosko1998Agent,mitaim-kosko2001ASAM}.  Our fuzzy prior approximations did not need many rules or extensive computation time because the fuzzy systems were 1-dimensional $( R \to R )$.  But iterative Bayesian inference can produce its own rule explosion (chapter \ref{ch:ExBAT}, \parencite{osoba-mitaim-kosko-FODM2012}).

\subsection{The Watkins Representation Theorem}
Fuzzy systems can exactly represent a bounded pdf with a known closed form. Watkins has shown that in many cases a SAM system $F$ can exactly represent a function $f$ in the sense that $F=f$. The Watkins Representation Theorem \parencite{watkins1995, watkins1994} states that  $F=f$ if $f$ is bounded and if we know the closed form of $f$. The results is stronger that this because the SAM system $F$ exactly represents $f$ with just \emph{two} rules with equal weights $w_1=w_2$ and equal then-part set volumes $V_1=V_2$:
\begin{align}
F(x)&=\frac{\sum^2_{j=1} w_j a_j (x) V_j c_j}{\sum^2_{j=1} w_j a_j (x) V_j} \\
&=\frac{ a_1 (x) c_1 +a_2(x) c_2}{a_1 (x) + a_2(x)} \\
&= f(x)
\end{align}
if $a_1(x)=\frac{\sup f - f(x)}{\sup f - \inf f}$, $a_2(x)=1 - a_1(x) = \frac{f(x) - \inf f}{\sup f - \inf f}$, $c_1=\inf f$, $c_2 =\sup  f$.

The representation technique builds $f$ directly into the structure of the two if-then rules. Let $h(\theta)$ be any bounded prior pdf such as the $\beta(8,5)$ pdf in the simulations below. Then $F(\theta)=h(\theta)$ holds for the all realizations of $\theta$ if the SAM's two rules have the form ``If $\Theta=A$ then $Y=B_1$''  and ``If $\Theta=\textrm{not-}A$ then $Y=B_2$'' for the if-part set function
\begin{equation}
a(\theta)= \frac{\sup h - h(x)}{\sup h -\inf h} = 1-\frac{11^{11}}{7^7 4^4}\theta^7 (1-\theta)^4
\end{equation}
The not-$A$ if-part set function is $1-a(\theta) = \frac{11^{11}}{7^7 4^4}\theta^7 (1-\theta)^4$. Then-part sets $B_1$ and $B_2$ can have any shape from rectangles to Gaussians so long as $0<V_1=V_2<\infty$ with centroids $c_1 = \inf h = 0$ and $c_2 = \sup h=\frac{\Gamma(13)}{\Gamma(8) \Gamma(5)}(\frac{7}{11})^7 (\frac{4}{11})^4$. So the Watkins Representation Theorem lets a SAM fuzzy system directly absorb a closed-form bounded prior $h(\theta)$ if it is available. The same holds for a bounded likelihood or posterior pdf.

\subsection{ASAM Learning Laws}\label{sec:asam-laws}

An adaptive SAM (ASAM) $F$ can quickly approximate a prior $h(\theta)$ (or likelihood) if the following supervised learning laws have access to adequate samples $h(\theta_1), h(\theta_2),\ldots$ from the prior.  This may mean in practice that the ASAM trains on the same numerical data that a user would use to conduct a chi-squared or Kolmogorov-Smirnov hypothesis test for a candidate pdf. Figure \ref{fig:EmpiricalCasesCombined} shows that an ASAM can learn the prior pdf even from noisy random samples drawn from the pdf.  Unsupervised clustering techniques can also train an ASAM if there is sufficient cluster data \parencite{kosko-nnfs,kosko-fuzeng,xu-wunsch2009}.  The ASAM prior simulations in the next section show how $F$ approximates $h(\theta)$ when the ASAM trains on random samples from the prior.  These approximations bolster the case that ASAMs will in practice learn the appropriate prior that corresponds to the available collateral data.

\begin{figure}[!ht]
\centerline{\includegraphics[width=\textwidth]{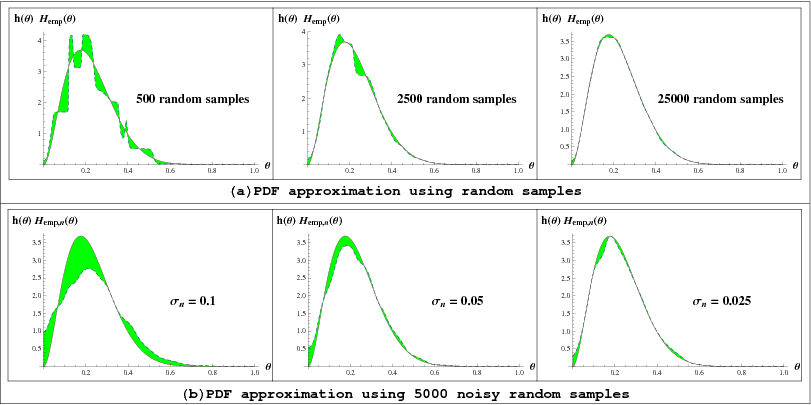} }
\caption[ASAMs can use a limited number of random samples or noisy random samples to estimate the sampling pdf]{ASAMs can use a limited number of random samples or noisy random samples to estimate the sampling pdf.  The ASAMs for these examples use the tanh set function with 15 rules and they run for 6000 iterations. The ASAMs approximate empirical pdfs from the different sets of random samples. The shaded regions represent the approximation error between the ASAM estimate and the sampling pdf.	Part (a) compares the $\beta(3,10.4)$ pdf with ASAM approximations for some $\beta(3,10.4)$ empirical pdfs. Each empirical pdf is a scaled histogram for a set of $N$ random samples. The figure shows comparisons for the cases $N=500,\, 2500,\, 25000$. 	Part (b) compares the $\beta(3,10.4)$ pdf with ASAM approximations of 3 $\beta(3,10.4)$ random sample sets corrupted by independent noise. Each set has $5000$ random samples. The noise is zero-mean additive white Gaussian noise. The standard deviations $\sigma_n$ of the additive noise are $0.1$, $0.05$, and $0.025$. The plots show that the ASAM estimate gets better as the number of samples increases. The ASAM has difficulty estimating tail probabilities when the additive noise variance gets large.}
\label{fig:EmpiricalCasesCombined}
\end{figure}

ASAM supervised learning uses gradient descent to tune the parameters of the set functions $a_j$ as well as the then-part areas $V_j$ (and weights $w_j$) and centroids $c_j$.  The learning laws follow from the SAM's convex-sum structure (8) and the chain-rule decomposition $\frac{\partial E}{\partial m_j} = \frac{\partial E}{\partial F} \frac{\partial F}{\partial a_j} \frac{\partial a_j}{\partial m_j}$  for SAM parameter $m_j$ and error $E$ in the generic gradient-descent algorithm \parencite{kosko-fuzeng,mitaim-kosko2001ASAM}
\begin{equation}
m_j(t+1) = m_j(t) - \mu_t\frac{\partial E}{\partial m_j}
\end{equation}
where $\mu_t$ is a learning rate at iteration $t$.  We seek to minimize the squared error
\begin{equation}
E(\theta) = \frac{1}{2}(f(\theta)-F(\theta))^2 = \frac{1}{2}\varepsilon(\theta)^2
\end{equation}
of the function approximation.  Let $m_j$ denote any parameter in the set function $a_j$.  Then the chain rule gives the gradient of the error function with respect to the respective if-part set parameter $m_j$, the centroid $c_j$, and the volume $V_j$ :
\begin{align}
\frac{\partial E}{\partial m_j} = \frac{\partial E}{\partial F}
\frac{\partial F}{\partial a_j}\frac{\partial a_j}{\partial m_j}
\label{eq:dajdmj}, ~~~
\frac{\partial E}{\partial c_j} = \frac{\partial E}{\partial F}
\frac{\partial F}{\partial c_j},
~~~\mbox{and}~~~
\frac{\partial E}{\partial V_j} = \frac{\partial E}{\partial F}
\frac{\partial F}{\partial V_j}
\end{align}
with partial derivatives \parencite{kosko-fuzeng,mitaim-kosko2001ASAM}
\begin{align}
\frac{\partial E}{\partial F} = -(f(\theta)-F(\theta)) = -\varepsilon(\theta)~~~\mbox{and}~~~
\frac{\partial F}{\partial a_j} = [c_j-F(\theta)]\frac{p_j(\theta)}{a_j(\theta)}.
\end{align}
The SAM ratio (\ref{eq:SAM}) with equal rule weights $w_1 = \cdots = w_m$ gives
\parencite{kosko-fuzeng,mitaim-kosko2001ASAM}
\begin{equation}
\frac{\partial F}{\partial c_j} = \frac{a_j(\theta)V_j}{\sum_{i=1}^m a_i(\theta)V_i} = p_j(\theta)
\end{equation}
\begin{equation}
\frac{\partial F}{\partial V_j} = \frac{a_j(\theta)[c_j-F(\theta)]}{\sum_{i=1}^m a_i(x)V_i} = [c_j-F(\theta)]\frac{p_j(\theta)}{V_j}.
\end{equation}
Then the learning laws for the then-part set centroids $c_j$ and volume $V_j$ have the final form
\begin{align}
c_j(t+1) &= c_j(t) + \mu_t\varepsilon(\theta)p_j(\theta)\\
V_j(t+1) &= V_j(t) + \mu_t\varepsilon(\theta)[c_j-F(\theta)]\frac{p_j(\theta)}{V_j}.
\end{align}
The learning laws for the if-part set parameters follow in like manner by expanding
$\frac{\partial a_j}{\partial m_j}$ in (\ref{eq:dajdmj}).

\begin{figure}[!ht]
\centerline{ \includegraphics[height=4in]{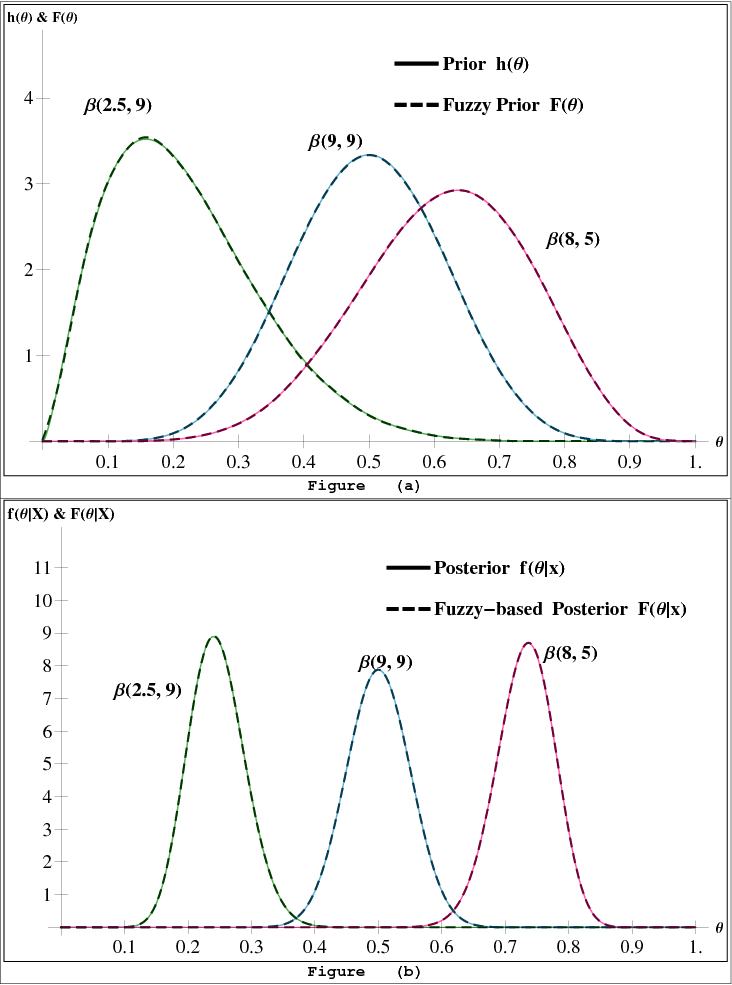} }
\caption[Comparison of conjugate beta priors and posteriors with their fuzzy approximators]{Comparison of conjugate beta priors and posteriors with their fuzzy approximators.  (a) an adapted sinc-SAM fuzzy system $F(\theta)$ with 15 rules approximates the three conjugate beta priors $h(\theta)$: $\beta(2.5 , 9)$, $\beta(9 , 9)$, and $\beta(8 , 5)$.  (b) the sinc-SAM fuzzy priors $F(\theta)$ in (a) produce the SAM-based approximators $F(\theta|x)$ of the three corresponding beta posteriors $f(\theta|x)$ for the three corresponding binomial likelihood $g(x|\theta)$ with $n = 80$:  $bin(20, 80)$, $bin(40, 80)$, and $bin(60, 80)$ where $g(x|\theta)= bin(x, 80) = \frac{80!}{x!(80-x)!} \theta^x (1 - \theta)^{80-x}$. So $X \sim bin(x, 80)$ and $X = 20$ mean that there were 20 successes out of 80 trials in an experiment where the probability of success was $\theta$.  Each of the three fuzzy approximations cycled 6,000 times through 500 uniform training samples from the corresponding beta priors.
}
\label{fig:BetaASAM}
\end{figure}

\begin{figure}[!ht]
\centerline{\includegraphics[height=4.5in]{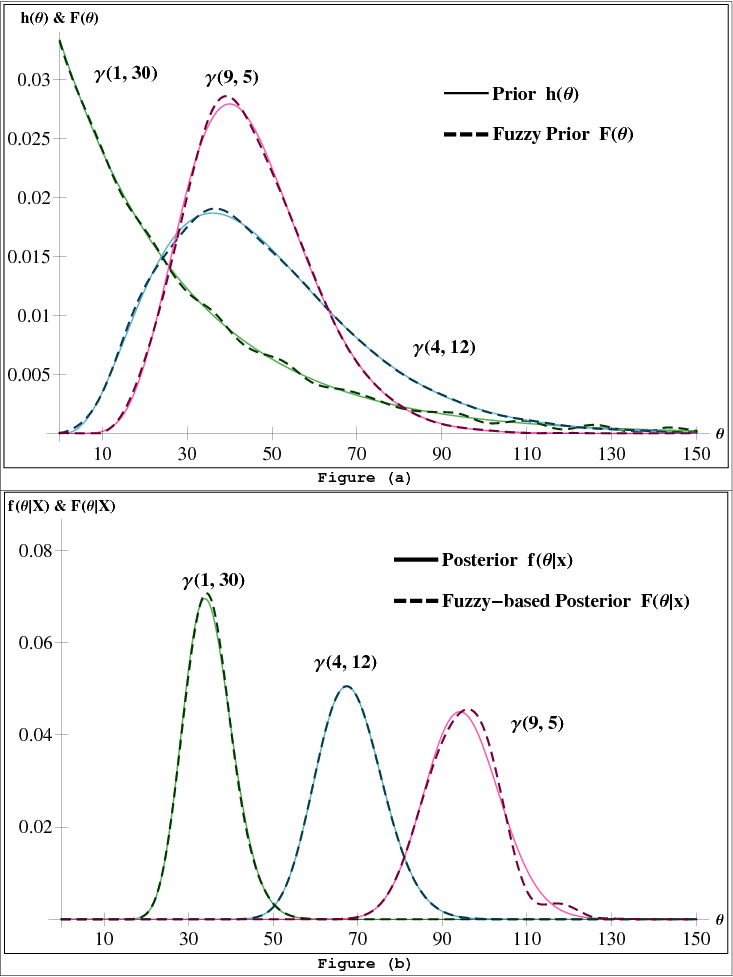} }
\caption[Comparison of conjugate gamma priors and posteriors with their fuzzy approximators]{Comparison of conjugate gamma priors and posteriors with their fuzzy approximators.  (a) an adapted sinc-SAM fuzzy system $F(\theta)$ with 15 rules approximates the three conjugate gamma priors $h(\theta)$: $\gamma(1 , 30)$, $\gamma(4, 12)$, and $\gamma(9 , 5)$.  (b) the sinc-SAM fuzzy priors $F(\theta)$ in (a) produce the SAM-based approximators $F(\theta|x)$ of the three corresponding gamma posteriors $f(\theta|x)$ for the three corresponding Poisson likelihoods $g(x|\theta)$:  $p(35)$, $p(70)$, and $p(105)$ where $g(x|\theta) =  p(x) =  \theta^x  e^{-\theta} / x!$.  Each of the three fuzzy approximations cycled 6,000 times through 1,125 uniform training samples from the corresponding gamma priors.}
\label{fig:GammaASAM}
\end{figure}

\begin{figure}[!ht]
\centerline{
\includegraphics[width=4in]{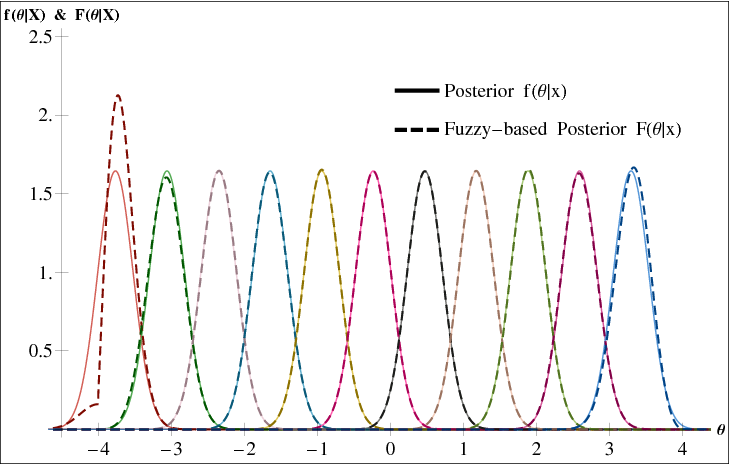}
}
\caption[Comparison of $11$ conjugate normal posteriors with their fuzzy-based approximators based on a standard normal prior and $11$ different normal likelihoods]{Comparison of 11 conjugate normal posteriors with their fuzzy-based approximators based on a standard normal prior and 11 different normal likelihoods. An adapted sinc-SAM approximator with 15 rules first approximates the standard normal prior $h(\theta) = N(0, 1)$ and then combines with the likelihood $g(x|\theta) = N(\theta, \sigma^2=\frac{1}{16})$.  The variance is $1/16$ because $x$ is the observed sample mean of 16 standard-normal random samples $X_k \sim N(0, 1)$.  The 11 priors correspond to the 11 likelihoods $g(x|\theta)$ with $x = -4$, $-3.25$, $-2.5$, $-1.75$, $-1$, $-0.25$, 0.5, 1.25, 2, 2.75, and 3.5. The fuzzy approximation cycled 6,000 times through 500 uniform training samples from the standard-normal prior.
}
\label{fig:NormalASAM}
\end{figure}

The simulations below tune the location $m_j$ and dispersion $d_j$ parameters of the if-part set functions $a_j$ for sinc, Gaussian, triangle, Cauchy, Laplace, and generalized hyperbolic tangent if-part sets.  Figure \ref{fg:2} shows an example of each of these six fuzzy sets with the following learning laws.

\subsubsection{Sinc ASAM learning law}
The sinc set function $a_j$ has the form 
\begin{equation}
a_j(\theta) = \sin\left(\frac{\theta-m_j}{d_j}\right)\Big/\left(\frac{\theta-m_j}{d_j}\right)
\end{equation}
with parameter learning laws \parencite{kosko-fuzeng,mitaim-kosko2001ASAM}
\begin{align}
m_j(t+1) &= m_j(t)+\mu_t\varepsilon(\theta)[c_j-F(\theta)]\frac{p_j(\theta)}{a_j(\theta)}\left(a_j(\theta)-\cos\left(\frac{\theta-m_j}{d_j}\right)\right)\frac{1}{\theta-m_j} \nonumber \\
d_j(t+1) &= d_j(t)+\mu_t\varepsilon(\theta)[c_j-F(\theta)]\frac{p_j(\theta)}{a_j(\theta)}\left(a_j(\theta)-\cos\left(\frac{\theta-m_j}{d_j}\right)\right)\frac{1}{d_j}. \nonumber
\end{align}

\subsubsection{Gaussian ASAM learning law} 
The Gaussian set function $a_j$ has the form
\begin{equation}
a_j(\theta) = \exp\left\{-\left(\frac{\theta-m_j}{d_j}\right)^2\right\}
\end{equation}
with parameter learning laws
\begin{align}
m_j(t+1) &= m_j(t)+\mu_t\varepsilon(\theta)p_j(\theta)[c_j-F(\theta)]\frac{\theta-m_j}{d_j^2}
\label{eq:GaussLearnmj}\\
d_j(t+1) &= d_j(t)+\mu_t\varepsilon(\theta)p_j(\theta)[c_j-F(\theta)]\frac{(\theta-m_j)^2}{d_j^3}.
\label{eq:GaussLearndj}
\end{align}

\subsubsection{Triangle ASAM learning law} 
The triangle set function has the form
\begin{equation}
a_j(\theta) = \left\{\begin{array}{ll}
1 - \frac{m_j-\theta}{l_j} ~~& \mbox{if $m_j-l_j \leq \theta \leq m_j$}\\
1 - \frac{\theta-m_j}{r_j} ~~& \mbox{if $m_j \leq \theta \leq m_j+r_j$}\\
0 ~~& \mbox{else}
\end{array}\right.
\end{equation}
with parameter learning laws
\begin{align}
m_j(t+1) &= \left\{\begin{array}{l}
m_j(t)-\mu_t\varepsilon(\theta)[c_j-F(\theta)]\frac{p_j(\theta)}{a_j(\theta)}\frac{1}{l_j}
~~~~~~~~~~\mbox{if $m_j-l_j < \theta < m_j$}\\
m_j(t)+\mu_t\varepsilon(\theta)[c_j-F(\theta)]\frac{p_j(\theta)}{a_j(\theta)}\frac{1}{r_j}
~~~~~~~~~~\mbox{if $m_j < \theta < m_j+r_j$}\\
m_j(t) ~~~\mbox{else}
\end{array}\right. \nonumber 
\\
l_j(t+1) &= \left\{\begin{array}{l}
l_j(t)+\mu_t\varepsilon(\theta)[c_j-F(\theta)]\frac{p_j(\theta)}{a_j(\theta)}\frac{m_j-\theta}{l_j^2}
~~~~~~~~~\mbox{if $m_j-l_j < \theta < m_j$}\\
l_j(t) ~~~\mbox{else}
\end{array}\right. \nonumber 
\\
r_j(t+1) &= \left\{\begin{array}{l}
r_j(t)+\mu_t\varepsilon(\theta)[c_j-F(\theta)]\frac{p_j(\theta)}{a_j(\theta)}\frac{\theta-m_j}{r_j^2}
~~~~~~~~~\mbox{if $m_j < \theta < m_j+r_j$}\\
r_j(t) ~~~\mbox{else}
\end{array}\right. \nonumber 
\end{align}
The Gaussian learning laws (\ref{eq:GaussLearnmj})-(\ref{eq:GaussLearndj}) can approximate the learning laws for the symmetric triangle set function $a_j(\theta) = \max\{0,1-\frac{|\theta-m_j|}{d_j}\}$.

\subsubsection{Cauchy ASAM learning law}
The Cauchy set function $a_j$ has the form
\begin{equation}
a_j(\theta) = \frac{1}{1+\left(\frac{\theta-m_j}{d_j}\right)^2}
\end{equation}
with parameter learning laws
\begin{align}
m_j(t+1) &= m_j(t)+\mu_t\varepsilon(\theta)p_j(\theta)[c_j-F(\theta)]\frac{\theta-m_j}{d_j^2}a_j(\theta)\\
d_j(t+1) &= d_j(t)+\mu_t\varepsilon(\theta)p_j(\theta)[c_j-F(\theta)]\frac{(\theta-m_j)^2}{d_j^3}a_j(\theta).
\end{align}

\subsubsection{Laplace ASAM learning law}
The Laplace or double-exponential set function $a_j$ has the form
\begin{equation}
a_j(\theta) = \exp\left\{-\frac{|\theta-m_j|}{d_j}\right\}
\end{equation}
with parameter learning laws
\begin{align}
m_j(t+1) &= m_j(t)+\mu_t\varepsilon(\theta)p_j(\theta)[c_j-F(\theta)]{\rm sign}(\theta-m_j)\frac{1}{d_j}\\
d_j(t+1) &= m_j(t)+\mu_t\varepsilon(\theta)p_j(\theta)[c_j-F(\theta)]{\rm sign}(\theta-m_j)\frac{|\theta-m_j|}{d_j^2}.
\end{align}

\subsubsection{Generalized hyperbolic tangent ASAM learning law}
The generalized hyperbolic tangent set function has the form
\begin{equation}
a_j(\theta) = 1 + \tanh\left(-\left(\frac{\theta-m_j}{d_j}\right)^2\right)
\end{equation}
with parameter learning laws
\begin{align}
m_j(t+1) &= m_j(t)+\mu_t\varepsilon(\theta)p_j(\theta)[c_j-F(\theta)](2-a_j(\theta))\frac{\theta-m_j}{d_j^2}\\
d_j(t+1) &= d_j(t)+\mu_t\varepsilon(\theta)p_j(\theta)[c_j-F(\theta)](2-a_j(\theta))\frac{(\theta-m_j)^2}{d_j^3}.
\end{align}

We can also reverse the learning process and adapt the SAM if-part and then-part set parameters by maximizing a given closed-form posterior pdf $f(\theta|x)$.  The basic Bayesian relation (\ref{eq:BayesPosterior}) above leads to the following application of the chain rule for a set parameter $m_j$:
\begin{equation}
\frac{\partial f(\theta|x)}{\partial m_j} ~\propto~ 
g(x|\theta)\frac{\partial F}{\partial m_j}
\label{eq:dBayesPosterior}
\end{equation}
since $\frac{\partial g}{\partial F} = 0$ because the likelihood $g(x|\theta)$ does not depend on the fuzzy system $F$.  The chain rule gives $\frac{\partial F}{\partial m_j} = \frac{\partial F}{\partial a_j}\frac{\partial a_j}{\partial m_j}$ and similarly for the other SAM parameters.  Then the above learning laws can eliminate the product of partial derivatives to produce a stochastic gradient ascent or maximum-a-posteriori or MAP learning law for the SAM parameters. 

\subsection{ASAM Approximation Simulations}\label{sec:ASAM-approx}

We simulated six different types of adaptive SAM fuzzy systems to approximate the three standard conjugate prior pdfs and their corresponding posterior pdfs.  The six types of ASAMs corresponded to the six if-part sets in Figure \ref{fg:2} and their learning laws above.  We combined C++ software for the ASAM approximations with Mathematica to compute the fuzzy-based  posterior $F(\theta|x)$ using (\ref{eq:BayesPosterior}). Mathematica's NIntegrate program computed the mean-squared errors between the conjugate prior $h(\theta)$ and the fuzzy-based prior $F(\theta)$ and between the posterior $f(\theta|x)$ and the fuzzy posterior $F(\theta|x)$.
            
Each ASAM simulation used uniform samples from a prior pdf $h(\theta)$.  The program evenly spaced the initial if-part sets and assigned them equal but experimental dispersion values.  The initial then-part sets had unit areas or volumes.  The initial then-part centroids corresponded to the prior pdf's value at the location parameters of the if-part sets.  A single learning iteration began with computing the approximation error at each uniformly spaced sample point.  The program cycled through all rules for each sample value and then updated each rule's if-part and then-part parameters according to the appropriate ASAM learning law.  Each adapted parameter had a harmonic-decay learning rate $\mu_t = \frac{c}{t}$ for learning iteration $t$.  Experimentation picked the numerator constants $c$ for the various parameters.

The approximation figures show representative simulation results. Figure \ref{fg:1} used Cauchy if-part sets for illustration only and not because they gave a smaller mean-squared error than sinc sets did. Figures \ref{fig:BetaASAM}-\ref{fig:NormalASAM} used sinc if-part sets even though we simulated all six types of if-part sets for all three types of conjugate priors. Simulations demonstrated that all 6 set functions produce good approximations for the prior pdfs.  The sinc ASAM usually performed best. We truncated the gamma priors at the right-side value of 150 and truncated the normal priors at $-4$ and 4 because the overlap between the truncated prior tails and the likelihoods $g(x|\theta)$ were small. The likelihood functions $g(x|\theta)$ had narrow dispersions relative to the truncated supports of the priors.  Larger truncation values or appended fall-off tails can accommodate unlikely $x$ values in other settings.  We also assumed that the priors were strictly positive. So we bounded the ASAM priors to a small positive value $(F(\theta) \geq 10^{-3})$ to keep the denominator integral in (\ref{eq:BayesPosterior}) well-behaved.

Figure \ref{fg:1} used only one fuzzy approximation. The fuzzy approximation of the $\beta(8,5)$ prior had mean-squared error $4.2\times 10^{-4}$.  The Cauchy-ASAM learning algorithm used 500 uniform samples for 6,000 iterations.  

The fuzzy approximation of the beta priors $\beta(2.5, 9)$, $\beta(9, 9)$, and $\beta(8, 5)$ in Figure \ref{fig:BetaASAM} had respective mean-squared errors $1.3\times 10^{-4}$, $2.3\times 10^{-5}$, and $1.4\times 10^{-5}$. The sinc-ASAM learning used 500 uniform samples from the unit interval for 6,000 training iterations. The corresponding conjugate beta posterior approximations had respective mean-squared errors $3.0\times 10^{-5}$, $6.9\times 10^{-6}$, and $3.8\times 10^{-5}$. 

The fuzzy approximation of the gamma priors $\gamma(1, 30)$, $\gamma(4, 12)$, and $\gamma(9, 5)$ in Figure \ref{fig:GammaASAM} had respective mean-squared errors $5.5\times 10^{-5}$, $3.6\times 10^{-6}$, and $7.9\times 10^{-6}$. The sinc-ASAM learning used 1,125 uniform samples from the truncated interval $[0, 150]$ for 6,000 training iterations. The corresponding conjugate gamma posterior approximations had mean-squared errors $2.3\times 10^{-5}$, $2.1 \times 10^{-7}$, and $2.3\times 10^{-4}$.  

The fuzzy approximation of the single standard-normal prior that underlies Figure \ref{fig:NormalASAM} had mean-squared error of $7.7\times 10^{-6}$. The sinc-ASAM learning used  500 uniform samples from the truncated interval $[-4,4]$ for 6,000 training iterations.  
\begin{table}[!h]\footnotesize
\begin{center}
\begin{tabular}{c@{\hspace{0.05in}}c}
\begin{tabular}{|c|c|}
\hline
Sample Mean&MSE \\ \hline
$-4$& $0.12$ \\ 
$-3.25$& $1.9 \times 10^{-3}$\\ 
$-2.5$& $3\times 10^{-4}$\\ 
$-1.75$& $1.5\times 10^{-4}$ \\ 
$-1$& $3.1\times 10^{-5}$\\ 
$-0.25$& $2.2\times 10^{-6}$\\ \hline
\end{tabular}
&
\begin{tabular}{|c|c|}
\hline
Sample Mean&MSE \\ \hline 
&\\
$0.5$& $1.1\times 10^{-5}$ \\
$1.25$& $6.5\times 10^{-5}$ \\
$2$& $1.6\times 10^{-4}$ \\ 
$2.75$& $3\times 10^{-4}$ \\
$3.5$&$7.6\times 10^{-3}$ \\ \hline
\end{tabular}
\end{tabular}
\caption{Mean squared errors for the $11$ normal posterior approximations}
\end{center}
\end{table}

The generalized-hyperbolic-tanh ASAMs in Figure \ref{fig:EmpiricalCasesCombined} learn the beta prior $\beta(3, 10.4)$ from both noiseless and noisy random-sample (i.i.d.) $x_1,x_2,\ldots$ draws from the ``unknown'' prior because the ASAMs use only the histogram or empirical distribution of the pdf.  The Glivenko-Cantelli Theorem \parencite{billingsley1995} ensures that the empirical distribution converges uniformly to the original distribution.  So sampling from the histogram of random samples increasingly resembles sampling directly from the unknown underlying pdf as the sample size increases.  This ASAM learning is robust in the sense that the fuzzy systems still learn the pdf if independent white noise corrupts the random-sample draws.

The simulation draws $N$ random samples $x_1,x_2,\ldots,x_N$ from the pdf $h(\theta) = \beta(3, 10.40)$ and then bins them into 50 equally spaced bins of length $\Delta\theta = 0.02$.  We generate an empirical pdf $h_{emp}(\theta)$ for the beta distribution by rescaling the histogram. The rescaling converts the histogram into a staircase approximation of the pdf $h(\theta)$:
\begin{align}
	h_{emp}(\theta) &= \sum_{m=1}^{\textrm{\scriptsize \# of bins}}{\frac{p[m] rect(\theta-\theta_{b}[m])}{N \Delta\theta}}
	\label{eq:EmpiricalPDF}
\end{align}
where $p[m]$ is the number of random samples in bin $m$ and where $\theta_{b}[m]$ is the central location of the $m^{th}$ bin. The ASAM generates an approximation $H_{emp}(\theta)$ for the empirical distribution $h_{emp}(\theta)$. Figure \ref{fig:EmpiricalCasesCombined}(a) shows comparisons between $H_{emp}(\theta)$ and $h(\theta)$.

The second example starts with $5,000$ random samples of the $\beta(3,10.4)$ distribution. We add zero-mean white Gaussian noise to the random samples. The noise is independent of the random samples. The examples use respective noise standard deviations of $0.1$, $0.05$, and $0.025$ in the three separate cases. The ASAM produces an approximation $H_{emp, n}(\theta)$ for this noise-modified function $h_{emp, n}(\theta)$.  Figure \ref{fig:EmpiricalCasesCombined}(b) shows comparisons between $H_{emp, n}(\theta)$ to $h(\theta)$. The approximands $h_{emp}$ and $h_{emp, n}$ in Figures \ref{fig:EmpiricalCasesCombined} (a) and (b) are random functions. So these functions and their ASAM approximators are sample cases.

\subsection{Approximating Non-conjugate Priors}\label{sec:ASAM-nonconj-approx}

We defined a prior pdf $h(\theta)$ as a convex bimodal mixture of normal and Maxwell pdfs: $h(\theta)=0.4N(10,1) + 0.3M(2) + 0.3M(5)$. The Maxwell pdfs have the form
\begin{align}
\theta \sim M(\sigma): \quad h(\theta) = \theta^2e^{-\frac{\theta^2}{2\sigma^2 }}~~~~ \quad \mbox{if~ $\theta > 0$.}
\end{align}
The prior pdf modeled a location parameter for the normal mixture likelihood function: $g(x|\theta) = 0.7N(\theta, 2.25) + 0.3N(\theta + 8, 1)$. The prior $h(\theta)$ is not conjugate with respect to this likelihood function $g(x|\theta)$. 

The ASAM used sinc set functions to generate a fuzzy approximator $H(\theta)$ for the prior $h(\theta)$. The ASAM used 15 rules and 6000 iterations on 500 uniform samples of $h(\theta)$. The two figures below show the quality of the prior and posterior fuzzy approximators. This example shows that fuzzy Bayesian approximation still works for non-conjugate pdfs.

\begin{figure}[ht!]
\centerline{
\includegraphics[width=\textwidth]{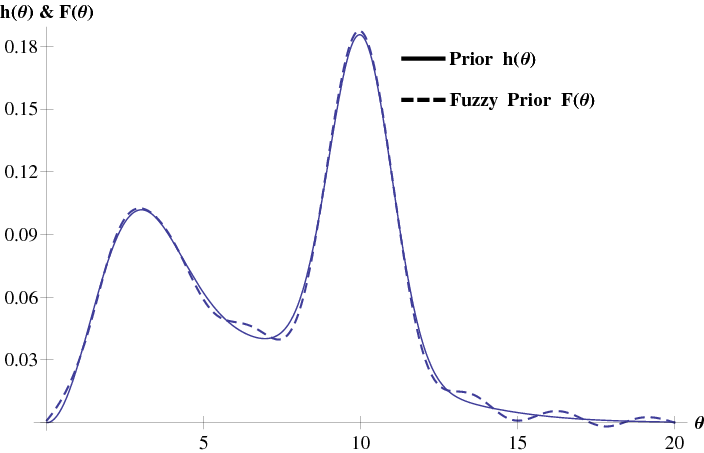}
}
\caption[Comparison of a non-conjugate prior pdf $h(\theta)$ and its fuzzy approximator $H(\theta)$]{Comparison of a non-conjugate prior pdf $h(\theta)$ and its fuzzy approximator $H(\theta)$. The pdf $h(\theta)$ is a convex mixture of normal and Maxwell pdfs: $h(\theta)=0.4N(10,1) + 0.3M(2) + 0.3M(5)$. The Maxwell pdf $M(\sigma)$ is $\theta^2e^{-\theta^2/2\sigma^2}$ for $\theta\geq0$ and $0$ for $\theta\leq0$. An adaptive sinc-SAM generated $H(\theta)$ using 15 rules and 6000 training iterations on 500 uniform samples of the $h(\theta)$.
}
\label{fig:NonConjugatePrior}
\end{figure}

\begin{figure}[ht!]
	\centerline{
\includegraphics[width=\textwidth]{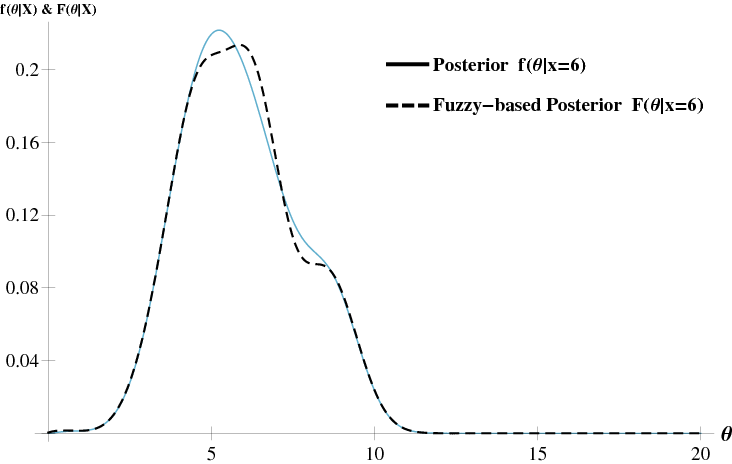}
}
\caption[Approximation of a non-conjugate posterior pdf. Comparison of a non-conjugate posterior pdf $f(\theta|x)$ and its fuzzy approximator $F(\theta|x)$]{Approximation of a non-conjugate posterior pdf. Comparison of a non-conjugate posterior pdf $f(\theta|x)$ and its fuzzy approximator $F(\theta|x)$. The fuzzy prior $H(\theta)$ and the mixture likelihood function $g(x|\theta)=0.7N(\theta|2.25)+0.3N(\theta+8, 1)$ produce the fuzzy approximator of the posterior pdf $F(\theta|x)$. The figure shows $F(\theta|x)$ for the single observation $x=6$.
}
\label{fig:NonConjugate-Posterior-x=6}
\end{figure}

\subsection{The SAM Structure of Fuzzy Posteriors}

The next theorem shows that the SAM's convex-weighted-sum structure passes over into the structure of the fuzzy-based posterior $F(\theta|x)$.  The result is a {\it generalized} SAM \parencite{kosko-fuzeng} because the then-part centroids $c_j$ are no longer constant but vary both with the observed data $x$ and the parameter value $\theta$.  This simplified structure for the posterior $F(\theta|x)$ comes at the expense in general of variable centroids that require several integrations for each observation $x$.

\begin{thm}\label{thm:Doubly-Posterior-SAM}{\bf Fuzzy Posterior Approximator is a SAM} \\
The fuzzy posterior approximator is a SAM:
\begin{align}
F(\theta|x) &=  \sum_{j=1}^m p_j(\theta) c_j'(x|\theta)
\label{eq:Sum pj cj'}
\end{align}
where the generalized then-part set centroids $c_j'(x|\theta)$ have the form
\begin{equation}
c_j'(\theta|x) = \frac{c_j g(x|\theta) }{\sum_{i=1}^m \int_D g(x|u) p_i(u)c_i\, du}
\label{eq:cj'}
~~~~~\mbox{for sample space $D$.}
\end{equation}
\end{thm}

\begin{proof}
The proof equates the fuzzy-based posterior $F(\theta|x)$ with the right-hand side of (\ref{eq:BayesPosterior}) and then expands according to Bayes Theorem:
\begin{align}
F(\theta|x)   
&= \frac{\displaystyle g(x|\theta) F(\theta)}
{\int_D g(x|u) F(u) du}  ~~~~~~~~~~~~~~\mbox{by (\ref{eq:BayesPosterior})}\\
&= \frac{g(x|\theta) \sum_{j=1}^m p_j(\theta)c_j}
{\int_D g(x|u) \sum_{j=1}^m p_j(u)c_j\, du} ~~\,\mbox{by (\ref{eq:SAM})} \\
&=  \frac{g(x|\theta) \sum_{j=1}^m p_j(\theta)c_j}
{\sum_{j=1}^m \int_D g(x|u) p_j(u)c_j\, du}  \\
&= \sum_{j=1}^m p_j(\theta) \left(\frac{c_j g(x|\theta) }{\sum_{i=1}^m \int_D g(x|u) p_i(u)c_i\, du}\right)~~ \\
\label{eq:SAMposterior} 
&= \sum_{j=1}^m p_j(\theta) c_j'(x|\theta).
\end{align}
\end{proof}

We next prove two corollaries that hold in special cases that avoid the integration in (\ref{eq:SAMposterior}) and thus are computationally tractable.

The first special case occurs when $g(x|\theta)$ approximates a Dirac delta function centered at $x$: $g(x|\theta) \approx \delta(\theta-x)$.  This can arise when $g(x|\theta)$ concentrates on a region $D_g \subset D$ if $D_g$ is much smaller than $D_{p_j} \subset D$ and if $p_j(\theta)$ concentrates on $D_{p_j}$.  
Then
\begin{align}
\int_D p_j(u)g(x|u)\,du ~\approx~ \int_{D_{p_j}} p_j(u)\delta(u-x)\,du
~=~ p_j(x).
\label{eq:Intpg1}
\end{align}
Then (\ref{eq:SAMposterior}) becomes
\begin{align}
F(\theta|x) &\approx \frac{\displaystyle g(x|\theta) F(\theta)} { \sum_{j=1}^m p_j(x)c_j} ~~~~~~~~~~~~\mbox{by (\ref{eq:Intpg1})}\\
&= g(x|\theta)\frac{\sum_{j=1}^m p_j(\theta)c_j}{F(x)}~~~~~\mbox{by (\ref{eq:SAM})}\\
&= \sum_{j=1}^m p_j(\theta) \left(\frac{c_j g(x|\theta)}{F(x)}\right) 
\label{eq:SAMposteriorI}
~=~ \sum_{j=1}^m p_j(\theta) c_j'(x|\theta).
\end{align}
So a learning law for $F(\theta|x)$ needs to update only each then-part centroid $c_j$ by scaling it with $g(x|\theta)/F(x)$ for each observation $x$. This involves a substantially lighter computation than does the integration in (\ref{eq:SAMposterior}).

The delta-pulse approximation $g(x|\theta) \approx \delta(\theta-x)$ holds for narrow bell curves such as normal or Cauchy pdfs when their variance or dispersion is small.  It holds in the limit as the equality  $g(x|\theta) = \delta(\theta-x)$ in the much more general case of alpha-stable pdfs \parencite{kosko-mitaim2001SR,nikias-shao1995} with any shape if $x$ is the location parameter of the stable pdf and if the dispersion $\gamma$ goes to zero.  Then the characteristic function is the complex exponential $e^{i x \omega}$ and thus Fourier transformation gives the pdf $g(x|\theta)$ exactly as the Dirac delta function \parencite{kosko-mitaimNN2003}:  $\displaystyle \lim_{\gamma \to 0} g(x|\theta) = \delta(\theta - x)$.  Then $F(\theta|x)$ equals the right-hand side of (\ref{eq:SAMposteriorI}).  The approximation fails for a narrow binomial $g(x|\theta)$ unless scaling maintains unity status for the mass of $g(x|\theta)$ in (\ref{eq:Intpg1}) for a given $n$. 

The second special case occurs when the SAM system $F$ uses several narrow rules to approximate the prior $h(\theta)$ and when the data is highly uncertain in the sense that $D_g$ is large compared with $D_{p_j}$. Then we can approximate $g(x|\theta$) as the constant $g(x|m_j)$ over $D_{p_j}$:
\begin{align}
\int_D p_j(u)\,g(x|u)\,du &\approx \int_{D_{p_j}} p_j(u)\,g(x|m_j)\,du~~
~=~ g(x|m_j)\,U_{p_j}
\label{eq:IntApproxII2}
\end{align}
where $U_{p_j} = \int_{D_{p_j}} p_j(u)du$. So (\ref{eq:SAMposterior}) becomes
\begin{align}
F(\theta|x) &\approx \frac{\displaystyle g(x|\theta) F(\theta)}
{\sum_{j=1}^m g(x|m_j)U_{p_j}c_j} ~~~~\mbox{by (\ref{eq:IntApproxII2})}\\
&= \sum_{j=1}^m p_j(\theta)\left(\frac{c_j g(x|\theta)}{\sum_{i=1}^m g(x|m_i)U_{p_i}c_i}\right)
\label{eq:SAMposteriorII} 
~=~ \sum_{j=1}^m p_j(\theta) c_j'(x|\theta)~~~.
\end{align}
We can pre-compute or estimate the if-part volume $U_{p_j}$ in advance.  So (\ref{eq:SAMposteriorII}) also gives a generalized SAM structure and another tractable way to adapt the variable then-part centroids $c_j'(x|\theta)$.

This second special case holds for the normal likelihood 
\begin{equation}g(x|\theta) = \frac{1}{\sqrt{2\pi}\sigma_0} e^{-(x-\theta)^2/2\sigma_0^2}\end{equation}
 if the widths or dispersions $d_j$ of the if-part sets are small compared with $\sigma_0$ and if there are a large number $m$ of fuzzy if-then rules that jointly cover $D_g$. This occurs if $D_g=(\theta-3\sigma_0,\theta+3\sigma_0)$ with if-part dispersions $d_j = \sigma_0/m$ and locations $m_j$.  Then $p_j(\theta)$ concentrates on some $D_{p_j} = (m_j-\epsilon,m_j+\epsilon)$ where $0 < \epsilon \ll \sigma_0$ and so $p_j(\theta) \approx 0$ for $\theta \notin D_{p_j}$.  Then $\frac{x-m_j\pm \epsilon}{\sigma_0} \approx \frac{x-m_j}{\sigma_0}$ since $\epsilon \ll \sigma_0$.  So $\frac{x-\theta}{\sigma_0} \approx \frac{x-m_j}{\sigma_0}$ for all $\theta \in D_{p_j}$  and thus
\begin{equation}
g(x|\theta) \approx \frac{1}{\sqrt{2\pi}\sigma_0}e^{-(x-m_j)^2/2\sigma_0^2} = g(x|m_j)
\end{equation}
for $\theta \in D_{p_j}$. 

This special case also holds for the binomial $g(x|\theta) = \binom{n}{k} \theta^x(1-\theta)^{n-x}$ for $x = 0,1,\ldots,n$ if $n \ll m$ and thus if there are fewer Bernoulli trials $n$ than fuzzy if-then rules $m$ in the SAM system. It holds because $g(x|\theta)$ concentrates on $D_g$ and because $D_g$ is wide compared with $D_{p_j}$ when $m \gg n$. This case also holds for the Poisson $g(x|\theta) = \frac{1}{x!}\theta^x e^{-\theta}$ if the number of times $x$ that a discrete event occurs is small compared with the number $m$ of SAM rules that jointly cover  $D_g = (\frac{x}{2},\frac{3x}{2})$ because again $D_g$ is large compared with $D_{p_j}$.  So (\ref{eq:IntApproxII2}) follows.

%%%%% New Section 

\subsection{Other Uniform Function Approximation Methods}

Fuzzy systems lend themselves well to prior pdf approximation because a lot of subjective information exists as linguistic rules. The approximation and representation theorems of Kosko \parencite*{kosko-fat1994} and Watkins \parencite{watkins1995} make fuzzy approximators well-suited for Bayesian inference. But there are other \emph{model-free} approximation schemes for approximating pdfs and likelihood. \emph{Uniform} approximation schemes are especially useful in practice because they guarantee that users can specify and achieve a global approximation error tolerance. Examples of model-free uniform approximators include multilayer feedforward (FF) networks and algebraic rings of polynomials. There are even more approximation methods available if we do not insist on uniform approximation quality. 

Hornik et al.~\parencite*{hornik-white1989} showed that multilayer FF networks can also uniformly approximate any Borel measurable function. FF networks can learn the approximand from sample data. Back-propagation training~\parencite{werbos1974,rumelhart-hinton-williams1986,werbos1990} is the standard method for training feedforward approximators with sample data. The approach is similar to the learning procedure for fuzzy ASAMs: gradient descent on an approximation-error surface (squared error or cross-entropy). But FF networks have very opaque internal structures compared to the rule-sets for fuzzy ASAMs. So tuning a FF network with linguistic information is not as simple as in the fuzzy case. Exact function representation with FF networks is possible~\parencite{hecht-nielsen1987}. But the FF representation uses highly nonsmooth, possibly non-parametric activation functions~\parencite{girosi-poggio1989, kurkova1991, kurkova1992}. Thus function representation is much harder with FF networks than with fuzzy systems.

Polynomial rings provide another method for universal function approximation on compact domains. Polynomial rings refer to the set of polynomials with real coefficients equipped with the natural addition and multiplication binary operations. Function approximation with polynomial rings usually relies on the Weierstrass Approximation theorem and its generalizations~\parencite{stone1948, rudin1977}. The approximation theorem guarantees the existence of uniform polynomial approximator for any continuous function over a compact domain. Applying the theorem involves selecting a polynomial basis for the approximation. The set of available polynomial bases includes splines, Bernstein Polynomials, Hermite polynomials, Legendre polynomials, etc. However polynomials tend to be unstable in the norm and are thus less attractive for implementing approximators.

%%%%%%%%%%
\section{Doubly Fuzzy Bayesian Inference: Uniform Approximation} \label{sec:doubly-fuzzy}
We will use the term {\it doubly fuzzy} to describe Bayesian inference where separate fuzzy systems $H(\theta)$ and $G(x|\theta)$ approximate the respective prior pdf $h(\theta)$ and the likelihood $g(x|\theta)$.  Theorem 3 below shows that the resulting fuzzy approximator $F$ of the posterior pdf $f(\theta|x)$ still has the convex-sum structure (8) of a SAM fuzzy system. 

The doubly fuzzy posterior approximator $F$ requires only $m_1 m_2$ rules if the fuzzy likelihood approximator $G$ uses $m_1$ rules and if the fuzzy prior approximator $H$ uses $m_2$ rules.  The $m_1m_2$ if-part sets of $F$ have a corresponding product structure as do the other fuzzy-system parameters.  Corollary 3.1 shows that using an exact 2-rule representation reduces the corresponding rule number $m_1$ or $m_2$ to two.  This is a tractable growth in  rules for a single Bayesian inference.  But the same structure leads in general to an exponential growth in posterior-approximator rules if the old posterior approximator becomes the new prior approximator in iterated Bayesian inference.

\begin{figure}[ht]
\centerline{	
\includegraphics[width=\textwidth]{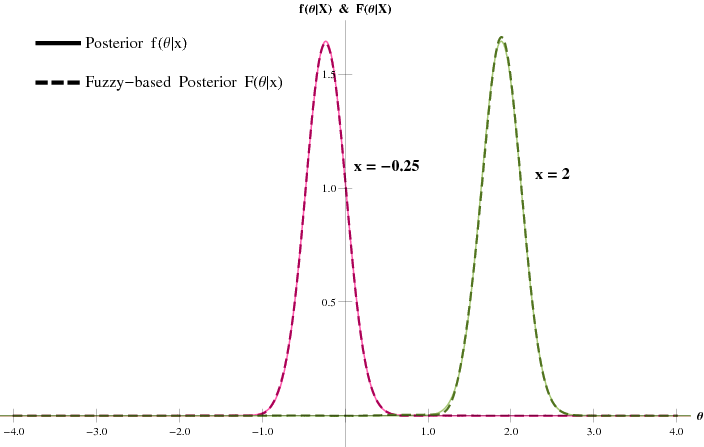}
}
\caption[Doubly fuzzy Bayesian inference: comparison of two normal posteriors and their doubly fuzzy approximators]{Doubly fuzzy Bayesian inference: comparison of two normal posteriors and their doubly fuzzy approximators. The doubly fuzzy approximations use fuzzy prior-pdf approximator $H(\theta)$ and fuzzy likelihood-pdf approximator $G(x|\theta)$. The sinc-SAM fuzzy approximator $H(\theta)$ uses 15 rules to approximate the normal prior $h(\theta) = N(0,1)$. The Gaussian-SAM fuzzy likelihood approximator $G(x|\theta)$ uses 15 rules to approximate the two likelihood functions $g(x|\theta) = N(-0.25,\frac{1}{16})$ and $g(x|\theta) = N(2, \frac{1}{16})$. The two fuzzy approximators used 6000 learning iterations based on 500 uniform sample points.
} 
\label{fig:FuzzyGauss-Fuzzylikelihood}
\end{figure}

Figure \ref{fig:FuzzyGauss-Fuzzylikelihood} shows the result of doubly fuzzy Bayesian inference for two normal posterior pdfs.  A 15-rule Gaussian SAM $G$ approximates two normal likelihoods while a 15-rule sinc SAM $H$ approximates a standard normal prior pdf. 

\subsection{The Bayesian Approximation Theorem}
We call the next theorem the Bayesian Approximation Theorem (BAT).  The BAT shows that doubly fuzzy systems can uniformly approximate posterior pdfs under some mild conditions.  The proof derives an approximation error bound for $F(\theta|x)$ that does not depend on $\theta$ or $x$. Thus $F(\theta| x)$ uniformly approximates $f(\theta|x)$.  The BAT holds in general for any uniform approximators of the prior or likelihood. Corollary \ref{cor:CentBoundPost} %(2.1) 
shows how the centroid and convex-sum structure of SAM fuzzy approximators $H$ and $G$ specifically bound the posterior approximator $F$. % Theorem 3 gives further insight into the induced SAM structure of the doubly fuzzy posterior approximator $F$.

The statement and proof of the BAT require the following notation. Let $\mathcal{D}$ denote the set of all $\theta$ and let $X$ denote the set of all $x$.  Assume that $\mathcal{D}$ and $X$ are compact. The prior is $h(\theta)$ and the likelihood is $g(x|\theta)$. $H(\theta)$ is a 1-dimensional SAM fuzzy system that uniformly approximates $h(\theta)$ in accord with the Fuzzy Approximation Theorem \parencite{kosko-fat1994,kosko-fuzeng}. $G(x|\theta)$ is a 2-dimensional SAM that uniformly approximates $g(x|\theta)$. Define the Bayes factors as $q(x) = \int_D h(\theta) g(x|\theta) d\theta$ and $Q(x) = \int_D H(\theta) G(x| \theta) d\theta$. Assume that $q(x)>0$ so that the posterior $f(\theta|x)$ is well-defined for any sample data $x$. Let $\Delta Z$ denote the approximation error $Z-z$ for an approximator $Z$. 

\vspace{6pt}

\begin{thm}{\bf{[Bayesian Approximation Theorem]:}}\label{thm:BAT}\\
	Suppose that $h(\theta)$ and $g(x|\theta)$ are bounded and continuous and that $H(\theta)G(x|\theta) \neq 0$ almost everywhere.
	
	Then the posterior pdf approximator $F(\theta|x) = HG/Q$ uniformly approximates $f(\theta|x)$. Thus for all $\epsilon > 0:~  | F(\theta|x)  -  f(\theta|x) |  <  \epsilon$ for all $x$ and $\theta$.
\end{thm}

\begin{proof}
Write the posterior pdf $f(\theta|x)$ as $f(\theta|x) = \frac{h(\theta) g(x|\theta)}{q(x)}$ and its approximator $F(\theta| x)$ as $F(\theta| x) = \frac{H(\theta) G(x|\theta)}{Q(x)}$. The SAM approximations for the prior and likelihood functions are uniform \parencite{kosko-fuzeng}. So they have approximation error bounds $\epsilon_h$ and $\epsilon_g$ that do not depend on $x$ or $\theta$:
\begin{align}
	|\Delta H| &<~ \epsilon_h    \label{eq:Defn-DelH}
	 ~~~\mbox{and}~~~\\
	|\Delta G| &<~ \epsilon_g 	\label{eq:Defn-DelG}
\end{align}
where $\Delta H = H(\theta)-h(\theta)$ and $\Delta G = G(x|\theta)-g(x|\theta)$. The posterior error $\Delta F$ is
\begin{equation}
	\Delta F ~=~ F-f ~=~ \frac{ H G}{Q(x)} - \frac{h g}{q(x)}.  
	\label{eq:DelF-Expr}
\end{equation}
Expand $HG$ in terms of the approximation errors to get
\begin{align}
	HG &= (\Delta H + h )(\Delta G +g)
	   ~=~ \Delta H \Delta G + \Delta H g +h \Delta G + h g.
\end{align} 
We have assumed that $HG \neq 0$ almost everywhere and so $Q \neq 0$.  We now derive an upper bound for the Bayes-factor error $\Delta Q$:
\begin{align}
\Delta Q &= Q-q ~=~ \int_\mathcal{D}(\Delta H \Delta G + \Delta H g +h \Delta G + h g - hg) ~d\theta.
\end{align}
So
\begin{align}
	|\Delta Q| &\leq\int_\mathcal{D}{|\Delta H \Delta G + \Delta H g +h \Delta G|} ~d\theta\\
	&\leq\int_\mathcal{D}\big(|\Delta H| |\Delta G| + |\Delta H| g + h |\Delta G|\big) ~d\theta\\
	&< \int_\mathcal{D} (\epsilon_h \epsilon_g + \epsilon_h g + h \epsilon_g) ~d\theta~~~~~
	\mbox{by (\ref{eq:Defn-DelH})}.
\end{align}

Parameter set $\mathcal{D}$ has finite Lebesgue measure $\displaystyle m(\mathcal{D}) = \int_\mathcal{D} ~d\theta < \infty$ because $\mathcal{D}$ is a compact subset of a metric space and thus \parencite{munkres2000} it is (totally) bounded. Then the bound on $\Delta Q$ becomes
\begin{align}
	|\Delta Q| & <  m(\mathcal{D}) \epsilon_h \epsilon_g + \epsilon_g + \epsilon_h \int_\mathcal{D}{g(x|\theta)}~d\theta
\end{align}
because $\displaystyle \int_\mathcal{D} h(\theta) d\theta = 1$.

We now invoke the extreme value theorem \parencite{folland1999}. The extreme value theorem states that a continuous function on a compact set attains both its maximum and minimum. The extreme value theorem allows us to use maxima and minima instead of suprema and infima. Now $\int_\mathcal{D}{g(x|\theta)}~d\theta$ is a continuous function of $x$ because $g(x|\theta)$ is a continuous nonnegative function. The range of $\int_\mathcal{D}{g(x|\theta)}~d\theta$ is a subset of the right half line $(0,\infty)$ and its domain is the compact set $\mathcal{D}$. So $\int_\mathcal{D}{g(x|\theta)}~d\theta$ attains a finite maximum value. Thus
\begin{equation}
	|\Delta Q| < \epsilon_q 
	\label{eq:DelQ-bound}
\end{equation}
where we define the error bound $\epsilon_q$ as
\begin{align}
\epsilon_q &=  m(\mathcal{D}) \epsilon_h \epsilon_g + \epsilon_g + \epsilon_h~\max_{x}\left\{ \int_\mathcal{D}{g(x|\theta)}~d\theta \right\}.
\end{align}
Rewrite the posterior approximation error $\Delta F$ as
\begin{multline}
\Delta F = \frac{qHG - Qhg}{q Q} =\frac{q(\Delta H \Delta G + \Delta H g +h \Delta G + hg)~ - ~ Q hg} {q(q + \Delta Q)}
\end{multline}
Inequality (\ref{eq:DelQ-bound}) implies that $-\epsilon_q < \Delta Q < \epsilon_q$ and that $(q - \epsilon_q)< (q + \Delta Q) < (q + \epsilon_q)$. Then (\ref{eq:Defn-DelH}) and (\ref{eq:Defn-DelG}) give similar inequalities for $\Delta H$ and $\Delta G$. So
\begin{multline}
 \frac{q[-\epsilon_h \epsilon_g - \text{min} (g) \epsilon_h - \text{min} (h) \epsilon_g ]~ - ~ \epsilon_q hg} {q(q - \epsilon_q)} < \Delta F \\ < \frac{q[\epsilon_h \epsilon_g + \text{max} (g) \epsilon_h + \text{max} (h) \epsilon_g ]~ + ~ \epsilon_q hg} {q(q - \epsilon_q)}.
	\label{eq:Ptwise-Bound}
\end{multline}
The extreme value theorem ensures that the maxima in (\ref{eq:Ptwise-Bound}) are finite. The bound on the approximation error $\Delta F$ does not depend on $\theta$. But $q$ still depends on the value of the data sample $x$. So (\ref{eq:Ptwise-Bound}) guarantees at best a pointwise approximation of $f(\theta|x)$ when $x$ is arbitrary. We can improve the result by finding bounds for $q$ that do not depend on $x$.  Note that $q(x)$ is a continuous function of $x \in X$ because $hg$ is continuous. So the extreme value theorem ensures that the Bayes factor $q$ has a finite upper bound and a positive lower bound. 

The term $q(x)$ attains its maximum and minimum by the extreme value theorem. The minimum of $q(x)$ is positive because we assumed $q(x)>0$ for all $x$. H\"older's inequality gives $|q| \leq \left( \int_D|h|d\theta\right) \left(\left\|g(x,\theta)\right\|_{\infty} \right) = \left\|g(x,\theta)\right\|_{\infty}$ since $h$ is a pdf. So the maximum of $q(x)$ is finite because $g$ is bounded: $0 ~<~ \min\{q(x)\} ~\leq~$ $ \max\{ q(x)\} ~<~ \infty$. Then
\begin{align}
\epsilon_{-} < \Delta F < \epsilon_{+}
\end{align}
if we define the error bounds $\epsilon_{-}$ and $\epsilon_{+}$ as
\begin{align} 
\epsilon_{-} &= \frac{\left(-\epsilon_h \epsilon_g - \min\{g\} \epsilon_h - \min\{h\} \epsilon_g\right) \min\{q\}  ~ - ~ hg \epsilon_q }{\min\{q\} \left( \min\{q\} - \epsilon_q\right) } \\
\epsilon_{+} &= \frac{\left(\epsilon_h \epsilon_g + \max\{g\} \epsilon_h + \max\{h\} \epsilon_g\right) \max\{g\}  ~ + ~ hg \epsilon_q }{\min\{q\} \left( \min\{q\} - \epsilon_q\right) }. 
	\label{eq:Final-Bound}
\end{align}

Now $\epsilon_q \rightarrow 0$ as $\epsilon_g \rightarrow 0$ and $\epsilon_h \rightarrow 0$. So $\epsilon_{-} \rightarrow 0$ and $\epsilon_{+} \rightarrow 0$. The denominator of the error bounds must be non-zero for this limiting argument. We can guarantee this when $\epsilon_q < \min\{q\}$. This condition is not restrictive because the functions $h$ and $g$ fix or determine $q$ independent of the approximators $H$ and $G$ involved and because $\epsilon_q \to 0$ when $\epsilon_h \to 0$ and $\epsilon_g \to 0$. So we can achieve arbitrarily small $\epsilon_q$ that satisfies $\epsilon_q < \min\{q\}$ by choosing appropriate $\epsilon_h$ and $\epsilon_g$. Then $\Delta F \rightarrow 0$ as $\epsilon_g \rightarrow 0$ and $\epsilon_h \rightarrow 0$. So $|\Delta F| \rightarrow 0$. 
\end{proof}

The BAT proof also shows how sequences of uniform approximators $H_n$ and $G_n$ lead to a sequence of posterior approximators $F_n$ that converges uniformly to $F$.  Suppose we have such sequences $H_n$ and $G_n$ that uniformly approximate the respective prior $h$ and likelihood $g$. Suppose $\epsilon_{h,n+1}<\epsilon_{h,n}$ and $\epsilon_{g,n+1} < \epsilon_{g,n}$ for all $n$. Define $F_n = \frac{H_n G_n}{\int{H_n G_n}}$. Then for all $\epsilon > 0$ there exists an $n_0 \in \mathbb{N}$ such that for all $n>n_0: \quad |F_n(\theta|x)-F(\theta|x)| < \epsilon$ for all $\theta$ and for all $x$. The positive integer $n_0$ is the first $n$ such that $\epsilon_{h,n}$ and $\epsilon_{g,n}$ satisfy (\ref{eq:Final-Bound}). Hence $F_n$ converges uniformly to $F$.

Corollary \ref{cor:CentBoundPost}%2.1 
below reveals the fuzzy structure of the BAT's uniform approximation when the prior $H$ and likelihood $G$ are uniform SAM approximators.  The corollary shows how the convex-sum and centroidal structure of $H$ and $G$ produce centroid-based bounds on the fuzzy posterior approximator $F$.  Recall first that Theorem 1 states that $F(\theta|x) = \sum_{j=1}^m{p_j(\theta)c'_j(x|\theta)}$ where $c'_j(x|\theta) = \frac{c_j g(x|\theta)}{\sum_{i=1}^m{\int_D{g(x|u)p_i(u)c_i du}}}$. Replace the likelihood $g(x|\theta)$ with its doubly fuzzy SAM approximator $G(x|\theta)$ to obtain the posterior  
\begin{align}
F(\theta| x) &= \sum_{j=1}^m{p_j(\theta)C'_j(x|\theta)}
\end{align}
where the if-part set centroids are
\begin{align}
C'_j(x|\theta) &= \frac{c_{h,j} G(x|\theta)}{\sum_{i=1}^m{\int_D{G(x|u)p_i(u)c_{h,i} du}}}.
\label{eq:centroidsc'}
\end{align}
The $\left\{c_{h,k}\right\}_k$ are the then-part set centroids for the prior SAM approximator $H(\theta)$. $G(x|\theta)$ likewise has then-part set centroids $\left\{c_{g,j}\right\}_j$. Each SAM is a convex sum of its centroids from (\ref{eq:Sum pj cj'}). This convex-sum structure induces bounds on $H$ and $G$ that in turn produce bounds on $F$. We next let the subscripts {\it max} and {\it min} denote the respective maximal and minimal centroids. The maximal centroids are positive. But the minimal centroids may be negative even though $h$ and $g$ are non-negative functions.  We also assume that the minimal centroids are positive. So define the maximal and minimal product centroids as
\begin{multline}
c_{gh,max} = \max_{j,k} c_{g,j}c_{h,k} 
= c_{g,max} c_{h,max} \\
\mbox{and}~~~~ c_{gh,min} = \min_{j,k} c_{g,j}c_{h,k} 
= c_{g,min} c_{h,min}.  
\end{multline}
Then the BAT gives the following SAM-based bound.

\vspace{14pt}
\begin{cor}{\bf Centroid-based bounds for the doubly fuzzy posterior $F$.}\label{cor:CentBoundPost}
Suppose that the set $D$ of all $\theta$ has positive Lebesgue measure.  Then the centroids of the $H$ and $G$ then-part sets bound the posterior $F$:
\begin{equation}
\frac{c_{gh,min}}{m(D)c_{gh,max}} ~~\leq~~ F(\theta| x) ~~\leq~~ \frac{c_{gh,max} }{m(D)c_{gh,min}}.
\end{equation}
\end{cor}

\begin{proof}
The convex-sum structure constrains the values of the SAMs:
$H(\theta) \in \left[c_{h,min},c_{h,max}\right]$ for all $\theta$ and  
$G(x| \theta) \in \left[c_{g,min},c_{g,max}\right]$ for all $x$ and $\theta$.  Then (\ref{eq:centroidsc'}) implies
\begin{align}
C'_j(x|\theta) &\geq \frac{c_{gh,min}}{c_{gh,max} \sum_{i=1}^m{\int_D{p_i(u)du}}} \\
&= \frac{c_{gh,min}}{m(D)c_{gh,max}} ~~~~\mbox{for all $x$ and $\theta$}
\label{eq:C-LowerBound} 
\end{align}
since $\sum_{i=1}^m{\int_D{p_i(u)du}} = \int_D{\sum_{i=1}^m{p_i(u)}du} = \int_D{du} = m(D)$ where $m(D)$ denotes the (positive) Lebesgue measure of $D$. The same argument gives the upper bound: 
\begin{align}
C'_j(x|\theta) ~\leq~ \frac{c_{gh,max}}{m(D)c_{gh,min}} ~~~~\mbox{for all $x$ and $\theta$}. 
\label{eq:C-UpperBound}
\end{align}
Thus (\ref{eq:C-LowerBound}) and (\ref{eq:C-UpperBound}) give bounds for all centroids:
\begin{align}
\frac{c_{gh,min}}{m(D)c_{gh,max} } ~\leq~ C'_j(x|\theta) ~\leq~ \frac{c_{gh,max}}{m(D)c_{gh,min}} ~~~~\mbox{for all $x$ and $\theta$}.
\label{eq:C-Interval}
\end{align}
This bounding interval applies to $F(\theta| x)$ because the posterior approximator also has a convex-sum structure. Thus
\begin{equation}
\frac{c_{gh,min}}{m(D)c_{gh,max}} ~\leq~ F(\theta| x) ~\leq~ \frac{c_{gh,max}}{m(D)c_{gh,min}} ~~~~\mbox{for all $x$ and $\theta$}. 
\label{eq:F-Interval}
\end{equation}
\end{proof}

The size of the bounding interval depends on the size of the set $D$ and on the minimal centroids of $H$ and $G$. The lower bound is more sensitive to minimal centroids than the upper bound because dividing by a maximum is more stable than dividing by a minimum close to zero. The bounding interval becomes $[0, \infty)$ if any of the minimal centroids for $H$ or $G$ equal zero.  The infinite bounding interval $[0, \infty)$ corresponds to the least informative case.

Similar centroid bounds hold for the multidimensional case.  Suppose that the SAM-based posterior $F$ is the multidimensional approximator $F: R \to R^p$ with $p > 1$.  Then the same argument applies to the components of the centroids along each dimension. There are $p$ bounding intervals 
\begin{align}
\frac{c^s_{gh,min}}{m(D)c^s_{gh,max}} ~\leq~ F_s(\theta| x) ~\leq~ \frac{c^s_{gh,max}}{m(D)c^s_{gh,min}}
\end{align}
for each dimension $s$ of the range $R^p$. These componentwise intervals define a bounding hypercube $\prod_{s=1}^p [\frac{c^s_{gh,min}}{m(D)c^s_{gh,max}}, \frac{c^s_{gh,max}}{m(D)c^s_{gh,min}}] \subset R^p$ for $F$.

\section{Conclusion}
Fuzzy systems allow users to encode prior and likelihood information through fuzzy rules rather than through only a handful of closed-form probability densities.  Gradient-descent learning algorithms also allow fuzzy systems to learn and tune rules based on the same type of collateral data that an expert might consult or that a statistical hypothesis might use.  Different learning algorithms should produce different bounds on the fuzzy prior or likelihood approximations and those in turn should lead to different bounds on the fuzzy posterior approximation. Hierarchical Bayes systems can model hyperpriors with fuzzy approximators or with other ``intelligent'' learning systems such as neural networks or semantic networks. An open research problem is how the use of semi-conjugate rules or other techniques can reduce the exponential rule explosion that doubly fuzzy Bayesian systems face in general in Bayesian iterative inference.

\clearpage

%\subfile{./Chapters/chapExBAT}

\chapter{Hierarchical and Iterative Bayesian Inference with Function Approximators}\label{ch:ExBAT} %%More Results on Bayesian Inference with Function Approximators

\blfootnote{
	This chapter features work done in collaboration with Prof. Sanya Mitaim and first presented in \parencite{osoba-mitaim-kosko-FODM2012,osoba-mitaim-kosko-triply2011,osoba-mitaim-kosko-SMC2011}.
}

This chapter presents some extensions of the approximate Bayesian inference scheme in Chapter \ref{ch:BAT}. The first extension is a demonstration of the fuzzy approximation on hierarchical Bayesian models where the user puts a second-order prior pdf or a hyperprior on one of the uncertain parameters in the original prior pdf. Hierarchical Bayesian models are more complex and require more approximators. The next result, the \emph{Extended Bayesian Approximation Theorem} (Theorem \ref{thm:ExBAT}), shows that the increased complexity of hierarchical Bayesian models does not weaken the theoretical strength of our approximation scheme. This theorem generalizes the Bayesian Approximation Theorem (Theorem \ref{thm:BAT}) to hierarchical and multi-dimensional Bayesian models.

The final set of results addresses the computational complexity of iterative Bayesian schemes. Conjugate Bayesian models track a constant number of pdf parameters with each iteration of Bayes theorem. We propose the related idea of \emph{Semi-Conjugacy} for fuzzy approximators as a way to limit the growth of pdf parameters for iterative Bayesian schemes. Semi-conjugacy refers to functional similarity between if-part sets of the posterior and prior approximators. It is a looser condition than conjugacy. Semi-conjugacy allows users to strike a balance between parameter explosion associated with non-conjugate models and constant parameter size associated with overly restrictive conjugate models.

\section{Approximate Hierarchical Bayesian Inference}
Function approximation can also apply to nested priors or so-called {\it hierarchical Bayesian} techniques \parencite{carlin-louis2009,hogg-mckean-craig2005}.  Here the user puts a new prior or {\it hyperprior} pdf on an uncertain parameter that appears in the original prior pdf.  This new hyperprior pdf can itself have a random parameter that leads to yet another new prior or hyper-hyperprior pdf and so on up the hierarchy of prior models.  Figure \ref{fig:hierarchical_bayes} demonstrates this hierarchical approximation technique in the common case where an inverse-gamma hyperprior pdf models the uncertainty in the unknown variance of a normal prior pdf.  This is the scalar case of the conjugate inverse Wishart prior \parencite{carlin-louis2009} that often models the uncertainty in the covariance matrix of a normal random vector.

The simple Bayesian posterior pdf is $f(\theta|x) \propto g(x|\theta)\,h(\theta)$ where the likelihood is $g(x|\theta)$ and the prior pdf is $h(\theta)$. But now suppose that the prior pdf $h(\theta)$ depends on an uncertain parameter $\tau:  h(\theta|\tau)$.  We will model the uncertainty involving $\tau$ by making $\tau$ a random variable $\mathcal{T}$ with its own pdf or hyperprior pdf $\pi(\tau)$. The probabilistic graphical model (PGM) in Figure \ref{fig:ExBAT-PGM} represents this data model succinctly. 
\begin{figure}
\centering
\begin{tikzpicture}
\tikzstyle{main}=[circle, minimum size = 10mm, thick, draw =black!80, node distance = 16mm]
\tikzstyle{connect}=[-latex, thick]
\tikzstyle{box}=[rectangle, draw=black!100]
  \node[main] (tau) [label=below:$\mathcal{\tau}$] { };
  \node[main] (theta) [right=of tau, label=below:$\theta$] { };
  \node[main, fill = black!10] (x) [right=of theta, label=below:$X$] { };
  \path (theta) edge [connect] (x);
  \path (tau) edge [connect] (theta);
  \node[rectangle, inner sep=0mm, fit=(x),label=below right:$\mathbf{n}$, xshift=0.75mm,yshift=-0.75mm] {};
  \node[rectangle, inner sep=5mm,draw=black!100, fit= (x)] {};
\end{tikzpicture}
\caption[Probabilistic graphical model for Bayesian data models in Chapter \ref{ch:ExBAT}.]{Probabilistic graphical model for all Bayesian data models in this chapter. We observe $\mathbf{n}$ samples of the data $X$ which depends on a hidden random parameter $\theta$. The likelihood function $g(x|\theta)$ captures this dependence. The hidden random parameter $\theta$ itself depends on another hidden hyperparameter $\tau$ which has an unconditional hyperprior pdf $\pi(\tau)$. The conditional prior $h(\theta)$ describes the distribution of the hidden parameter $\theta$ and its dependence on the hyperparameter $\tau$.}
\label{fig:ExBAT-PGM}
\end{figure}
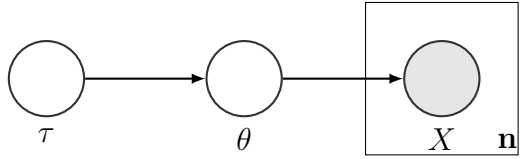

Conditioning the original prior $h(\theta)$ on $\tau$ adds complexity and a new dimension to the posterior pdf:
\begin{equation}
f(\theta, \tau| x) = \frac{g(x|\theta)\,h(\theta|\tau)\,\pi(\tau)}{\iint g(x|\theta)\,h(\theta|\tau)\,\pi(\tau) ~d\theta ~d\tau}.
\end{equation}
Marginalizing or integrating over $\tau$ removes this extra dimension and returns the posterior pdf for the parameter of interest $\theta$:
\begin{align}
f(\theta|x) &\sim \int g(x|\theta)\,h(\theta|\tau)\,\pi(\tau) \,d\tau.
\end{align}
Thus hierarchical Bayes has the benefit of working with a more flexible and descriptive prior but at the computational cost of a new integration.  The approach of empirical Bayes \parencite{carlin-louis2009,hogg-mckean-craig2005} would simply replace the random variable $\tau$ with a numerical proxy such as its most probable value. That approach is simpler to compute but ignores a lot of information in the hyperprior pdf.

Figure \ref{fig:hierarchical_bayes} shows marginal posterior approximations for a hierarchical Bayesian model. The posterior approximations use fuzzy approximators for the prior and hyperprior pdfs. The data model for the posterior pdf is a hierarchical conjugate normal model with a variance hyperparameter. The likelihood is Gaussian with unknown mean $g(x|\theta) = N\left(x=\frac{1}{16}|\theta\right)$. The prior pdf $h$ for $\theta$ is a zero-mean Gaussian with unknown variance $\tau$.  So $h(\theta|\tau)$ is $N(0,\tau)$. We model the pdf of $\tau$ with an inverse gamma (IG) hyperprior pdf: $\tau \sim IG(2,1)$ where $IG(\alpha, \beta) = \pi(\tau) = \frac{\beta^\alpha e^{-\beta/\tau}}{\Gamma(\alpha) \tau^{\alpha+1}}$. The inverse gamma prior is conjugate to the normal likelihood and so the resulting posterior is inverse gamma. Thus we have conjugacy in both the mean and variance parameters. The posterior approximator $F(\theta|x)$ uses a fuzzy approximation of the truncated hyperprior $\pi(\tau)$.  Figure \ref{fig:IG-FuzzyApprox} shows the sinc-SAM fuzzy approximator $\Pi(\tau)$ for the truncated hyperprior $\pi(\tau)$.

\begin{figure}[ht!]
\centerline{	\includegraphics[width=\textwidth]{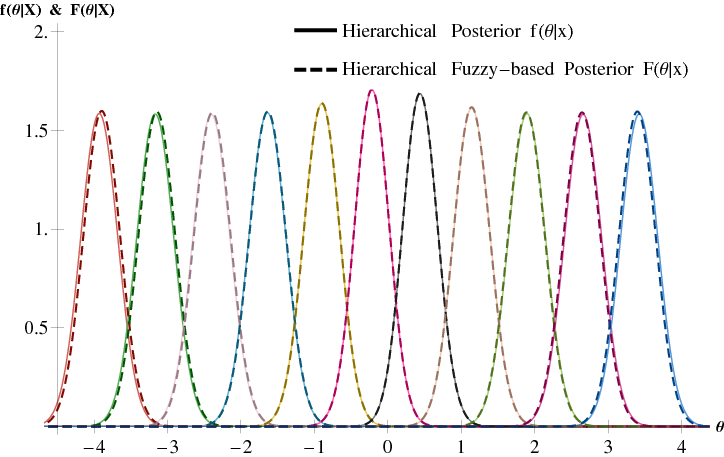} }
\caption[Hierarchical Bayes posterior pdf approximation using a fuzzy hyperprior]{Hierarchical Bayes posterior pdf approximation using a fuzzy hyperprior.  The plot shows the fuzzy approximation for 11 normal posterior pdfs.  These posterior pdfs use 2 levels of prior pdfs.  The first prior pdf $h(\theta|\tau)$ is $N(0,\tau)$ where $\tau$ is a random variance hyperparameter.  The distribution of $\tau$ is the inverse-gamma (IG) hyperprior pdf.  $\tau \sim IG(2,1)$ where $IG(\alpha, \beta) \equiv \pi(\tau) = \frac{\beta^\alpha e^{-\beta/\tau}}{\Gamma(\alpha) \tau^{\alpha+1}}$.  The likelihood function is $g(x|\theta) = N(\theta|\frac{1}{16})$.  The 11 pdfs are posteriors for the observations $x = -4, -3.25, -1.75, -1, -0.25, 0.5, 1.25, 2.75$, and 3.5.  The approximate posterior $F(\theta|x)$ uses a fuzzy approximation for the inverse-gamma hyperprior $\pi(\tau)$ (1000 uniform sample points on the support [0, 4], 15 rules, and 6000 learning iterations).  The posterior pdfs show the distribution of $\theta$ given the data $x$.}
\label{fig:hierarchical_bayes}
\end{figure}

\begin{figure}[ht]
\centerline{	\includegraphics[width=0.75\textwidth]{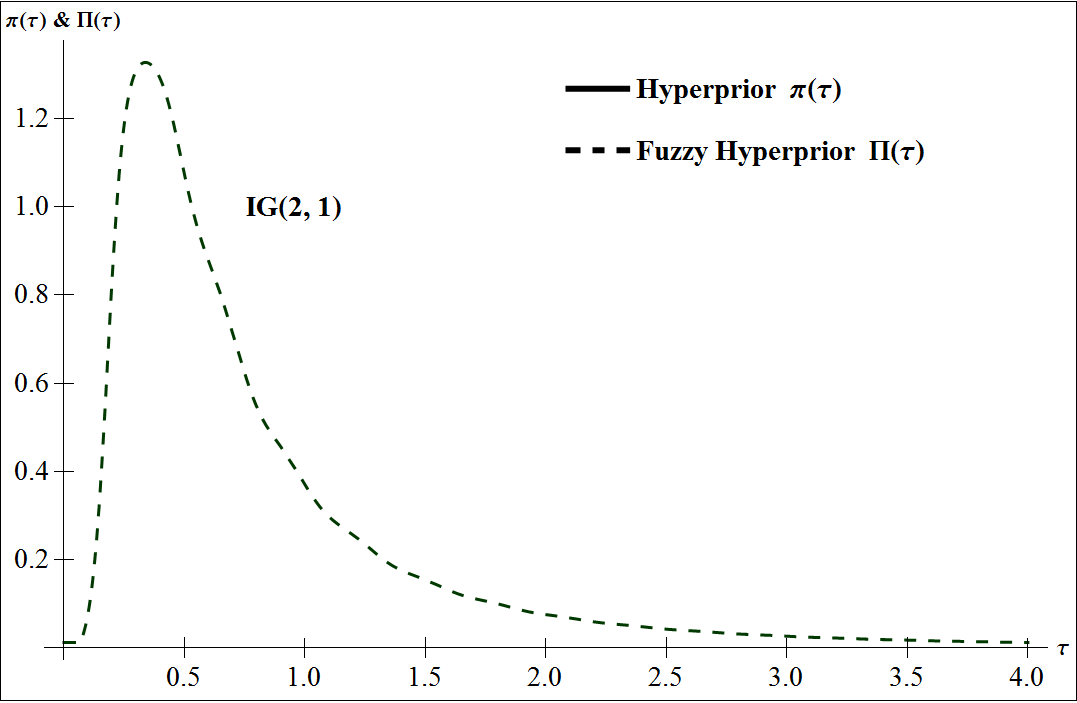}	}
\caption[Comparison between inverse-gamma (IG) hyperprior $\pi(\tau)$ and its fuzzy approximation]{Comparison between inverse-gamma (IG) hyperprior $\pi(\tau)$ and its fuzzy approximation. The hyperprior pdf is the $IG(2,1)$ pdf that describes the random parameter $\tau$ that appears as the variance in a normal prior. The approximating fuzzy hyperprior $\Pi(\tau)$ used 15 rules in a SAM with Gaussian if-part sets. The fuzzy approximator used 1000 uniform samples from $[0, 4]$ and $6000$ training iterations. 
}
\label{fig:IG-FuzzyApprox}
\end{figure}

These figures show that approximate hierarchical Bayesian inference is feasible with function approximators like fuzzy ASAMs.  The next section gives a theoretical guarantee that this kind of uniform posterior approximation is always possible.

\section{Uniform Approximation for Hierarchical Bayesian Inference}%(Triply Fuzzy Bayesian Inference: Uniform Approximation)
The posterior approximation scheme for hierarchical Bayes uses uniform approximators $\Pi(\tau)$, $H(\theta|\tau)$, and $G(x|\theta)$ in place of their respective approximands: the hyperprior pdf $\pi(\tau)$, the prior pdf $h(\theta|\tau)$, and the likelihood $g(x|\theta)$. This triple function approximation scheme builds on the double function approximation scheme for non-hierarchical models in  Chapter \ref{ch:BAT}. The double function approximation scheme uses approximators for just the prior pdf $h(\theta)$ and the likelihood $g(x|\theta)$.

Theorem \ref{thm:ExBAT} below proves that the triple function approximation scheme produces uniform posterior pdf approximations when the input model-function approximators are uniform. This result extends the  Bayesian Approximation Theorem in Chapter \ref{ch:BAT}. So we call the theorem the \emph{Extended Bayesian Approximation Theorem}.  The proof of this theorem is quite general and does not depend on the structure of the uniform approximators.  So the approximators can be fuzzy systems, neural networks or polynomials or any other uniform function approximators.

\subsection{The Extended Bayesian Approximation Theorem}
The statement and proof of the Extended Bayesian approximation theorem requires the following notation.The hyperprior pdf is $\pi(\tau)$. The prior is $h(\theta|\tau)$ and the likelihood is $g(x|\theta)$. The 2-D pdf $p(\theta, \tau) = h(\theta|\tau) \pi(\tau)$ describes the dependence between $\theta$ and $\tau$. $P(\theta,\tau)$ is a 2-dimensional uniform approximator for $p(\theta, \tau) = h(\theta|\tau)\pi(\tau)$. $G(x|\theta)$ is a 2-dimensional uniform approximator for $g(x|\theta)$. Let $\mathcal{D}$ denote the set of all $(\theta,\tau)$ and let $\mathcal{X}$ denote the set of all $x$.  Assume that $\mathcal{D}$ and $\mathcal{X}$ are compact. Define the Bayes factors as $q(x) = \int_{\mathcal{D}} p(\theta, \tau) g(x|\theta) d\tau d\theta$ and $Q(x) = \int_{\mathcal{D}} P(\theta, \tau) G(x| \theta) d\tau d\theta$. Assume that $q(x)>0$ so that the posterior $f(\theta, \tau|x)$ is well-defined for any sample data $x$.

We can now state the Extended Bayesian Approximation Theorem.  The proof relies on the Extreme Value Theorem \parencite{munkres2000}.  The proof applies to \emph{any} kind of uniform approximator.  The vector structure of the proof also allows the hyperprior prior to depend on its own hyperprior and so on.

\newpage

\begin{thm}{\bf{[Extended Bayesian Approximation Theorem]:}}\label{thm:ExBAT}\\
	Suppose that $h(\theta|\tau)$, $\pi(\tau)$, and $g(x|\theta)$ are bounded and continuous. Suppose that 
	\begin{equation}
		\Pi(\tau)H(\theta|\tau)G(x|\theta)  = P(\theta, \tau)G(x|\theta) \neq 0 
	\end{equation} almost everywhere. 
	
	Then the posterior pdf approximator $F(\theta,\tau|x) = P G/Q$ uniformly approximates $f(\theta,\tau|x)$. Thus for all $\epsilon > 0:~  | F(\theta,\tau|x)  -  f(\theta,\tau|x) |  <  \epsilon$ for all $x$ and all $(\theta,\tau)$.
\end{thm}

%%%%%%%%%%%%%%%%%%%%%%%%%%%%%%
\begin{proof}
Write the posterior pdf $f(\theta, \tau|x)$ as $f(\theta, \tau|x) = \frac{p(\theta, \tau) g(x|\theta)}{q(x)}$ and its approximator $F(\theta, \tau|x)$ as $F(\theta, \tau|x) = \frac{P(\theta, \tau) G(x|\theta)}{Q(x)}$. The SAM approximations for the prior and likelihood functions are uniform \parencite{kosko-fuzeng}. So they have approximation error bounds $\epsilon_p$ and $\epsilon_g$ that do not depend on $x$ or $\theta$:
\begin{align}
	|\Delta P| &<~ \epsilon_p \label{eq:Defn-DelP} \\
	\mbox{and}& \\
	|\Delta G| &<~ \epsilon_g	\label{eq:Defn-DelG-ExBAT}
\end{align}
where $\Delta P = P(\theta, \tau)-p(\theta, \tau)$ and $\Delta G = G(x|\theta)-g(x|\theta)$. The posterior error $\Delta F$ is
\begin{equation}
	\Delta F ~=~ F-f ~=~ \frac{ P G}{Q(x)} - \frac{p g}{q(x)}.  
	\label{eq:DelF-Expr-ExBAT}
\end{equation}
Expand $PG$ in terms of the approximation errors to get
\begin{align}
	PG &= (\Delta P + p )(\Delta G +g) \\
	   &= \Delta P \Delta G + \Delta P g +p \Delta G + p g.
\end{align} 
We have assumed that $PG \neq 0$ almost everywhere and so $Q \neq 0$.  We now derive an upper bound for the Bayes-factor error $\Delta Q = Q - q$:
\begin{align}
\Delta Q &= \int_\mathcal{D}(\Delta P \Delta G + \Delta P g +p \Delta G + p g - pg) ~d\tau ~d\theta. \hspace{.2in}
\end{align}
So
\begin{align}
	|\Delta Q| &\leq \int_\mathcal{D}{|\Delta P \Delta G + \Delta P g +p \Delta G|} ~d\tau ~d\theta\\
	&\leq \int_\mathcal{D}\big(|\Delta P| |\Delta G| + |\Delta P| g + p |\Delta G|\big) ~d\tau ~d\theta\\
	&< \int_\mathcal{D} (\epsilon_p \epsilon_g + \epsilon_p g + p \epsilon_g) ~d\tau ~d\theta~~~~~
	\mbox{by (\ref{eq:Defn-DelP})}.
\end{align}

Parameter set $\mathcal{D}$ has finite Lebesgue measure $\displaystyle m(\mathcal{D}) = \int_\mathcal{D} ~d\tau ~d\theta < \infty$ because $\mathcal{D}$ is a compact subset of a metric space and thus \parencite{munkres2000} it is (totally) bounded.
Then the bound on $\Delta Q$ becomes
\begin{align}
	|\Delta Q| & <  m(\mathcal{D}) \epsilon_p \epsilon_g + \epsilon_g + \epsilon_p \int_\mathcal{D}{g(x|\theta)} ~d\theta
\end{align}
because $\displaystyle \int_\mathcal{D} p(\theta, \tau) d\tau ~d\theta = 1$ and $g$ has no dependence on $\tau$.

We now invoke the extreme value theorem \parencite{folland1999}. The extreme value theorem states that a continuous function on a compact set attains both its maximum and minimum. The extreme value theorem allows us to use maxima and minima instead of suprema and infima. Now $\int_\mathcal{D}{g(x|\theta)} ~d\theta$ is a continuous function of $x$ because $g(x|\theta)$ is a continuous nonnegative function. The range of $\int_\mathcal{D}{g(x|\theta)}~d\theta$ is a subset of the right half line $(0,\infty)$ and its domain is the compact set $\mathcal{D}$. So $\int_\mathcal{D}{g(x|\theta)}~d\theta$ attains a finite maximum value. Thus
\begin{equation}
	|\Delta Q| < \epsilon_q 
	\label{eq:DelQ-bound-ExBAT}
\end{equation}
where we define the error bound $\epsilon_q$ as
\begin{align}
\epsilon_q &=  m(\mathcal{D}) \epsilon_p \epsilon_g + \epsilon_g + \epsilon_p~\max_{x}\left\{ \int_\mathcal{D}{g(x|\theta)} d\theta \right\}.
\end{align}
Rewrite the posterior approximation error $\Delta F$ as
\begin{align}
\Delta F & =  \frac{qPG - Qpg}{q Q}\\ 
&= \frac{q(\Delta P \Delta G + \Delta P g +p \Delta G + pg)~ - ~ Q pg} {q(q + \Delta Q)}
\end{align}
Inequality (\ref{eq:DelQ-bound-ExBAT}) implies that $-\epsilon_q < \Delta Q < \epsilon_q$ and that $(q - \epsilon_q)< (q + \Delta Q) < (q + \epsilon_q)$. Then (\ref{eq:Defn-DelP})
gives similar inequalities for $\Delta P$ and $\Delta G$. So
\begin{multline}
\lefteqn{ \frac{q[-\epsilon_p \epsilon_g - \text{min} (g) \epsilon_p - \text{min} (h) \epsilon_g ]~ - ~ \epsilon_q pg} {q(q - \epsilon_q)} < ~\Delta F~} \\
 <~ \frac{q[\epsilon_p \epsilon_g + \text{max} (g) \epsilon_p + \text{max} (h) \epsilon_g ]~ + ~ \epsilon_q pg} {q(q - \epsilon_q)} \;.	\label{eq:Ptwise-Bound-ExBAT}
\end{multline}
The extreme value theorem ensures that the maxima in (\ref{eq:Ptwise-Bound-ExBAT}) are finite. The bound on the approximation error $\Delta F$ does not depend on $\theta$. But $q$ still depends on the value of the data sample $x$. So (\ref{eq:Ptwise-Bound-ExBAT}) guarantees at best a pointwise approximation of $f(\theta, \tau|x)$ when $x$ is arbitrary. We can improve the result by finding bounds for $q$ that do not depend on $x$.  Note that $q(x)$ is a continuous function of $x \in X$ because $pg$ is continuous. So the extreme value theorem ensures that the Bayes factor $q$ has a finite upper bound and a positive lower bound. 

The term $q(x)$ attains its maximum and minimum by the extreme value theorem. The minimum of $q(x)$ is positive because we assumed $q(x)>0$ for all $x$. H\"older's inequality gives $|q| \leq \left( \int_{\mathcal{D}}|p|~d\tau ~d\theta\right) \left(\left\|g(x,\theta)\right\|_{\infty} \right) = \left\|g(x,\theta)\right\|_{\infty}$ since $p$ is a pdf. So the maximum of $q(x)$ is finite because $g$ is bounded: 
$0 ~<~ \min\{q(x)\} ~\leq~$ $ \max\{ q(x)\} ~<~ \infty$.
Then
\begin{align}
\epsilon_{-} < \Delta F < \epsilon_{+}
\end{align}
if we define the error bounds $\epsilon_{-}$ and $\epsilon_{+}$ as
\begin{align} 
\epsilon_{-} &= \frac{\left(-\epsilon_p \epsilon_g - \min\{g\} \epsilon_p - \min\{p\} \epsilon_g\right) \min\{q\}   -  pg \epsilon_q }{\min\{q\} \left( \min\{q\} - \epsilon_q\right) } \\
\epsilon_{+} &= \frac{\left(\epsilon_p \epsilon_g + \max\{g\} \epsilon_p + \max\{p\} \epsilon_g\right) \max\{g\}  +  pg \epsilon_q }{\min\{q\} \left( \min\{q\} - \epsilon_q\right) }. \hspace{.2in}
	\label{eq:Final-Bound-ExBAT}
\end{align}

Now $\epsilon_q \rightarrow 0$ as $\epsilon_g \rightarrow 0$ and $\epsilon_p \rightarrow 0$. So $\epsilon_{-} \rightarrow 0$ and $\epsilon_{+} \rightarrow 0$.
The denominator of the error bounds must be non-zero for this limiting argument. We can guarantee this when $\epsilon_q < \min\{q\}$. This condition is not restrictive because the functions $p$ and $g$ fix or determine $q$ independent of the approximators $P$ and $G$ involved and because $\epsilon_q \to 0$ when $\epsilon_p \to 0$ and $\epsilon_g \to 0$. So we can achieve arbitrarily small $\epsilon_q$ that satisfies $\epsilon_q < \min\{q\}$ by choosing appropriate $\epsilon_p$ and $\epsilon_g$.
Then $\Delta F \rightarrow 0$ as $\epsilon_g \rightarrow 0$ and $\epsilon_p \rightarrow 0$. 
So $|\Delta F| \rightarrow 0$.

Theorem \ref{thm:ExBAT} now follows from lemma \ref{eq:UniLemma} (below). Lemma (\ref{eq:UniLemma}) implies Theorem \ref{thm:ExBAT} because we have uniform approximators for $f(\theta, \tau|x)$. We can integrate the nuisance parameter $\tau$ away to get a posterior approximation in terms of $\theta$ alone. Thus $F \longrightarrow f$ uniformly implies $\int F~d\tau \longrightarrow \int f~d\tau$ uniformly.
\end{proof}
%%%%%%%%%%%%%%%%%%%%%%%%%%%%%%
\begin{lem}{\bf{[Uniform Integral Approximation Lemma]:}}\label{eq:UniLemma}\\
If $Y$ is compact and $f_n \longrightarrow f$ uniformly then
\begin{equation}
\int_Y f_n(x,y,\vec{z})~dy \longrightarrow \int_Y f(x,y,\vec{z})~dy \quad \textrm{uniformly.} 
\end{equation}
\end{lem}

\begin{proof}
The uniform convergence of the sequence $f_n$ to $f$ implies that for all $\epsilon >0$ there is an $n\in \mathbb{N}$ such that
\[|f_n(x,y,\vec{z})-f(x,y,\vec{z})| ~<~ \epsilon\] 
for all $(x,y,\vec{z}) \in X\!\times\!Y\!\times\!\prod Z_i$. Thus
\begin{equation}
-\epsilon ~<~ f_n(x,y,\vec{z}) - f(x,y,\vec{z}) ~<~ \epsilon \textrm{.}
\end{equation}
Thus
\[-\int_Y \epsilon ~dy ~<~ \int_Y f_n(x,y,\vec{z})  ~dy - \int_Y f(x,y,\vec{z})  ~dy \]
and
\[\int_Y f_n(x,y,\vec{z})  ~dy - \int_Y f(x,y,\vec{z})  ~dy ~<~ \int_Y \epsilon ~dy \;. \]
$Y$ has finite Lebesgue measure $m(Y) = \int_Y dy$ because $Y$ is a compact set on a finite-dimensional metric space. Define $s_n(x,\vec{z}) =\int_Y f_n(x,y,\vec{z})~dy$ and $s(x,\vec{z}) =\int_Y f(x,y,\vec{z})~dy$. Then
\begin{align}
 -\epsilon ~m(Y) ~<~ s_n(x,\vec{z}) - s(x,\vec{z}) ~<~ \epsilon ~m(Y) \;. \\
 \text {Thus} \quad |s_n(x,\vec{z}) - s(x,\vec{z})| ~<~ \epsilon ~m(Y) \;.
\end{align}
Define $\epsilon'$ as $\epsilon' = \epsilon m(Y)$. Then for all $\epsilon'>0$ there exists an $n\in \mathbb{N}$ such that \[|s_n(x,\vec{z}) - s(x,\vec{z})| ~<~ \epsilon'\] 
for all $(x,\vec{z}) \in X \times \prod Z_i$.

\noindent Therefore 
\begin{equation}
\int_Y f_n(x,y,\vec{z})~dy \longrightarrow \int_Y f(x,y,\vec{z})~dy
\end{equation}
uniformly in $x$ and $\vec{z}$. 
\end{proof}
Lemma \ref{eq:UniLemma} guarantees that uniform approximation still holds after marginalizing a multidimensional uniform approximator. 

The proofs of Theorem \ref{thm:ExBAT} and Lemma \ref{eq:UniLemma} also extend to $n$-dimensional fuzzy posterior approximators for $n$-dimensional functions $f$ i.e.  data models with more than one hyperparameter (PGMs with longer linear chains of unobserved parameters than Figure \ref{fig:ExBAT-PGM}). Thus an $n$-dimensional approximator $F$ and its marginalizations uniformly approximate the $n$-dimensional posterior $f$ and its marginal pdfs.

\subsection{Adaptive Fuzzy Systems for Hierarchical Bayesian Inference}
Hierarchical Bayesian inference uses priors with hyperparameters or multi-dimensional priors and likelihoods. These model functions are multi-dimensional and so are their approximators. Standard additive model (SAM) fuzzy systems \parencite{kosko-fat1994,Kosko1995opt,kosko-fuzeng,mitaim-kosko2001ASAM} can uniformly approximate multi-dimensional functions according to the Fuzzy Approximation Theorem \parencite{kosko-fat1994}. Theorem \ref{thm:ExBAT} guarantees that these fuzzy approximators will produce a uniform approximator for the posterior.

A prior parameter $\theta$ with one hyperparameter $\tau$ will have a conditional pdf $h(\theta|\tau)$ which has a $2$--D functional representation. A $2$--D adaptive SAM (ASAM) $H$ can learn to approximate $h$ uniformly.  $2$--D ASAMs fuzzy SAMs tune $2$--D set functions (like the $2$--D sinc and Gaussian set functions below) using the appropriate learning laws. The $2$--D sinc set function $a_j$ has the form
\vspace{-.1in}
\begin{align}
a_j(x,y) &=\textrm{sinc}\left(\frac{x-m_{x,j}}{d_{x,j}}\right)\textrm{sinc}\left(\frac{y-m_{y,j}}{d_{y,j}}\right) 
\end{align}
with parameter learning laws \parencite{kosko-fuzeng,mitaim-kosko2001ASAM} for the location $m$
\begin{multline}
\lefteqn{m_{x,j}(t+1) = m_{x,j}(t)+ \mu_t\varepsilon(x,y)[c_j-F(x,y)] \left(\frac{p_j(x,y)}{a_j(x,y)} \right) \times} \\
\left(a_j(x,y)-\cos\left(\frac{x-m_{x,j}}{d_{x,j}}\right) \textrm{sinc}\left(\frac{y-m_{y,j}}{d_{y,j}}\right) \right) \left(\frac{1}{x-m_{x,j}}\right)
\end{multline}
and the dispersion $d$
\begin{multline}
\lefteqn{d_{x,j}(t+1) ~=~ d_{x,j}(t) + \mu_t\varepsilon(x,y)[c_j-F(x,y)]\times} \\
\left(a_j(x,y)-\cos\left(\frac{x-m_{x,j}}{d_{x,j}}\right) \textrm{sinc}(\frac{y-m_{y,j}}{d_{y,j}}) \right) \frac{p_j(x,y)}{a_j(x,y)}\left(\frac{1}{d_{x,j}}\right).
\end{multline}
The Gaussian learning laws have the same functional form as its $1$--D case. We replace $a_j(\theta)$ with $a_j(x,y)$:
\begin{align}
a_j(x,y) = \exp\left[-\left(\frac{x-m_{x,j}}{d_{x,j}}\right)^2-\left(\frac{y-m_{y,j}}{d_{y,j}}\right)^2\right]
\vspace{-.1in}
\end{align}
with parameter learning laws
\vspace{-.1in}
\begin{align}
m_{x,j}(t+1) &= m_j(t)+\mu_t\varepsilon(x,y)
p_j(x,y)[c_j-F(x,y)] \frac{x-m_{x,j}}{d_{x,j}^2}
\label{eq:GaussLearnmj2D}\\
d_{x,j}(t+1) &= d_j(t)+\mu_t\varepsilon(x,y)
p_j(x,y)[c_j-F(x,y)] \frac{(x-m_{x,j})^2}{d_{x,j}^3}.
\hspace{.25in}
\label{eq:GaussLearndj2D}
\end{align}
These learning laws are the result of applying gradient descent on the squared approximation error function for the associated SAM system.

\subsection{Triply Fuzzy Bayesian Inference}
The term {\it triply fuzzy} describes the special case where $\Pi(\tau)$, $H(\theta|\tau)$, and $G(x|\theta)$ are uniform fuzzy approximators. Triply fuzzy approximations allow users to state priors and hyper-priors in words or rules as well as to adapt them from sample data.

Figure \ref{fig:2D-Post} shows a simulation instance of triply fuzzy function approximation in accord with the Extended Bayesian Approximation Theorem. It shows that the 2-D fuzzy approximator $F(\mu,\tau|x)$ approximates the posterior pdf $f(\mu,\tau|x)\propto g(x|\mu) h(\mu|\tau) \pi(\tau)$ for hierarchical Bayesian inference. The sample data $x$ is normal. A normal prior distribution $h(\mu|\tau) = N(1,\sqrt \tau)$ models the population mean $\mu$ of the data. An inverse gamma $IG(2,1)$ hyperprior models the variance $\tau$ of the prior. An inverse gamma hyperprior $\pi(\tau) = IG(\alpha,\beta)$ has the form $\pi(\tau) = e^{-\frac{\beta }{\tau}} \left(\frac{\beta }{\tau}\right)^{\alpha } \Big/ \tau \Gamma(\alpha)$ for $\tau > 0$ where $\Gamma$ is the gamma function. The posterior fuzzy approximator $F(\mu,\tau|x)$ is proportional to the triple-product approximator $G(x|\mu)H(\mu|\tau)\Pi(\tau)$. These three adaptive SAMs  separately approximate the three corresponding Bayesian pdfs. $G(x|\mu)$ approximates the 1-D likelihood $g(x|\mu)$. $H(\mu|\tau)$ approximates the 2-D conditional prior pdf $h(\mu|\tau)$. And $\Pi(\tau) $ approximates the 1-D hyperprior pdf $\pi(\tau)$. 

Figure \ref{fig:2D-Post}(a) shows the approximand or the original posterior pdf. Figure \ref{fig:2D-Post}(b) shows the adapted triply fuzzy approximator of the posterior pdf using a conditional 2-D approximator $H(\mu|\tau)$ for $h(\mu|\tau)$ and a separate 1-D approximator $\Pi(\tau)$ for $\pi(\tau)$. Figure \ref{fig:2D-Post}(c) shows a simulation instance where the posterior approximator $F(\mu,\tau|x)$ uses a single 2-D approximator $P(\mu,\tau)$ for the joint prior pdf $p(\mu, \tau)=h(\mu|\tau) \pi(\tau)$. Both fuzzy posterior approximators $F(\mu,\tau|x) \propto G(x|\mu)H(\mu|\tau)\Pi(\tau)$ and $F(\mu,\tau|x) \propto G(x|\mu)P(\mu,\tau)$ uniformly approximate the posterior pdf $f(\mu,\tau|\theta)$.

Figure \ref{fig:TripFuzz-nonconj} shows an example of triply fuzzy function approximation for an arbitrary non-conjugate Bayesian model. The likelihood is zero-mean Gaussian with unknown standard deviation $\sigma$. The prior $h(\sigma|\tau)$ for the standard deviation $\sigma$ is an arbitrary beta distribution that depends on a hyperparameter $\tau$ with a mixture-beta hyperprior $\pi(\tau)$. The 2-D fuzzy approximator $F(\sigma,\tau|x)$ uniformly approximates the posterior pdf $f(\sigma,\tau|x)$ for this arbitrary model. The quality of this arbitrary posterior approximation shows that approximation feasibility does not require conjugacy in the Bayesian data model.

\begin{figure}[ht!]
\centerline{\includegraphics[width=\textwidth]{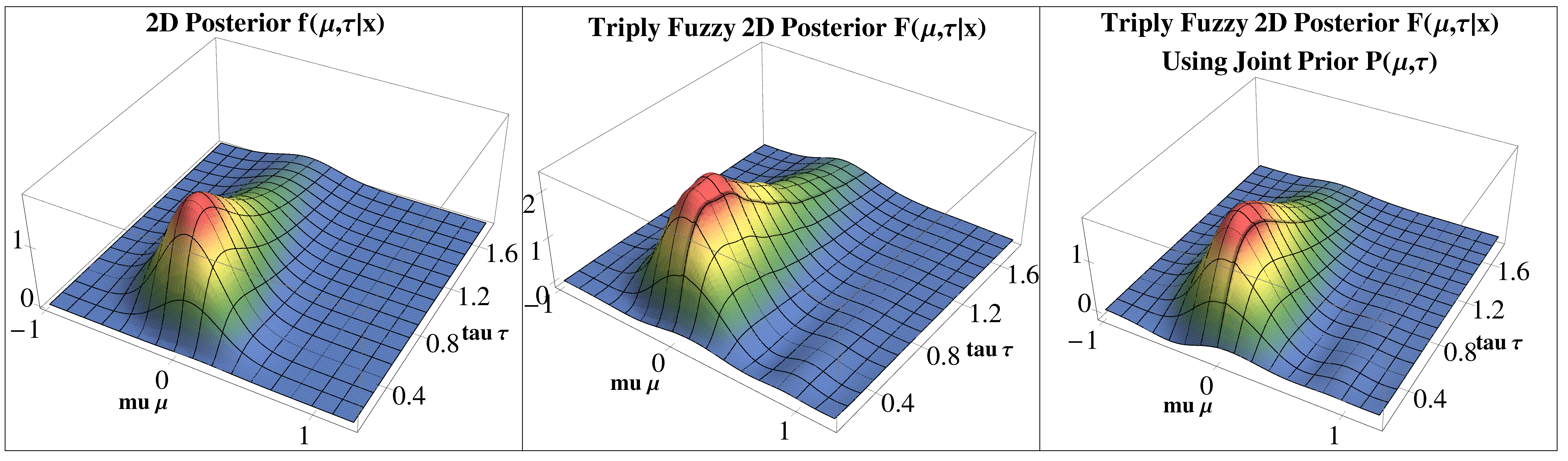}}
\centerline{(a)\hspace{1.4in}(b)\hspace{1.4in}(c)}
\caption[Triply fuzzy Bayesian inference: comparison between a 2-D posterior $f(\mu,\tau|x)$ and its triply fuzzy approximator $F(\mu,\tau|x)$]{Triply fuzzy Bayesian inference: comparison between a 2-D posterior $f(\mu,\tau|x) \propto g(x|\mu)h(\mu|\tau)\pi(\tau)$ and its triply fuzzy approximator $F(\mu,\tau|x)$. The first panel shows the approximand $f(\mu,\tau|x)$. The second panel shows a triply fuzzy approximator $F(\mu,\tau|x)$ that used a 2-D fuzzy approximation $H(\mu|\tau)$ for the conditional prior $h(\mu|\tau)$ and a 1-D fuzzy approximation $\Pi(\tau)$ for the hyperprior pdf $\pi(\tau)$ and a 1-D fuzzy likelihood-pdf approximator $G(x|\mu)$. The third panel shows a triply fuzzy approximator $F(\mu,\tau|x)$ that used a 2-D fuzzy approximation $P(\mu,\tau)=(H \times \Pi)(\mu,\tau)$ for the joint prior $p(\mu,\tau)=(h \times \pi)(\mu,\tau)$. The likelihood approximation is the same as in the second panel. The sinc-SAM fuzzy approximators $H(\mu|\tau)$ and $P(\mu,\tau)$ use 6 rules to approximate the respective 2-D pdfs $h(\mu|\tau) = N(1,\sqrt \tau)$ and $ h(\mu|\tau) \pi(\tau) = N(1,\sqrt \tau) IG(2,1)$. The hyperprior Gaussian-SAM approximator $\Pi(\tau)$ used 12 rules to approximate an inverse-gamma pdf $\pi(\tau) = IG(2,1)$. The Gaussian-SAM fuzzy likelihood approximator $G(x|\mu)$ used 15 rules to approximate the likelihood function $g(x|\mu) = N(\mu, \frac{1}{16})$ for $x = -0.25$. The 2-D conditional prior fuzzy approximator $H(\mu|\tau)$ used 15000 learning iterations based on 6000 uniform sample points. The hyperprior fuzzy approximator $\Pi(\tau)$ used 6000 iterations on 120 uniform sample points. The likelihood fuzzy approximator used 6000 iterations based on 500 uniform sample points.}
\label{fig:2D-Post}
\end{figure}

\begin{figure}[ht!]
%\centerline{\includegraphics[width=0.8\textwidth]{Figures/FODM/TriplyFuzzy-NonConjugate-v2.eps}}
\centerline{\includegraphics[width=0.8\textwidth]{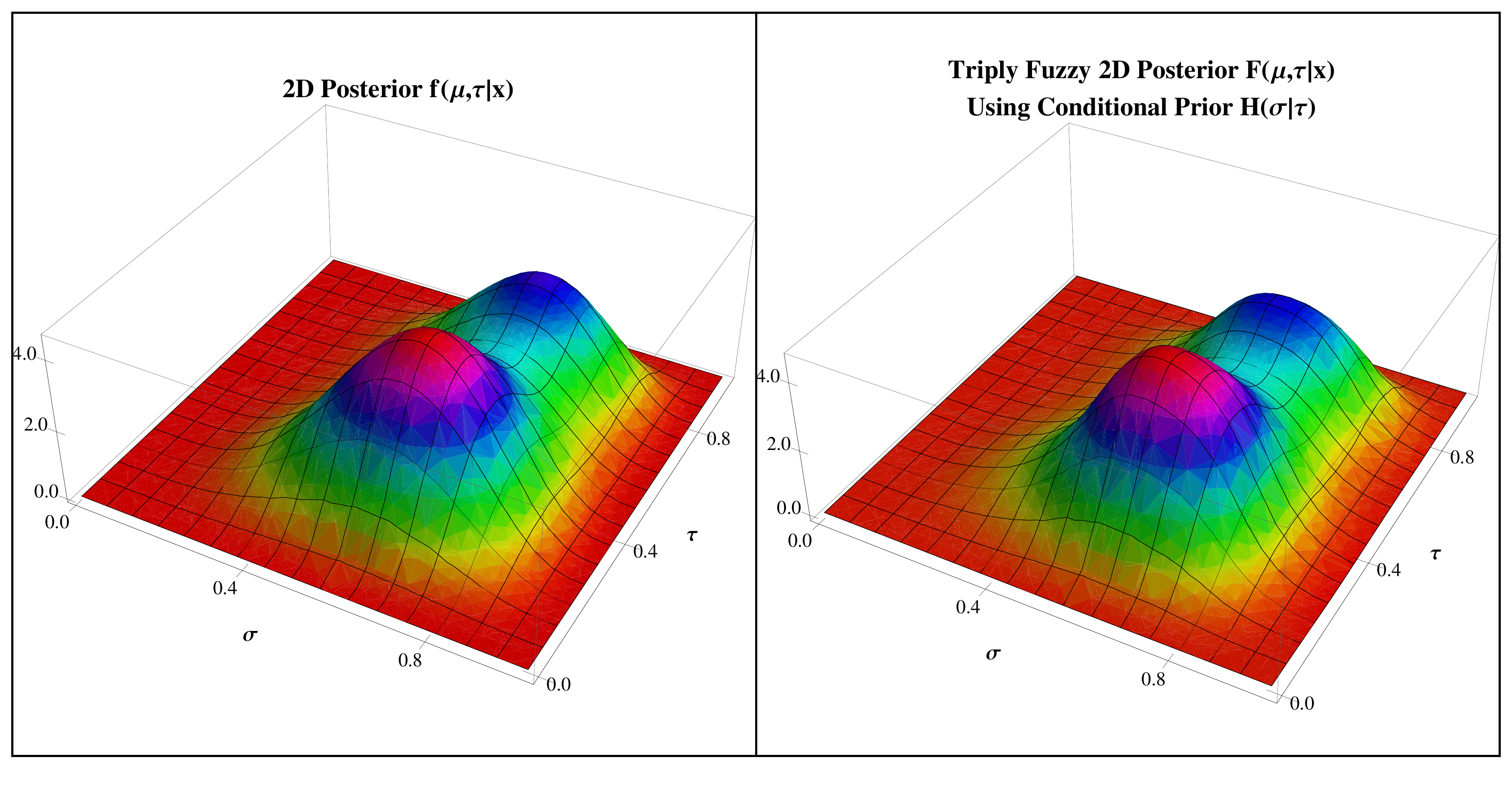}}
\centerline{(a)\hspace{2.1in}(b)}
\caption[Triply fuzzy Bayesian inference for a non-conjugate posterior: comparison between a 2-D non-conjugate posterior and its triply fuzzy approximator]{Triply fuzzy Bayesian inference: comparison between a 2-D non-conjugate posterior $f(\sigma,\tau|x) \propto g(x|\sigma)h(\sigma|\tau)\pi(\tau)$ and its triply fuzzy approximator $F(\sigma,\tau|x)$. The first panel shows the approximand $f(\sigma,\tau|x)$. The second panel shows a triply fuzzy approximator $F(\sigma,\tau|x)$ that used a 2-D fuzzy approximation $H(\sigma|\tau)$ for the conditional prior $h(\sigma|\tau)$ and a 1-D fuzzy approximation $\Pi(\tau)$ for the hyperprior pdf $\pi(\tau)$ and a 1-D fuzzy likelihood-pdf approximator $G(x|\sigma)$. The Gaussian-SAM fuzzy approximator $H(\sigma|\tau)$ used 6 rules to approximate the 2-D pdf $h(\sigma|\tau) = \beta(6+2\tau,4)$. The hyperprior Gaussian-SAM approximator $\Pi(\tau)$ used 12 rules to approximate a beta mixture pdf $\pi(\tau) = \frac{1}{3}\beta(12,4) + \frac{2}{3}\beta(4,7)$. The Gaussian-SAM fuzzy likelihood approximator $G(x|\sigma)$ used 12 rules to approximate the likelihood function $g(x|\sigma) = N(0,\sigma)$ for $x=0.25$. The 2-D conditional prior fuzzy approximator $H(\sigma|\tau)$ used $6000$ learning iterations based on $3970$ uniform sample points. The hyperprior fuzzy approximator $\Pi(\tau)$ used $15000$ iterations on $1000$ uniform sample points. The likelihood fuzzy approximator $G(x|\sigma)$ used $15000$ iterations based on $300$ uniform sample points.}
\label{fig:TripFuzz-nonconj}
\end{figure}

\section{Semi-conjugacy in Fuzzy Posterior Approximation}\label{sec:semi-conj}

Iterative fuzzy Bayesian inference can lead to rule explosion. Iterative Bayesian updates will propagate the convex SAM structure from the fuzzy approximators to the posterior approximator as Theorem \ref{thm:fuzzypost-prodrules} (and \ref{thm:Doubly-Posterior-SAM}) shows. But the updates produce exponentially increasing sets of rules and parameters. This also true for Bayesian models with more than two nested levels. 

Standard Bayesian applications avoid a similar parameter explosion by using conjugate models. Conjugacy keeps the number parameters constant with each update. We define the related idea of \emph{semi-conjugacy} for fuzzy-approximators in Bayes model. Semi-conjugacy is the property by which the if-part sets of the posterior fuzzy approximator inherit the shape of the if-part sets of the prior approximators.

Theorem \ref{thm:semi-conj} and  Corollaries \ref{cor:sconj-Gauss}---\ref{cor:sconj-Lapl} show that updates also preserve the shapes of the if-part sets (semi-conjugacy of if-part set functions) if the SAM fuzzy systems use if-part set functions that belong to conjugate families in Bayesian statistics. Figure \ref{fig:productsets_conjugacy} shows examples of such if-part sets from Corollaries \ref{cor:sconj-Gauss}---\ref{cor:sconj-Lapl}. The conjugacy of Gaussian if-part sets is straightforward.  The conjugacy of the beta, gamma, and Laplace if-part sets is only partial (semi-conjugacy) because we cannot combine the functions' exponents and because two beta set functions or two gamma set functions need not share the same supports.

\begin{figure}[ht!]
%\centerline{\includegraphics[width=0.45\textwidth]{Figures/FODM/Fig3a.eps}
\centerline{\includegraphics[width=0.45\textwidth]{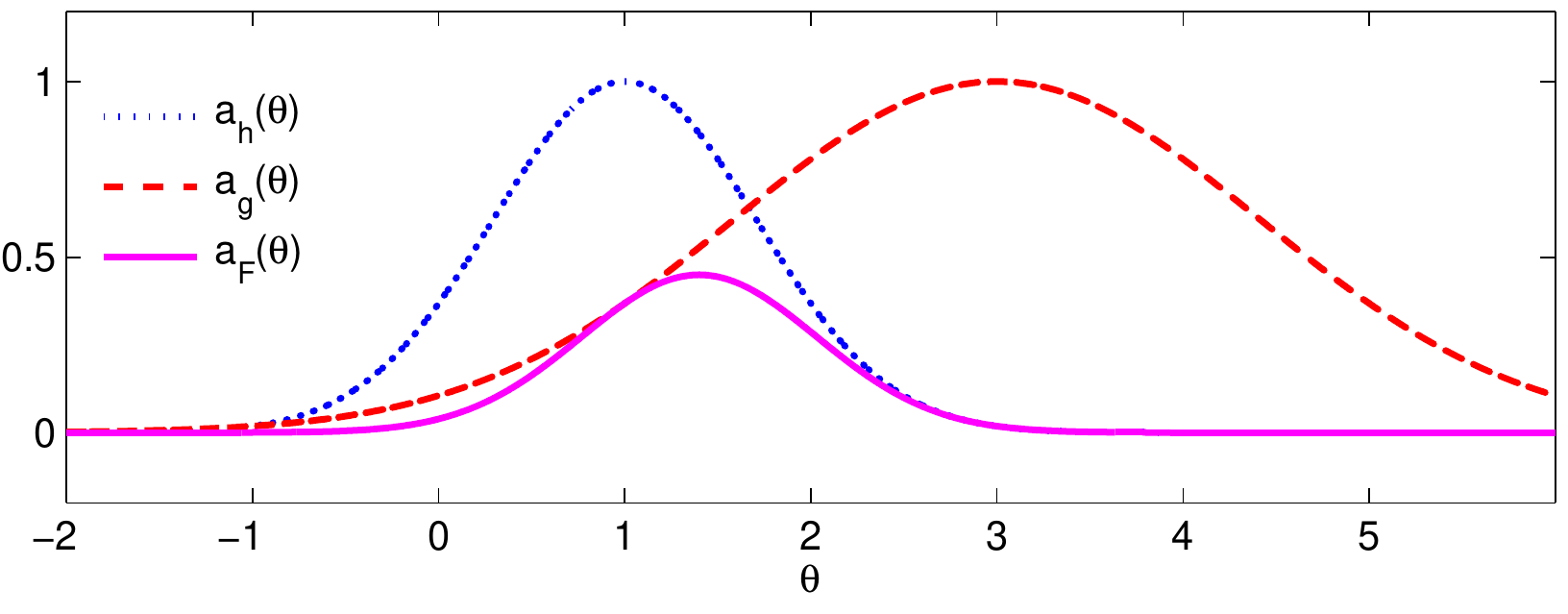}
%\hspace{.15in}\includegraphics[width=0.45\textwidth]{Figures/FODM/Fig3b.eps}}%2.2in
\hspace{.15in}\includegraphics[width=0.45\textwidth]{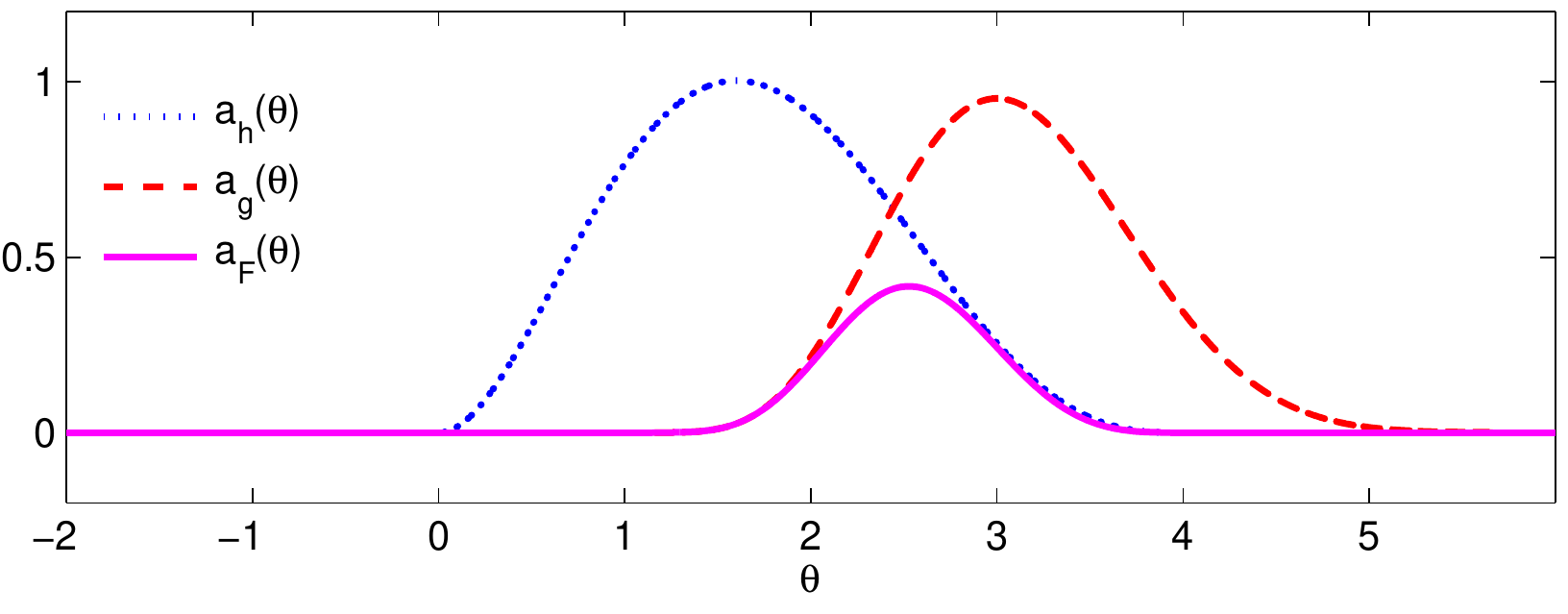}}%2.2in
\vspace{-.1in}
\centerline{\small (a) Gaussian set functions \hspace{1in}(b) beta set functions }
\vspace{.2in}
%\centerline{\includegraphics[width=0.452\textwidth]{Figures/FODM/Fig3c.eps}
\centerline{\includegraphics[width=0.452\textwidth]{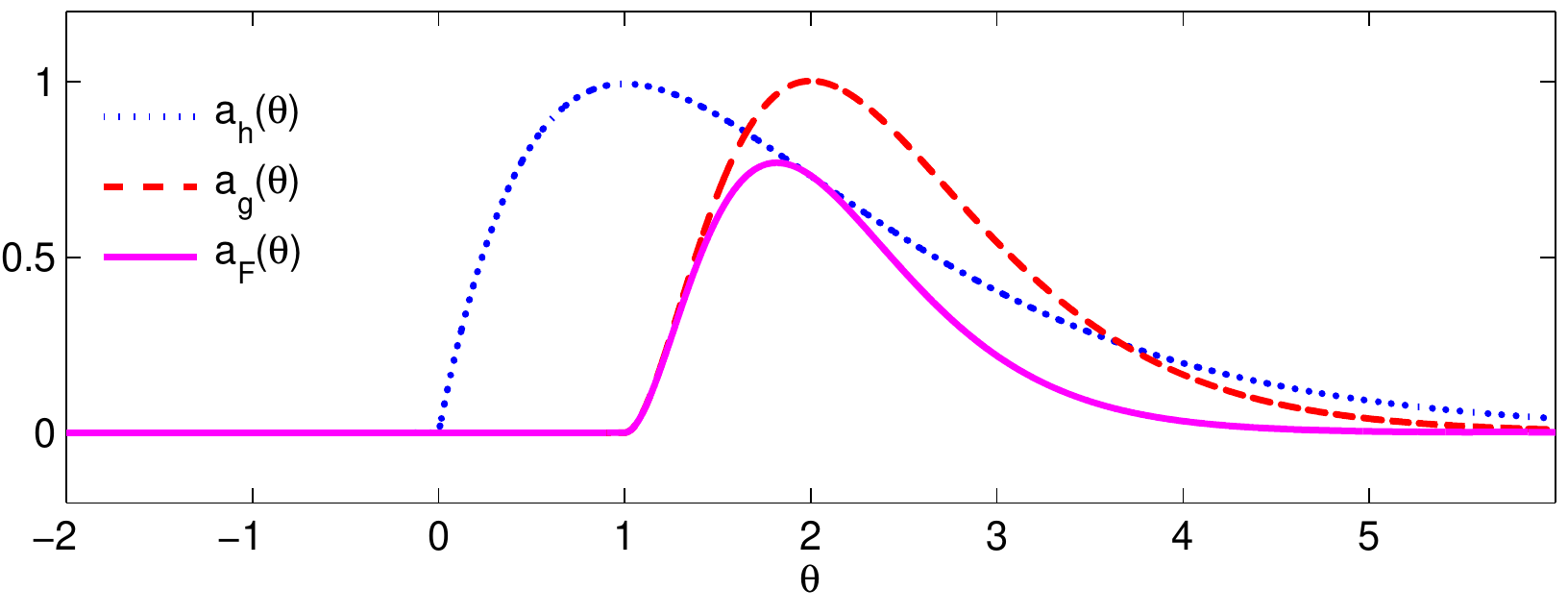}
\hspace{.15in}
\includegraphics[width=0.45\textwidth]{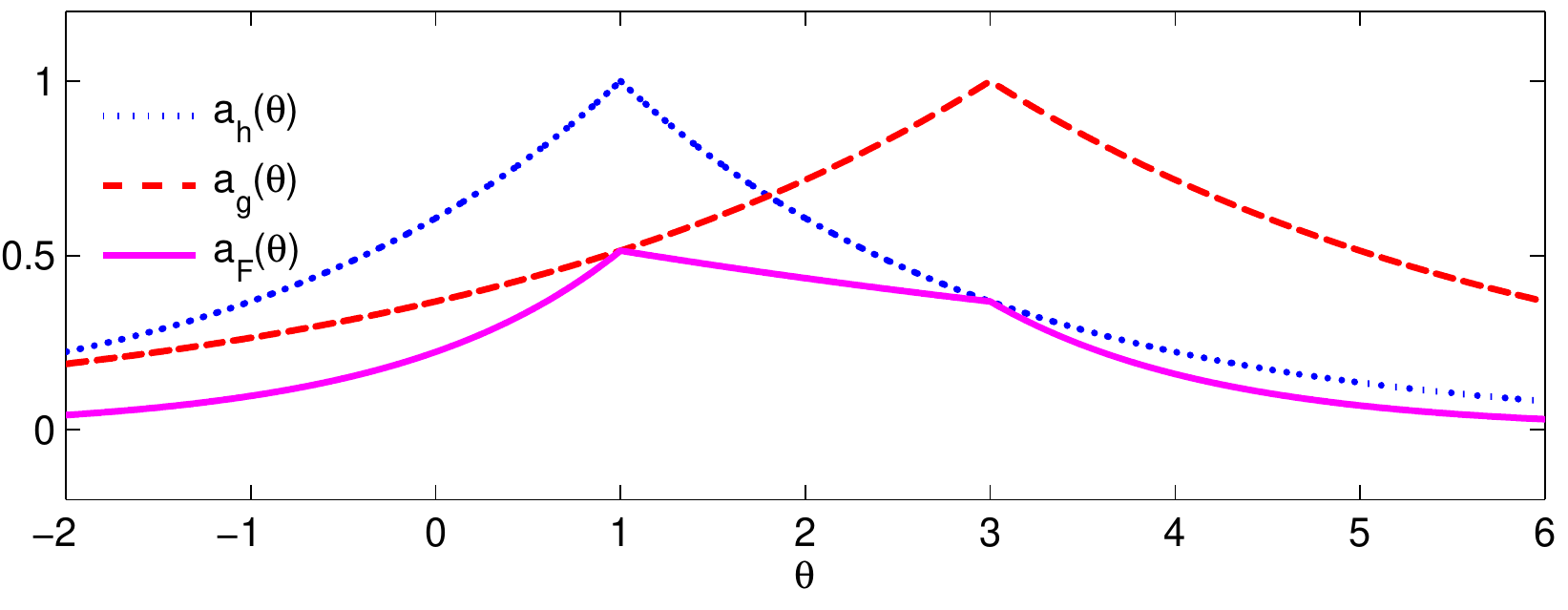}}
\vspace{-.1in}
\centerline{\small (c) gamma set functions \hspace{1in} (d) Laplace set functions}
%\vspace{.2in}
\caption[Conjugacy and semi-conjugacy of the doubly fuzzy posterior if-part set functions]{Conjugacy and semi-conjugacy of the doubly fuzzy posterior if-part set functions $a_F(\theta) = a_h(\theta)a_g(x|\theta)$. 
(a) Gaussian if-part set functions have the form of (\ref{eq:gaussiansets}) where $a_h(\theta) = G(1,1,1;\theta)$ and $a_g(\theta) = G(3,2,1;\theta)$ give Gaussian $a_F(\theta) = G(\frac{7}{5},\frac{4}{5},e^{-4/5};\theta)$. 
(b) beta if-part set functions have the form of (\ref{eq:betasets}) where $a_h(\theta) = B(0,4,2,3,29;\theta)$ and $a_g(\theta) = B(1,6,6,12,9\times10^4;\theta)$ give semi-beta $a_F(\theta)$.
(c) gamma if-part set functions have the form of (\ref{eq:gammasets}) where $a_h = G(0,1,2,3,2.7;\theta)$ and $a_g = G(1,1,2,0.5,7.4;\theta)$ give semi-gamma $a_F(\theta)$.
(d) Laplace if-part set functions have the form of (\ref{eq:laplacesets}) where $a_h(\theta) = L(1,2;\theta)$ and $a_g(\theta) = L(3,3;\theta)$ give semi-Laplace $a_F(\theta)$.
}
\label{fig:productsets_conjugacy}
\end{figure}

\vspace{.2in}

\begin{thm}\label{thm:fuzzypost-prodrules}{\bf [Preservation of the SAM convex structure in fuzzy Bayesian inference]}

\noindent {\bf (i)} Doubly fuzzy posterior approximators are SAMs with product rules.

%\noindent
Suppose an $m_1$-rule SAM fuzzy system $G(x|\theta)$ approximates (or represents) a likelihood $g(x|\theta)$ and another $m_2$-rule SAM fuzzy system $H(\theta)$ approximates (or represents) a prior $h(\theta)$ pdf with $m_2$ rules:
\begin{align}
G(x|\theta) &=
\frac{\sum_{j=1}^{m_1}\,w_{g,j}a_{g,j}(\theta)V_{g,j}c_{g,j}}{\sum_{i=1}^{m_1} w_{g,j}a_{g,j}(\theta) V_{g,j}} 
= \sum_{j=1}^{m_1} p_{g,j}(\theta) c_{g,j}
\\
H(\theta) &= 
\frac{\sum_{j=1}^{m_2}\,w_{h,j}a_{h,j}(\theta)V_{h,j}c_{h,j}}{\sum_{j=1}^{m_2} w_{h,j}a_{h,j}(\theta) V_{h,j}}
=
\sum_{j=1}^{m_2} p_{h,j}(\theta) c_{h,j} \hspace{.2in}
\end{align}
where $p_{g,j}(\theta) = \frac{w_{g,j} a_{g,j}(\theta)V_{g,j}}{\sum_{i=1}^{m_1} w_{g,j}a_{g,j}(\theta) V_{g,j}}$ and $p_{h,j}(\theta) = \frac{w_{h,j} a_{h,j}(\theta)V_{h,j}}{\sum_{i=1}^{m_2} w_{h,j} a_{h,j}(\theta) V_{h,j}}$ are convex coefficients:
$\sum_{j=1}^{m_1} p_{g,j}(\theta) = 1$ and $\sum_{j=1}^{m_2} p_{h,j}(\theta) = 1$. Then (a) and (b) hold: 

\noindent {\bf (a)} The fuzzy posterior approximator $F(\theta|x)$ is a SAM system with $m = m_1 m_2$ rules:
%if we ignore the normalizing term:
\begin{align}
F(\theta|x) &= \frac{\sum_{i=1}^{m}\,w_{F,i}\,a_{F,i}(\theta)\, V_{F,i}\,c_{F,i}}
{\sum_{i=1}^{m}\,w_{F,i}\,a_{F,i}(\theta)\,V_{F,i}}.
\label{eq:SAMposterior_Doubly}
\end{align}

\noindent {\bf (b)} The $m$ if-part set functions $a_{F,i}(\theta)$ of the fuzzy posterior approximator $F(\theta|x)$ are the
products of the likelihood approximator's if-part sets $a_{g,j}(\theta)$ and the prior approximator's
if-part sets $a_{h,j}(\theta)$:
\begin{align}
a_{F,i}(\theta) &= a_{g,j}(\theta)a_{h,k}(\theta).
\end{align}
for $i = m_2(j-1)+k$, $j = 1,\ldots,m_1$, and $k = 1,\ldots, m_2$.
The weights $w_{F_i}$, then-part set volumes $V_{F_i}$, and centroids $c_{F_i}$ also have the same likelihood-prior product form:
\begin{align}
w_{F_i} &= w_{g,j}w_{h,k} \\ %, ~~~~~
V_{F_i} &= V_{g,j}V_{h,k} \\ %, ~~~~\mbox{and}~~~
c_{F_i} &= \frac{c_{g,j}c_{h,k}}{Q(x)}.
\end{align}

\vspace{.1in}

\noindent {\bf (ii)} Triply fuzzy posterior approximators and $n$-many fuzzy posterior approximators are SAMs with product rules.

%\noindent
Suppose an $m_1$-rule SAM fuzzy system $G(x|\theta)$ approximates (or represents) a likelihood $g(x|\theta)$, an $m_2$-rule SAM fuzzy system $H(\theta,\tau)$ approximates (or represents) a prior pdf $h(\theta|\tau)$ with $m_2$ rules, an  $m_3$-rule SAM fuzzy system $\Pi(\theta)$ approximates (or represents) a hyper-prior pdf $pi(\tau)$ with $m_3$ rules:
\begin{align}
G(x|\theta) &=
\frac{\sum_{j=1}^{m_1}\,w_{g,j}a_{g,j}(\theta)V_{g,j}c_{g,j}}{\sum_{i=1}^{m_1} w_{g,j}a_{g,j}(\theta) V_{g,j}} 
&=& \sum_{j=1}^{m_1} p_{g,j}(\theta) c_{g,j} \\
H(\theta,\tau) &=
\frac{\sum_{j=1}^{m_2}\,w_{h,j}a_{h,j}(\theta,\tau)V_{h,j}c_{h,j}}{\sum_{j=1}^{m_2} w_{h,j}a_{h,j}(\theta,\tau) V_{h,j}}
&=& \sum_{j=1}^{m_2} p_{h,j}(\theta,\tau) c_{h,j} \\
\Pi(\tau) &=
\frac{\sum_{j=1}^{m_3}\,w_{\pi,j}a_{\pi,j}(\tau)V_{h,j}c_{h,j}}{\sum_{j=1}^{m_3} w_{\pi,j}a_{\pi,j}(\tau) V_{\pi,j}}
&=& \sum_{j=1}^{m_3} p_{h,j}(\tau) c_{h,j} 
\end{align}
where 
\begin{align}
p_{g,j}(\theta) &= \frac{w_{g,j} a_{g,j}(\theta)V_{g,j}}{\sum_{i=1}^{m_1} w_{g,i}a_{g,i}(\theta) V_{g,i}}, \\
p_{h,j}(\theta,\tau) &= \frac{w_{h,j} a_{h,j}(\theta,\tau)V_{h,j}}{\sum_{i=1}^{m_2} w_{h,i} a_{h,i}(\theta,\tau) V_{h,i}}, \quad \text{and} \\
p_{\pi,j}(\tau) &= \frac{w_{\pi,j} a_{\pi,j}(\tau)V_{\pi,j}}{\sum_{i=1}^{m_3} w_{\pi,i} a_{pi,i}(\tau) V_{h,i}}
\end{align}
are convex coefficients:
$\sum_{j=1}^{m_1} p_{g,j}(\theta) = 1$,  $\sum_{j=1}^{m_2} p_{h,j}(\theta,\tau) = 1$,  and $\sum_{j=1}^{m_2} p_{\pi,j}(\theta) = 1$. Then (a) and (b) hold: 

\noindent {\bf (a)} The fuzzy posterior approximator $F(\theta,\tau|x)$ is a SAM system with $m = m_1 m_2 m_3$ rules:
%if we ignore the normalizing term:
\begin{align}
F(\theta,\tau|x) &= \frac{\sum_{i=1}^{m}\,w_{F,i}\,a_{F,i}(\theta)\, V_{F,i}\,c_{F,i}}
{\sum_{i=1}^{m}\,w_{F,i}\,a_{F,i}(\theta)\,V_{F,i}}.
\label{eq:SAMposterior_Triply}
\end{align}

\noindent {\bf (b)} The $m$ if-part set nctions $a_{F,i}(\theta,\tau)$ of the fuzzy posterior approximator $F(\theta,\tau|x)$ are the
products of the likelihood approximator's if-part sets $a_{g,j}(\theta)$, the prior approximator's
if-part sets $a_{h,j}(\theta,\tau)$, and the hyper-prior approximators's if-part sets  $a_{\pi,j}(\tau)$:
\begin{align}
a_{F,i}(\theta,\tau) &= a_{g,j}(\theta)a_{h,k}(\theta,\tau) a_{\pi,l}(\tau)
\end{align}
for $i = l+m_3(k-1)+m_2 m_3(j-1)$, $j = 1,\ldots,m_1$, $k = 1,\ldots, m_2$, and $l = 1,\ldots,m_3$.
The weights $w_{F_i}$, then-part set volumes $V_{F_i}$, and centroids $c_{F_i}$ also have the same likelihood-prior-hyper-prior product form:
\begin{align}
w_{F_i} &= w_{g,j}w_{h,k} w_{\pi,l}\\ %, ~~~~~
V_{F_i} &= V_{g,j}V_{h,k}V_{\pi,l} \\ %, ~~~~\mbox{and}~~~
c_{F_i} &= \frac{c_{g,j}c_{h,k}c_{\pi,l}}{Q(x)}
\end{align}
where $Q(x) = \int_{\mathcal{D}}  G(x| \theta) H(\theta, \tau)  \Pi(\tau)\,d\tau d\theta$.

{This implies that the $n$-many fuzzy posterior approximators are also SAMs with product rules.}
\end{thm}

\begin{proof} {\bf (i) Doubly fuzzy case.}\\
The fuzzy system $F(\theta|x)$ has the form
\begin{align}
F(\theta|x) 
&= \frac{H(\theta)G(x|\theta)}{\int_{\cal D} H(t)G(x|t)\,dt}\\
&= \frac{1}{Q(x)} 
\left(\sum_{j=1}^{m_1} \, p_{g,j}(\theta)\, c_{g,j}\right)\left(\sum_{j=1}^{m_2}\, p_{h,j}(\theta)\, c_{h,j}\right)
\\
&= 
\sum_{j=1}^{m_1}\sum_{k=1}^{m_2}\,p_{g,j}(\theta)p_{h,k}(\theta)\, \frac{c_{g,j}\,c_{h,k}}{Q(x)}
\label{eq:convex-sumF1}
\end{align}
\begin{align}
&=
\sum_{j=1}^{m_1}\sum_{k=1}^{m_2}\,
\frac{w_{g,j}a_{g,j}(\theta)V_{g,j}}{\displaystyle \sum_{i=1}^{m_1} w_{g,i}a_{g,i}(\theta) V_{g,i}} 
\frac{w_{h,k}a_{h,k}(\theta)V_{h,k}}{\displaystyle \sum_{i=1}^{m_2} w_{h,i}a_{h,i}(\theta) V_{h,i}}
\, \frac{c_{g,j}\,c_{h,k}}{Q(x)}
\\
&=  
\frac{\displaystyle\sum_{j=1}^{m_1}\sum_{k=1}^{m_2}\,w_{g,j}\,w_{h,k}\,a_{g,j}(\theta)a_{h,k}(\theta)\,V_{g,j}\,V_{h,k}\, \displaystyle \frac{c_{g,j}\,c_{h,k}}{Q(x)}}
{\displaystyle \sum_{j=1}^{m_1}\sum_{k=1}^{m_2}\,w_{g,j}\,w_{h,k}\,a_{g,j}(\theta)a_{h,k}(\theta)\,V_{g,j}\,V_{h,k}}
\\
&= \frac{\displaystyle \sum_{i=1}^{m}\,w_{F,i}\,a_{F,i}(\theta)\, V_{F,i}\,c_{F,i}}
{\displaystyle \sum_{i=1}^{m}\,w_{F,i}\,a_{F,i}(\theta)\,V_{F,i}} \\
F(\theta|x) &= \sum_{i=1}^{m}\,p_{F,i}(\theta)\,c_{F,i}
\label{eq:convex-sumF2}
\end{align}
where $p_{F,i}(\theta)$ are the convex coefficients defined as
\begin{equation}
p_{F,i}(\theta)=\frac{ w_{F,i}\,a_{F,i}(\theta)\, V_{F,i} }{\sum_{i=1}^{m}\,w_{F,i}\,a_{F,i}(\theta)\,V_{F,i}} \;.
\end{equation}
\end{proof}

\begin{proof} {\bf (ii) Triply fuzzy case.}\\
The fuzzy system $F(\theta,\tau|x)$ has the form:

\begin{align}
\lefteqn{F(\theta,\tau|x) ~=~ \frac{G(x|\theta)H(\theta,\tau)\Pi(\tau)}{\int_{\cal D_{\theta}\times D_{\tau}} G(x|t)H(t,s)\Pi(s)\,dt\,ds}}\\
&= \frac{1}{Q(x)} \left(\sum_{j=1}^{m_1} \, p_{g,j}(\theta)\, c_{g,j}\right) \left(\sum_{j=1}^{m_2}\, p_{h,j}(\theta,\tau)\, c_{h,j}\right) \left(\sum_{j=1}^{m_3}\, p_{\pi,j}(\theta)\, c_{\pi,j}\right) \\
&= \sum_{j=1}^{m_1}\sum_{k=1}^{m_2}\,p_{g,j}(\theta)p_{h,k}(\theta,\tau) p_{\pi,l}(\tau)\, \frac{c_{g,j}\,c_{h,k}c_{\pi,l}}{Q(x)} \label{eq:convex-sumF3} \\
&= \sum_{j=1}^{m_1}\sum_{k=1}^{m_2}\sum_{l=1}^{m_3}\, \frac{w_{g,j}a_{g,j}(\theta)V_{g,j}}{\displaystyle \sum_{i=1}^{m_1} w_{g,i}a_{g,i}(\theta) V_{g,i}} \frac{w_{h,k}a_{h,k}(\theta,\tau)V_{h,k}}{\displaystyle \sum_{i=1}^{m_2} w_{h,i}a_{h,i}(\theta,\tau) V_{h,i}} \times \nonumber \\ 
& \quad \frac{w_{\pi,l}a_{\pi,l}(\tau)V_{\pi,k}}{\displaystyle \sum_{i=1}^{m_3} w_{\pi,i}a_{\pi,i}(\tau) V_{\pi,i}} \, \frac{c_{g,j}\,c_{h,k}\,c_{\pi,l}}{Q(x)} \\
&=  \frac{\displaystyle\sum_{j=1}^{m_1}\sum_{k=1}^{m_2}\sum_{l=1}^{m_3}\,
\begin{array}{ll}
w_{g,j}w_{h,k}
w_{\pi,k}\,a_{g,j}(\theta) a_{h,k}(\theta,\tau) a_{\pi,l}(\tau) %\times \\
V_{g,j} V_{h,k} V_{\pi,l}\, \displaystyle \frac{c_{g,j}\,c_{h,k}\,c_{\pi,l}}{Q(x)}
\end{array}}
{\displaystyle \sum_{j=1}^{m_1}\sum_{k=1}^{m_2}\sum_{l=1}^{m_3}
\begin{array}{ll}
w_{g,j}w_{h,k}
w_{\pi,k}\,a_{g,j}(\theta) a_{h,k}(\theta,\tau) a_{\pi,l}(\tau)%\times\\
V_{g,j} V_{h,k} V_{\pi,l}
\end{array}}~~~~~
\end{align}
Therefore $F(\theta,\tau|x)$ has the SAM convex structure
\begin{align}
F(\theta,\tau|x) = \frac{\displaystyle \sum_{i}^{m}\,w_{F,i}\,a_{F,i}(\theta)\, V_{F,i}\,c_{F,i}} {\displaystyle \sum_{i}^{m}\,w_{F,i}\,a_{F,i}(\theta)\,V_{F,i}} = \sum_{i} p_{F,i}(\theta) c_{F,i} \;.
\end{align}
The index $i$ is shorthand for the triple summation indices $(j,k,l)$. The parameters for the centroids $c_{F,i}$ and the convex coefficients $p_{F,i}$ are:
\begin{align}
w_{F,i} &= w_{g,j}w_{h,k} w_{\pi,k}\\
a_{F,i} &= a_{g,j}(\theta) a_{h,k}(\theta,\tau) a_{\pi,l}(\tau)\\
V_{F,i} &= V_{g,j} V_{h,k} V_{\pi,l} \\
\text{and} \quad c_{F,i} &= \frac{c_{g,j}\,c_{h,k}\,c_{\pi,l}}{Q(x)}\;.
\end{align}
\end{proof}

\vspace{8pt}

\begin{cor}\label{cor:2-ruleRep-g} %{\bf Posterior approximation with two-rule representation of $g(x|\theta)$ and $m$-rule fuzzy $H(\theta)$.} \\

Suppose a 2-rule fuzzy system $G(x|\theta)$ represents a likelihood $g(x|\theta)$ and an $m$-rule system $H(\theta)$ approximates the prior pdf $h(\theta)$.  Then the fuzzy-based posterior (or ``updated'' system) $F(\theta|x)$ is a SAM fuzzy system with $2m$ rules.
\end{cor}

\vspace{.2in}

\begin{proof}
Suppose a 2-rule fuzzy system $G(x|\theta)$ represents a likelihood $g(x|\theta)$:
\begin{align}
G(x|\theta) &= \sum_{j=1}^2 p_{g,j}(\theta)c_{g,j} ~=~ \sum_{k=1}^2 a_{g,j}(\theta)c_{g,j}
\label{eq:2-ruleG}
\end{align}
where the if-part set functions have the form (from the Watkins Representation Theorem)
\begin{align}
a_{g,1}(x|\theta) &= \frac{\displaystyle g(x|\theta)-\inf g(x|\theta)}{\displaystyle\sup g(x|\theta) - \inf g(x|\theta)}\\
a_{g,2}(x|\theta) &= a_{g,1}^c(\theta) ~=~ 1-a_{g,1}(x|\theta) \\
&= \frac{\displaystyle \sup g(x|\theta) - g(x|\theta)}{\displaystyle \sup g(x|\theta) - \inf g(x|\theta)}
\label{eq:2-ruleG_ag2}
\end{align}
and the centroids are $c_{g,1} = \sup g$ and  $c_{g,2} = \inf g$.  And suppose that an $m$-rule fuzzy system $H(\theta)$ with equal weights $w_i = \cdots = w_m$ and volumes $V_i = \cdots = V_m$ approximates (or represents) the prior $h(\theta)$. Then (\ref{eq:SAMposterior_Doubly}) becomes
\begin{align}
F(\theta|x) 
&=
\frac{\displaystyle \sum_{j=1}^{m}\sum_{k=1}^{2}\,a_{g,k}(x|\theta)\, a_{h,j}(\theta)\, 
\displaystyle \frac{c_{g,k}\,c_{h,j}}{Q(x)}}
{\displaystyle \sum_{j=1}^{m}\sum_{k=1}^{2}\,a_{g,k}(x|\theta)\, a_{h,j}(\theta)}
\\
&= 
\frac{\begin{array}{l} \displaystyle \sum_{j=1}^{m}\, a_{g,1}(x|\theta)a_{h,j}(\theta) \, 
\displaystyle \frac{c_{g,1}c_{h,j}}{Q(x)}  
%\\ \hspace{.5in} 
+ a_{g,2}(x|\theta)a_{h,j}(\theta) \, \displaystyle \frac{c_{g,2}c_{h,j}}{Q(x)} \end{array}}
{\displaystyle \sum_{j=1}^{m}\,a_{g,1}(x|\theta)\, a_{h,j}(\theta) + a_{g,2}(x|\theta)\, a_{h,j}(\theta)}
\end{align}
\begin{align}
&= \frac{\begin{array}{l} \displaystyle \sum_{j=1}^{m}\, a_{g,1}(x|\theta)a_{h,j}(\theta) \, 
\displaystyle \frac{c_{g,1}c_{h,j}}{Q(x)} 
%\\  \hspace{.5in} 
+ (1-a_{g,1}(x|\theta))a_{h,j}(\theta) \,
 \displaystyle \frac{c_{g,2}c_{h,j}}{Q(x)}\end{array}}
{\displaystyle \sum_{j=1}^{m}\,a_{g,1}(x|\theta)\, a_{h,j}(\theta) + (1-a_{g,1}(x|\theta))\, a_{h,j}(\theta)} \hspace{.2in}
%\hspace{.2in} \mbox{}  
\end{align}
\end{proof}

\vspace{.2in}

The above results imply that the number $m$ of rules of a fuzzy system $F(\theta|x)$ after $n$ stages will be $m_1 m_2^n = 2^n m$ rules. So the iterative fuzzy posterior approximator will in general suffer from exponential rule explosion.

At least one practical special case avoids this exponential rule explosion and produces
only a linear or quadratic growth in fuzzy-posterior rules in iterative Bayesian inference. Suppose that we can keep track of past data involved in the Bayesian inference and that $g(x_1,\ldots,x_n|\theta) = g(\bar{x}_n|\theta)$.  Then we can compute the likelihood $g(\bar{x}_{n-1}|\theta)$ from $g(\bar{x}_n|\theta)$ for any new data $x_n$.
Then we can update the original prior $H(\theta)$  
and keep the number of rules at $2m$ (or $m^2$) if the fuzzy system uses two rules (or $m$ rules).

\vspace{8pt}

\begin{cor}\label{cor:2-ruleRep-hg}{\bf Posterior representation with fuzzy representations of $h(\theta)$ and $g(x|\theta)$.} \\
Suppose a 2-rule fuzzy system $G(x|\theta)$ represents a likelihood function $g(x|\theta)$ and a $2$-rule system $H(\theta)$ represents the prior $h(\theta)$.  Then the fuzzy-based posterior $F(\theta|x)$ is a SAM fuzzy system with 4 $(2\times 2)$ rules.
\end{cor}

\vspace{0.1in}

\begin{proof} %{\bf Proof:} 
Suppose a 2-rule fuzzy system $G(x|\theta)$ represents a likelihood $g(x|\theta)$ as in 
(\ref{eq:2-ruleG})-(\ref{eq:2-ruleG_ag2}).  The 2-rule fuzzy system $H(\theta)$ likewise represents the prior pdf $h(\theta)$:
\begin{align}
H(\theta) &= \sum_{k=1}^2 p_{h,k}(\theta)c_{h,k} ~=~ \sum_{k=1}^2 a_{h,k}(\theta)c_{h,k}.
\end{align}
The Watkins Representation Theorem implies that the if-part set functions have the form

\begin{align}
a_{h,1}(\theta) &= \frac{\displaystyle h(\theta)-\inf h(\theta)}{\displaystyle \sup h(\theta) - \inf h(\theta)}\\
a_{h,2}(\theta) &= a_{h,1}^c(\theta) ~=~ 1-a_{h,1}(\theta) \\
&= \frac{\displaystyle \sup h(\theta) - h(\theta)}{\displaystyle \sup h(\theta) - \inf h(\theta)}
\end{align}
with centroids $c_{h,1} = \sup h$ and  $c_{h,2} = \inf h$.  Then the SAM posterior $F(\theta|x)$ in (\ref{eq:SAMposterior_Doubly}) represents $f(\theta|x)$ with 4 rules: 
\begin{align}
F(\theta|x) 
&=
\frac{ \sum_{j=1}^{2}\sum_{k=1}^{2}\,a_{g,j}(x|\theta)\, a_{h,k}(\theta)\, 
\displaystyle \frac{c_{g,j}\,c_{h,k}}{Q(x)}}
{\sum_{j=1}^{2}\sum_{k=1}^{2}\,a_{g,j}(x|\theta)\, a_{h,k}(\theta)}
\label{eq:F(theta|x)_2-ruleg,h} \hspace{.2in} \\
&= 
\sum_{j=1}^{2}\sum_{k=1}^{2}\,a_{g,j}(x|\theta)\, a_{h,k}(\theta)\, 
\displaystyle \frac{c_{g,j}\,c_{h,k}}{q(x)} \\
&= \sum_{i=1}^{4}\,a_{F,i}(\theta)\,c_{F,i}
\end{align}
because $\sum a_{g,j}(x|\theta) = \sum a_{h,k}(\theta) = 1$ and $Q(x) = q(x)$ in (\ref{eq:F(theta|x)_2-ruleg,h}).
\end{proof}

\vspace{.1in}

Figure \ref{fig:doubleSAMbeta} shows the if-part sets $a_{h,k}(\theta)$ of the 2-rule SAM $H(\theta)$ that represents the beta prior $h(\theta) \sim \beta(9,9)$ and the if-part sets $a_{g,j}(\theta)$ of the 2-rule SAM $G(x|\theta)$ that represents the binomial likelihood $g(20|\theta) \sim bin(20,80)$. The resulting SAM posterior $F(\theta|20)$ that represents $f(\theta|20) \sim \beta(29,69)$ has four rules with if-part sets $a_{F,i}(\theta) = a_{g,j}(\theta)a_{h,k}(\theta)$. The next theorem gives the main result on the conjugacy structure of doubly and triply fuzzy systems.

\begin{figure}[ht!]
\centerline{
\includegraphics[width=\textwidth]{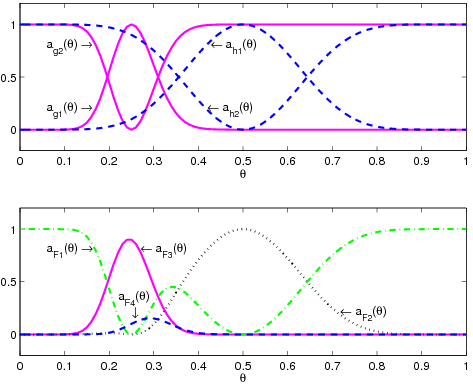}
}
\caption[Doubly fuzzy posterior representation]{Doubly fuzzy posterior representation. Top: Two if-part sets $a_{g,j}(\theta)$ of the two-rule SAM likelihood representation $G(x|\theta) = g(20|\theta) \sim bin(20,80)$ and two if-part sets $a_{h,k}(\theta)$ of the 2-rule SAM prior representation $H(x|\theta) = h(\theta) \sim \beta(9,9)$. Bottom: Four if-part sets $a_{F,i}(\theta) = a_{g,j}(\theta)a_{h,k}(\theta)$ of the 4-rule SAM posterior representation $F(\theta|x) = f(\theta|x)$.}
\label{fig:doubleSAMbeta}
\end{figure}

\vspace{0.2in}

\begin{thm}\label{thm:semi-conj}{\bf Semi-Conjugacy}\\
(i) The if-part sets of a doubly fuzzy posterior approximator are conjugate to the
if-part sets of the fuzzy prior approximator.
The product fuzzy if-part set functions $a_{F,i}(\theta)$ in Theorem \ref{thm:fuzzypost-prodrules}.i(b) have the same functional form as the if-part prior set functions $a_{h,k}$ if $a_{h,k}$ is conjugate to the if-part likelihood set function $a_{g,j}$. % and if $a_{h,k}$ is functionally similar to $a_{g,j}$.

\noindent(ii) The if-part sets of a triply fuzzy posterior approximator are conjugate to the if-part sets of the fuzzy prior approximator. The product fuzzy if-part set functions $a_{F,i}(\theta)$ in Theorem \ref{thm:fuzzypost-prodrules}.ii(b) have the same functional form as the if-part prior set functions $a_{h,k}$ if $a_{h,k}$ is conjugate to the if-part likelihood set function $a_{g,j}$ and if-part likelihood set function $a_{\pi,l}$. % and if $a_{h,k}$ is functionally similar to $a_{g,j}$ and .
\end{thm}

\vspace{.2in}

\begin{proof}
The product $a_{F,i}(\theta)= a_{g,j}(\theta) a_{h,k}(\theta)$ of two conjugate functions $a_{g,j}$ and $a_{h,k}$ will still have the same functional form as $a_{g,j}(\theta)$ and $a_{h,k}(\theta)$.  Then the $n$ parameters $\alpha_1,\ldots, \alpha_n$ define the if-part likelihood set function: $a_{g,j}(\theta) = f(\alpha_1,\ldots,\alpha_n;\theta)$. The $n$ parameters $\beta_1,\ldots,\beta_n$ likewise define the if-part prior set function $a_{h,k}(\theta)$ with the same functional form: $a_{h,k}(\theta) = f(\beta_1,\ldots,\beta_n;\theta)$.  {Then $a_{F,i}(\theta)$ also has the same functional form $f$ given the $n$ parameters $\gamma_1,\ldots,\gamma_n$: $a_{F,i}(\theta) = f(\gamma_1,\ldots,\gamma_n;\theta)$ where $\gamma_l = g_l(\alpha_1,\ldots,\alpha_n,\beta_1,\ldots,\beta_n)$ for $l = 1\ldots,n$ for some functions $g_1,\ldots,g_n$ that do not depend on $\theta$.}  
\end{proof}

\vspace{0.2in}

\begin{cor}{\bf Conjugacy of Gaussian if-part sets.} \label{cor:sconj-Gauss}\\
{\bf (i)} Doubly fuzzy case.

Suppose that the SAM-based prior $H(\theta)$ uses Gaussian if-part sets 
$a_{h,k}(\theta) = G(m_{h,k},d_{h,k},\nu_{h,k};\theta)$ 
and the SAM-based likelihood $G(x|\theta)$ also uses Gaussian if-part sets 
$a_{g,j}(\theta) = G(m_{g,j},d_{g,j},\nu_{g,j};\theta)$ where
\begin{align}
G(m,d,\nu;\theta) = \nu e^{-(\theta-m)^2/d^2}
\label{eq:gaussiansets}
\end{align}
for some positive constant $\nu > 0$.
Then $F(\theta|x)$ in (\ref{eq:SAMposterior_Doubly}) will have set functions $a_{F,i}(\theta)$ that are also Gaussian:
\begin{align}
a_{F,i}(\theta) = G(m_{F,i},d_{F,i},\nu_{F,i};\theta) \;.
\end{align}

\vspace{.1in}

\noindent {\bf (ii)} Triply fuzzy case.

Suppose that the SAM-based prior $H(\theta,\tau)$ uses factorable (product) Gaussian if-part sets 
\begin{equation}
a_{h\theta k}(\theta,\tau) = G(m_{h\theta k},d_{h\theta k},\nu_{h\theta k};\theta) G(m_{h\tau k},d_{h\tau k},\nu_{h\tau k};\tau)\;,\end{equation}
the SAM-based likelihood $G(x|\theta)$  uses Gaussian if-part sets \begin{equation}
a_{g\theta j}(\theta) = G(m_{g\theta j},d_{g\theta j},\nu_{g\theta j};\theta)\;,
\end{equation}
and the SAM-based hyper-prior $\Pi(\tau)$ also uses Gaussian if-part sets \begin{equation}a_{h\tau l}(\tau) = G(m_{h\tau l},d_{h\tau l},\nu_{h\tau l};\tau)\;.\end{equation}
Then $F(\theta,\tau|x)$ in (\ref{eq:SAMposterior_Triply}) will have set functions $a_{F,i}(\theta,\tau)$ that are products of two Gaussian sets:
\begin{align}
a_{F,i}(\theta,\tau) &= G(m_{F\theta i},d_{F\theta i},\nu_{F\theta i};\theta) G(m_{F\tau i},d_{F\tau i},\nu_{F\tau i};\tau)
\end{align}
\end{cor}

\begin{proof}
Gaussian if-part sets are self-conjugate because of their exponential structure.

\noindent{\bf (i)} Doubly fuzzy case.

\begin{align}
a_{F,i}(\theta) &= a_{g,j}(\theta)\, a_{h,k}(\theta) \\
&= \nu_{F,i} e^{-(\theta-m_{F,i})^2/d_{F,i}^2} \\
&= G(m_{F,i},d_{F,i},\nu_{F,i};\theta)
\end{align}
where
\begin{align}
m_{F,i} &= \frac{d_{g,j}^2m_{h,k}+d_{h,k}^2m_{g,j}}{d_{g,j}^2+d_{h,k}^2}\\
d_{F,i}^2 &= \frac{d_{g,j}^2d_{h,k}^2}{d_{g,j}^2+d_{h,k}^2}\\
\nu_{F,i} &= \nu_{h,k}\nu_{g,j}\exp\{-\frac{(m_{h,k}-m_{g,j})^2}{d_{g,j}^2+d_{h,k}^2}\}. 
\end{align}
for $j=1,\ldots,m_1$, $k=1,\ldots,m_2$, and $i = m_2(j-1)+k$.

\vspace{.1in}

\noindent {\bf (ii)} Triply fuzzy case.
\begin{align}
a_{F,i}(\theta,\tau) &= a_{g,j}(\theta)a_{h,k}(\theta,\tau) a_{\pi,l}(\tau) \\
&= 
\nu_{F\theta i} e^{-(\theta-m_{F\theta i})^2/d_{F\theta i}^2}  
\nu_{F\tau i} e^{-(\tau-m_{F\tau i})^2/d_{F\tau i}^2}\\
&= 
G(m_{F\theta i},d_{F\theta i},\nu_{F\theta i};\theta) 
G(m_{F\tau i},d_{F\tau i},\nu_{F\tau i};\tau)
\end{align}
where
\begin{align}
m_{F\theta i} &= \frac{d_{g\theta j}^2m_{h\theta k}+d_{h\theta k}^2m_{g\theta j}}{d_{g\theta j}^2+d_{h\theta k}^2}
\\
d_{F\theta i}^2 &= \frac{d_{g\theta j}^2d_{h\theta k}^2}{d_{g\theta j}^2+d_{h\theta k}^2}
\\
\nu_{F\theta i} &= 
\nu_{h\theta k}\nu_{g\theta j}\exp\{-\frac{(m_{h\theta k}-m_{g\theta j})^2}{d_{g\theta j}^2+d_{h\theta k}^2}\}
\\
m_{F\tau i} &= \frac{d_{\pi\tau l}^2m_{h\tau k}+d_{h\tau k}^2m_{\pi\tau l}}{d_{\pi\tau l}^2+d_{h\tau k}^2}
\\
d_{F\tau i}^2 &= \frac{d_{\pi\tau l}^2d_{h\tau k}^2}{d_{\pi\tau l}^2+d_{h\tau k}^2}\\
\nu_{F\tau i} &= 
\nu_{h\tau k}\nu_{\pi\tau l}\exp\{-\frac{(m_{h\tau k}-m_{\pi\tau l})^2}{d_{\pi\tau l}^2+d_{h\tau k}^2}\}.
\end{align}
for $j=1,\ldots,m_1$, $k=1,\ldots,m_2$, $l=1,\ldots,m_3$, and $i = l+m_3(k-1) + m_2m_3(j-1)$.
\end{proof}

%%%%%%%%%%%%%%%%

\vspace{.1in}

Corollary \ref{cor:sconj-Gauss} also shows that if the fuzzy approximator $H(\theta,\tau)$ uses product if part set functions $a_h(\theta,\tau) = a_{h\theta}(\theta) a_{h\tau}(\tau)$  then the fuzzy posterior $F(\theta,\tau|x)$ also has product if-part sets $a_F(\theta,\tau) = a_{F\theta}(\theta) a_{F\tau}(\tau)$. This holds for higher dimension fuzzy approximators for Bayesian inference. Thus the corollaries below only state the results for doubly fuzzy cases.

\vspace{.2in}

\begin{cor}{\bf Semi-conjugacy of beta if-part sets.}\label{cor:sconj-beta}\\
Suppose that the SAM-based prior $H(\theta)$ uses beta (or binomial) if-part sets 
$a_{h,k}(\theta) =$
$B(m_{h,k},d_{h,k},\alpha_{h,k},\beta_{h,k},\nu_{h,k};\theta)$ and the SAM-based likelihood $G(x|\theta)$ also uses beta (or binomial) if-part sets 
$a_{g,j}(\theta) = B(m_{g,j},d_{g,j},\alpha_{g,j},\beta_{g,j},\nu_{g,j};\theta)$ where
\begin{align}
B(m_{},d_{},\alpha_{},\beta_{},\nu;\theta) = \nu \Big(\frac{\theta-m_{}}{d_{}}\Big)^{\alpha_{}}
\Big(1 - (\frac{\theta-m_{}}{d_{}})\Big)^{\beta_{}}
%\hspace{.3in} \mbox{if $0 < \frac{\theta-m_{h,k}}{d_{h,k}} < 1$}
\label{eq:betasets}
\end{align}
if $0 < \frac{\theta-m}{d} < 1$ and for some constant $\nu > 0$.
Then the posterior $F(\theta|x)$ in (\ref{eq:SAMposterior_Doubly}) will have if-part set functions $a_{F,i}(\theta)$ of semi-beta form:
\begin{align}
a_{F,i}(\theta) = \nu_{F,i} \Big(\frac{\theta-m_{h,k}}{d_{h,k}}\Big)^{\alpha_{h,k}+\alpha_{g,j}\lambda_{jk}(\theta)}\, \Big(1 - (\frac{\theta-m_{h,k}}{d_{h,k}})\Big)^{\beta_{h,k}+\beta_{g,j}\gamma_{jk}(\theta)}
\end{align}
\end{cor}

\begin{proof}
\begin{align}
\lefteqn{a_{F,i}(\theta) 
~=~ a_{g,j}(\theta)\, a_{h,k}(\theta)} \\
&= \nu_{F,i} \Big(\frac{\theta-m_{h,k}}{d_{h,k}}\Big)^{\alpha_{h,k}}
\Big(1 - (\frac{\theta-m_{h,k}}{d_{h,k}})\Big)^{\beta_{h,k}}
\Big(\frac{\theta-m_{g,j}}{d_{g,j}}\Big)^{\alpha_{g,j}} 
\Big(1 - (\frac{\theta-m_{g,j}}{d_{g,j}})\Big)^{\beta_{g,j}} \\
&= \nu_{F,i} \Big(\frac{\theta-m_{h,k}}{d_{h,k}}\Big)^{\alpha_{h,k}+\alpha_{g,j}\lambda_{jk}(\theta)}\,
\Big(1 - (\frac{\theta-m_{h,k}}{d_{h,k}})\Big)^{\beta_{h,k}+\beta_{g,j}\gamma_{jk}(\theta)}
\label{eq:fuzzybeta-binomial1}
\end{align}
if $0 < \frac{\theta-m_{h,k}}{d_{h,k}} < 1$ and $0 < \frac{\theta-m_{g,j}}{d_{g,j}} < 1$ or if $\theta \in (m_{h,k}, m_{h,k} + d_{h,k} ) \cap (m_{g,j}, m_{g,j}+d_{g,j} )$ where
\begin{align}
\lambda_{jk}(\theta) &= 
\log_{(\frac{\theta-m_{h,k}}{d_{h,k}})}(\frac{\theta-m_{g,j}}{d_{g,j}}) \\
\gamma_{jk}(\theta) &=
\log_{(1-\frac{\theta-m_{h,k}}{d_{h,k}})}(1-\frac{\theta-m_{g,j}}{d_{g,j}}).
\end{align}
\end{proof}

\vspace{.1in}

A special case occurs if $m_{h,k} = m_{g,j}$ and $d_{h,k} = d_{g,j}$.  Then $a_{F,i}$ has the beta conjugate form:
\begin{align}
a_{F,i}(\theta) &= 
\nu_{F,i} \Big(\frac{\theta-m_{h,k}}{d_{h,k}}\Big)^{\alpha_{F,i}}
\Big(1 - (\frac{\theta-m_{h,k}}{d_{h,k}})\Big)^{\beta_{F,i}} \\
&=  B(m_{h,k},d_{h,k},\alpha_{F,i},\beta_{F,i},\nu_{F,i};\theta)
%hspace{.3in} \mbox{if $0 < \frac{\theta-m_{h,k}}{d_{h,k}} < 1$}
\end{align}
if $0 < \frac{\theta-m_{h,k}}{d_{h,k}} < 1$.  Here $\alpha_{F,i} = \alpha_{h,k}+\alpha_{g,j}$, 
$\beta_{F,i} = \beta_{h,k}+\beta_{g,j}$, and $\nu_{F,i} = \nu_{h,k}\nu_{g,j}$.

The if-part fuzzy sets of the posterior approximation in (\ref{eq:fuzzybeta-binomial1}) have beta-like form but with exponents that also depend on $\theta$.  Suppose we repeat the  updating of the prior-posterior. Then the final posterior will still have the beta-like if-part sets of the form
\begin{align}
a_{F,s}(\theta)
&= \nu_{F,s} \Big(\frac{\theta-m_{h,k}}{d_{h,k}}\Big)^{\alpha_{h,k}+\sum_i \alpha_{g,i}\lambda_{ik}(\theta)}
% \times \nonumber \\ && \hspace{.1in}
\Big(1 - (\frac{\theta-m_{h,k}}{d_{h,k}})\Big)^{\beta_{h,k}+\sum_i\beta_{g,i}\gamma_{ik}(\theta)}
\end{align}
for $\theta \in D = \cap_i (m_{g,i}, m_{g,i} + d_{g,i}) \cap (m_{h,k}, m_{h,k} + d_{h,k} )$.

\vspace{10pt}

\begin{cor}{\bf Semi-conjugacy of gamma if-part sets.}\label{cor:sconj-Gamma}\\
Suppose that the SAM-based prior $H(\theta)$ uses gamma (or Poisson) if-part sets 
$a_{h,k}(\theta) = G(m_{h,k},d_{h,k},\alpha_{h,k},\beta_{h,k},\nu_{h,k};\theta)$ and the SAM-based likelihood $G(x|\theta)$ also uses gamma (or Poisson) if-part sets 
$a_{g,j}(\theta) = G(m_{g,j},d_{g,j},\alpha_{g,j},\beta_{g,j},\nu_{g,j};\theta)$ where
\begin{align}
G(m_{},d_{},\alpha_{},\beta_{},\nu_{};\theta) = \nu \Big(\frac{\theta-m_{}}{d_{}}\Big)^{\alpha_{}}
e^{-(\frac{\theta-m_{}}{d_{}})/\beta_{}}
%\hspace{.3in} \mbox{if $ \frac{\theta-m_{h,k}}{d_{h,k}} > 0$ (or if $\theta > m_{h,k}$)}
\label{eq:gammasets}
\end{align}
if $ \frac{\theta-m_{}}{d_{}} > 0$ (or if $\theta > m_{}$) for some constant $\nu > 0$. Then the posterior $F(\theta|x)$ in (\ref{eq:SAMposterior_Doubly}) will have set functions $a_{F,i}(\theta)$ of semi-gamma form
\begin{multline}
a_{F,i}(\theta) = \nu_{F,i}  
\Big(\frac{\theta-m_{h,k}}{d_{h,k}}\Big)^{\alpha_{h,k}+\alpha_{g,j}\log_{(\frac{\theta-m_{h,k}}{d_{h,k}})}(\frac{\theta-m_{g,j}}{d_{g,j}})} \times \\
e^{-(\theta-\frac{\beta_{g,j}d_{g,j} m_{h,k}  + \beta_{h,k}d_{h,k} m_{g,j}} {\beta_{g,j}d_{g,j}+\beta_{h,k} d_{h,k}} )/ \frac{\beta_{g,j}\beta_{h,k} d_{g,j}d_{h,k} } {\beta_{g,j}d_{g,j}+\beta_{h,k} d_{h,k}}}
\end{multline}
\end{cor}

\begin{proof}
\begin{align}
\lefteqn{a_{F,i}(\theta) ~=~ a_{g,j}(\theta)\, a_{h,k}(\theta)} \\
&= \nu_{F,i} 
\Big(\frac{\theta-m_{h,k}}{d_{h,k}}\Big)^{\alpha_{h,k}}\,
e^{-(\frac{\theta-m_{h,k}}{d_{h,k}})/\beta_{h,k}} 
%\times\nonumber \\ && 
\Big(\frac{\theta-m_{g,j}}{d_{g,j}}\Big)^{\alpha_{g,j}}
e^{-(\frac{\theta-m_{g,j}}{d_{g,j}})/\beta_{g,j}}~~~
\\
&= \nu_{F,i}  
\Big(\frac{\theta-m_{h,k}}{d_{h,k}}\Big)^{\alpha_{h,k}+\alpha_{g,j}\log_{(\frac{\theta-m_{h,k}}{d_{h,k}})}(\frac{\theta-m_{g,j}}{d_{g,j}})}
 %\times \nonumber \\ && 
e^{-(\frac{\theta-m_{h,k}}{d_{h,k}})/\beta_{h,k} -(\frac{\theta-m_{g,j}}{d_{g,j}})/\beta_{g,j}}
\\
&= \nu_{F,i}  
\Big(\frac{\theta-m_{h,k}}{d_{h,k}}\Big)^{\alpha_{h,k}+\alpha_{g,j}\log_{(\frac{\theta-m_{h,k}}{d_{h,k}})}(\frac{\theta-m_{g,j}}{d_{g,j}})} 
\times \nonumber \\ & 
e^{-(\theta-\frac{\beta_{g,j}d_{g,j} m_{h,k}  + \beta_{h,k}d_{h,k} m_{g,j}} {\beta_{g,j}d_{g,j}+\beta_{h,k} d_{h,k}} )/ \frac{\beta_{g,j}\beta_{h,k} d_{g,j}d_{h,k} } {\beta_{g,j}d_{g,j}+\beta_{h,k} d_{h,k}}}~~~~~~
\label{eq:fuzzygamma-poisson1}
\end{align}
if $\theta > m_{h,k}$ and $\theta > m_{g,j}$ (or $\theta > \max\{m_{h,k},m_{g,j}\}$). 
\end{proof}

\vspace{.1in}

A special case occurs if $m_{h,k} = m_{g,j}$ and $d_{h,k} = d_{g,j}$. Then $a_{F,i}$ has gamma form
\begin{align}
a_{F,i}(\theta) &= \nu_{F,i} \Big(\frac{\theta-m_{h,k}}{d_{h,k}}\Big)^{\alpha_{F,i}} e^{-(\frac{\theta-m_{h,k}}{d_{h,k}})/\beta_{F,i}} \nonumber \\
&= G(m_{h,k},d_{h,k},\alpha_{F,i},\beta_{F,i},\nu_{F,i};\theta)
\end{align}
if $\theta > m_{h,k}$. 
Here $\alpha_{F,i} = \alpha_{h,k}+\alpha_{g,j}$, 
$\beta_{F,i} = \frac{\beta_{g,j}\beta_{h,k}}{\beta_{g,j}+\beta_{h,k}}$, and $\nu_{F,i} = \nu_{h,k}\nu_{g,j}$.

\vspace{8pt}

\begin{cor}{\bf Semi-conjugacy of Laplace if-part sets.} \label{cor:sconj-Lapl} \\
Suppose that the SAM-based prior $H(\theta)$  uses Laplace if-part sets $a_{h,k}(\theta) = L(m_{h,k},d_{h,k};\theta)$ and the SAM-based likelihood $G(x|\theta)$ also uses Laplace if-part sets $a_{g,j}(\theta) = L(m_{g,j},d_{g,j};\theta)$ where
\begin{align}
L(m_{},d_{};\theta) &= e^{-|\frac{\theta-m_{}}{d_{}}|}.
\label{eq:laplacesets}
\end{align}
Then $F(\theta|x)$ in (\ref{eq:SAMposterior_Doubly}) will have set functions $a_{F,i}(\theta)$ of the (semi-Laplace) form
\begin{align}
a_{F,i}(\theta) &= a_{g,j}(\theta)\, a_{h,k}(\theta) 
~=~ e^{-|\frac{\theta-m_{h,k}}{d_{h,k}}| - |\frac{\theta-m_{g,j}}{d_{g,j}}|}. 
\end{align}
\end{cor}

\vspace{.1in}

A special case occurs if $m_{h,k} = m_{g,j}$ and $d_{h,k} = d_{g,j}$. Then $a_{F,i}$ is of Laplace form
\begin{align}
a_{F,i}(\theta) 
&= e^{-|\frac{\theta-m_{h,k}}{d_{h,k}/2}|}.
\end{align}

\begin{figure}[ht!]
%\centerline{\includegraphics[width=0.9\textwidth]{./Figures/FODM/Laplacian-SemiConjugates3x3.eps}}
\centerline{\includegraphics[width=0.9\textwidth]{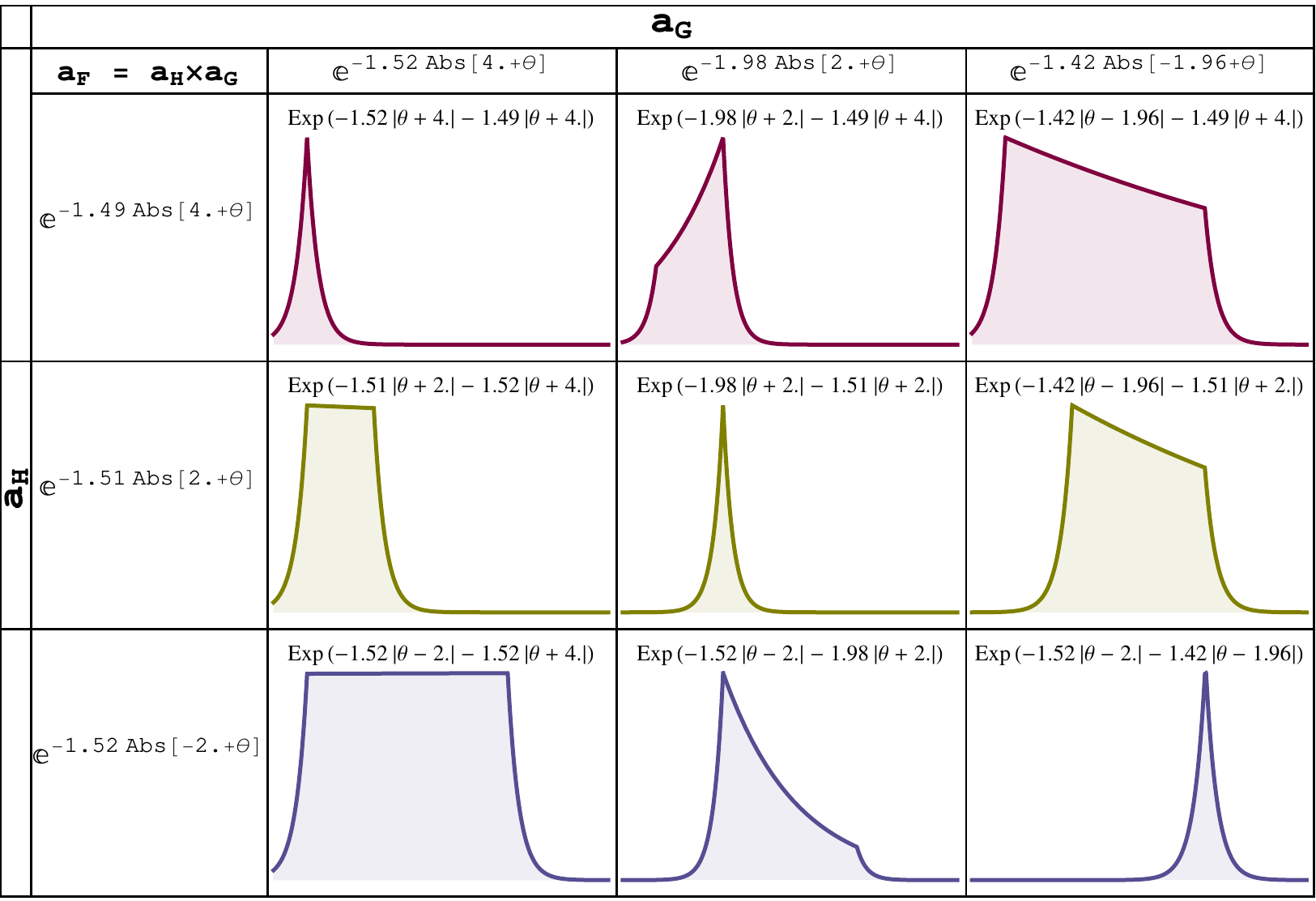}}
\caption[Semi-conjugacy for Laplace set functions]{
	Laplace semi-conjugacy: Plots show examples of semi-conjugate Laplace set functions for the doubly fuzzy posterior approximator $F(\theta|x) \propto H(\theta) G(x|\theta)$. The approximator uses five Laplace set functions $a_H$ for the prior approximator $H(\theta)$ and five Laplace set functions $a_G$ for the likelihood approximator $G(x|\theta)$. Thus $F(\theta|x)$ is a weighted sum of $25$ Laplacian semi-conjugate set functions of the form $a_{F,i}(\theta) = \exp \left( {-|\frac{\theta-m_{h,k}}{d_{h,k}}| - |\frac{\theta-m_{g,j}}{d_{g,j}}|} \right)$. The plots show that the Laplace semi-conjugate function can have a variety of shapes depending on the location and dispersion parameters of the prior and likelihood set functions.
}
\label{fig:Lapl-SemiConj}
\end{figure}

Such semi-conjugacy differs from outright conjugacy in a crucial respect:  The parameters of semi-conjugate if-part sets increase with each iteration or Bayesian update. The conjugate Gaussian sets in Corollary \ref{cor:sconj-Gauss} avoid this parameter growth while the semi-conjugate beta, gamma, and Laplace sets in Corollaries \ref{cor:sconj-beta}---\ref{cor:sconj-Lapl} incur it.  The latter if-part sets do not depend on a fixed number of parameters such as centers and widths as in the Gaussian case.  Only the set functions with the same centers $m_j$ and widths $d_j$ (in the special cases) will result in set functions for posterior approximation with the same fixed number of parameters.  Coping with this form of ``parameter explosion'' remains an open area of research in the use of fuzzy systems in iterative Bayesian inference.

\section{Conclusion}
We have shown that additive fuzzy systems can uniformly approximate a Bayesian posterior even in the hierarchical case when the prior pdf depends on its own uncertain parameter with its own hyperprior.  This gives a triply fuzzy uniform function approximation.  That hyperprior can in turn have its own hyperprior.  The result will be a quadruply fuzzy uniform function approximation and so on.  This new theorem substantially extends the choice of priors and hyperpriors from well-known closed-form pdfs that obey conjugacy relations to arbitrary rule-based priors that depend on user knowledge or sample data.  An open research problem is whether semi-conjugate rules or other techniques can reduce the exponential rule explosion that both doubly and triply fuzzy Bayesian systems face in general Bayesian iterative inference.
 
\clearpage

%\subfile{./Chapters/chapFinal}
\chapter{Conclusion and  Future Directions}\label{ch:finale} % AKA the ``Blue Skies'' chapter...

\section{Conclusion}

The prevalence of high-speed ubiquitous computing devices has led to a phenomenal growth in the amount data available for algorithmic manipulation. One side effect of this growth is the commensurate growth in the amount of corrupted or incomplete data. Increased diversity in source and structure of recorded data also means there is a higher risk of applying the wrong statistical model for analysis. There are costs associated with the inefficient use of corrupted data and the improper use of statistical models. These costs will only grow with our increasing reliance on algorithms in daily life. And yet the systematic study of data and model corruption lags behind other statistical signal processing research efforts.

This dissertation addresses both problems: the statistical analysis of corrupted or incomplete data and the effects of model misspecification. 

The first part of the dissertation showed how to improve the speed of a general algorithm (the EM algorithm) for analyzing corrupted data. The speed improvement relies on a theoretical result (Theorem \ref{thm:NEM}) about noise-benefits in EM algorithms. This dissertation demonstrated how this result affects three important statistical applications: data-clustering, sequential data analysis with HMMs, and supervised neural network training.

The second part of the dissertation showed that the use of uniform approximators can limit the sensitivity of Bayesian inference schemes to model misspecification. This dissertation presented theorems showing that Bayes theorem on uniform approximators for model functions (prior pdfs and likelihood functions) produce uniform approximations for the posterior pdf. So Bayes theorem preserves the approximation quality of approximate data models. The dissertation also demonstrated an efficient data-driven fuzzy approximation scheme for uniformly approximating these statistical models functions. The approximation method is particularly interesting because it can \emph{exactly reproduce} known traditional closed-form statistical models. 

The next sections discuss future extensions to the results in this dissertation.

\section{Future Work}
The EM algorithm has many diverse applications. Noisy expectation maximization may help reduce the training time of time-intensive EM applications. Table ~\ref{tab:EM-Apps} lists a few areas to explore. The subsections below describe some of these application areas in detail.

\begin{table}[htb]
\begin{center}
%\rowcolors{1}{gray}{pink}
\scalebox{0.825}{
\begin{tabular}{|c|p{2in}|c|}
			\hline
			\textbf{Application} & \textbf{Prior EM Work} & \textbf{NEM Extension} \\
			\hline \hline
			 Big Data Clustering &  Celeux \& Govaert, Hofmann~\parencite{macqueen1967, celeux-govaert1992, hofmann1999}. &  Chapter \ref{ch:NEM-CNBT}, \parencite{osoba-kosko2013}\\
			\hline
			 Feedforward NN Training &  Ng \& McLachlan~\parencite*{ng-mclachlan2004}, Cook \& Robinson~\parencite*{cook-robinson1995}. &  Chapter \ref{ch:NEM-BP} \parencite{audhkhasi-osoba-kosko-BP2013} \\
			\hline
			 HMM Training &  Baum \&  Welch~\parencite{baum-et-al1970,baum-petrie-soules-weiss1970,welch2003}. &  Chapter \ref{ch:NEM-HMM} \parencite{audhkhasi-osoba-kosko-HMM2013} \\
			\hline
			 Deep Neural Network Training &  None. & Extension based on Chapter \ref{ch:NEM-BP}\\ %\parencite{audhkhasi-osoba-kosko2013b}
			 \hline
			 Genomics: Motif Identification &  Lawrence \& Reilly~\parencite*{lawrence-reilly1990}, Bailey \& Elkan~\parencite*{bailey-elkan1994}. & Proposed below \\
			\hline
			 PET/SPECT &  Vardi \& Shepp\parencite*{shepp-vardi1982}. & Proposed below \\
			\hline
			 MRI Segmentation &  Zhang et al.\parencite*{zhang-brady-smith2001}. & Proposed below \\
			\hline
		\end{tabular}
}
\end{center}
\caption{EM algorithm applications with possible NEM extensions}
\label{tab:EM-Apps}
\end{table}%\todo{Add more entries to EM Application table}

\subsection{NEM for Deep Learning}\label{subsec:deepNN}

Deep learning~\parencite{hinton2012deep} refers to the use ``deep'' stacks of restricted Boltzmann machines (RBMs)~\parencite{smolensky1986information, hinton2006fast, hinton2012deep} to perform machine learning tasks. An RBM is a stochastic neural network with a bipartite graph structure and an associated Gibbs energy function $E(\mathbf{x,h}|\Theta)$ for the whole ensemble of neurons. RBMs are simple bidirectional associative memories (BAMs)~\parencite{kosko-nnfs} with synchronous neuron-update. The Gibbs energy functions are just the BAM energy or Lyapunov functions~\parencite{kosko1988bidirectional, kosko1987adaptive, kosko-nnfs} for the network. RBMs derive their convergence properties from the general BAM convergence theorem~\parencite{kosko1988bidirectional, kosko1987adaptive, kosko-nnfs}. Thus the convergence of any deep learning algorithm is a direct consequence of the general BAM convergence theorem.

Chapter \ref{ch:NEM-BP} showed that backpropagation training of feedforward neural networks is a GEM procedure. The GEM training formalism extends to the training of RBMs~\parencite{audhkhasi-osoba-kosko-BP2013}. Thus we can apply NEM noise--benefit to speed up RBM training and deep learning.

An RBM has a visible layer with $I$ neurons and a hidden layer with $J$ neurons. Let $x_i$ and $h_j$ denote the values of the $i^{th}$ visible and $j^{th}$ hidden neuron. 
\begin{figure}[ht!]
	\centering
	\begin{tikzpicture}[shorten >=1pt,->,draw=black!50, node distance=\layersep]
		\tikzstyle{every pin edge}=[<-,shorten <=2pt]
		\tikzstyle{neuron}=[circle,fill=black!25,minimum size=18pt,inner sep=0pt]
		\tikzstyle{input neuron}=[neuron, fill=blue!50];
		\tikzstyle{hidden neuron}=[neuron, fill=red!50];
		\tikzstyle{annot} = [text width=4em, text centered]
		\foreach \name / \y in {1,...,5}
			\node[input neuron] (I-\name) at (-\y-0.4,0) {};
		\foreach \name / \y in {2,...,5}
				\node[hidden neuron] (H-\name) at (-\y cm, \layersep) {};
		
		\node[inner sep=0,minimum size=0] (phantom) at (-3.5 cm, 1.75cm) {}; % invisible node
		\node[inner sep=0,minimum size=0,left of=phantom] (Wup)  {$\mathbf{W}$}; 
		\node[inner sep=0,minimum size=0,right of=phantom](Wdown) {$\mathbf{W}^{T}$}; 
		\draw [->, line width=0.2cm] (-0.5,2.8) -- (-0.5,0.8);
		\draw [->, line width=0.2cm] (-6,0.8) -- (-6,2.8);
		
		\foreach \source in {1,...,5}
			\foreach \dest in {2,...,5}
				\draw[-, very thick] (I-\source) edge (H-\dest);
				
		\node[annot,left of=H-5, node distance=1.8cm] (hl) {Hidden Layer, $\mathbf{h}$};
		\node[annot,below of=hl] {Visible layer, $\mathbf{x}$};
	\end{tikzpicture}
	\caption[A restricted Boltzmann machine]{A restricted Boltzmann machine (RBM) with visible-to-hidden connection matrix $\mathbf{W}$. The network is \emph{restricted} because there are no connections between neurons in the same layer. The RBM is a BAM with forward connection matrix $\mathbf{W}$ and backward connection matrix $\mathbf{W}^T$. }
\end{figure}
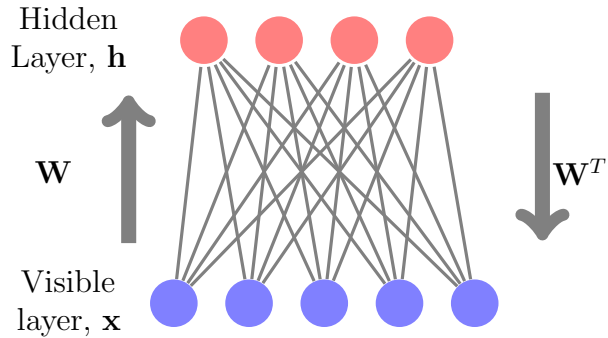

\begin{figure}[ht!]
	\centering
	\begin{tikzpicture}[shorten >=1pt,->,draw=black!50, node distance=\layersep]
		\tikzstyle{every pin edge}=[<-,shorten <=2pt]
		\tikzstyle{neuron}=[circle,fill=black!25,minimum size=18pt,inner sep=0pt]
		\tikzstyle{input neuron}=[neuron, fill=blue!50];
		\tikzstyle{hidden neuron}=[neuron, fill=red!50];
		\tikzstyle{output neuron}=[neuron, fill=green!60];
		\tikzstyle{annot} = [text width=4em, text centered]
		\foreach \name / \y in {1,...,5}
			\node[input neuron] (I-\name) at (-\y-0.4,0) {};
		\foreach \name / \y in {2,...,5}
				\node[hidden neuron] (H-\name) at (-\y cm, \layersepp) {};
		\foreach \name / \y in {2,...,5}
				\node[hidden neuron] (H2-\name) at (-\y cm, \layersep-0.2cm) {};
		\foreach \name / \y in {2,...,5}
				\node[output neuron] (O-\name) at (-\y cm, \layersep+\layersepp-0.3cm) {};
		
		\foreach \source in {1,...,5}
			\foreach \dest in {2,...,5}
				\draw[-, very thick] (I-\source) edge (H-\dest);

		\foreach \source in {2,...,5}
			\foreach \dest in {2,...,5}
				\draw[-, very thick] (H2-\source) edge (O-\dest);
				
		\node[inner sep=0,minimum size=0] (phantom) at (-3.5 cm, 1.75cm) {}; % invisible node
		\node[annot, above of=phantom, node distance=0.75cm](dots) {$\mathbf{\vdots}$};
		\node[annot, right of=dots, node distance=1cm] {$\mathbf{\vdots}$};
		\node[annot, left of=dots, node distance=1cm] (hs) {$\mathbf{\vdots}$};
				
		\node[annot,left of=hs, node distance=2.2cm] (hl) {Hidden layers};
		\node[annot,left of=I-5, node distance=1.4cm] {Visible layer};
		\node[annot,left of=O-5, node distance=1.8cm] {Output layer};
		
	\end{tikzpicture}
	\caption[A deep neural network]{A Deep Neural Network consists of a stack of restricted Boltzmann machines (RBMs) or bidirectional associative memories (BAMs).}
\end{figure}

Let $E(\mathbf{x,h}|\Theta)$ be the energy function for the network. Then the joint probability density function of $\mathbf{x}$ and $\mathbf{h}$ is the Gibbs distribution:
\begin{align}
	p(\mathbf{x},\mathbf{h}|\Theta) &= \frac{\exp\Big{(}-E(\mathbf{x},\mathbf{h}|\Theta)\Big{)}}{Z(\Theta)} \label{eq:Gibbs-Dist}\\
	\text{where } Z(\Theta) &= \sum_{\mathbf{x}}\sum_{\mathbf{h}} \exp(-E(\mathbf{x},\mathbf{h}|\Theta)) \;. \label{eq:Part-Func}
\end{align}
The Gibbs energy function $E(\mathbf{v}, \mathbf{h}| \Theta)$ depends on the type of RBM. A Bernoulli(visible)-Bernoulli(hidden) RBM has logistic conditional PDFs at the hidden and visible layers and has the following Gibbbs energy:
\begin{align}
	E(\mathbf{x},\mathbf{h}|\Theta) = -\sum_{i=1}^I\sum_{j=1}^J w_{ij} x_i h_j  - \sum_{i=1}^I b_i x_i - \sum_{j=1}^J a_j h_j \label{eq:Energy-LL}
\end{align}
where $w_{ij}$ is the weight of the connection between the $i^{th}$ visible and $j^{th}$ hidden neuron, $b_i$ is the bias for the $i^{th}$ visible neuron, and $a_j$ is the bias for the $j^{th}$ hidden neuron. A Gaussian(visible)-Bernoulli(hidden) RBM has Gaussian conditional PDFs at the visible layer, logistic conditional PDFs at the hidden layer, and the energy function~\parencite{hinton2006reducing, hinton2006fast}
\begin{align}
	E(\mathbf{x},\mathbf{h}|\Theta) &= -\sum_{i=1}^I\sum_{j=1}^J w_{ij} x_i h_j  + \frac{1}{2} \sum_{i=1}^I (x_i - b_i)^2  - \sum_{j=1}^J a_j h_j \;.
	\label{eq:Energy-NL}
\end{align}

Constrastive divergence (CD)~\parencite{hinton2006fast} is an \emph{approximate} gradient ascent algorithm for finding ML parameters for the RBMs. This ML gradient ascent approach means that CD is an instance of generalized EM~\parencite{audhkhasi-osoba-kosko-BP2013} -- just like backpropagation is an instance of GEM. So we can derive NEM conditions for RBM training too.

The NEM theorem implies that noise benefits occur on average if
\begin{align}
	\mathbb{E}_{\mathbf{x},\mathbf{h},\mathbf{n}|\Theta^*}& \Big{\{} -E(\mathbf{x}+\mathbf{n},\mathbf{h}|\Theta) + E(\mathbf{x},\mathbf{h}|\Theta) \Big{\}} \geq 0 \label{eq:NEM-EnergyDiff} \;.
\end{align}
This condition depends on the activation function in the visible and hidden layers. The next two theorems specify the NEM conditions for two common RBM configurations: the Bernoulli-Bernoulli and the Gaussian-Bernoulli RBM.
\pagebreak

\begin{thm}{\bf{[Forbidden Hyperplane Noise--Benefit Condition]:}} \label{thm:hyp-ndnn}\\
	The NEM positivity condition holds for Bernoulli-Bernoulli RBM training if
	\begin{align}
		\mathbb{E}_{\mathbf{x},\mathbf{h},\mathbf{n}|\Theta^*}& \Big{\{} \mathbf{n}^T (\mathbf{W}\mathbf{h} + \mathbf{b}) \Big{\}} \geq 0 \;. \label{eq:NEM-LLRBM}
	\end{align}
\end{thm}
\begin{proof}
	The noise--benefit for a Bernoulli(visible)-Bernoulli(hidden) RBM results if we apply the energy function from (\ref{eq:Energy-LL}) to the expectation in (\ref{eq:NEM-EnergyDiff}) to get
	\begin{equation}
		\mathbb{E}_{\mathbf{x},\mathbf{h},\mathbf{n}|\Theta^*} \Big{\{} \sum_{i=1}^I\sum_{j=1}^J w_{ij}n_{i}h_j + \sum_{i=1}^I n_{i} b_i \Big{\}} \geq 0 \;.
	\end{equation}
	The term in brackets is equivalent to
	\begin{align}
		\sum_{i=1}^I\sum_{j=1}^J w_{ij}n_{i}h_j + \sum_{i=1}^I n_{i} b_i &= \mathbf{n}^T (\mathbf{W}\mathbf{h} + \mathbf{b}) \;.
	\end{align}
	So the noise--benefit sufficient condition becomes
	\begin{align}
	\mathbb{E}_{\mathbf{x},\mathbf{h},\mathbf{n}|\Theta^*}& \Big{\{} \mathbf{n}^T (\mathbf{W}\mathbf{h} + \mathbf{b}) \Big{\}} \geq 0 \;.
	\end{align}
\end{proof}

The above sufficient condition stated separation condition that guarantees a noise--benefit in the Bernoulli-Bernoulli RBM. The next theorem states a spherical separation condition that guarantees a noise--benefit in the Gaussian-Bernoulli RBM.

\begin{thm}{\bf{[Forbidden Sphere Noise--Benefit Condition]:}} \label{thm:sph-ndnn}\\
The NEM positivity condition holds for Gaussian-Bernoulli RBM training if
\begin{align}
	\mathbb{E}_{\mathbf{x},\mathbf{h},\mathbf{n}|\Theta^*}& \Big{\{} \frac{1}{2}\|\mathbf{n}\|^2 - \mathbf{n}^T (\mathbf{W}\mathbf{h} + \mathbf{b} - \mathbf{x})\Big{\}} \leq 0 \;.  \label{eq:NEM-NLRBM}
\end{align}
\end{thm}
\begin{proof}
Putting the energy function in (\ref{eq:Energy-NL}) into (\ref{eq:NEM-EnergyDiff}) gives the noise--benefit condition for a Gaussian(visible)-Bernoulli(hidden):
\begin{equation}
	\mathbb{E}_{\mathbf{x},\mathbf{h},\mathbf{n}|\Theta^*} \Big{\{} 
	\sum_{i=1}^I\sum_{j=1}^J w_{ij}n_{i}h_j 
	+ \sum_{i=1}^I n_{i} b_i
	-\frac{1}{2} \sum_{i=1}^I n_i^2 
	- \sum_{i=1}^I n_i x_i
	\Big{\}} \geq 0 \;.
\end{equation}
The term in brackets equals the following matrix expression:
\begin{multline}
	\sum_{i=1}^I\sum_{j=1}^J w_{ij}n_{i}h_j 
	+ \sum_{i=1}^I n_{i} b_i
	-\frac{1}{2} \sum_{i=1}^I n_i^2 
	- \sum_{i=1}^I n_i x_i 
	= \mathbf{n}^T (\mathbf{W}\mathbf{h} + \mathbf{b} - \mathbf{x}) - \frac{1}{2}\|\mathbf{n}\|^2 \;.
\end{multline}
So the noise--benefit sufficient condition becomes
\begin{align}
	\mathbb{E}_{\mathbf{x},\mathbf{h},\mathbf{n}|\Theta^*}& \Big{\{}\frac{1}{2}\|\mathbf{n}\|^2 - \mathbf{n}^T (\mathbf{W}\mathbf{h} + \mathbf{b} - \mathbf{x}) \Big{\}} \leq 0 \;. 
\end{align}
\end{proof}

These noise-benefit conditions predict speed improvements for deep learning via the careful application of noise injection.

\subsection{NEM for Genomics: DNA Motif Identification}\label{subsec:genomics}

An organism's genome contains all the instructions necessary to manufacture the proteins its cells need to function. The information resides in the coding DNA sequences (or genes) within the genome. Transcription is the nuclear process by which the cell copies protein-code from a gene to a strand of messenger-RNA (mRNA). mRNA takes this code outside the nucleus to ribosomes where protein-synthesis occurs.

The human genome contains between $30000$ and $40000$ of such protein-coding genes. Most advanced cells need only a small portion of these genes to be active at any given time. So eukaryotic cells have adapted by keeping most genes switched off unless they receive a signal for expression~\parencite{brown-brown2004}. One of the functions of non-coding (or upstream) DNA sequences is to control the initiation and rate of gene expression. The cell transcribes (or turns on) a specific gene when transcription factors (promoters and enhancers\footnote{There are also \emph{negative} transcription factors that inhibit specific gene expression.}) bind to that gene's associated binding sites located upstream in the genome.

The basic question is~\parencite{bailey-elkan1994,lawrence-reilly1990,robison-mcguire-church1998}: how do we identify transcription factor binding sequences for specific genes? The problem is more complicated because the recognition of a binding site DNA sequence is often fuzzy; inexact promoter sequence matches can activate gene expression. Sabatti and Lange \parencite{sabatti-lange2002} give an example of $11$ experimentally-verified admissible variations of a single promoter sequence (the CRP binding site) that activates a specific gene in \textit{E. coli}. These fuzzy sequence patterns compose a single \emph{Motif}\parencite{zambelli-pesole-giulio2013} for CRP binding. Motif search is a fertile area for EM methods \parencite{stormo-hartzell1989,bailey-elkan1994,lawrence-reilly1990} because of the problem's missing information structure. The motif positions are missing and the exact motif sequence is not precise. The idea of a motif also extends from DNA sequences for genes to amino acid sequences for proteins.

Bailey and Elkan~\parencite{bailey-elkan1994, bailey-elkan1995a,bailey-elkan1995b,bailey2006} developed a mixture model EM for motif identification which they called \emph{Multiple EM for Motif Elicitation} (MEME). Likelihood models on DNA sequences assume sequence entries are independent discrete random variables that take values from the $4$ possible nucleotides\footnote{
	A similar likelihood model applies for motif identification in biopolymer sequences like proteins. Proteins just use a symbol lexicon of amino acids instead of nucleotides.
}: $\{$Adenine (A), Thymine (T), Guanine (G), Cytosine (C)$\}$. MEME models the data as a mixture of two different discrete populations~\parencite{bailey-elkan1994}: motif sequences and background (non-motif) sequences. An EM algorithm learns this mixture model which then informs the unsupervised sequence classification. Lawrence and Reilly~\parencite{lawrence-reilly1990} had introduced the first explicit EM model. But their EM algorithm does a supervised position search for fixed-length motifs in sequences that contain a single of the motif. Sabatti and Lange~\parencite{sabatti-lange2002} also developed an MM generalization of Lawrence and Reilly's method that allows motifs to have variable length sequences.

These approaches have had some success. But they work slowly on the huge human genome data set. And they miss some binding sites. Perturbing these likelihood-based estimation methods with noise can alleviate some of these problems. The success of a randomized Gibbs' sampling~\parencite{geman-geman1984,tanner1996} version of Lawrence and Reilly's supervised method~\parencite{lawrence-et-al1993} strongly suggests that the principled use of noise injection (e.g. NEM) should improve EM-based approaches to motif identification.

\subsection{NEM for PET \& SPECT}\label{subsec:PET-NEM}
We can write a NEM condition for PET reconstruction via EM. The $Q$-function for the PET-EM model from \Sec\ref{subsec:PET-EM} was:
\begin{align}
	Q(\Lambda | \Lambda(t)) &= \sum_j \sum_i -p_{ij} \lambda_i + \E_{\mathbf{Z}| \mathbf{Y,} \Lambda(t)}[Z_{ij}] \ln (p_{ij} \lambda_i) - \E_{\mathbf{Z}| \mathbf{Y,} \Lambda(t)}[\ln(z_{ij}!)].
\end{align}
The noise--benefit condition for general exponential-family EM models is:
\begin{align}
	\E_{Z,N|\Lambda_*}\left[ \ln \left( \frac{f(\mathbf{Z} + \mathbf{N}| \Lambda) }{f(\mathbf{Z}| \Lambda) } \right) \right] \geq 0
\end{align}
The likelihood function for a Poisson model is:
\begin{align}
	\ln f(z| \lambda_z)  = -\lambda_z + z \ln(\lambda_z) - \ln(z!)
\end{align}
The log-ratio in the expectation reduces to the following for each data sample $z$:
\begin{equation}
	\ln \left( \frac{f(z + N| \lambda_z) }{f(z| \lambda_z) } \right) = N \ln \lambda_z  + \ln(z!) - \ln[(z+N)!].
\end{equation}
This simplifies to the following condition:
\begin{align}
	N  \geq \frac{1}{\ln \lambda_z} \ln \frac{\Gamma(z+N+1)}{\Gamma(z+1)}. \label{eq:DomCondn}
\end{align}
The noise--benefit holds if the inequality (\ref{eq:DomCondn}) holds only on average. Thus
\begin{equation}
	\E_{N} [N]  \geq \frac{1}{\ln(\lambda_z)} \E_N \left[ \ln \left( \frac{\Gamma(z+N+1)}{\Gamma(z+1)} \right) \right]. \label{eq:GammaRatio}
\end{equation}
The noise condition for the PET model replaces $\lambda_z$ in (\ref{eq:GammaRatio}) with $p_{ij}\lambda_i$ in (\ref{eq:Q-PET}).
\begin{equation}
	\E_{N} [N]  \geq \frac{1}{\ln(p_{ij}\lambda_i)} \E_N \left[ \ln \left( \frac{\Gamma(z+N+1)}{\Gamma(z+1)} \right) \right]. \label{eq:PET-GammaRatio}
\end{equation}

This improvement applies to the PET using EM. Th PET-EM approach was state-of-the-art almost $30$ years ago when Shepp and Vardi~\parencite*{shepp-vardi1982} introduced the idea of ML reconstruction for emission tomography. There have been recent advances in statistical reconstruction for tomography. These advances include the use of MAP estimation, the introduction of regularization, the use of  Gibbs priors, and hybrid methods that take frequency-domain details into account. The NEM theorem and its generalizations may accommodate some of these improvements.

\subsection{NEM for Magnetic Resonance Image Segmentation}\label{subsec:MRI-NEM}
The Expectation Maximization (EM) algorithm figures prominently in the segmentation of magnetic resonance (MR) images of brain and cardiac tissue \parencite{kapur-et-al-1998},\parencite{yzhang-et-al-segmentation-2001},\parencite{lorenzo-et-al-2003},\parencite{diplaros-et-al-2007}. For example, a brain image segmentation problem takes a MR image of the brain and classifies voxels or parts of voxels as grey matter, white matter, or cerebro-spinal fluid. 

There have been recent efforts~\parencite{hu-nayak2008,hu-nayak-goran2011,joshi-hu-et-al2013} to develop robust automated methods for differentiating and measuring differences between water and fat contents in a MR image. MR is optimal for this task because MR can differentiate between fat and water-saturated tissue unlike most other medical imaging modalities. But Water-Fat segmentation is difficult even with MR sensing. This is because subcutaneous and visceral (within organs) adipose tissue have different distribution characteristics. Current water/fat segmentation uses post-processing methods based on semi-automated image registration techniques~\parencite{joshi-hu-et-al2013} or human annotation. These work well for identifying large adipose masses. But they are not well-suited for identifying for the smaller adipose masses that are characteristic of visceral fat. Accurate and automatic water/fat MRI segmentation would be an important tool for assessing patients' risk of type II diabetes, heart disease, and other obesity-related ailments.

A simple EM segmentation application would use of EM to fit the pixel intensity levels to different tissue classes \parencite{kapur-et-al-1996}, \parencite{huang-et-al-2006} under a Gaussian mixture model (GMM) assumption. Pixel intensity classification model does not take spatial information into account. This limits the model to well-defined, low-noise applications \parencite{yzhang-et-al-segmentation-2001}. 

A more general approach is to apply a spatially-aware image segmentation method based on hidden Markov Random Field model (HMRF).The HMRF model is an non-causal generalization of the hidden Markov model. HMMs usually apply to signals with localized dependence in a time-like dimension. While HMRFs model signals with localized dependence in its spatial dimensions. The traditional Baum-Welch algorithm cannot train HMRF models because of the non-causal nature of signals involved. But there are mean-field EM algorithms that can train HMRFs.

There is prior work showing EM-based image segmentation \parencite{zhang1992},\parencite{zhang1993} on the HMRF model. There is also prior work showing how the HMRF model applies to image segmentation for MR images \parencite{kapur-et-al-1998},\parencite{yzhang-et-al-segmentation-2001},\parencite{diplaros-et-al-2007}. The EM algorithm on HMRF models is computationally expensive and can be slow. Thus a speed-up in the algorithm would be useful.

We propose to generalize the Noisy Expectation Maximization result to such mean-field algorithms and apply their noisy variants to HMRF training. The goal would be to show faster convergence of the HMRF-NEM algorithm or to show that NEM improves classification/segmentation results for a given number of EM iterations.

\subsection{Rule Explosion in Approximate Bayesian Inference}
The Bayesian approximation results apply to generic uniform approximations. Chapters \Sec\ref{ch:BAT} and \Sec\ref{ch:ExBAT} examined the behavior of fuzzy approximators in Bayesian applications. Neural network pdf approximators may have interesting properties. Other uniform approximators methods may also be worth exploring.

The main problem with the Bayesian approximation results is approximator explosion in the iterative Bayesian setting. This manifests as rule explosion when we use fuzzy function approximators. Semi-conjugacy (\Sec\ref{sec:semi-conj}) is a restricted approach to controlling rule explosion problem. Other approximation methods may lead to more robust strategies for reducing rule explosion in iterative Bayesian inference. 

Further work could also combine the idea of a noise-benefit with iterative Bayesian inference. A noise--benefit may speed up the convergence of Gibbs samplers.

\backmatter

\addcontentsline{toc}{chapter}{Bibliography}
\printbibliography
\end{document}